\documentclass[11pt,twoside]{article}

\usepackage{geometry}
\usepackage[square,numbers]{natbib}

\usepackage[utf8]{inputenc} 
\usepackage[T1]{fontenc}    
\usepackage{hyperref}       
\usepackage{url}            
\usepackage{booktabs}       
\usepackage{amsfonts}       
\usepackage{nicefrac}       
\usepackage{microtype}      
\usepackage{xcolor}         

\usepackage{tikz}
\usepackage{amsthm}
\usepackage[T1]{fontenc}
\usepackage{lmodern}
\usepackage{tcolorbox,wrapfig}
\usepackage[labelfont={bf},font=small]{caption}
\usepackage{subcaption}
\usepackage[ruled,vlined]{algorithm2e}
\tcbuselibrary{skins}
\usepackage[utf8]{inputenc} 
\usepackage[T1]{fontenc}    
\usepackage{hyperref}       
\usepackage{url}            
\usepackage{amsfonts}       
\usepackage{nicefrac}       
\usepackage{microtype}      
\usepackage{xcolor}         
\usepackage{tikz}
\usepackage{graphicx}
\usepackage{lipsum}
\usepackage{makecell}
\usepackage{tabu}
\usepackage{wrapfig}
\usepackage{amsmath}

\usepackage{longtable}

\usepackage{hyperref,graphicx,amsmath,amsfonts,amssymb,bm,url,breakurl,epsfig,epsf,color,MnSymbol,mathbbol,fmtcount,semtrans,caption,multirow,comment, boldline}
\usepackage{wrapfig}
\usepackage{enumitem}
\usepackage{amssymb}

\usepackage[utf8]{inputenc} 
\usepackage[T1]{fontenc}    
\usepackage{url}            
\usepackage{booktabs}       
\usepackage{amsfonts}       
\usepackage{nicefrac}       
\usepackage{microtype}      
\usepackage{dsfont}

\usepackage{xcolor}

\newlist{myitemize}{enumerate}{1}
\setlist[myitemize]{font=\color{darkred}\bfseries\itshape}

  \usepackage{mathtools}

\usepackage{titlesec}

\usepackage{tikz}
\usepackage{pgfplots}
\usetikzlibrary{pgfplots.groupplots}

\usepackage{movie15}

\usepackage{caption}
\usepackage[bottom,hang,flushmargin]{footmisc} 

\newcommand{\tsn}[1]{{\left\vert\kern-0.25ex\left\vert\kern-0.25ex\left\vert #1 
    \right\vert\kern-0.25ex\right\vert\kern-0.25ex\right\vert}}

\definecolor{darkred}{RGB}{150,0,0}
\definecolor{darkgreen}{RGB}{0,150,0}
\definecolor{darkblue}{RGB}{0,0,200}
\hypersetup{colorlinks=true, linkcolor=darkred, citecolor=darkgreen, urlcolor=darkblue}

\newcommand{\darkblue}{\color{darkblue}}

\newenvironment{fminipage}%
  {\begin{Sbox}\begin{minipage}}%
  {\end{minipage}\end{Sbox}\fbox{\TheSbox}}

\newtheorem{theorem}{Theorem}

\newtheorem{lemma}{Lemma}[section]
\newtheorem{corollary}{Corollary}[theorem]
\newtheorem{propo}{Proposition}
\newtheorem{definition}{Definition}

\newtheorem{remark}{Remark}[section]

\newcommand{\Cos}[2]{\operatorname{Cos}(#1,#2)}

\newcommand{\wmaj}{\w_{\rm{maj}}}
\newcommand{\wmin}{\w_{\rm{minor}}}
\newcommand{\hmaj}{\h_{\rm{maj}}}
\newcommand{\hmin}{\h_{\rm{minor}}}

\newcommand{\mubmaj}{\boldsymbol{\mu}_{\rm{maj}}}
\newcommand{\mubmin}{\boldsymbol{\mu}_{\rm{minor}}}

\newcommand{\minor}{{\text{minor}}}
\newcommand{\maj}{{\text{maj}}}

\newcommand{\rank}{\operatorname{rank}}

\newcommand{\Ghat}{\hat{\G}}

\newcommand{\nmin}{n_{\text{min}}}

\newcommand{\Deltab}{\boldsymbol{\Delta}}

\newcommand{\rhobar}{\overline{\rho}}

\newcommand{\SEL}{{SEL}}
\newcommand{\SELI}{{SELI}}

\newcommand{\Sec}{Sec.}
\newcommand{\Fig}{Fig.}

\DeclareMathOperator{\tr}{tr}

\DeclareMathOperator{\diag}{diag}

\newcommand{\cut}[1]{\textcolor{red}{}}
\newcommand{\W}{\mathbf{W}}

\newcommand{\M}{\mathbf{M}}
\newcommand{\Z}{\mathbf{Z}}
\newcommand{\Ub}{\mathbf{U}}
\newcommand{\nb}{\mathbf{n}}
\newcommand{\G}{\mathbf{G}}
\newcommand{\Rb}{\mathbf{R}}

\newcommand{\Hb}{{\mathbf{H}}}
\newcommand{\Hhat}{{\hat{\Hb}}}
\newcommand{\What}{{\hat{\W}}}

\newcommand{\Zhat}{{\hat{\Z}}}
\newcommand{\zhat}{{\hat{\z}}}
\newcommand{\Sb}{\mathbf{S}}
\newcommand{\X}{\mathbf{X}}

\newcommand{\Y}{\mathbf{Y}}
\newcommand{\Vb}{\mathbf{V}}
\newcommand{\A}{\mathbf{A}}
\newcommand{\Db}{\mathbf{D}}
\newcommand{\Pb}{\mathds{P}}
\newcommand{\Pbf}{\mathbf{P}}
\newcommand{\Qbf}{\mathbf{Q}}
\newcommand{\Pibf}{\boldsymbol{\Pi}}
\newcommand{\Xibf}{\boldsymbol{\Xi}}

\newcommand{\omegab}{\boldsymbol{\omega}}

\newcommand{\pbf}{\mathbf{p}}
\newcommand{\qbf}{\mathbf{q}}

\newcommand{\mub}{\boldsymbol{\mu}}

\newcommand{\x}{\mathbf{x}}
\newcommand{\ub}{\mathbf{u}}
\newcommand{\w}{\mathbf{w}}

\newcommand{\vb}{\mathbf{v}}

\newcommand{\Bb}{\mathbf{B}}

\newcommand{\eb}{\mathbf{e}}
\newcommand{\y}{\mathbf{y}}

\newcommand{\z}{\mathbf{z}}

\newcommand{\ab}{\mathbf{a}}

\newcommand{\h}{\mathbf{h}}

\newcommand{\Sc}{{\mathcal{S}}}

\newcommand{\Nc}{\mathcal{N}}

\newcommand{\Lc}{\mathcal{L}}

\newcommand{\beq}{\begin{equation}}
\newcommand{\eeq}{\end{equation}}
\newcommand{\bea}{\begin{align}}
\newcommand{\eea}{\end{align}}

\newcommand{\vp}{\vspace{5pt}}

\newcommand{\R}{\mathbb{R}}

\newcommand{\nn}{\notag}

  \newcommand{\Sigmab}{\boldsymbol\Sigma}
    
  \newcommand{\la}{{\lambda}}                     
  \newcommand{\eps}{\epsilon}
    \newcommand{\Lambdab}{\boldsymbol{\Lambda}}                     
\newcommand{\Xhat}{\hat{\X}}
 \newcommand{\deltab}{\boldsymbol\delta}

\DeclarePairedDelimiterX{\inp}[2]{\langle}{\rangle}{#1, #2}

\newcommand{\wt}{\widetilde}

\newcommand{\Id}{\mathds{I}}
\newcommand{\ones}{\mathds{1}}

\newcommand{\zeros}{\mathbf{0}}

\newcommand{\blue}[1]{\textcolor{blue}{#1}}
\newcommand{\red}[1]{\textcolor{red}{#1}}
\newcommand{\new}[1]{{{#1}}}

\author{%
Christos Thrampoulidis$^\star$, Ganesh R. Kini$^\dagger$, Vala Vakilian$^\star$, Tina Behnia$^\star$ \footnote{This work is supported by an NSERC Undergraduate Student Research Grant, an NSERC Discovery Grant,  NSF Grant CCF-2009030, and by a CRG8-KAUST award. The authors also acknowledge use of the Sockeye cluster by UBC Advanced Research Computing. 
}
 \vspace{8pt}
 \\
$^\star$University of British Columbia, Canada 
\vspace{3pt}
\\
$^\dagger$University of California, Santa Barbara, USA
}

\title{
Imbalance Trouble: Revisiting Neural-Collapse Geometry
}

\begin{document}

\maketitle

\begin{abstract}
Neural Collapse refers to the remarkable structural properties characterizing the geometry of class embeddings and classifier weights, found by deep nets when trained beyond zero training error. However, this characterization only holds for balanced data. Here we thus ask whether it can be made invariant to class imbalances. Towards this end, we adopt the unconstrained-features  model (UFM), a recent theoretical model for studying neural collapse, and introduce $\text{\emph{Simplex-Encoded-Labels Interpolation}}$ (SELI) as an invariant characterization of the neural collapse phenomenon. Specifically, we prove for the UFM with cross-entropy loss and vanishing regularization that, irrespective of class imbalances, the embeddings and classifiers always interpolate a simplex-encoded label matrix and that their individual geometries are determined by the SVD factors of this same label matrix. We then present extensive experiments on synthetic and real datasets that confirm convergence to the SELI geometry. However, we caution that convergence worsens with increasing imbalances. We theoretically support this finding by showing that unlike the balanced case, when minorities are present, ridge-regularization plays a critical role in tweaking the geometry. This defines new questions and motivates further investigations into the impact of class imbalances on the rates at which first-order methods converge to their asymptotically preferred solutions.
\end{abstract}

\addtocontents{toc}{\protect\setcounter{tocdepth}{0}}

\section{Introduction}\label{sec:intro}

 {What are the unique structural properties of models learned by training deep neural networks to zero training error? Is there an \emph{implicit bias} towards solutions of certain geometry? How does this vary across training instances, architectures, and data? These questions are at the core of understanding the optimization landscape of deep-nets. Also, they are naturally informative about the role of models since different parameterizations might affect preferred geometries. Ultimately, such understanding makes progress towards explaining generalization of overparameterized models. 

Recently, remarkable new progress in answering these questions has been made by \citet{NC}, who empirically discover and formalize the so-called \emph{Neural-collapse} (NC) phenomenon. NC describes  geometric properties of the learned embeddings (aka last-layer features) and of the classifier weights of deep-nets, trained with cross-entropy (CE) loss and \emph{balanced} data far into the zero training-error regime. The NC phenomenon produces a remarkably simple description of a particularly symmetric geometry: (i) The embeddings of each class collapse to their class mean (see \emph{\ref{NC}} property); and (ii) The class means align with the classifier weights and they form a simplex equiangular tight frame (see \emph{\ref{ETF}} property). Importantly, as noted by  \citet{NC}, this {simple} geometry appears to be ``cross-situational \emph{invariant}''
across different architectures and different \emph{balanced} datasets.

In this paper, we study Neural collapse with {imbalanced} classes:
\emph{Is there a (ideally equally simple) description of the geometry that is invariant across \underline{class-imbalanced datasets}? 
}

\begin{figure}
	\begin{subfigure}{0.7\linewidth}
		\centering
		\includegraphics[scale=0.6]{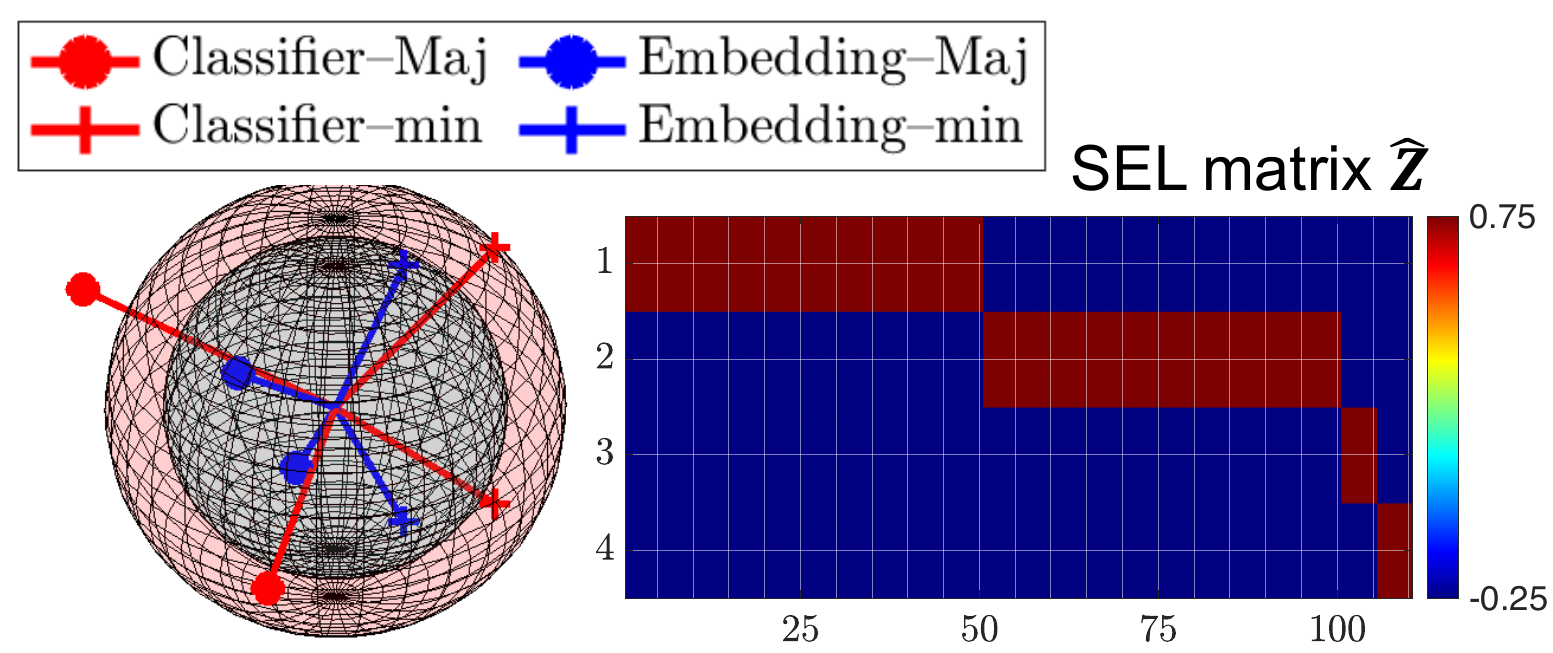}
		\caption{\SELI~geometry;\\ $(10,\nicefrac{1}{2})$-STEP imbalance, $k=4$ classes}
		\label{fig:intro_SEL_seli}
	\end{subfigure}
\begin{subfigure}{0.24\linewidth}
	\centering
	\vspace{1cm}
\includegraphics[scale=0.52]{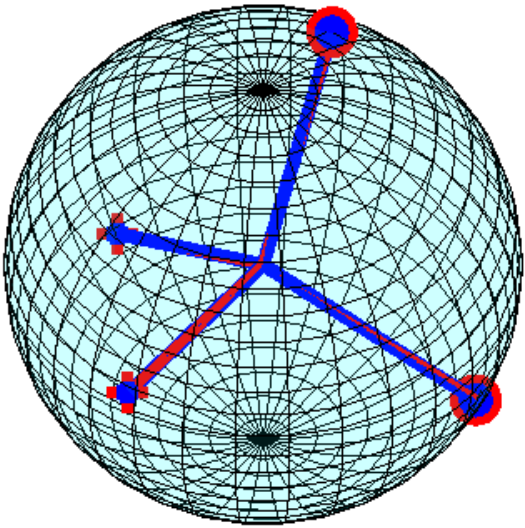}
\captionsetup{width=0.9\linewidth}\caption{ETF~geometry;\\ $k=4$ balanced classes}
\label{fig:intro_SEL_etf}
\end{subfigure}
\caption{Visualization of the SELI and ETF geometries.}
\label{fig:intro_SEL}
\end{figure}

\vp
\noindent\textbf{Contributions.} We propose a new description capturing the geometric structure of learned-embeddings and classifier-weights on possibly \emph{class-imbalanced data}, which we call the \emph{Simplex-Encoded-Labels Interpolation}  \emph{\ref{SELI}} geometry. This new geometry is a generalization of the {ETF} geometry: It recovers the latter when data are balanced   {or when there are only two classes}, and also, unlike ETF, it remains invariant across different \emph{imbalance levels}. Importantly, it too, has a simple description: The matrix of learned logits interpolates a simplex-encoded label (SEL) matrix $\Zhat$, and, the individual  geometries of the embeddings and classifiers are determined by the SVD factors of this same \SEL~matrix. Because the particular arrangement of columns of the \SEL~matrix changes with the imbalance level, this also impacts the geometric arrangement of the embedding and classifier vectors.   {Overall, the norms and angles of these vectors admit simple closed-form expressions in terms of the imbalance characteristics and the number of classes.}

We use an example to illustrate this. \Fig~\ref{fig:intro_SEL_seli} depicts the \SEL~matrix $\Zhat\in\R^{4\times 110}$ for a STEP-imbalanced $k=4$-class dataset with two majority classes of $50$ examples each, and, two minority classes of $5$ examples each. Each column of $\Zhat$ includes the $k$ learned logits for each one of the $110$ examples in the dataset. Each such column has exactly one entry equal to $1-1/k=0.75$ and  three entries equal to $-1/k=-0.25$. The corresponding geometry of the embeddings and classifiers, shown in the 3D plot, is found by an SVD of $\Zhat$: the left eigenvectors determine the classifiers and the right ones the embeddings. Note that $\Zhat$ is rank 3, hence the geometry is  3D. Since embeddings collapse to their class means (see \emph{\ref{NC}} property), we only show the four {\color{blue}class-mean embeddings} and the corresponding four {\color{red}classifiers}. Two of each correspond to majorities (``$\bullet$'' marker) and two to minorities (``$+$'' marker).   {The radii of the two concentric spheres are equal to the norms of the minority classifiers (red sphere) and of the minority embeddings (blue sphere), respectively. Note that the norms of minorities and majorities are different, and so are the angles. Moreover, the classifiers are \emph{not} aligned with the embeddings. Overall, the geometry  is different compared to the ETF geometry seen in the balanced case, which is shown in Fig. \ref{fig:intro_SEL_etf}.}   {What remains invariant across class-imbalances is that the logits (i.e. inner products between classifiers and embeddings) only take values either $1-1/k$  or  $-1/k$, so that the logit matrix is equal to the \SEL~matrix. Equivalently, the learned model interpolates the simplex-encoding of the labels.}

Below we explain the conception of this geometry and our contributions in  detail. The initial major challenge was: \emph{Assuming a class-imbalance-invariant geometry exists, how to find  it?} 

%
%

To answer this question, we adopted the \emph{Unconstrained Feature Model (UFM}) previously introduced in the literature as a  two-layer proxy model to theoretically justify neural collapse \cite{mixon2020neural,fang2021exploring,zhu2021geometric}. Motivated by deep-learning practice and by studies on implicit bias of gradient descent (GD) for unregularized CE minimization, we analyze the geometry of solutions to an unconstrained-features Support Vector Machines (UF-SVM) problem. We prove, for STEP-imbalanced data, any solution of the UF-SVM follows the SELI geometry. Thus, the learned end-to-end model always interpolates a simplex labels encoding. We show that \emph{\ref{ETF}}$\implies$\emph{\ref{SELI}}. However, \emph{\ref{SELI}}$\not\implies$\emph{\ref{ETF}} unless classes are balanced or there is just two of them ($k=2$). 

Next, we analyze training of the UFM with ridge-regularized CE. Unlike previous studies, we find in the presence of imbalances that regularization matters as it changes the geometry of solutions. In fact, we show that there is \emph{no} finite regularization that leads to the  SELI geometry. However, we also show that as regularization vanishes, the solutions do interpolate the \SEL~matrix (after appropriate normalization.) Finally, 
we show that the SELI geometry differs from the minority-collapse phenomenon \cite{fang2021exploring}, since the latter does not correspond to solutions with zero training error. In fact, we show for minority collapse that: (i)  It does \emph{not} occur for small finite regularization   {and finite imbalance ratio, and (ii) It occurs asymptotically for vanishing regularization, but only asymptotically as the imbalance ratio grows.}   

We numerically test convergence to the proposed SELI geometry in both synthetic and real class-imbalanced datasets. For different imbalance levels, the learned geometries approach the SELI geometry significantly faster compared to the ETF geometry. See \Fig~\ref{fig:intro_CIFAR}.
However, we also observe that convergence \emph{worsens} with increasing level of imbalance. A plausible theoretical justification is that as we show regularization plays critical role under imbalances. We also consistently get better convergence rate for the classifiers. We believe our observations strongly motivate further investigations regarding potential frailties of ``asymptotic'' implicit bias characterizations and how these might vary in multiclass and possibly imbalanced settings.

\begin{figure*}[t]
	\vspace{-10pt}
	\centering
	\hspace{-40pt} \begin{subfigure}{0.3\textwidth}
		\centering
		\begin{tikzpicture}
			\node at (-1.4,-1.4) 
			{\includegraphics[scale=0.26]{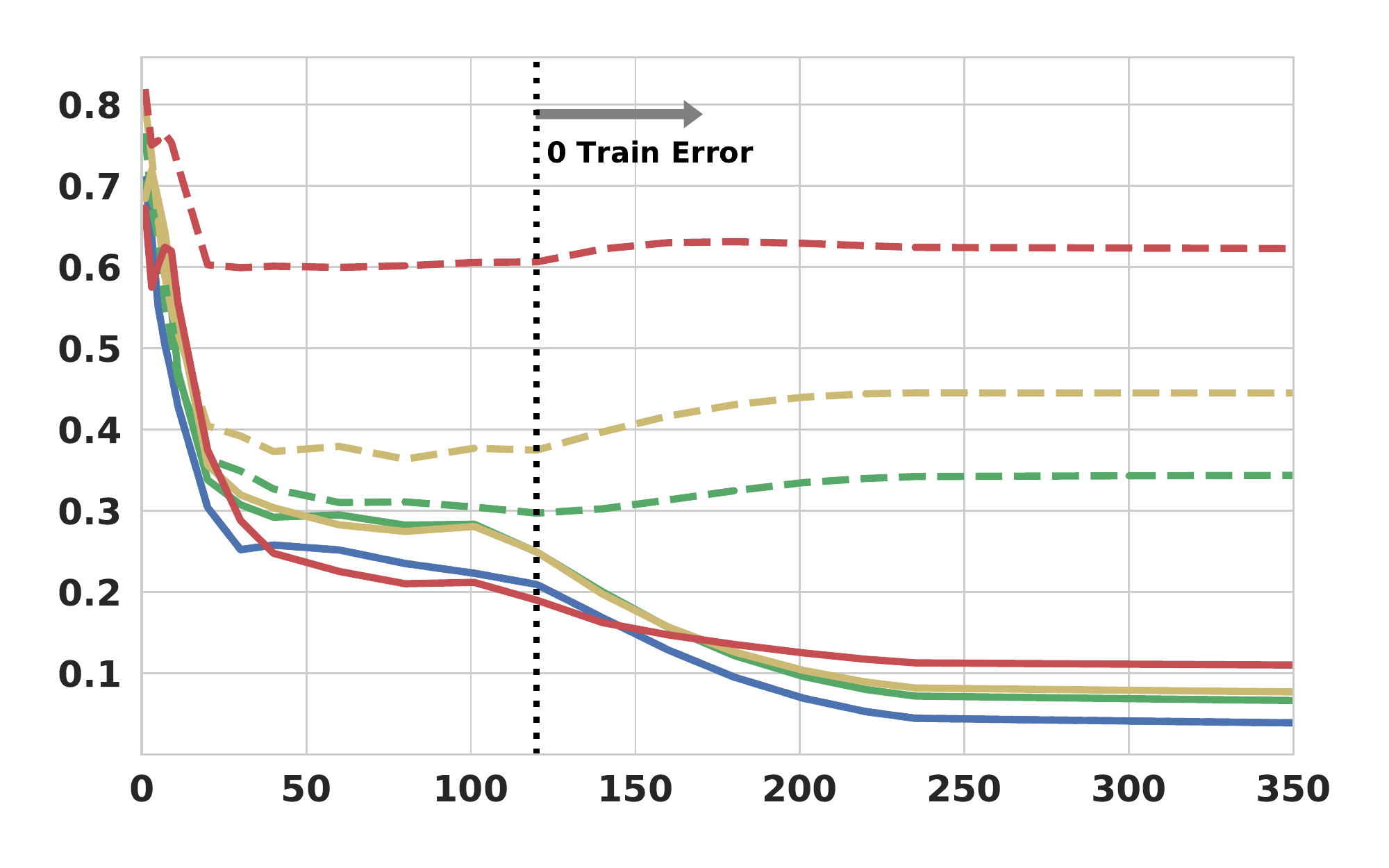}};
			\node at (-4.1,-1.4)  [scale=0.7, rotate=90]{Distance to SELI/ETF};
		\end{tikzpicture}
	\end{subfigure}\hspace{13pt}\begin{subfigure}{0.3\textwidth}
		\centering
		\begin{tikzpicture}
			\node at (0,-1.4) {\includegraphics[scale=0.26]{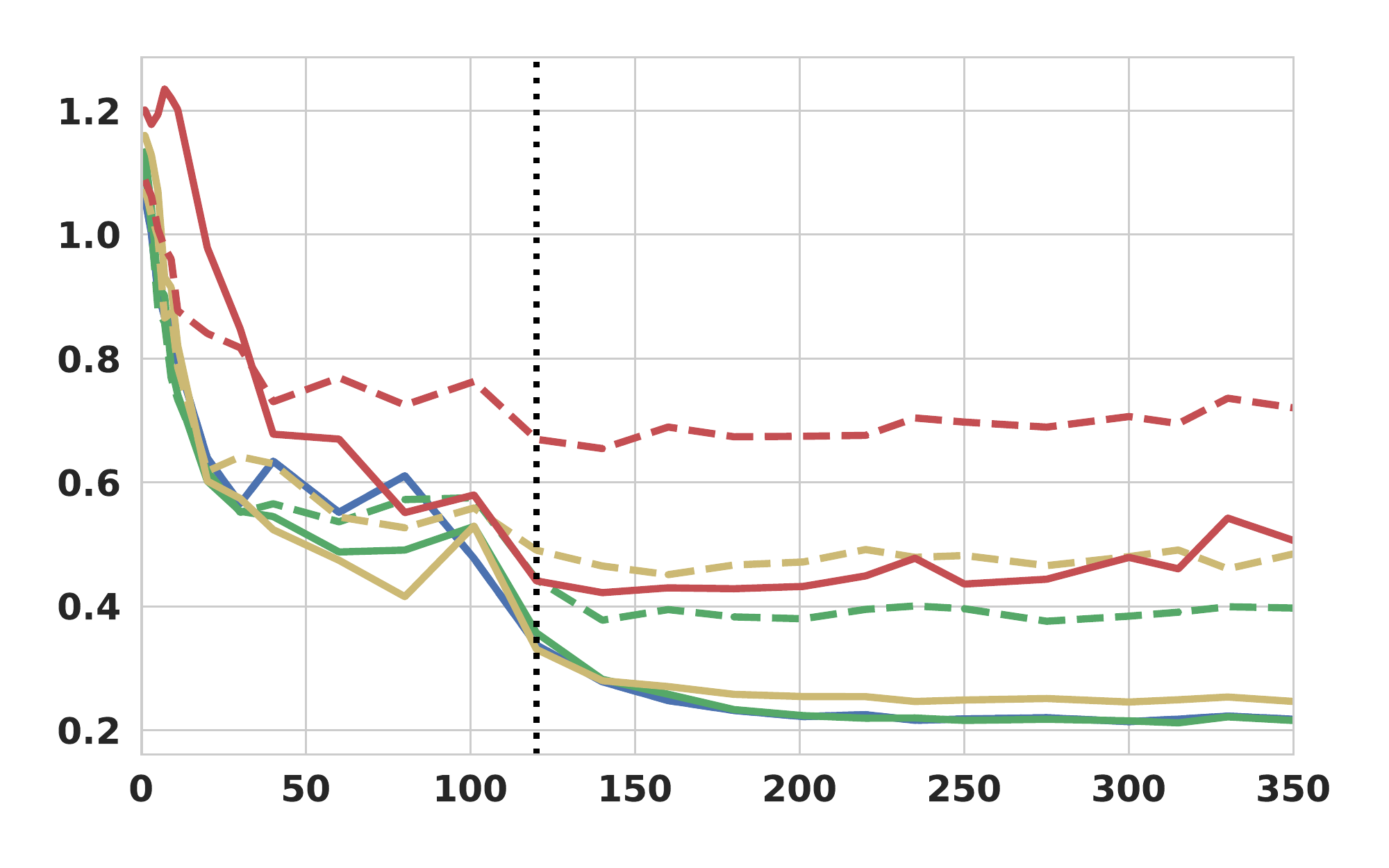}};
			\node at (0,0.2) [scale=0.9]{\textbf{CIFAR10}};
			\node at (-2.7,-1.4) [scale=0.7, rotate=90]{};
		\end{tikzpicture}
	\end{subfigure}\hspace{13pt}\begin{subfigure}{0.3\textwidth}
		\centering
		\begin{tikzpicture}
			\node at (0,-1.4) {\includegraphics[scale=0.26]{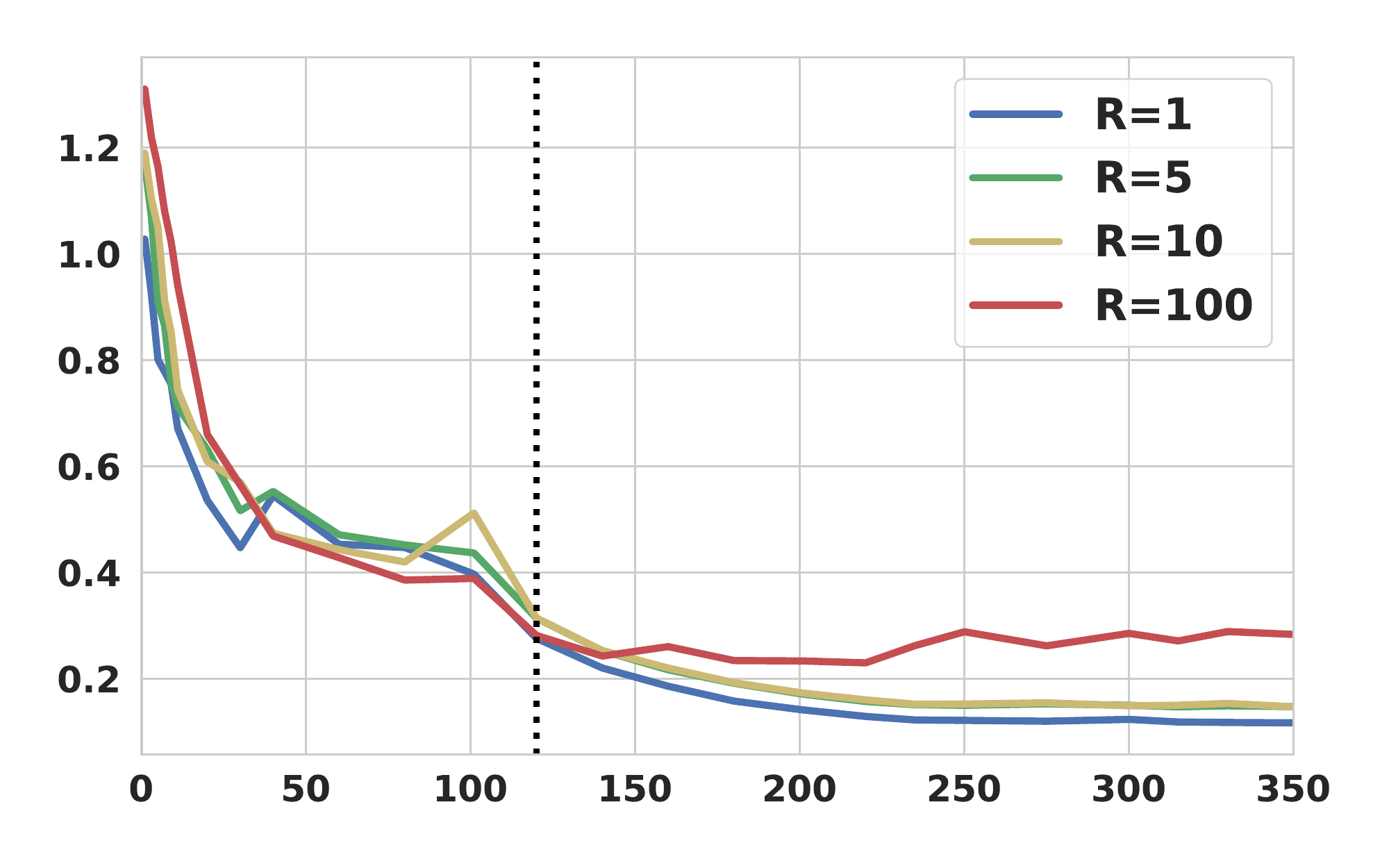}};
			\node at (-2.7,-1.4) [scale=0.7, rotate=90]{};
		\end{tikzpicture}
	\end{subfigure}\vspace{-10pt}
	
	\vspace{5pt}
	\centering
	\hspace{-40pt} \begin{subfigure}{0.3\textwidth}
		\centering
		\begin{tikzpicture}
			\node at (-1.4,-1.4)
			{\includegraphics[scale=0.26]{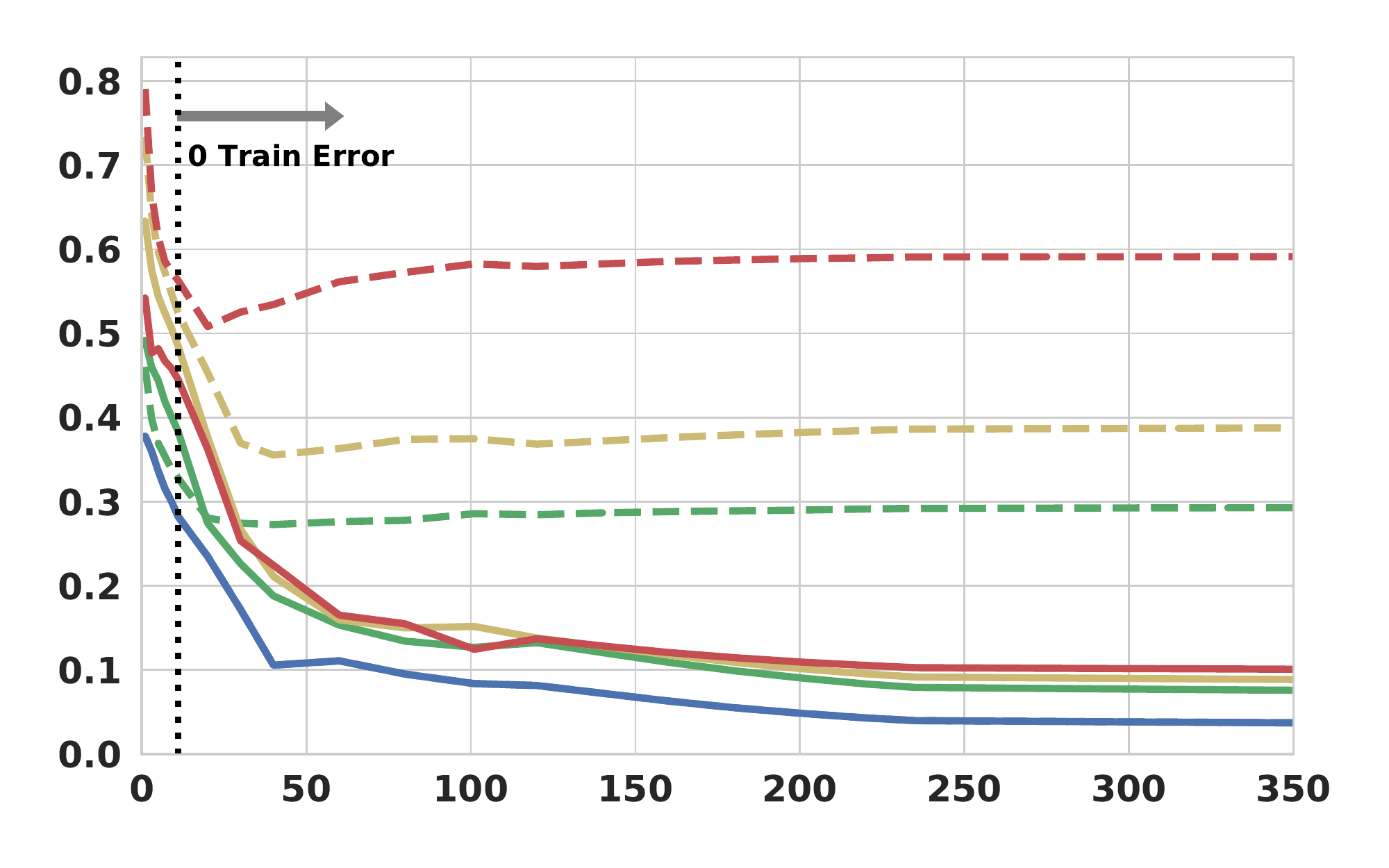}};
			\node at (-4.1,-1.4)  [scale=0.7, rotate=90]{Distance to SELI/ETF};
		\end{tikzpicture}
	\end{subfigure}\hspace{13pt}\begin{subfigure}{0.3\textwidth}
		\centering
		\begin{tikzpicture}
			\node at (0,-1.4) {\includegraphics[scale=0.26]{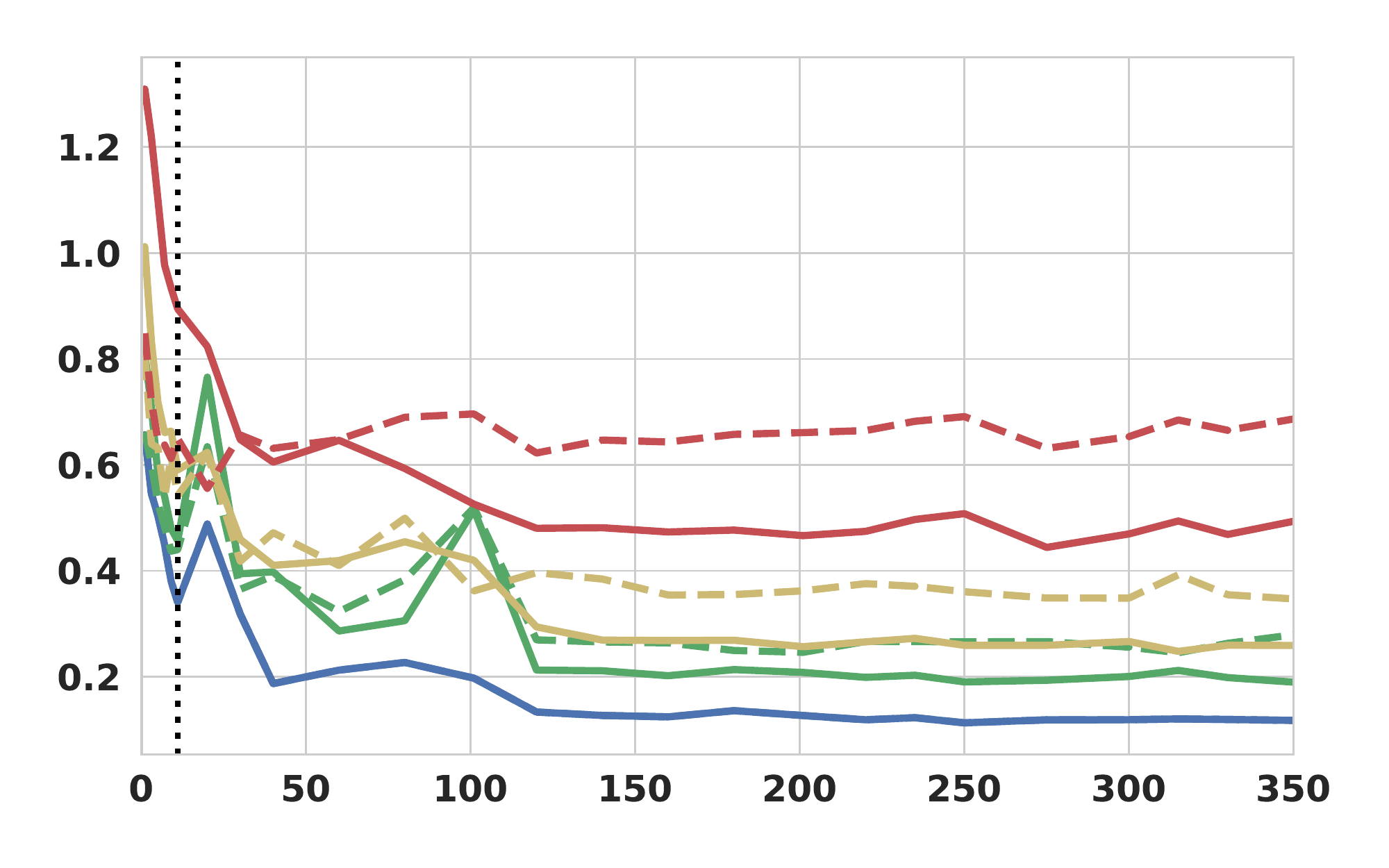}};
			\node at (0,0.2) [scale=0.9]{\textbf{MNIST}};
			\node at (-2.7,-1.4) [scale=0.7, rotate=90]{};
		\end{tikzpicture}
	\end{subfigure}\hspace{13pt}\begin{subfigure}{0.3\textwidth}
		\centering
		\begin{tikzpicture}
			\node at (0,-1.4) {\includegraphics[scale=0.26]{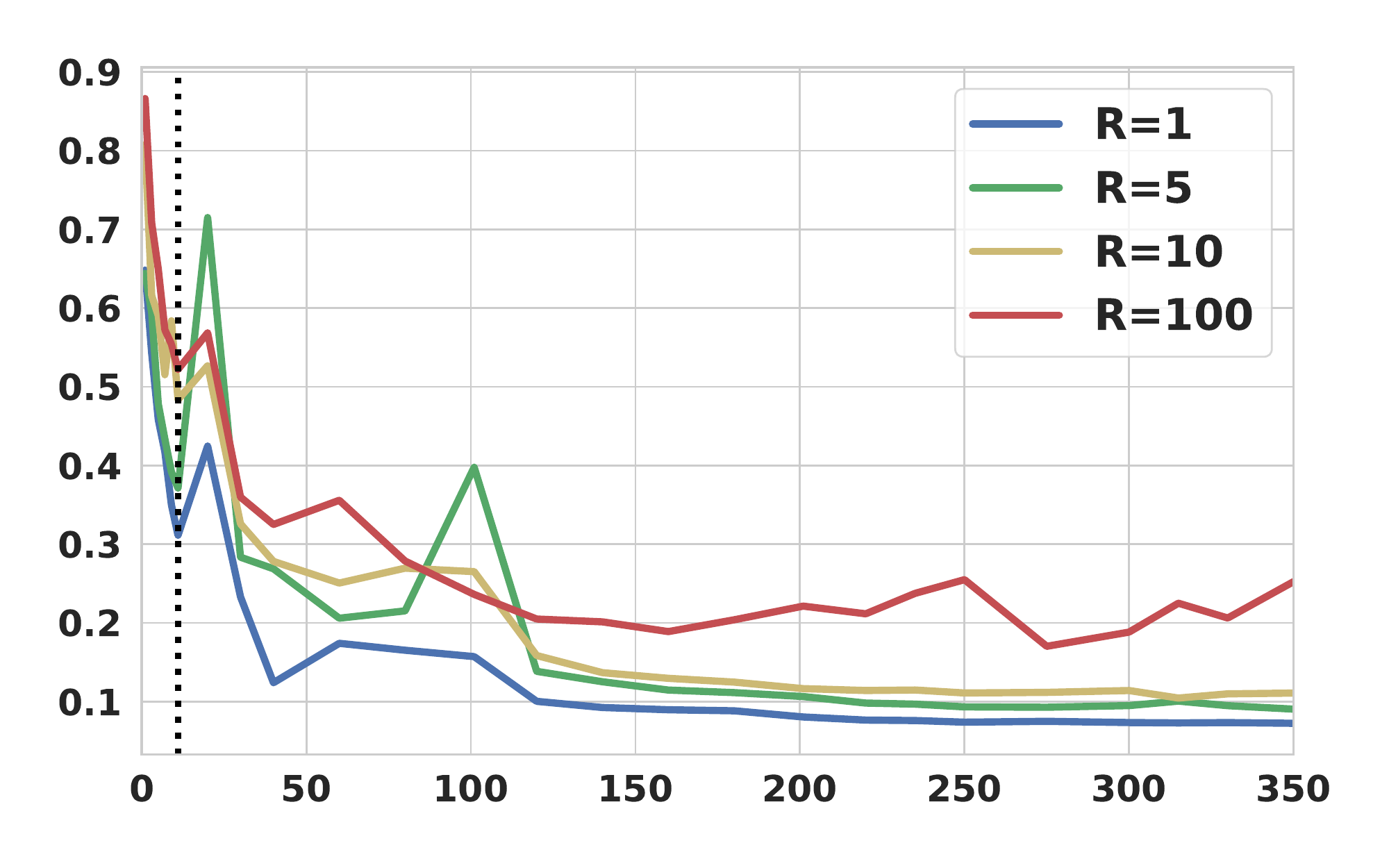}};
			\node at (-2.7,-1.4) [scale=0.7, rotate=90]{};
		\end{tikzpicture}
	\end{subfigure}\vspace{-10pt}
	
	\vspace{5pt}
	\centering
	\hspace{-40pt} \begin{subfigure}{0.3\textwidth}
		\centering
		\begin{tikzpicture}
			\node at (-1.4,-1.4)
			{\includegraphics[scale=0.26]{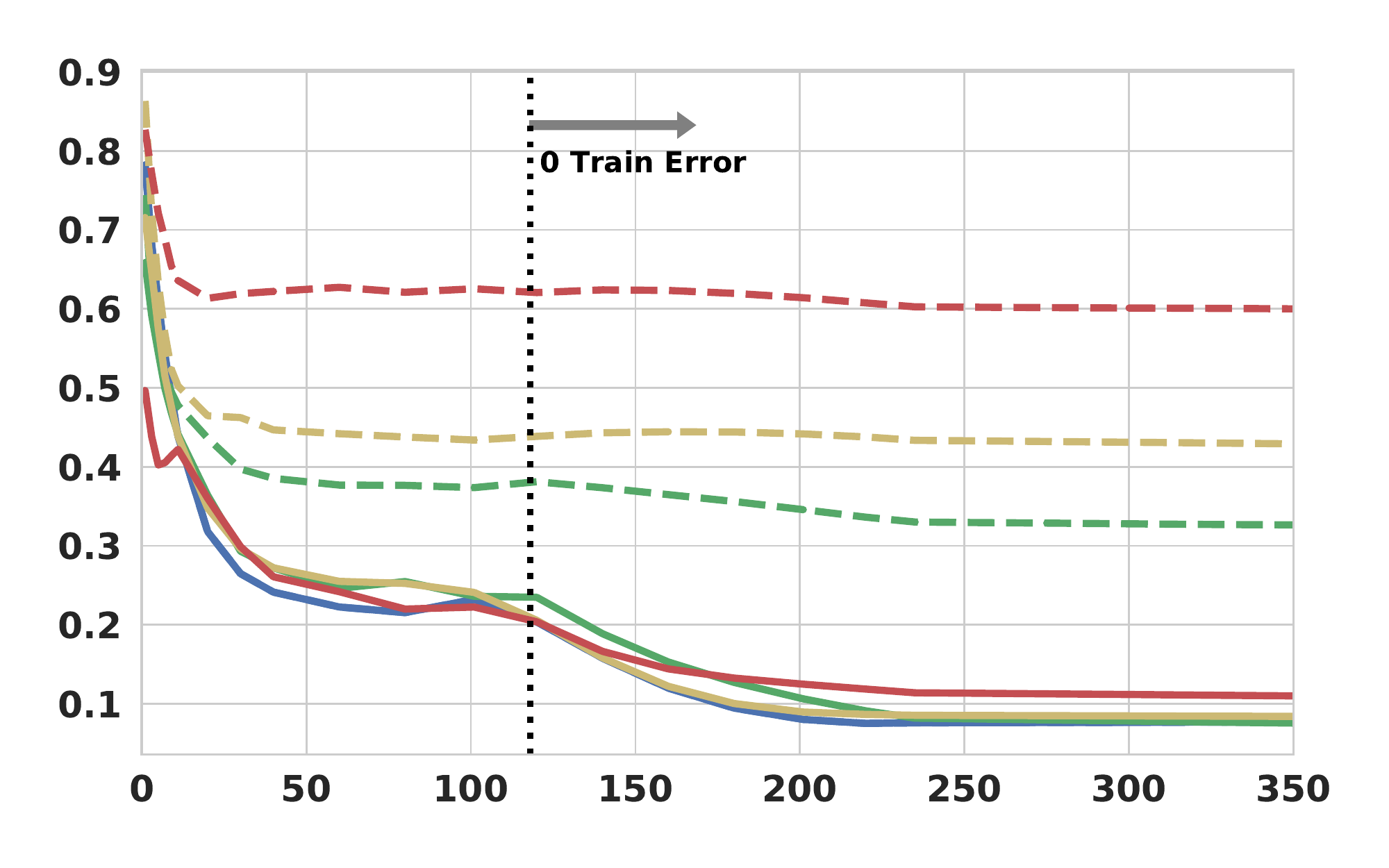}};
			\node at (-1.4,-3.1) [scale=0.7]{Epochs};
			\node at (-4.1,-1.4)  [scale=0.7, rotate=90]{Distance to SELI/ETF};
		\end{tikzpicture}\vspace{-0.2cm}
		\captionsetup{width=0.6\linewidth}\caption{Classifiers}
	\end{subfigure}\hspace{13pt}\begin{subfigure}{0.3\textwidth}
		\centering
		\begin{tikzpicture}
			\node at (0,-1.4) {\includegraphics[scale=0.26]{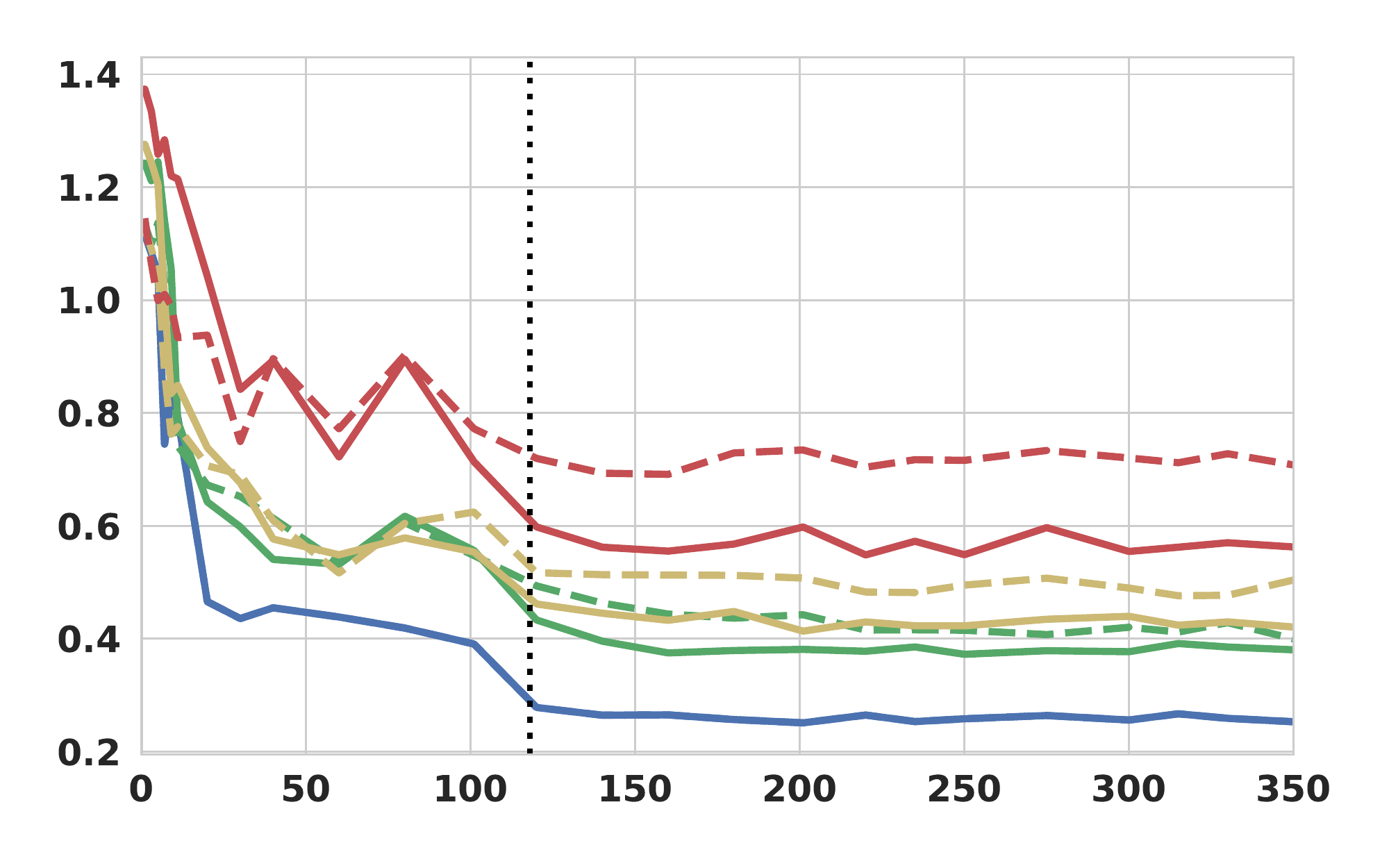}};
			\node at (0,0.2) [scale=0.9]{\new{\textbf{Fashion-MNIST}}};
			\node at (0,-3.1) [scale=0.7]{Epochs};
			\node at (-2.7,-1.4) [scale=0.7, rotate=90]{};
		\end{tikzpicture}\vspace{-0.2cm}
		\captionsetup{width=0.6\linewidth}\caption{Embeddings}
	\end{subfigure}\hspace{13pt}\begin{subfigure}{0.3\textwidth}
		\centering
		\begin{tikzpicture}
			\node at (0,-1.4) {\includegraphics[scale=0.26]{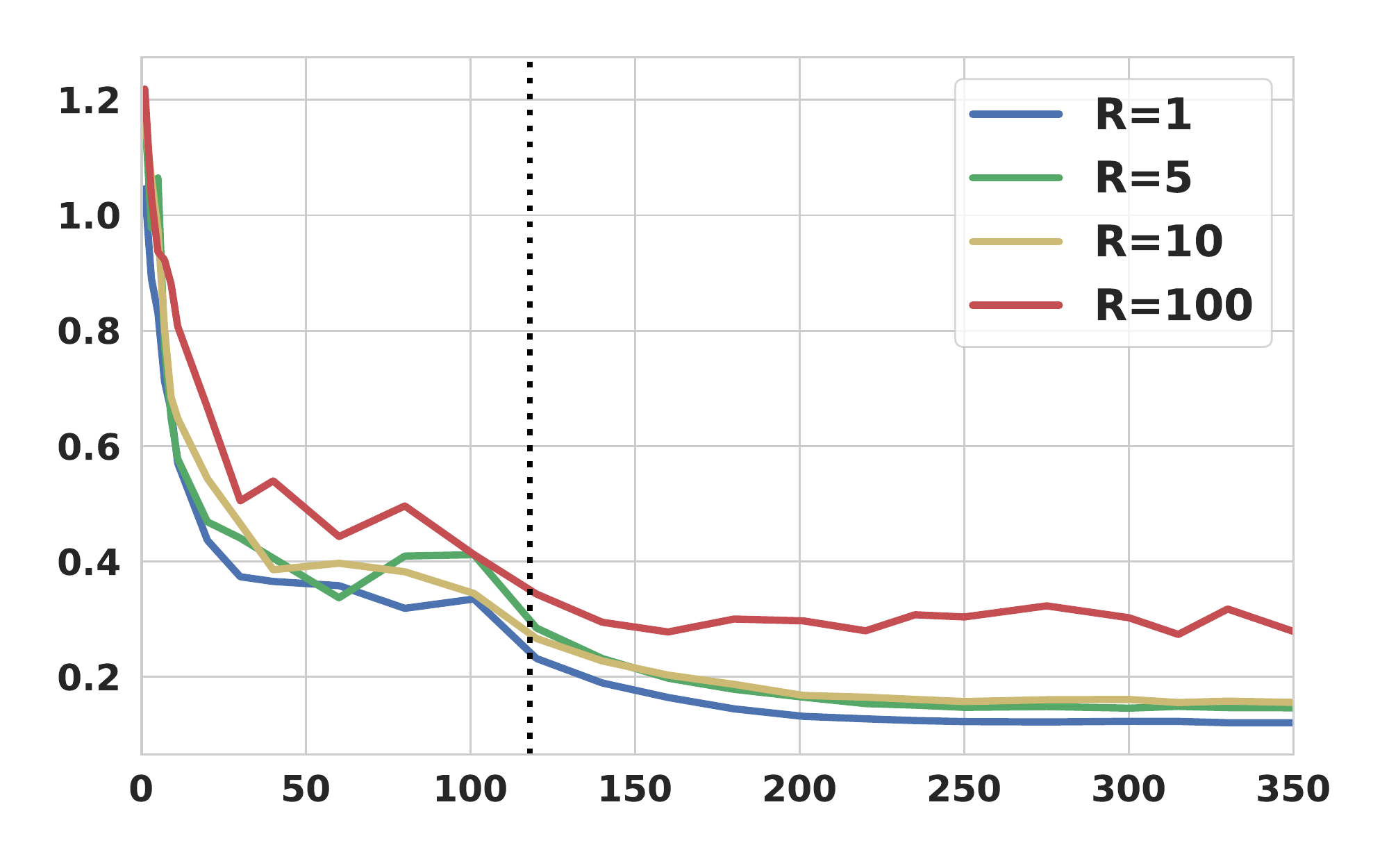}};
			\node at (0,-3.1) [scale=0.7]{Epochs};
			\node at (-2.7,-1.4) [scale=0.7, rotate=90]{};
		\end{tikzpicture}\vspace{-0.2cm}
		\captionsetup{width=0.6\linewidth}\caption{Logits}
	\end{subfigure}\vspace{-5pt}\caption{Convergence of learned classifiers, embeddings and corresponding logits to the SELI (solid lines) vs ETF (dashed lines) geometries, measured using a ResNet-18 model, trained far beyond zero training error on \textbf{\underline{STEP-Imbalanced}} CIFAR10, MNIST and Fashion-MNIST, for different imbalance ratios $R$; see \Sec~\ref{sec:rel} for metrics and discussion. 
	}
	\label{fig:intro_CIFAR}
\end{figure*}
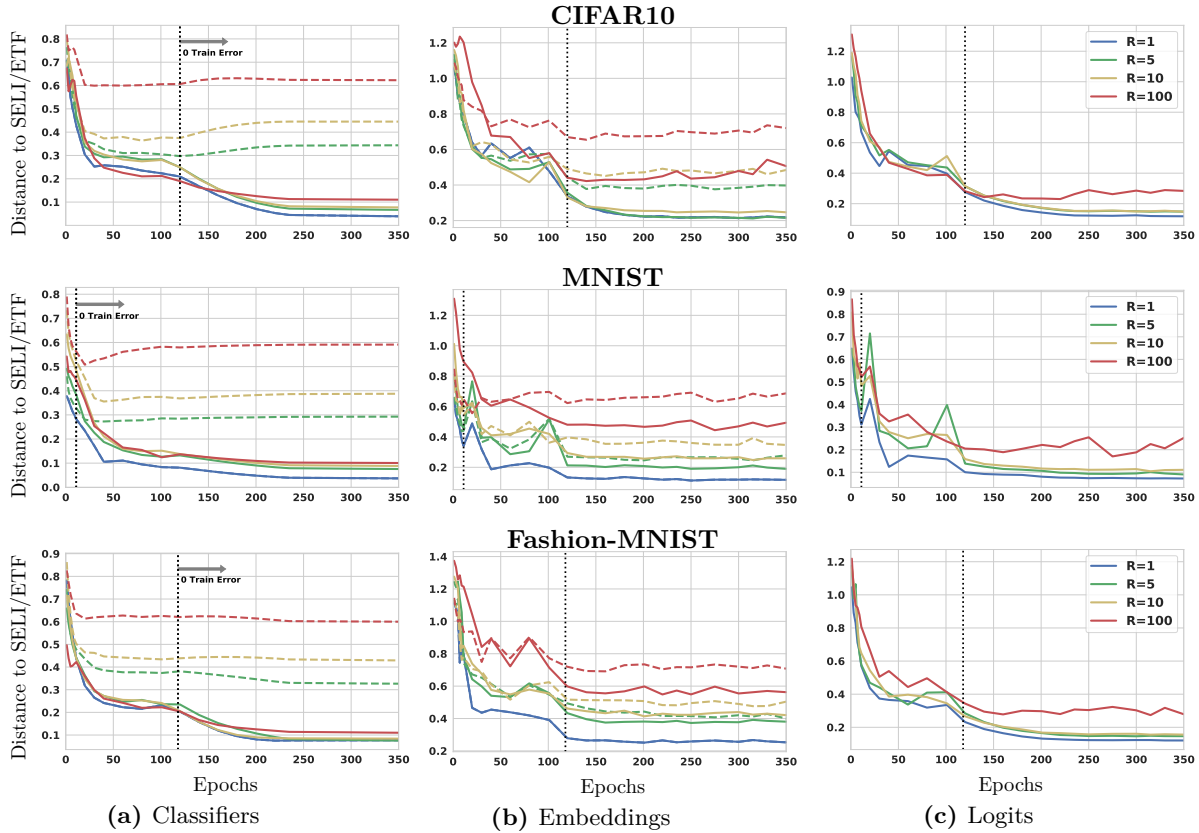

\vp
\noindent\textbf{Why imbalanced?}~~Investigating structural geometric properties under data imbalances is important for the following reasons. First, data-imbalances appear in various learning tasks more often than not. Thus, it is relevant understanding what deep neural networks learn in these settings. Second, the previously discovered, perfectly symmetric, ETF geometry only holds when the data distribution is  itself symmetric. This naturally raises a question, which we answer here: when this data symmetry breaks, is it still possible to characterize potentially non-symmetric geometries? We argue this result is nontrivial, particularly so because the new geometry is richer than the ETF as it is parameterized by not only the number of classes, but also by the imbalance ratio and the fraction of minorities. Third, while previous work has shown the simplified unconstrained-features model is powerful to predict the exact geometry when classes are balanced, it is not a priori known whether the model is also able to predict the corresponding geometry when classes are imbalanced. Our paper shows this to be the case and also uncovers several unique features when data are imbalanced: (i) the UFM solution depends on the regularization parameter, and, (ii) epoch-wise convergence of SGD slows down with increasing imbalance and is generally better for classifiers compared to embeddings. Finally, it is known that data imbalances affect generalization since vanilla training with cross-entropy loss results in poor generalization for minority classes. We envision that understanding how the properties of the learnt geometries vary for minorities versus majorities could help create links between structural results and generalization towards explaining the effectiveness of existing or inspiring new techniques for mitigating imbalances.

\begin{table}[t]
\begin{minipage}[b]{\linewidth}\centering
\scalebox{0.75}{\hspace*{-1.9cm}
\begin{tabular}{|c|c|c|c | c| }
\multicolumn{5}{c}{\large{Geometry of global minimizers for UFM}} \\
\cline{2-5}
\multicolumn{1}{c|}{}&\multicolumn{2}{|c|}{\textbf{UF-RidgeCE} (Eqn.~\eqref{eq:CE})}
&\multicolumn{2}{|c|}{\textbf{UF-SVM} (Eqn.~\eqref{eq:svm_original})} \\
\hline
 \makecell{\makecell{~\\
 \textbf{Balanced}
 \\ ~
 }}
 & 
 \multicolumn{2}{|c|}{\makecell{
\textbf{ETF}\\
\cite{lu2020neural,graf2021dissecting,zhu2021geometric,fang2021exploring,tirer2022extended}
 }}
 & 
 \multicolumn{2}{|c|}{\makecell{
\textbf{ETF}\\
\cite{ULPM}, \textbf{\darkblue [Cor.~\ref{cor:balanced}]}
 }}
 \\ 
 \hline
  \makecell{\makecell{~\\
 \textbf{Imbalanced}
 \\ ~
 }}
 & 
 \multicolumn{2}{|c|}{\makecell{
\textbf{
$\forall\la$: NO SELI
{\darkblue [Prop.~\ref{propo:imbalance_reg}]}
}
\\
\textbf{$\la\rightarrow0$: \textbf{\red{SELI}}
 \darkblue [Prop.~\ref{propo:reg_path}]}
\\
\textbf{$\la<\frac{1}{2}$: NO minority collapse
[\darkblue Sec. \ref{sec:minority}]}
 }}
 & 
 \multicolumn{2}{|c|}{\makecell{
\textbf{\red{SELI}}\\
\textbf{\darkblue [Thm.~\ref{thm:SVM}]}
 }}
 \\
\hline
 \end{tabular}
}
\vspace{-0.0cm}
\caption{\small{Summary of contributions and comparison to most-closely related work. 
}}
\label{table:intro}
\end{minipage}
\end{table}

\vspace{-0.15cm}
\subsection{Related works}\label{sec:intro_related}
\vspace{-0.15cm}
The original contribution by \citet{NC} has attracted lots of attention resulting  in numerous followups within short time period, e.g., \cite{zhu2021geometric,ULPM,fang2021exploring,han2021neural,lu2020neural,mixon2020neural,graf2021dissecting,zhou2022optimization,tirer2022extended}. (See also \cite[\Sec~E]{han2021neural} for a  review of the recent literature.) Several works have proposed and/or used the UFM with CE training to 
analyze theoretical abstractions of NC  \cite{zhu2021geometric,ULPM,fang2021exploring,graf2021dissecting}. \new{Other works analyze the UFM with square loss \cite{mixon2020neural,han2021neural,zhou2022optimization,tirer2022extended} and recent model extensions accounting for additional layers and nonlinearities are studied in \cite{tirer2022extended}. Here,} we drove particular inspiration from \citet{zhu2021geometric}, who  presented a particularly transparent and complete analysis of the optimization landscape of ridge-regularized CE minimization for the UFM under balanced data.  In the same spirit, we also relied on the UFM.  
 However, our work is, to the best of our knowledge, the first explicit geometry analysis for class-imbalanced data. See also Table \ref{table:intro} for a comparison. 

The only previous work on neural collapse with imbalances is  
\cite{fang2021exploring}, which was the first to note that collapse of the embeddings is preserved, but otherwise the geometry might skew away from ETF. Also, \citet{fang2021exploring} first proposed studying the new geometry using the UFM and appropriate convex relaxations. With this setup, they presented an intriguing finding, which they termed \emph{minority collapse}: for large imbalance levels, the minorities' classifiers collapse to the same vector. As  mentioned above, our focus is on the, particularly relevant for deep-learning practice, zero-training error scenarios. This excludes minority collapse by definition. More importantly, we derive an explicit geometric characterization of \emph{both}  embeddings and  classifiers for \emph{both} majorities and minorities and for \emph{all} imbalance levels.  {Specializing these findings to vanishing regularization and imbalance ratio growing to infinity recovers and gives new insights to minority collapse. 
}


Our results also draw from and relate to the literatures on implicit bias, matrix factorization, and imbalanced deep-learning.  {We defer a detailed discussion on these to \Sec~\ref{sec:rel2} of the Supplementary Material. }

\subsection{Organization}

In \Sec~\ref{sec:setup} we setup some necessary terminology and  introduce the unconstrained features model. In \Sec~\ref{sec:UF-SVM} we formally introduce the SELI geometry as the geometry of the global minimizers of the non-convex max-margin minimization over the UFM. We also relate the new SELI geometry to the ETF and derive closed-form expressions for the former in terms of the level of imbalance. Next in \Sec~\ref{Sec:reg},  we investigate the impact of data imbalances on the structure of the ridge-regularized CE minimization with the UFM. In \Sec~\ref{sec:rel} we present experimental results corroborating our theoretical findings. Specifically, we conduct experiments on both synthetic data under the UFM  and on benchmark class-imbalanced datasets. Concluding remarks and some directions for future research are included in \Sec~\ref{sec:out}. 

The proofs of our results are presented in Secs.~\ref{sec:SEL_mat_properties}--\ref{sec:proofs_UF-SVM} of the Supplementary Material (SM). Several additional experimental results on the UFM and on real data are also included in \Sec s \ref{sec:UFM} and \ref{sec:real_exp_SM}, respectively. \Sec~\ref{sec:minority} of the SM discusses in detail implications of our findings to minority collapse. Finally, \Sec~\ref{sec:rel2} includes additional remarks and comparisons to previous works.

\vp
\noindent\textbf{Notation.}~For matrix $\Vb\in\R^{m\times n}$, $\Vb[i,j]$ denotes its $(i,j)$-th entry, $\vb_j$ denotes the $j$-th \emph{column}, $\Vb^T$ its transpose and $\Vb^\dagger$ its Moore-Penrose pseudoinverse.
We denote $\|\Vb\|_F, \|\Vb\|_2$, and, $\|\Vb\|_*$ the Frobenius, spectral, and, nuclear norms of $\Vb$.  $\tr(\Vb)$ denotes the trace of $\Vb$. $\odot$ and $\otimes$ denote Hadammard and Kronecker products, respectively. \new{$\Vb\succ0$ denotes $\Vb$ is positive semidefinite and $\Vb\geq 0$ that $\Vb$ has nonnegative entries.} 
 $\nabla_\Vb\Lc \in\R^{m\times n}$ is the gradient of a scalar function $\Lc(.)$ with respect to $\Vb$. 
We use $\ones_m$ to denote an  $m$-dimensional vector of all ones and $\Id_m$ for the $m$-dimensional identity matrix. For vectors/matrices with all zero entries, we simply write $0$, as dimensions are easily understood from context. $\eb_{j}$ denotes a column with a single non-zero entry of $1$ in the $j$-th entry.


\section{Problem setup}\label{sec:setup}


\vspace{-0.1cm}
We adopt the  \emph{unconstrained feature model} (UFM)  \citep{mixon2020neural,fang2021exploring} in a $k$-class classification setting. Let $
\W_{d\times k} = [\w_1,\w_2, \cdots,\w_k]
$ be the matrix of classifier weights corresponding to the $k$ classes. Here, $d$ is the feature dimension. We assume throughout that $d\geq k-1$. Next, we let 
$
\Hb_{d\times n} = [\h_1, \h_2, \cdots, \h_n]
$ denote a matrix of $n$ feature embeddings, each corresponding to a different example in the training set. We assume each class $c\in[k]$ has $n_c\geq 1$  examples (thus, $n_c$ embeddings) so that $\sum_{c\in[k]}n_c=n.$ Without loss of generality, we assume examples are ordered. Formally, we assume that examples $i=1,\ldots,n_1$ have labels $y_i=1$, examples $i=n_1+1,\ldots,n_1+n_2$ have labels $y_i=2$, and so on. 
The UFM trains the features $\h_i,i\in[n]$ (jointly with the weights $\w_c,c\in[k]$) without any further constraints, i.e., by minimizing the ridge-regularized cross-entropy (CE) loss as follows \cite{zhu2021geometric}:
\begin{align}\label{eq:CE}
(\What_\la,\Hhat_\la):=\arg\min_{\W,\Hb}~\Lc(\W^T\Hb) + \frac{\la}{2} \|\W\|_F^2+\frac{\la}{2}\|\Hb\|_F^2,
\end{align}
where $\Lc(\W^T\Hb):= \sum_{i\in[N]} \log\big(1+\sum_{c\neq y_i}e^{-(\w_{y_i}-\w_c)^T\h_i}\big)$ is the CE loss.

\vp
\noindent\textbf{UFM as two-layer linear net.}~The formulation above does not explicitly specify inputs for each example. 
An alternative view follows by considering training a 2-layer linear net with  hidden dimension $d$, first layer $\Hb$, and second layer $\W$, over $n$ examples  {with $n$-dimensional inputs} $\x_i=\eb_i \in\R^n, \in[n]$ and labels $y_i$ as above: $y_i=1$ for $i=1,\ldots,n_1$, $y_i=2$ for $i=n_1+1,\ldots,n_1+n_2$, and so on. 

\subsection{Unconstrained-features SVM (UF-SVM)}

Since neural-collapse is observed when training with small / vanishing regularization \cite{NC}, it is reasonable to consider an unregularized version of \eqref{eq:CE}. In this special case, gradient descent (with sufficiently small step size) on \eqref{eq:CE} produces iterates that diverge in norm, but converge in direction \cite{lyu2019gradient,ULPM}. In fact, it has been recently shown that 
the GD solutions converge in direction 
  to a KKT  point of the following max-margin classifier \cite{lyu2019gradient,ULPM}: 
\begin{align}\label{eq:svm_original}
(\What,\Hhat)\in\arg\min_{\W,\Hb}~\frac{1}{2}\|\W\|_F^2+ \frac{1}{2}\|\Hb\|_F^2
\quad\quad\text{sub. to}\quad (\w_{y_i}-\w_c)^T\h_i \geq 1,~i\in[n], c\neq y_i.
\end{align}
For convenience, we refer to the optimization problem in \eqref{eq:svm_original} as unconstrained-features SVM (UF-SVM). This minimization (unlike `standard' SVM) is non-convex. Hence, KKT points (thus, GD convergence directions) are not necessarily global minimizers; see discussion in \Sec~\ref{sec:out}. 

\subsection{Class-imbalance model}
To streamline the presentation, we focus on a setting with STEP imbalances. This includes balanced data as special case by setting $R=1$. 

\begin{definition}[$(R,\rho)$-STEP imbalance] In a $(R,\rho)$-STEP imbalance setting with label-imbalance ratio $R\geq 1$ and minority fraction $\rho\in(0,1)$, the following hold. All minority (resp. majority) classes have the same sample size $\nmin$ (resp. $R\nmin$). There are $(1-\rho)k$ majority and $\rho k$ minority classes. Without loss of generality, we assume classes $\{1,\ldots,(1-\rho)k\}$ are majorities.
\end{definition}



\section{Global structure of the UF-SVM: SELI geometry}\label{sec:UF-SVM}

In this section, we characterize the global minimizers of the non-convex program in \eqref{eq:svm_original}. Perhaps surprisingly, we show that they take a particularly simple form that is best described in terms of a \emph{simplex-encoding} of the labels. 

\begin{definition}[\SEL~matrix]\label{def:SEL} The \emph{simplex-encoding label (\SEL) matrix} $\Zhat_{k\times n}$
is such that
\begin{align}
\forall c\in[k], i\in[n]\,:\,\,\,\,\Zhat[c,i] = \begin{cases}
1-1/k & ,\,c=y_i \\
-1/k & ,\,c\neq y_i
\end{cases}.
\end{align}
Onwards, let $\Zhat^T=\Ub\Lambdab\Vb^T$ be the compact SVD of $\Zhat^T$. Specifically, $\Lambdab$ is a positive $(k-1)$-dimensional diagonal matrix and $\Ub_{n\times (k-1)}$, $\Vb_{k\times(k-1)}$ have orthonormal columns. 

\end{definition}

Each column $\zhat_i\in\R^{k}$ of $\Zhat$ represents a class-membership  encoding of datapoint $i\in[n]$. This differs from the vanilla one-hot encoding $\hat{\y}_i=\eb_{y_i}$ in that $\zhat_i=\hat{\y}_i-\frac{1}{k}\ones_k$.
Specifically, $\Zhat$ has exactly $k$ different and affinely independent columns, which together with the zero vector form a $k$-dimensional simplex, motivating the \SEL~name. Finally, note that $\Zhat^T\ones_k=0$; thus,  $\operatorname{rank}(\Zhat)=k-1$.  {We gather useful properties about the eigenstructure of $\Zhat$ in \Sec~\ref{sec:SEL_mat_properties}.}


\begin{theorem}[Structure of the UF-SVM minimizers]\label{thm:SVM}
 Suppose $d\geq k-1$ and a $(R,\rho)$-STEP imbalance setting.  Let $(\What,\Hhat)$ be any solution and $\mathrm{p}_* $ the optimal cost of the UF-SVM in \eqref{eq:svm_original}.
Then, $\mathrm{p}_* = {\|\Zhat\|_*}={\|\Hhat\|_F^2}={\|\What\|_F^2}$. Moreover, the following statements  characterize the geometry of global minimizers in terms of the the \SEL~matrix and its SVD.

\begin{enumerate}[label={(\roman*)},itemindent=0em]

\item\label{thm:global_logits}
 For the optimal logits it holds that
$\What^T\Hhat = \Zhat$. 

\item\label{thm:global_gram}
The Gram matrices satisfy
$
\Hhat^T\Hhat = \Ub\Lambdab\Ub^T
$ and 
$\What^T\What = \Vb\Lambdab\Vb^T.
$

\item\label{thm:global_W}
For some partial orthonormal matrix $\Rb\in\R^{(k-1)\times d}$, 
$
\What=\Rb^T\Lambdab^{1/2}\Vb^T
$
and
$
\Hhat=\Rb^T\Lambdab^{1/2}\Ub^T.
$
\end{enumerate}
\end{theorem}

We outline the theorem's proof in \Sec~\ref{sec:SVM_proof_sketch} and defer the details to \Sec~\ref{sec:proof_SVM}. 
The theorem provides an explicit characterization of the geometry of optimal embeddings and classifiers that relies around the key finding that the optimal logit matrix is always equal to the \SEL~matrix. We highlight the following key features of this characterization.

\vp
\noindent\textbf{Simplicity.} The lack of symmetry in the imbalanced setting, makes it a priori unclear whether a simple geometry description is still possible, as in the balanced case. But, the theorem shows this to be the case. The key observation is that the optimal logit matrix $\What^T\Hhat$ equals $\Zhat$ (cf. Statement \ref{thm:global_logits}). Then, the Gram matrices of embeddings and classifiers are given simply in terms of the singular factors of the \SEL~matrix (cf. Statements \ref{thm:global_gram},\ref{thm:global_W}).

\vp
\noindent\textbf{Invariance to imbalances.} The theorem's characterization is valid for all types of $(R,\rho)$-STEP imbalances. In particular, equality of the optimal logit matrix to the \SEL~matrix is the key invariant characterization across changing imbalances.  This also implies that at optimality all margins are equal irrespective of the imbalance type. The description of Gram matrices in terms of the SVD of $\Zhat$ is also invariant. Of course, the particular arrangement of columns of $\Zhat$ itself depends on the values of $(R,\rho)$. 
In turn, the singular factors determining the geometry of embeddings and classifiers depend implicitly on the same parameters. Thus, as we show next, the geometry differs for different imbalance levels; see \Fig~\ref{fig:intro_SEL} for an example.

\subsection{Invariant properties: NC and SELI}\label{sec:NC_and_SELI}

Here, we further discuss the geometry of embeddings and classifiers induced by the SVD of the \SEL~matrix. The first realization is that the embeddings collapse under \emph{all} settings.

\begin{corollary}\label{cor:NC}
The UF-SVM solutions satisfy the following property irrespective of imbalance:
\begin{myitemize}[label=\textbf{\emph{(NC)}},itemindent=0.5em]
\item\label{NC} The embeddings collapse to their class means $\hat\h_i=\hat\mub_c:=\frac{1}{n_c}\sum_{j:y_j=c}\hat\h_j$, $\forall c\in[k]$,  $i:y_i=c$.
\end{myitemize}
\end{corollary}

This statement can be inferred from Theorem \ref{thm:SVM} (specifically from Statement \ref{thm:global_W} and that $\Ub$ has repeated columns.) A more straightforward argument is possible by directly inspecting the UF-SVM optimization in \eqref{eq:svm_original}. For any fixed (say optimal) $\What$, the minimization over $\h_i$ is: (i) separable and identical for all $i\,:\,y_i=c$ in same class $c$, and (ii)  strongly convex. Hence, for all $i\,:\,y_i=c$, there is unique minimizer corresponding to the fixed $\What$; this must be their class mean. 

Beyond {{(NC)}}, Theorem \ref{thm:SVM} specifies the exact geometry of solutions. The corollary below is a restatement of Theorem \ref{thm:SVM} under the following formalization of what we call the SELI geometry. 


\begin{definition}[SELI geometry]\label{def:SELI} The  embedding and classifier matrices $\Hb_{d\times n}$ and $\W_{d\times k}$ follow the simplex-encoded-labels interpolation  geometry when for some scaling $\alpha>0$:  
\begin{myitemize}[label=\textbf{\emph{(SELI)}},itemindent=1.4em]
\item\label{SELI} 
$
\begin{bmatrix} \W^T \\ {\Hb}^T \end{bmatrix} \begin{bmatrix} \W & \Hb \end{bmatrix} = \alpha \begin{bmatrix}\Vb\Lambdab\Vb^T & \Zhat \\ \Zhat^T & \Ub\Lambdab\Ub^T\end{bmatrix}
,$~~ where $\Zhat=\Vb\Lambdab\Ub^T$ is the \SEL~matrix.
\end{myitemize}
\end{definition}


\begin{corollary}\label{cor:SELI}
The UF-SVM solutions follow the {SELI} geometry,  irrespective of imbalance.
\end{corollary}

 {The \emph{\ref{SELI}} geometry characterization specifies (up to a global positive scaling) the Gram matrices $\G_\W:=\W^T\W$ and $\G_\Hb:=\Hb^T\Hb$, and the logit matrix $\Z:=\W^T\Hb$. Specifically, the diagonals of the two Gram matrices specify the \emph{norms} of the classifiers and of the embeddings \footnote{\new{To simplify the exposition, we assume throughout that the regularization strength is same for embeddings and classifiers. For completeness we treat the general case in \Sec~\ref{sec:different_hyperparams} in the SM, where it is shown that different regularization values do \emph{not} change the SELI geometry as per Definition \ref{def:SELI} apart from introducing a (global) relative scaling factor between the norms of the embeddings and classifiers.}}. These, together with their off-diagonal entries, further specify the \emph{angles} between different classifiers and between different embeddings. Because of the \emph{\ref{NC}} property, the norms and angles of the embeddings $\h_i, i\in[n]$ are uniquely determined in terms of the norms and angles of the mean-embeddings $\mub_c, c\in[k]$. In other words, the Gram matrix $\G_\Hb=\Hb^T\Hb\in\R^{n\times n}$ is such that for all $i\in[n]$, $\G_{\Hb}[i,j]=\G_{\Hb}[i,\ell]$ if $y_j=y_\ell$.  Finally, the norms together with the entries of the logit matrix determine the angles between the two sets of: (a) the $k$ classifiers and (b) the $k$ mean embeddings. Thus, they specify the degree of alignment between the two sets of vectors. In the next section, we show that it is in fact possible to obtain explicit closed-form formulas describing the norms, angles and alignment of classifier and embedding vectors, in terms of the imbalance characteristics and the number of classes.}

\begin{remark}[Why ``SELI''?]
For the UFM, $\Z=\W^T\Hb$ is the learned end-to-end model. According to its definition, the SELI geometry implies $\W^T\Hb=\alpha\Zhat$. Thus, the learned model \emph{interpolates} (a scaling of) the \SEL~matrix, motivates the naming in Definition \ref{def:SELI}. 
\end{remark}

\subsubsection{Special case: Balanced or binary data}

For the special cases of balanced or binary data, Theorem \ref{thm:SVM} recovers the  ETF structure, i.e.  \emph{\ref{SELI}}$\equiv$\emph{\ref{ETF}}. Let $\hat\M=[\hat\mub_1,\ldots,\hat\mub_k]$ denote the matrix of mean embeddings.

\begin{corollary}[$R=1$ or $k=2$]\label{cor:balanced}
Assume balanced data ($R=1$) or binary classification ($k=2$).
 Then, any UF-SVM solution $(\hat\W,\hat\Hb)$  follows the ETF geometry as defined in \cite{NC}:
\begin{myitemize}[label=\textbf{\emph{(ETF)}},itemindent=1em]
\setcounter{enumi}{1}

\item\label{ETF}
$\hat\W=\hat\M$ and $\hat\M^T\hat\M=\hat\W^T\hat\W=\Id_k-\frac{1}{k}\ones_k\ones_k^T$.
\end{myitemize}
\end{corollary}

Thus, when data are balanced or binary: (i) the norms of the classifiers and of the embeddings are all equal; (ii) the angles between any two classifiers or any two embeddings are all equal to $\frac{-1/k}{(k-1)/k}=-\frac{1}{k-1}.$; and, (iii) the set of classifiers and the set of embeddings are aligned. 

\subsection{How the SELI geometry changes with imbalances}\label{sec:SELI_changes}

For $k>3$ and $R>1$, \emph{\ref{SELI}}$\not\implies$\emph{\ref{ETF}}.  In this general case, the norms and angles specifying the geometry are determined in terms of the SVD factors of the \SEL~matrix as per Definition \ref{def:SELI}. In \Sec~\ref{sec:SEL_mat_properties}, we give an explicit characterization of these SVD factors. Notably, this allows us to obtain explicit closed-form formulas for the norms, angles and alignment of the SELI geometry in terms of $R, \rho$ and  $k, n_{\min}$. The following lemma is an example: it gives a formula for the ratio of majority and minority norms for the classifiers and embeddings. For simplicity, we focus on the default case of equal numbers of minorities and majorities, i.e. $\rho=1/2$.

\begin{lemma}[Norm ratios]\label{lem:norms} Assume $(R,\nicefrac{1}{2})$-STEP imbalance. Suppose $(\W,\Hb)$ satisfies the \ref{SELI} property. Let $\wmaj,\hmaj$ (resp. $\wmin,\hmin$) denote majority (resp. minority) classifiers and embeddings, respectively. Then,
\begin{align*}
\frac{\|\wmaj\|_2^2}{\|\wmin\|_2^2} = \frac{(1-2/k)\sqrt{R}+\frac{\sqrt{(R+1)/2}}{k}}{(1-2/k)+\frac{\sqrt{(R+1)/2}}{k}}~~\text{ and }~~\frac{\|\hmaj\|_2^2}{\|\hmin\|_2^2} = \frac{\frac{1}{\sqrt{R}}(1-2/k)+\frac{1}{k\sqrt{(R+1)/2}}}{(1-2/k)+\frac{1}{k\sqrt{(R+1)/2}}}.
\end{align*}
Thus, $\|\wmaj\|_2\geq \|\wmin\|_2$ and $\|\hmaj\|_2\leq\|\hmin\|_2$, with equalities if and only if $R=1$ or $k=2$.
\end{lemma}

The fact that CE learns classifiers of larger norm for majorities (i.e. $\|\wmaj\|_2\geq \|\wmin\|_2$), has been empirically observed in the deep imbalanced-learning literature, e.g. by  \citet{kang2020decoupling,KimKim}. Lemma \ref{lem:norms}, not only provides a  theoretical justification for this empirical observation, but it also precisely quantifies the ratio. Moreover, it specifies the norm-ratio between majorities and minorities, not only for classifiers, but also for the learned embeddings.

\begin{figure}[t]
     \centering
     \begin{subfigure}[b]{0.49\textwidth}
         \centering
         \includegraphics[width=0.75\textwidth]{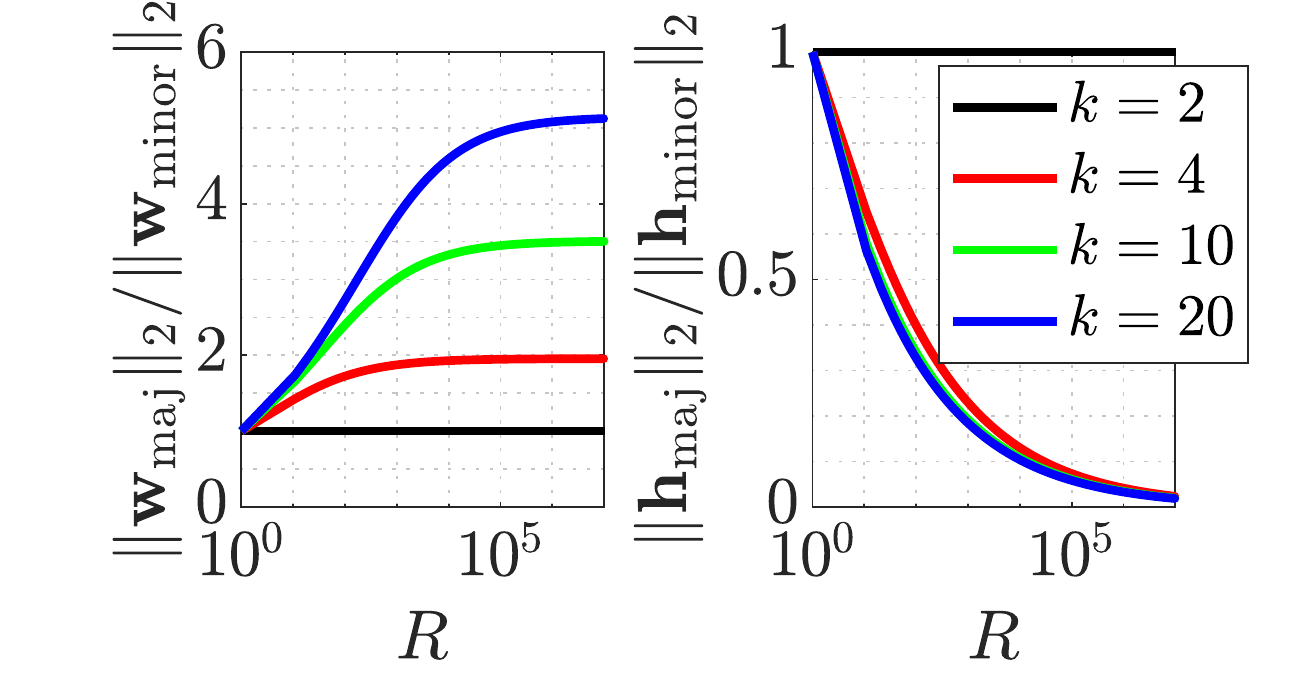}
         \caption{Norms of classifiers and embeddings.}
         \label{fig:SELI_theory_norms_inf}
     \end{subfigure}
     \hfill
     \begin{subfigure}[b]{0.49\textwidth}
         \centering
         \includegraphics[width=0.75\textwidth]{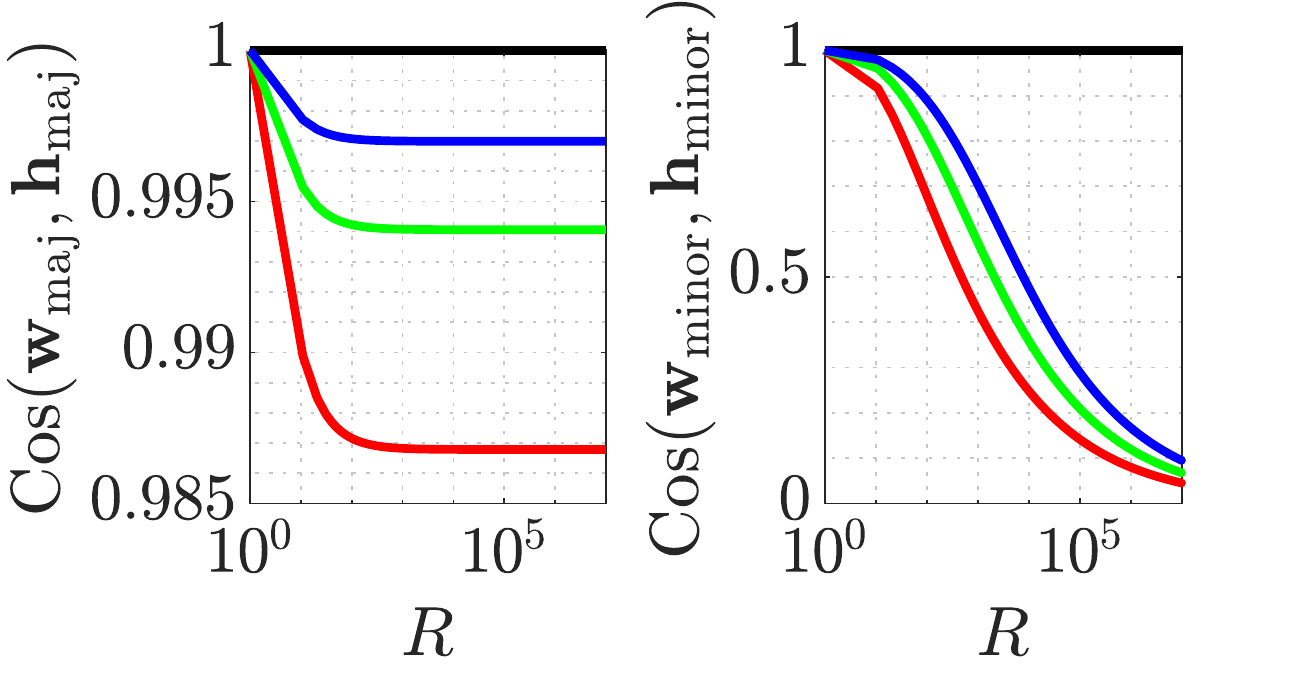}
         \caption{Alignment of classifiers and embeddings.}
         \label{fig:SELI_theory_angles_wh_inf}
     \end{subfigure}
     \\
          \begin{subfigure}[b]{0.75\textwidth}
         \centering
              \includegraphics[width=\linewidth]{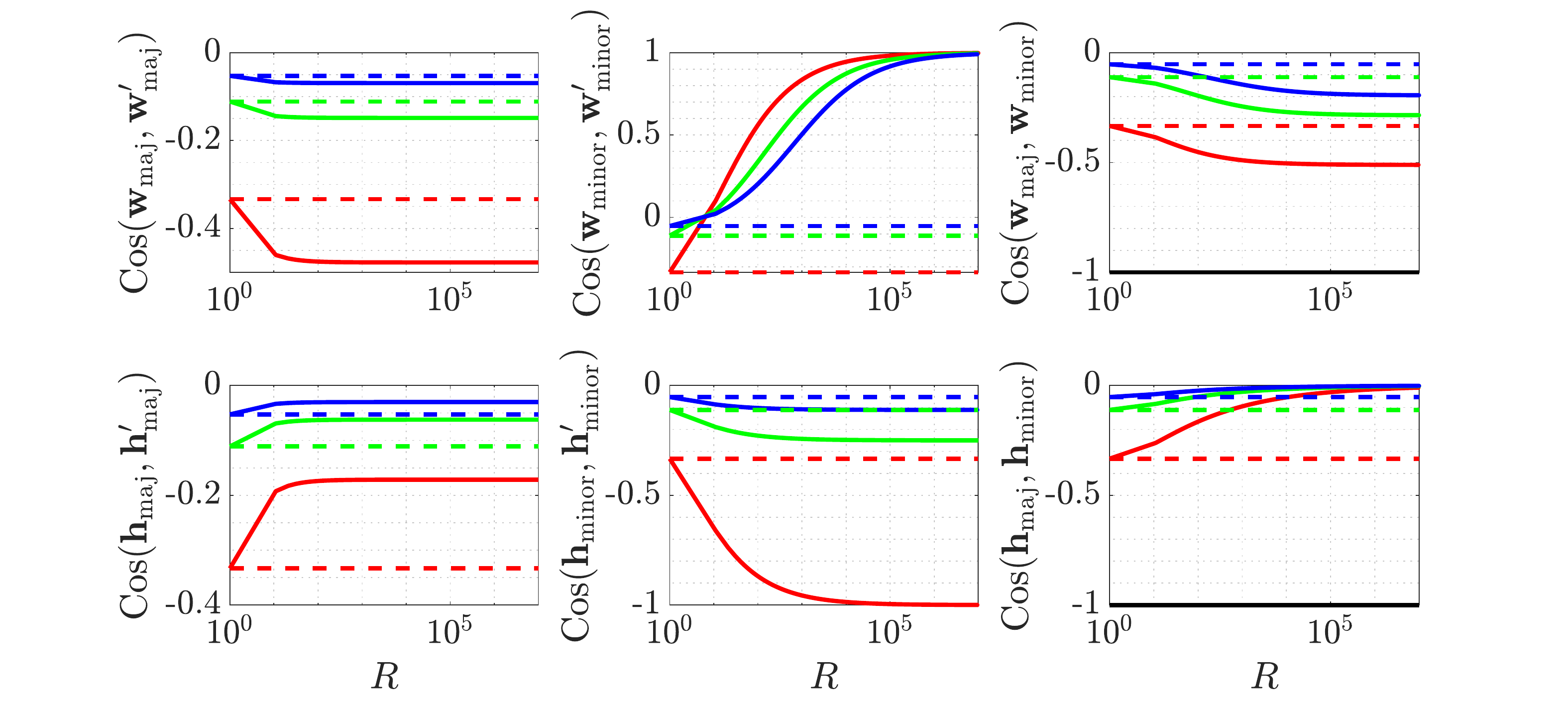}
         \caption{Angles of classifiers and embeddings. Dashed lines show the value $-1/(k-1)$ characterizing the ETF geometry.}
         \label{fig:SELI_theory_angles_inf}
     \end{subfigure}
        \caption{Illustration of how norms and angles of the SELI geometry vary with the number of classes $k$ and the imbalance ratio $R$. The minorities fraction is set to $\rho=1/2$. Unless $R=1$ or $k=2$, the geometry differs from the ETF geometry. For example, majority classifiers have larger norms than minorities (see Fig. \ref{fig:SELI_theory_norms_inf}), minority classifiers are less aligned with their corresponding class mean-embeddings  (see \Fig~\ref{fig:SELI_theory_angles_wh_inf}), and the angle between minority classifiers decreases (see \Fig~\ref{fig:SELI_theory_angles_inf}). The depicted values are computed thanks to closed-form formulas. See \Sec~\ref{sec:SELI_changes} and \Sec~\ref{sec:SELI_properties}.
        }
        \label{fig:SELI_theory_inf}
\end{figure}

\Sec~\ref{sec:SELI_properties} includes additional closed-form formulas for the angles and alignment of classifiers and embeddings. Thanks to these, it is easy to precisely quantify the changes in the geometry as a function of the imbalance ratio and of number of classes. As another example of this, besides Lemma \ref{lem:norms}, the SELI geometry leads to majority / minority angles for the classifiers satisfying the following formulas: (with $(R,\nicefrac{1}{2})$-STEP imbalance; see \Sec~\ref{sec:SELI_angles} for details)
\begin{align}
\Cos{\wmaj}{\wmaj'}&=\frac{-2\sqrt{R}+\sqrt{(R+1)/2}}{(k-2)\sqrt{R}+\sqrt{(R+1)/2}} \nn
\\
\Cos{\wmin}{\wmin'}&=\frac{R-7}{R-7+2k\big(2+\sqrt{(R+1)/2}\big)}\,.\label{eq:cos_min_min}
\end{align}
Here, $\wmaj,\wmaj'$ are two classifier vectors corresponding to any two majority classes (similarly for the minorities.) It is easy to see that both these formulas evaluate to $-1/(k-1)$ for $R=1$. Furthermore, the cosine of majorities is strictly decreasing, while the cosine of minorities is strictly increasing. That is, with increasing imbalance ratio majority classifiers go further away from each other, while minority classifiers come closer. These properties, together with the fact from Lemma \ref{lem:norms} that majority classifiers have larger norms compared to minorities, is visualized in \Fig~\ref{fig:intro_SEL} for $R=10$ and $k=4$. Additionally, \Fig~\ref{fig:SELI_theory_inf} shows how all the norms and angles of the geometry vary with the imbalance ratio $R$ for several values of $k=2,4,10,20.$

\begin{remark}[Asymptotics]
While we focus on understanding the geometry at finite values of $R$, it is possible to evaluate limits for our formulas giving asymptotic characterizations as $R\rightarrow\infty.$ The asymptotic behaviors can also be observed in \Fig~\ref{fig:SELI_theory_inf}. As an example, it is easy to see from \eqref{eq:cos_min_min} that the angle between the minority classifiers collapses to zero in that limit. This phenomenon is called ``minority collapse'' by \citet{fang2021exploring}. Here, we recover it as a special case of Theorem \ref{thm:SVM} and of the SELI characterization. Note also that the rate at which the minority angle collapses is rather slow (e.g. see \Fig~\ref{fig:SELI_theory_angles_inf}). Additional details and discussion on the SELI geometry are included in \Sec~\ref{sec:SELI_properties}. 
\end{remark}

\subsection{Proof sketch}\label{sec:SVM_proof_sketch}
We start from the following (by now, relatively standard) convex relaxation of the UF-SVM \cite{srebro2004maximum,haeffele2019structured,zhu2021geometric,fang2021exploring}: 
\begin{align}\label{eq:nuc_norm_SVM}
\min_{\Z\in\R^{k\times n}} \|\Z\|_*\quad\text{ subj. to }~ \Z[y_i,i]-\Z[c,i]\geq 1,\,\forall c\neq y_i, i\in[n].
\end{align}
The relaxation follows by setting $\Z=\W^T\Hb$, thus $\Z$ is the logit matrix (also, the end-to-end model) of the non-convex UF-SVM. Our key technical innovation is proving that $\Zhat$ is the unique minimizer of \eqref{eq:nuc_norm_SVM}. There are three key ingredients in this. First, is a clever re-parameterization of the dual program to \eqref{eq:nuc_norm_SVM}, introducing the \SEL~matrix $\Zhat$ in the dual:
\begin{align}\label{eq:dual_main}
\max_{\Bb\in\R^{n\times k} }~~~\tr(\Zhat\Bb)\quad\text{sub.~to}~\|\Bb\|_2\leq 1, ~~
 \Bb\ones_k=0, ~~ \Bb\odot\Zhat^T \geq 0\,.
\end{align}
Second, we prove that $\hat\Bb=\Ub\Vb^T$ is the unique maximizer of the re-parameterized dual problem in \eqref{eq:dual_main}. While it is not hard to check that $\hat\Bb$ optimizes a relaxation of \eqref{eq:dual_main}, it is far from obvious that $\hat\Bb$ is unique, and, even more that it satisfies the third constraint. The key technical challenge here is that the third constraint acts entry-wise on $\Bb$. In fact, to proceed with the proof we need that the constraint is \emph{not} active, i.e. $\hat\Bb\odot\Zhat^T > 0$, or equivalently, that the sign pattern of the entries of $\hat\Bb$ agrees with the sign pattern of the transpose SEL~matrix $\Zhat^T$. We prove this by an explicit construction of the singular factors $\Ub,\Vb$ exploiting the structure of the \SEL~matrix. Once we have shown that $\hat\Bb$ is the unique maximizer and is strictly feasible, we use the KKT conditions to prove that $\Zhat$ is the unique minimizer of the nuclear-norm relaxation in \eqref{eq:nuc_norm_SVM}. To do this, we leverage that strict feasibility of $\hat\Bb$ implies by complementary slackness all constraints in the primal \eqref{eq:nuc_norm_SVM} must be active at the optimum. The proof of the theorem completes by arguing that the relaxation \eqref{eq:dual_main} is tight when $d\geq k-1$ allowing us to connect the UF-SVM minimizers to the SEL~matrix $\Zhat$. See \Sec~\ref{sec:proof_SVM} for details.

\begin{remark}[Comparison to literature]\label{remark:comparison}
The common analysis strategy in \emph{all} other related works is deriving tight bounds on the CE loss (or related quantities, such as the minimum margin), and, then identifying the structure in the parameters that achieves those bounds. For example, \cite{zhu2021geometric,graf2021dissecting,fang2021exploring} lower bound the CE loss and \cite{ULPM} upper bounds the minimum margin, all using a similar elegant argument based on Cauchy-Schwartz and Jensen inequalities. The ETF geometry is then uncovered by recognizing that it uniquely achieves those bounds. It is not clear how to employ such exercises in the presence of imbalances, due to the absence of  symmetry properties (e.g. alignment of classifiers with embeddings). Our proof of Theorem \ref{thm:SVM} is more direct and is in large enabled by identifying the key role played by the logit~matrix.
\end{remark}




\section{The role of regularization}\label{Sec:reg}
In this section, we focus on the ridge-regularized CE loss minimization in \eqref{eq:CE}. Specifically, we  study the geometry of solutions $(\What_\la,\Hhat_\la)$ as a function of both the imbalance and the regularization parameter $\la$.

\subsection{Global minimizers as solutions to a convex relaxation}\label{sec:nuc_norm_ridge}

The regularized CE minimization in \eqref{eq:CE} is non-convex. Yet, its  landscape is benign and the global solution can be described in terms of the solution to a convex relaxation program \cite{zhu2021geometric}. 

\begin{theorem}[Reformulated from \cite{zhu2021geometric}]\label{thm:regularized}
Let $\la>0$, $d> k-1$ and a $(R,\rho)$-STEP imbalance setting. Let $\Zhat_\la\in\R^{k\times n}$ be the unique minimizer of the convex nuclear-norm-regularized loss minimization,
\begin{align}\label{eq:nuc_norm_reg}
\Zhat_\la := \arg\min_\Z\, \Lc(\Z)+\la\|\Z\|_*\,,
\end{align}
and, denote $\Zhat_\la=\Vb_\la\Lambdab_\la\Ub_\la^T$ its SVD. Any stationary point of \eqref{eq:CE} satisfies  exactly one of the following two. Either it is a strict saddle, or, it is a global minimizer $(\What_\la,\Hhat_\la)$ and satisfies 
\begin{align}\label{eq:geometry_reg}
\begin{bmatrix} \What_\la^T \\ {\Hhat_\la}^T \end{bmatrix} \begin{bmatrix} \What_\la & \Hhat_\la \end{bmatrix} = 
\begin{bmatrix}\Vb_\la \\ \Ub_\la \end{bmatrix}\Lambdab_\la \begin{bmatrix}\Vb_\la^T & \Ub_\la^T\end{bmatrix}.
\end{align}
\end{theorem}

First, the statement ensures that any first-order method escaping strict saddles finds a stationary point that is a global minimizer \cite{zhu2021geometric}. Moreover, 
it describes the structure of the global minimizers of \eqref{eq:CE} in terms of $\Zhat_\la$, the solution to the \emph{convex}  minimization in \eqref{eq:nuc_norm_reg}.
Structurally, the characterization in \eqref{eq:geometry_reg} resembles the characterization in Theorem \ref{thm:SVM} regarding the UF-SVM. However, Theorem \ref{thm:SVM} goes a step further and gives an \emph{explicit} form for the logit matrix, namely the \SEL~matrix $\Zhat$. Instead, $\Zhat_\la$ in Theorem~\ref{thm:regularized} is given implicitly as the solution to a convex program. In the remaining of this section, we ask: \emph{how does $\Zhat_\la$ compare to $\Zhat$ for different values of the regularizer?} Also, \emph{how does the answer depend on the imbalance level?}

\begin{remark}\label{rem:nuc_norm_ridge} 
Although not stated explicitly in this form, Theorem \ref{thm:regularized} is essentially retrieved from the proof of \citep[Theorem 3.2]{zhu2021geometric} with two small adjustments. First, \citet{zhu2021geometric} only considers balanced data. Here, we realize  their proof actually carries over to the imbalanced setting. Second (less important), we relax their assumption $d> k$ to $d> k-1$ thanks to a simple observation: $\ones_{k}^T\nabla_{\Z} \Lc(\Z) = \zeros$, hence the CE gradient drops rank (see \Sec~\ref{sec:proof_regularized_thm} for details).
\end{remark}

\subsection{Regularization matters}\label{sec:reg_matters}

For balanced data, previous works have shown that the minimizers $\What_\la,\Hhat_\la$ of \eqref{eq:CE} satisfy the \emph{\ref{NC}} and \emph{\ref{ETF}} properties (up to scaling by a constant) for \emph{every} value of the regularization parameter $\la>0$ \cite{zhu2021geometric} (see also \cite{graf2021dissecting,fang2021exploring,lu2020neural}.) In our language, \textbf{for all $\la>0$, there exists} scalar $\alpha_\la$ such that a scaling $(\alpha_\la\What_\la,\alpha_\la\Hhat_\la)$ of any global solution of the regularized CE minimization in \eqref{eq:CE} satisfies the ETF geometry. 
Thus,\emph{
for balanced data, up to a global scaling, the geometry is: (i) insensitive to $\la>0$ and (ii) the same as that of the UF-SVM minimizers. 
}


Here, we show that the situation changes drastically with imbalances: the regularization now plays a critical role and the solution is never the same as that of UF-SVM for finite $\la$. 
\begin{propo}[Imbalanced data: Regularization matters]\label{propo:imbalance_reg}
Assume imbalanced data \emph{and} $k>2$. There \textbf{does not exist} finite $\lambda>0$ and corresponding scaling $\alpha_\la$ such that the scaled solution $(\alpha_\la\What_\la,\alpha_\la\Hhat_\la)$ of \eqref{eq:CE} follows the \ref{SELI} geometry. Equivalently, there does not exist $\la>0$ and $\alpha_\la$ such that a scaling of the UF-SVM solution solves \eqref{eq:CE}.
\end{propo}

\begin{proof}The proof relies on Theorem \ref{thm:SVM} as follows. For the sake of contradiction assume there exists $\la>0$ and some $\alpha_\la>0$ such that the scaled UF-SVM minimizer $(\alpha_\la\What,\alpha_\la\Hhat)$ solves \eqref{eq:CE}. Since then  $(\alpha_\la\What,\alpha_\la\Hhat)$ is a stationary point, it satisfies 
\begin{align*}
\nabla_\W\Lc(\alpha_\la^2\What^T\Hhat)+\la\alpha_\la\hat\W=0 
&\implies \alpha_\la\Hhat\, \left(\nabla_\Z\Lc(\alpha_\la^2\What^T\Hhat)\right)^T = -\la \What
\\
&\implies \alpha_\la\What^T\Hhat\, \left(\nabla_\Z\Lc(\alpha_\la^2\What^T\Hhat)\right)^T = -\la \What^T\What
.
\end{align*}
But, by Theorem \ref{thm:SVM}: $\What^T\Hhat=\Zhat$ and $\What^T\What=\Vb\Lambdab\Vb^T$. Moreover, thanks to the special structure of $\Zhat$ we can check that $\nabla_\Z\Lc(\alpha_\la^2\Zhat)=-\alpha_\la^\prime\Zhat$ for $\alpha_\la^\prime:={k}\big/({\exp({\alpha_\la^2})+k-1)}$; see Lemma \ref{lem:Zhat}\ref{state:SEL_prop_CE}. With these, and denoting $\alpha_\la^{\prime\prime}:=\alpha_\la\alpha_\la^\prime$, we arrive at the following  about the singular values of $\Zhat$:
$$
\alpha_\la^{\prime\prime}\Zhat\Zhat^T = \la\What^T\What \implies
\alpha_\la^{\prime\prime}\Vb\Lambdab^2\Vb^T = \la\Vb\Lambdab\Vb^T 
\implies \alpha_\la^{\prime\prime}\Lambdab^2 = \la\Lambdab \implies \Lambdab = \big({\la}/{\alpha_\la^{\prime\prime}}\big)\Id_{k-1}.
$$
Thus, all singular values of $\Zhat$ must be the same. However, we show in Lemma \ref{lem:Zhat_SVD_general} that this is \emph{not} the case unless data are balanced or $k=2$. Thus, we arrive at a contradiction completing the proof of the proposition.
\end{proof}

\subsection{Vanishing regularization}
As $\la$ vanishes, it is not hard to check that  minimizers diverge in norm and the relevant question becomes: where do they converge in direction? The following answers this.

\begin{propo}[Regularization path leads to UF-SVM]\label{propo:reg_path}
 Suppose $d> k-1$ and $(R,\rho)$-STEP imbalance. 
 It then holds that $\lim_{\la\rightarrow0}\nicefrac{\What_\la^T\Hhat_\la}{{(\|\\\What_\la\|_F^2/2+\|\Hhat_\la\|_F^2/2)}}=\nicefrac{\Zhat}{\|\Zhat\|_*}$.
\end{propo}

Put together with the content of the previous section: For balanced data, the solution is always the same up to global scaling. However, for imbalanced data, the solution changes with $\la$ and only in the limit of $\la\rightarrow0$ does it align with that of the UF-SVM. 

Regarding the proof of the proposition, we note that thanks to Theorems \ref{thm:SVM} and \ref{thm:regularized}, it suffices that the solution $\Zhat_\la$ of \eqref{eq:nuc_norm_reg} converges in direction to the \SEL~matrix $\Zhat$; see Proposition \ref{propo:vanish_cvx} in \Sec~\ref{sec:nuc_norm_SM}. To show  this, we critically use from Theorem \ref{thm:SVM} that $\Zhat$ is  \emph{unique} minimizer of  \eqref{eq:nuc_norm_SVM}. See  \Sec~\ref{sec:proof_reg_path_noncvx} for details. The most closely related results are those of \cite{rosset2003margin,ji2020gradient}, who studied the regularization path of $p$-norm regularized CE. 

\vspace{-0.2cm}
\subsection{Imbalance emphasizes the impact of non-convexity}\label{sec:linear_ridge}
\vspace{-0.2cm}
Recall interpreting the UFM as a two-layer linear net trained on  the standard basis $\eb_i\in\R^n.$ Suppose instead that we train a simple $k$-class linear classifier $\Xibf_{k\times n}$  on the same data by minimizing ridge regularized CE:
$\min_{\Xibf}~\Lc(\Xibf)+\frac{\la}{2}\|\Xibf\|_F^2.$
It is easy to check that (after scaling) $\Zhat$ satisfies first-order optimality conditions. 
 Thus, the optimal linear classifier is such that for all $\la>0$, there exists $\alpha_\la$ such that $\hat{\Xibf}_\la=\alpha_\la\Zhat$. 
%
Contrasting this to Proposition \ref{propo:imbalance_reg}, we find that the end-to-end models minimizing ridge-regularized CE for a linear versus a two-layer linear network are the same (in direction) when data are balanced, but differ under imbalances. 



\section{Experiments}\label{sec:rel}
%

In our experiments we choose $(R,\nicefrac{1}{2})$-STEP imbalances with varying imalance ratio $R$. In all cases we measure convergence to either the \emph{\ref{SELI}} or the \emph{\ref{ETF}} geometries, in terms of the three metrics below corresponding to classifiers, embeddings, and logits, respectively. Denote $\underline{\A}=\nicefrac{\A}{\|\A\|_F}$ the Euclidean normalization  and $\G_\A=\A^T\A$ the Gram matrix of matrix $\A$.  
\begin{itemize}
	\item[$\square$] \textbf{~Classifiers:} We measure $\|\underline{\G_\W}-\underline{\Ghat_\W}\|_F$, where $\Ghat_\W^{\text{ETF}}=\Id_k-\frac{1}{k}\ones_k\ones_k^T=:\G^\star$ and $\Ghat_\W^{\text{SELI}}=\Vb\Lambdab\Vb^T$ (see Definition \ref{def:SELI}). 
	\item[$\square$]\textbf{Embeddings:} Because of the NC property, it suffices to work with the class means $\mub_c$ and the corresponding matrix $\M$ of mean embeddings (see \Sec ~\ref{sec:NC_and_SELI}.) Specifically, we measure  $\|\underline{\G_\M}-\underline{\Ghat_\M}\|_F$, where $\Ghat_\M^{\text{ETF}}=\G^\star$ and $\Ghat_\M^{\text{SELI}}$ is computed from the $n$-dimensional $\Ghat_\Hb^{\text{SELI}}=\Ub\Lambdab\Ub^T$ by only keeping the $k$ columns/rows corresponding to the first example of each class. 
	\item[$\square$]\textbf{Logits:} We measure $\|\underline{\W^T\M}-\underline{\G^\star}\|_F$. Note that, when NC holds this metric is essentially analogous to measuring $\|\underline{\W^T\Hb}-\underline{\Zhat}\|_F.$
\end{itemize}

\new{\noindent\textbf{Centering for deep-net experiments.}~In the experiments in Sec. \ref{sec:deep-net-main},  when investigating the embeddings' geometry, we employ an additional centering of the class means with their  (balanced) global mean. Specifically, we compute for each $c\in[k]:$
$
\overline{\mub_c} = \mub_c - \mub_G $
with $\mub_G = \frac{1}{k} \sum_{c \in [k]} \mub_c$. Thus, at each epoch, we compute $\overline\M = [\overline{\mub_1}, \overline{\mub_2}, ..., \overline{\mub_k}]\in \mathbb{R}^{d \times k}$ and compare, after normalization, $\G_\M=\overline{\M}^T\overline{\M}$ to $\hat\G_\M$, which we calculate as described above for the ETF and SELI geometries, respectively. On the other hand, for the logits calculations, we compute $\Z=\W^T\M$ without centering. The discrepancy between the UFM solutions being already centered, while deep-net embeddings require centering before computation of geometric measures is a common denominator in all previous works on the UFM, e.g.  \cite{zhu2021geometric,ULPM,fang2021exploring,han2021neural,lu2020neural,mixon2020neural,graf2021dissecting,zhou2022optimization}. Here, the chosen centering $\mub_G=\frac{1}{k} \sum_{i\in[n]}\frac{1}{n_{y_i}}\h_i$ is in general different (and the same only when classes are balanced) from the global centering $\h_G:=\frac{1}{n}\sum_{i\in[n]}\h_i$ used by \citet{NC}. Our choice of the former is motivated by the fact that the SELI geometry (as predicted by the UFM) satisfies $\sum_{c\in[k]}\overline{\mub_c}=0$ (see Eqn. \eqref{eq:center_mu} in Sec. ), but not always $\frac{1}{n}\sum_{i\in[n]}{\h_i}=0.$}

\begin{figure*}
\vspace{-30pt}
    \centering
   \hspace{-60pt} \begin{subfigure}{0.26\textwidth}
    \centering
        \begin{tikzpicture}
        \node at (-1.4,-1.4) {\includegraphics[scale=0.24]{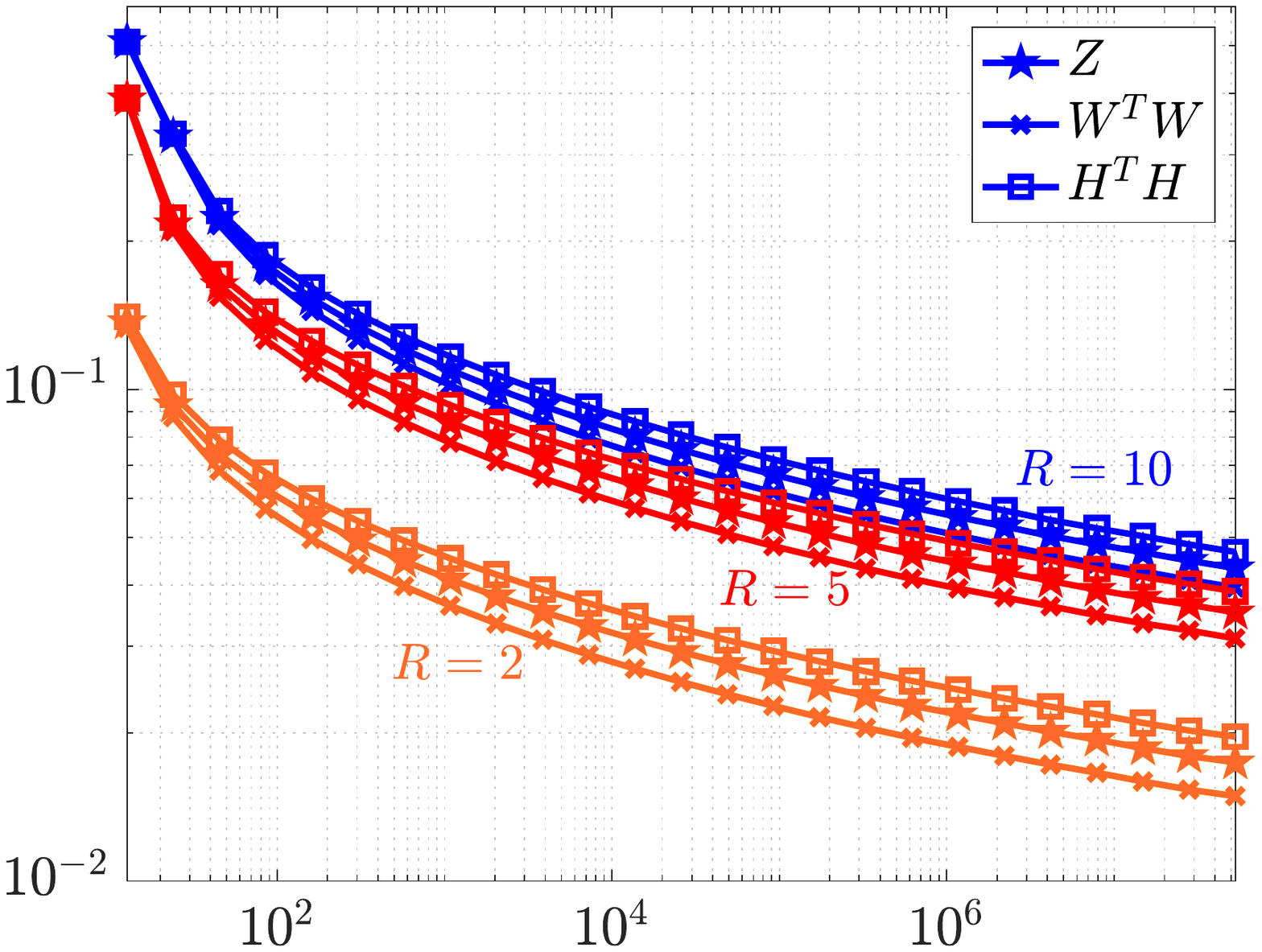}};
        \node at (-1.3,-3.1) [scale=0.9]{$n/\lambda$};
        \node at (-3.7,-1.4)  [scale=0.8, rotate=90]{Distance to SELI};
        \end{tikzpicture}\vspace{-50pt}\captionsetup{width=0.9\linewidth}\caption{SELI~Convergence.}\label{fig:hr_global_local}
    \end{subfigure}\hspace{25pt}\begin{subfigure}{0.26\textwidth}
    \centering
        \begin{tikzpicture}
        \node at (0,0) {\includegraphics[scale=0.24]{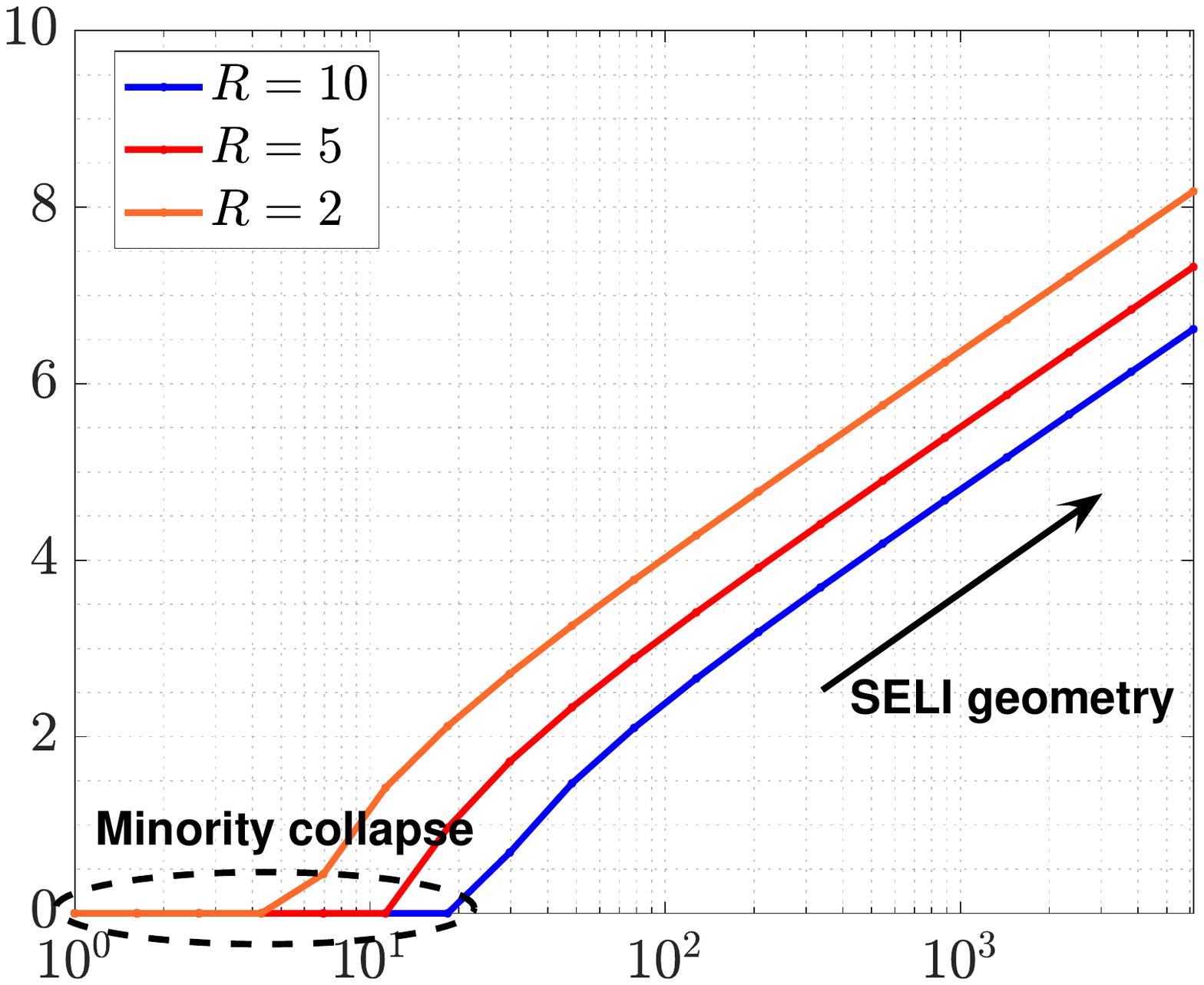}};
\node at (0,-1.8) [scale=0.9]{$n/\lambda$};
        \node at (-2.1,0) [scale=0.8, rotate=90]{Min. margin};
        \end{tikzpicture}\vspace{-50pt}\captionsetup{width=0.9\linewidth}\caption{Minimum margin.}\label{fig:UFM_nuc_norm_lambda_margin}
    \end{subfigure}\hspace{25pt}\begin{subfigure}{0.26\textwidth}
    \centering
        \begin{tikzpicture}
        \node at (0,0) {\includegraphics[scale=0.24]{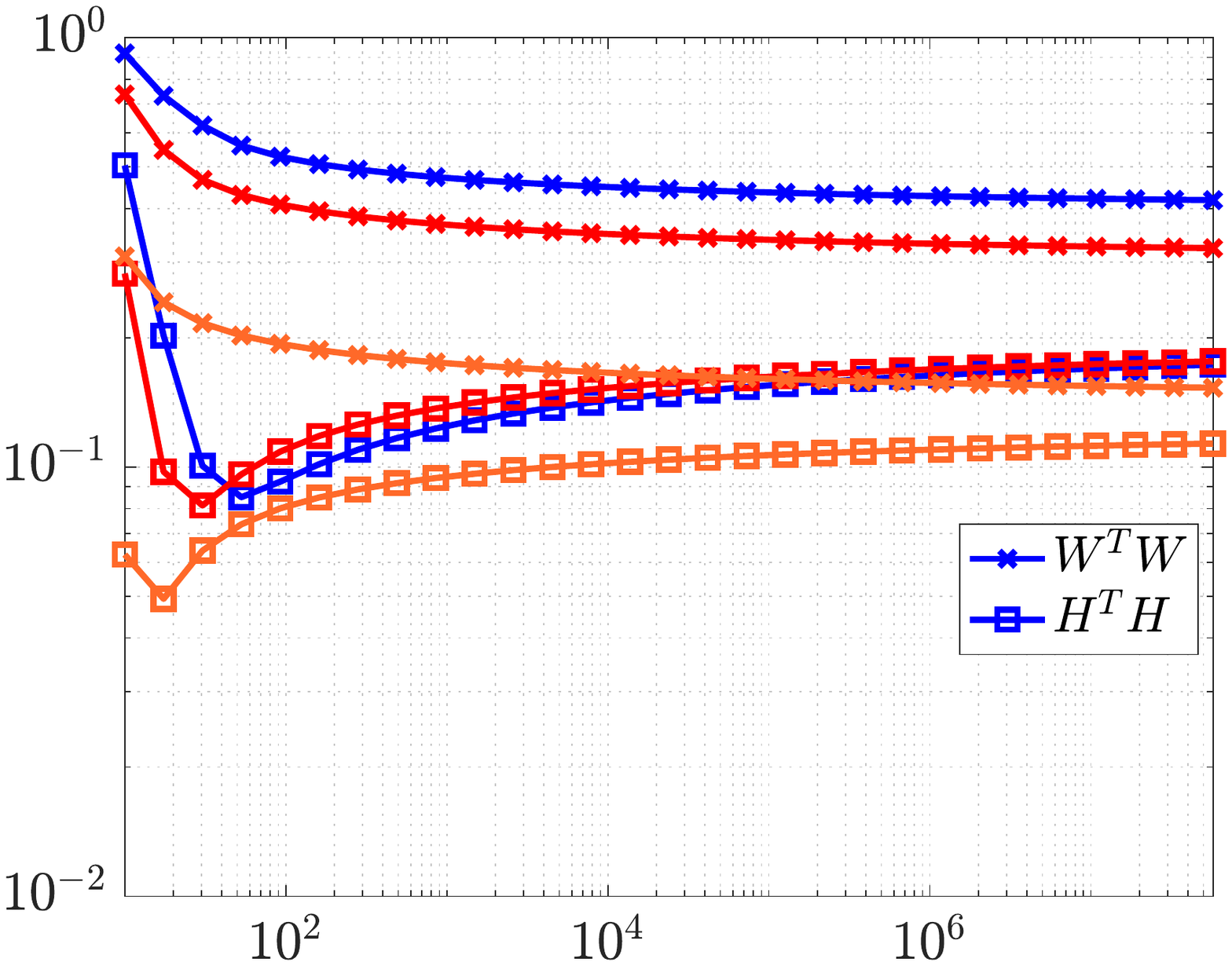}};
        \node at (0,-1.7) [scale=0.9]{$n/\lambda$};
        \node at (-2.1,0.1) [scale=0.8, rotate=90]{Distance to ETF};
        \end{tikzpicture}\vspace{-50pt}\captionsetup{width=0.9\linewidth}\caption{ETF convergence.}\label{fig:minority_collapse}
    \end{subfigure}\vspace{-5pt}\caption{Numerical study of global solutions $(\What_\la,\Hhat_\la)$ of \eqref{eq:CE} across the regularization path.}\label{fig:UFM_nuc_norm_lambda}
\end{figure*}


\subsection{UFM: Global minimizers}
\Fig~\ref{fig:UFM_nuc_norm_lambda} numerically investigates the behavior of global minimizers of regularized CE \eqref{eq:CE} for the UFM and $k=4$ classes. Thanks to Theorem~\ref{thm:regularized}, we obtain such minimizers by solving the convex program \eqref{eq:nuc_norm_reg} with CVX \cite{cvx}, and then, using \eqref{eq:geometry_reg} to infer the Gram matrices $\G_\W,\G_\M$ and logits $\W^T\Hb$. \Fig~\ref{fig:minority_collapse} shows that the distance to ETF is large and not approaching zero for any value of $\la$. On the other hand, \Fig~\ref{fig:hr_global_local} numerically validates Propositions~\ref{propo:imbalance_reg} and \ref{propo:reg_path}: the distance to SELI for all three metrics is non-zero for any finite $\la>0$, but converges to zero as $\la\rightarrow0$. However, this convergence is slow and the rate becomes even worse as $R$ increases. Finally, \Fig~\ref{fig:UFM_nuc_norm_lambda_margin} depicts the minimum margins of solutions across $\la$. Note that for all sufficiently small $\la$ values, the minimum margin is strictly positive. We prove this in Lemma \ref{lem:small_la} in \Sec~\ref{sec:small_la}.  As a byproduct, this shows that the ``minority collapse'' of \cite{fang2021exploring} can only possibly occur for large $\la$; see also \Sec~\ref{sec:minority}.

\begin{figure*}
	\centering
	\hspace{-40pt} \begin{subfigure}{0.3\textwidth}
		\centering
		\begin{tikzpicture}
			\node at (-1.4,-1.4) 
			{\includegraphics[scale=0.26]{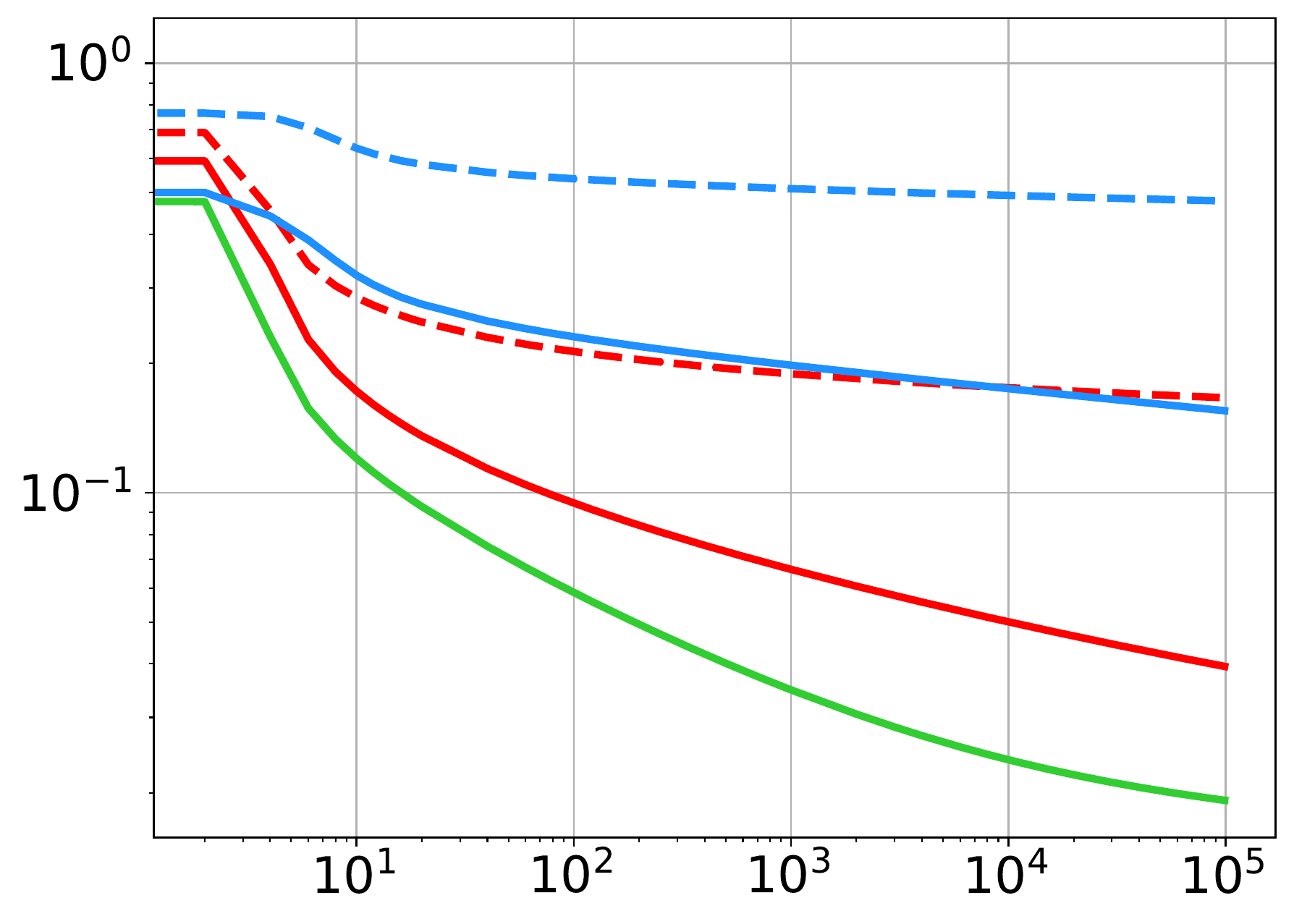}};
			\node at (-1.4,-3.5) [scale=0.7]{Epochs};
			\node at (-4.1,-1.4)  [scale=0.7, rotate=90]{Distance to SELI/ETF};
		\end{tikzpicture}\caption{\small{Classifiers}}\label{fig:UFM_W_hat}
	\end{subfigure}\hspace{25pt}\begin{subfigure}{0.3\textwidth}
		\centering
		\begin{tikzpicture}
			\node at (0,-1.4) {\includegraphics[scale=0.26]{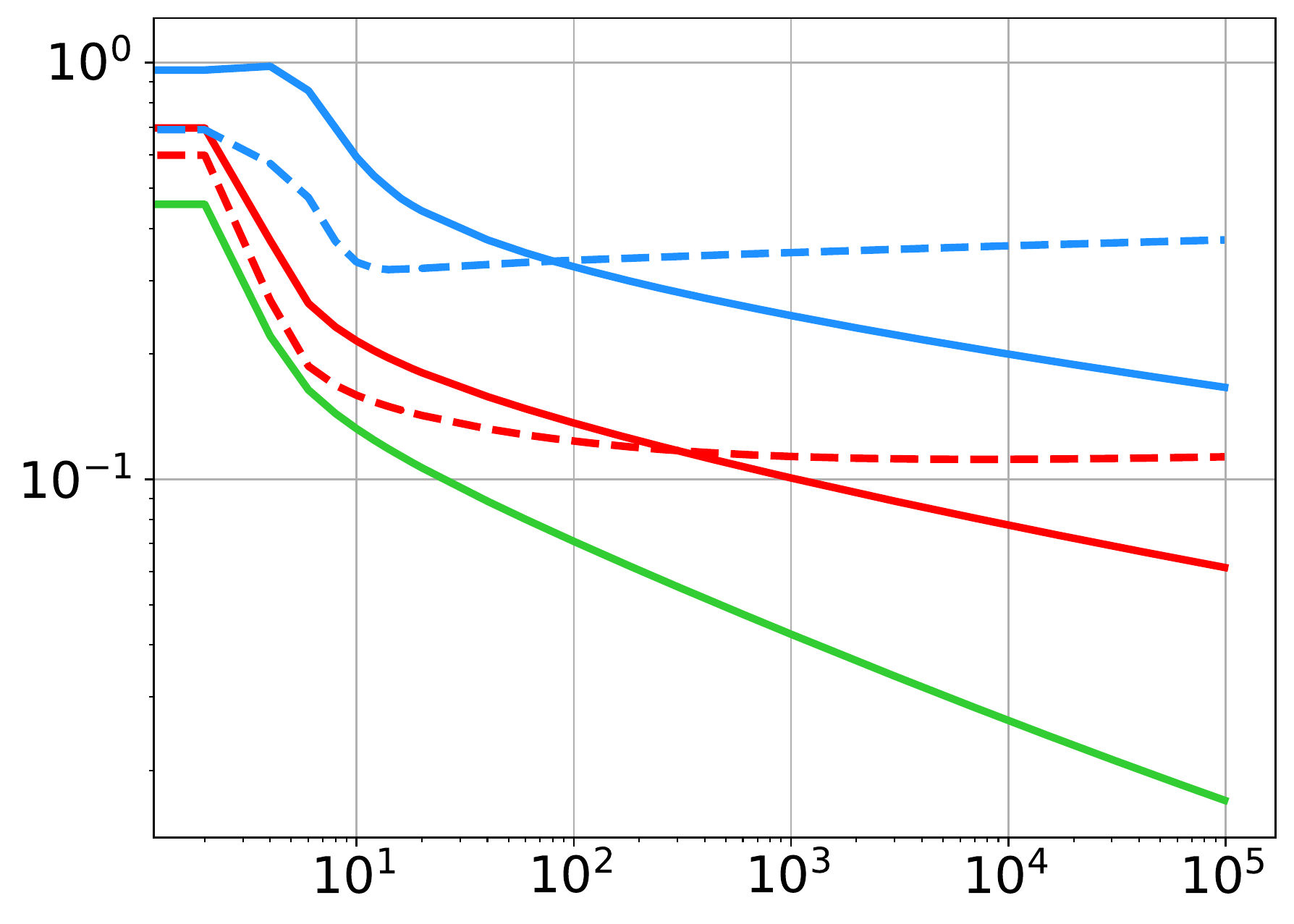}};
			\node at (0,-3.5) [scale=0.7]{Epochs};
			\node at (-2.7,-1.4) [scale=0.7, rotate=90]{};
		\end{tikzpicture}\caption{\small{Embeddings}}\label{fig:UFM_H_hat}
	\end{subfigure}\hspace{25pt}\begin{subfigure}{0.3\textwidth}
		\centering
		\begin{tikzpicture}
			\node at (0,-1.4) {\includegraphics[scale=0.26]{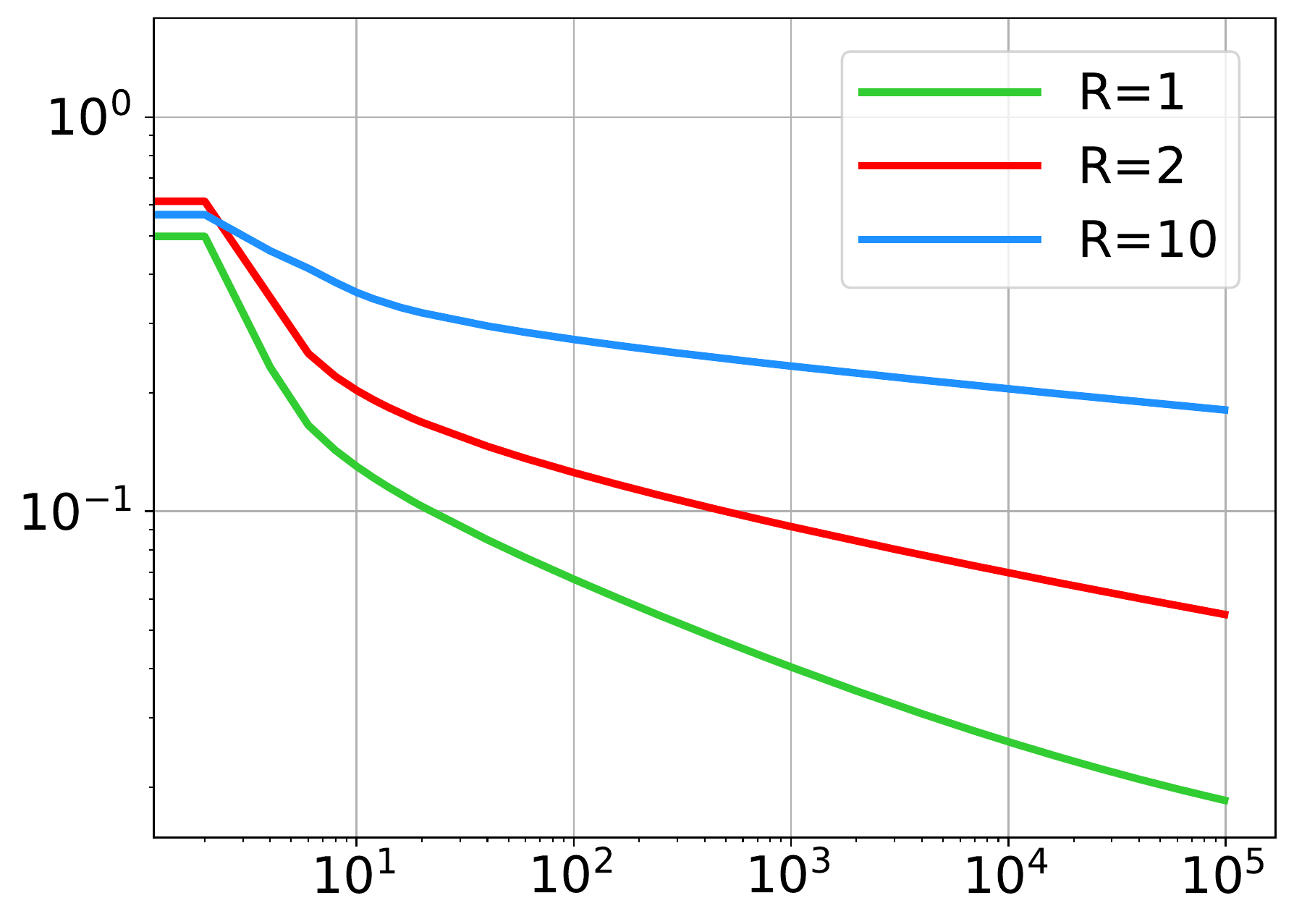}};
			\node at (0,-3.5) [scale=0.7]{Epochs};
			\node at (-2.7,-1.4) [scale=0.7, rotate=90]{};
		\end{tikzpicture}\caption{\small{Logits}}\label{fig:UFM_Z_hat}
	\end{subfigure}
	\vspace{-5pt}\caption{Geometry of SGD solutions on minimizing CE for the  UFM; SELI(Solid)/ETF(Dashed). 
	}
	\label{fig:UFM_convergence}
\end{figure*}
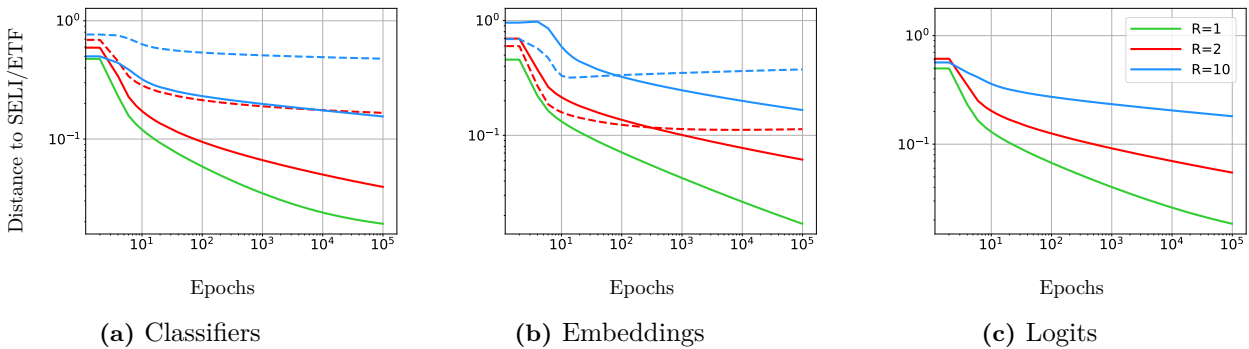

\subsection{UFM: SGD solutions}
\Fig~\ref{fig:UFM_convergence} investigates whether the solutions found by SGD are consistent with the prediction of Theorem~\ref{thm:SVM} about global minimizers of the UF-SVM. Specifically, we fix $k=4$ and, for each $R$, we select the sample size $\nmin$ for minorities so that the total number of samples $n=\big((R+1)/2\big)k\nmin$ is $\approx400$. 
 The weights of the UFM are optimized using SGD with constant learning rate $0.4$, batch size $4$ and no weight decay. We train for $10^5$ epochs, much beyond zero training error and plot the distance to SELI and ETF geometries for classifiers, embeddings and logits over time. We highlight the following three observations: \textbf{(i)} SGD iterates favor the SELI, instead of the ETF geometry. As a matter of fact, the distance to SELI is decreasing with epochs, suggesting an implicit bias of SGD towards global minimizers of the UF-SVM. 
\textbf{(ii)} However, convergence is rather slow and rates get worse with increasing imbalance. 
\textbf{(iii)} Also, the embeddings convergence is more elusive compared to that of the classifiers. Interestingly, the last two observations are reminiscent of the trends we observed in \Fig~\ref{fig:hr_global_local}, suggesting connections between regularization path and (S)GD iterates, worth investigating further. 

We refer the reader to \Sec~\ref{sec:UFM} for additional numerical results on the UFM, such as experiments with varying weight-decay and regularization choices.

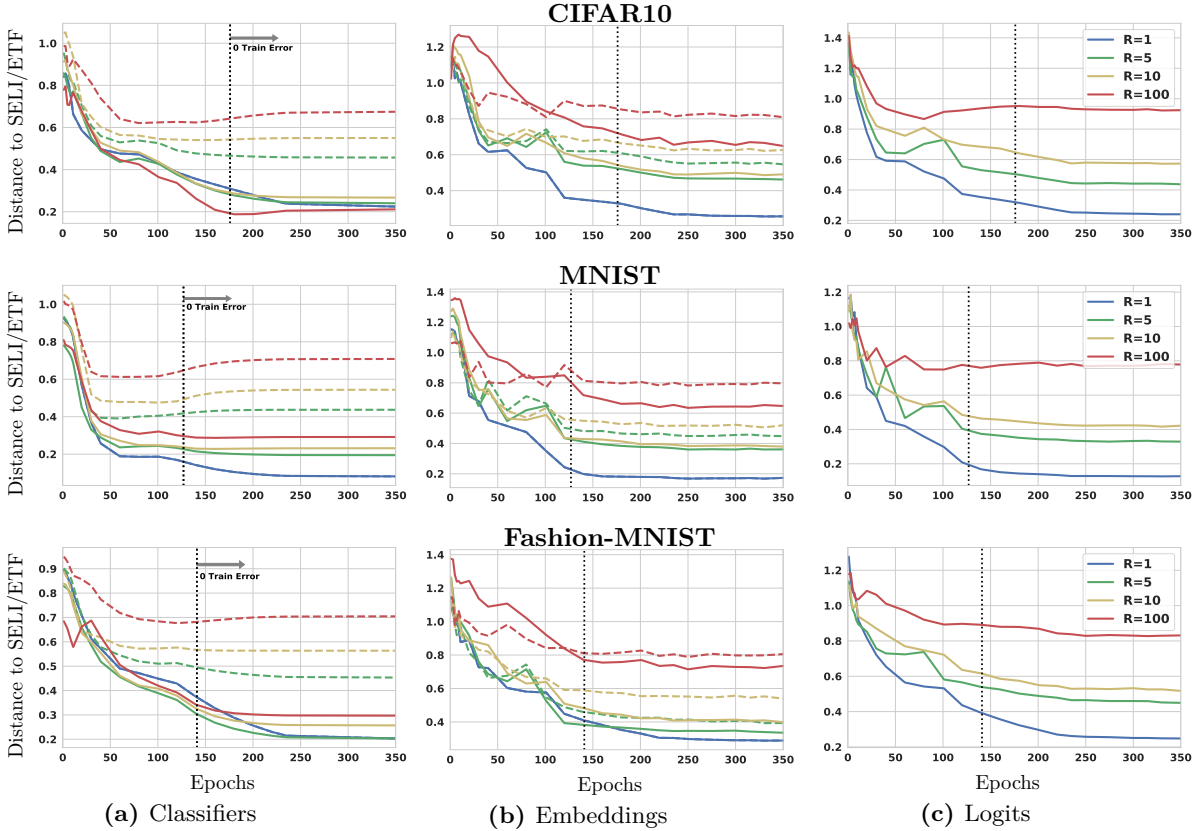
\begin{figure*}
	\vspace{-10pt}
	\centering
	\hspace{-40pt} \begin{subfigure}{0.3\textwidth}
		\centering
		\begin{tikzpicture}
			\node at (-1.4,-1.4) 
			{\includegraphics[scale=0.26]{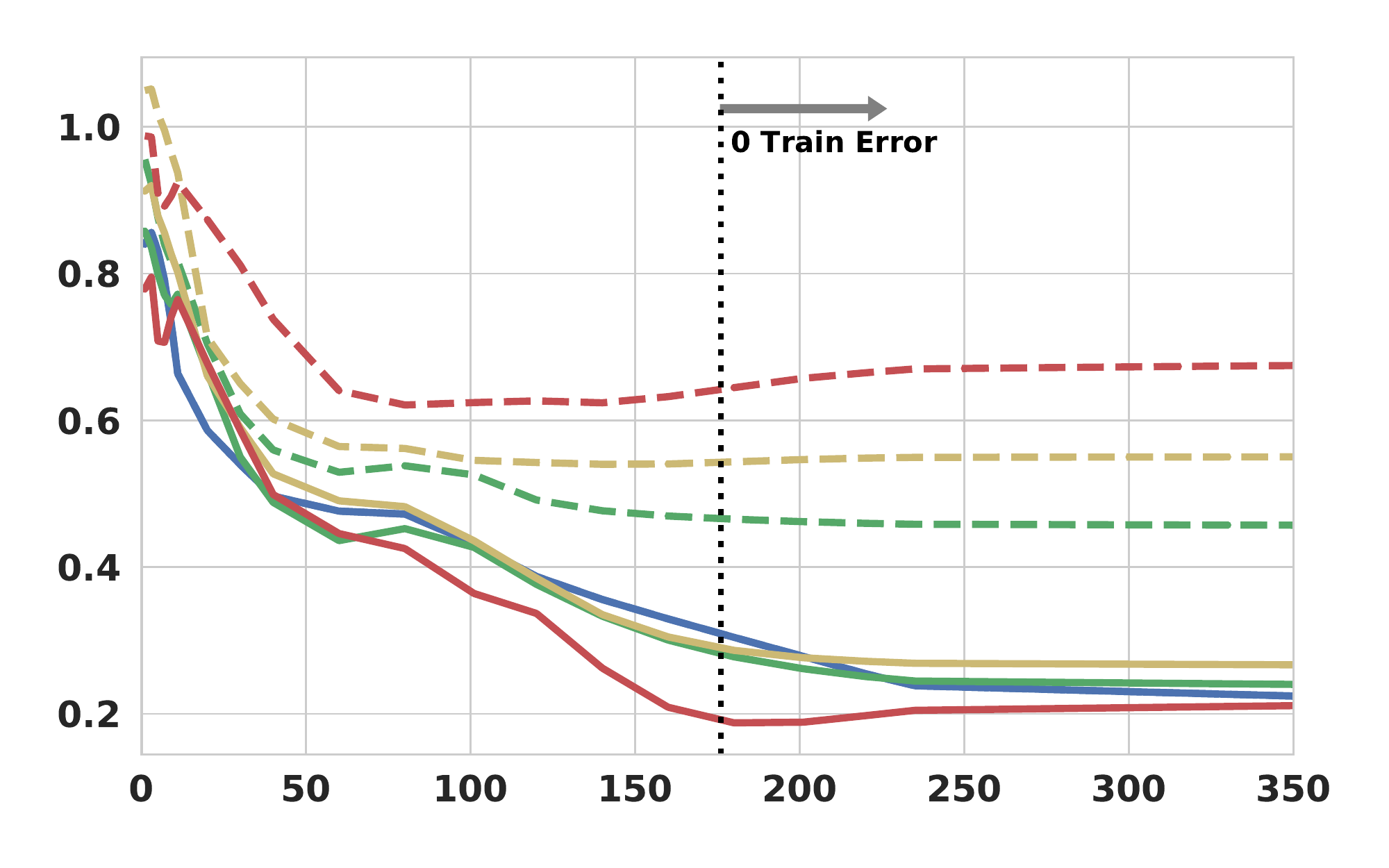}};
			\node at (-4.1,-1.4)  [scale=0.7, rotate=90]{Distance to SELI/ETF};
		\end{tikzpicture}
	\end{subfigure}\hspace{13pt}\begin{subfigure}{0.3\textwidth}
		\centering
		\begin{tikzpicture}
			\node at (0,-1.4) {\includegraphics[scale=0.26]{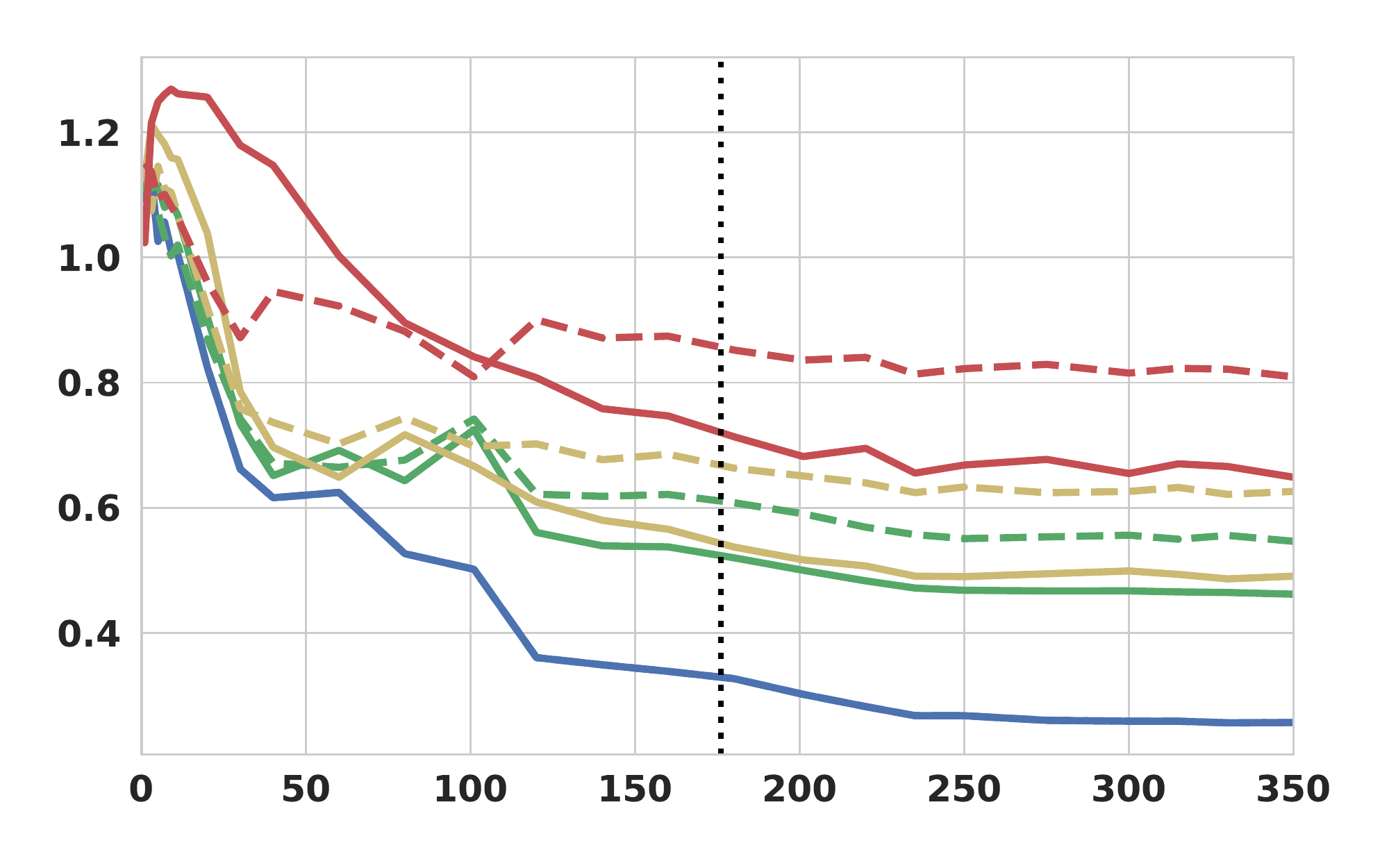}};
			\node at (0,0.2) [scale=0.9]{\textbf{CIFAR10}};
			\node at (-2.7,-1.4) [scale=0.7, rotate=90]{};
		\end{tikzpicture}
	\end{subfigure}\hspace{13pt}\begin{subfigure}{0.3\textwidth}
		\centering
		\begin{tikzpicture}
			\node at (0,-1.4) {\includegraphics[scale=0.26]{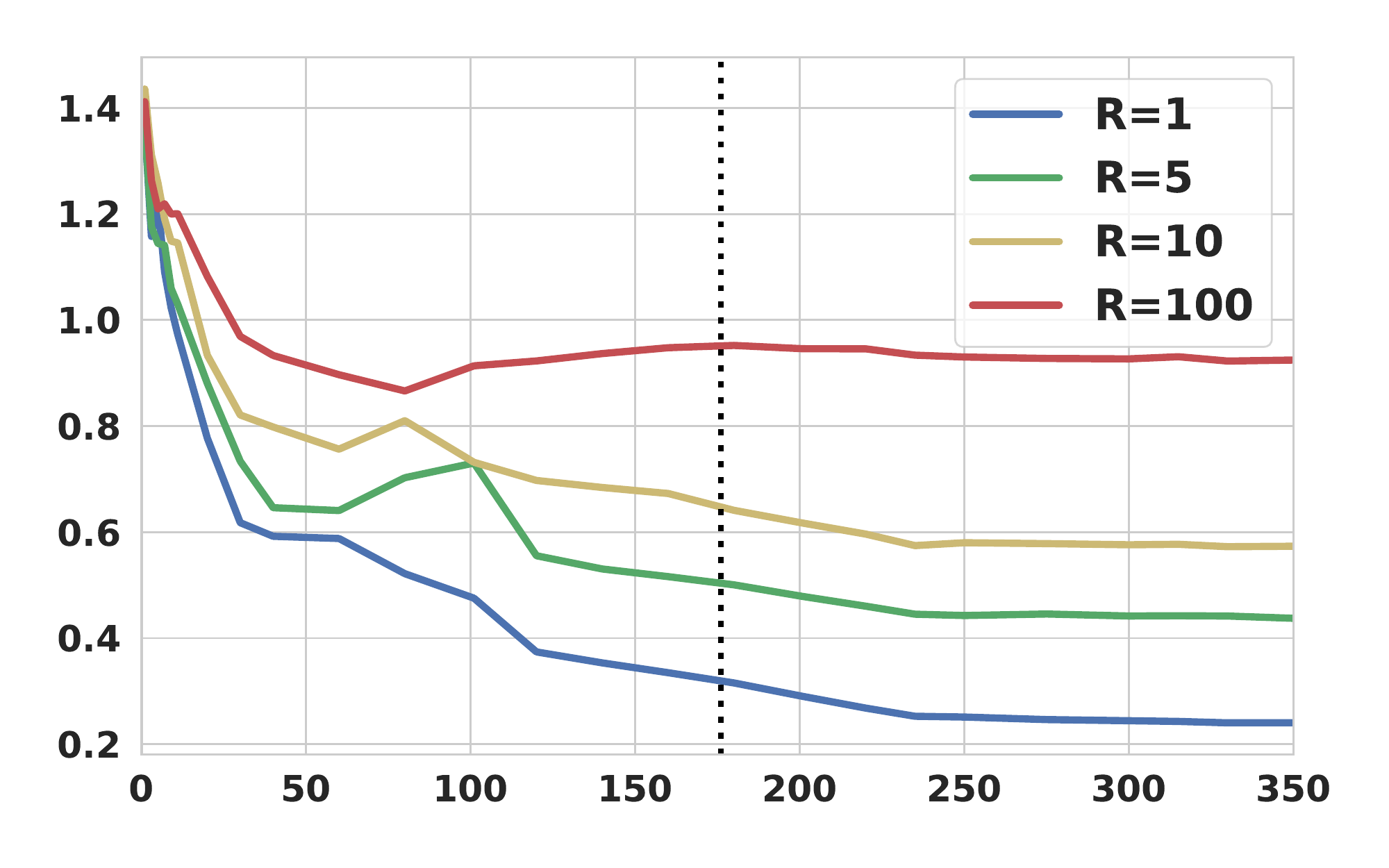}};
			\node at (-2.7,-1.4) [scale=0.7, rotate=90]{};
		\end{tikzpicture}
	\end{subfigure}\vspace{-10pt}

	\vspace{5pt}
	\centering
	\hspace{-40pt} \begin{subfigure}{0.3\textwidth}
	\centering
	\begin{tikzpicture}
		\node at (-1.4,-1.4)
		{\includegraphics[scale=0.26]{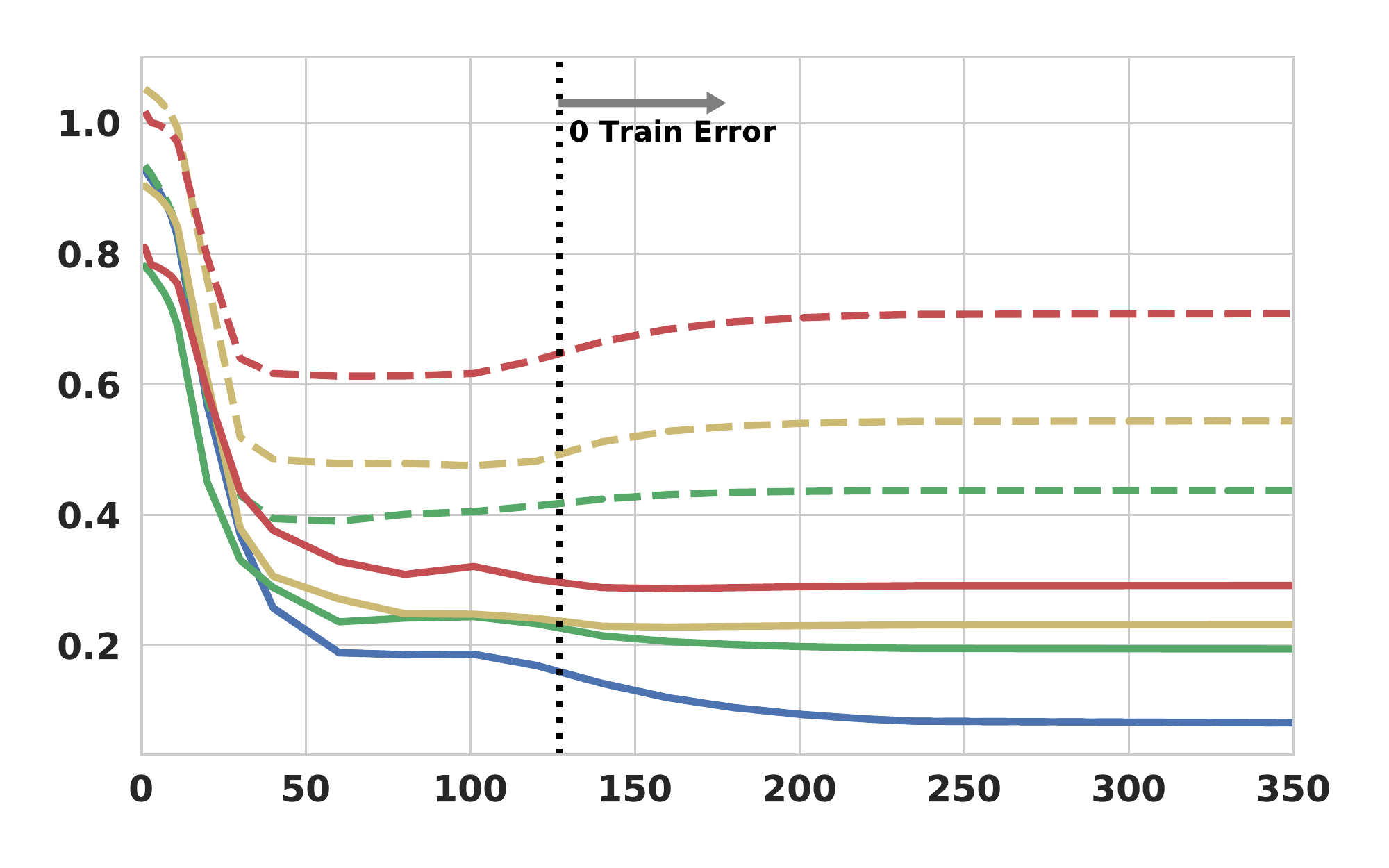}};
		\node at (-4.1,-1.4)  [scale=0.7, rotate=90]{Distance to SELI/ETF};
	\end{tikzpicture}
	\end{subfigure}\hspace{13pt}\begin{subfigure}{0.3\textwidth}
	\centering
	\begin{tikzpicture}
		\node at (0,-1.4) {\includegraphics[scale=0.26]{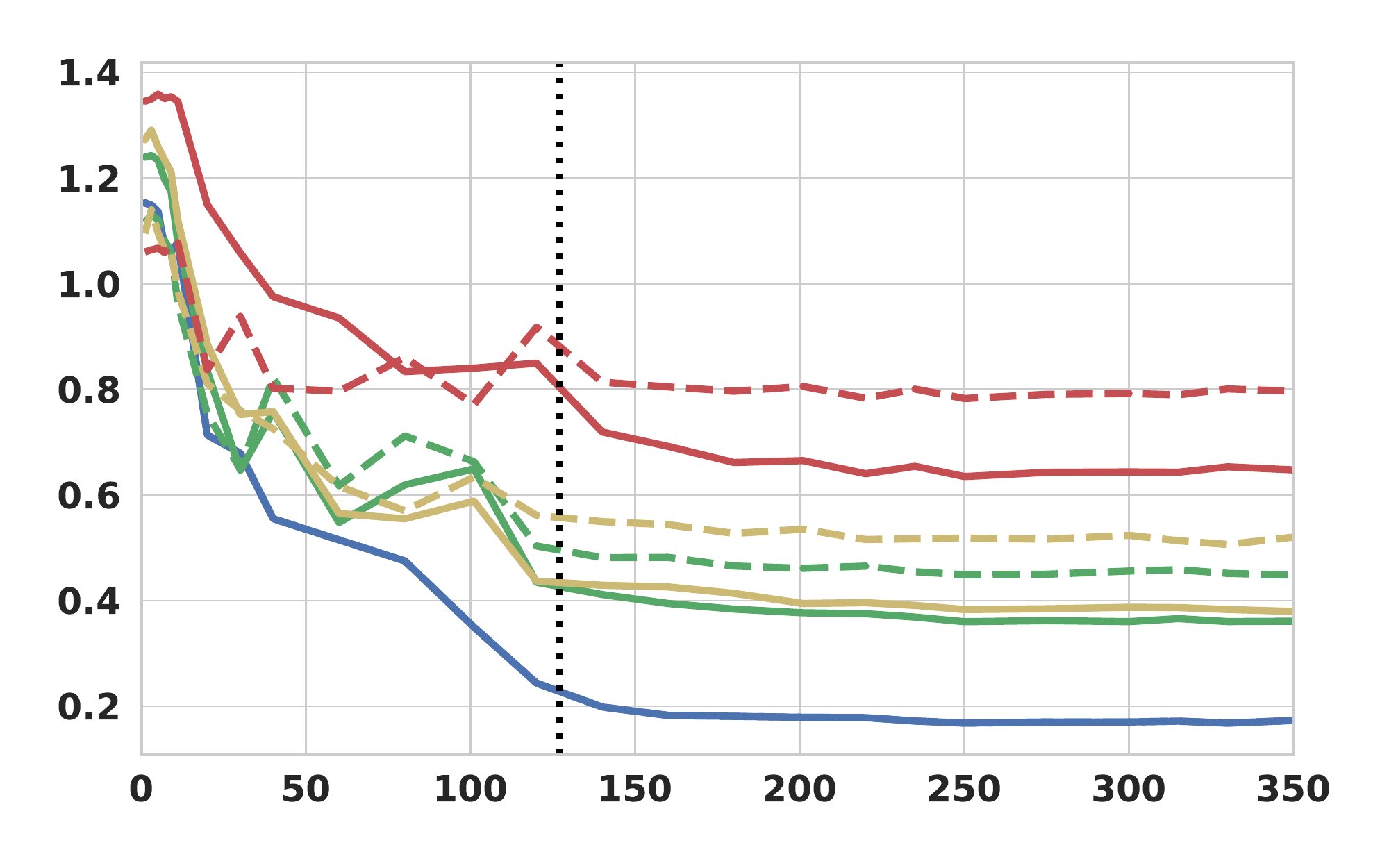}};
		\node at (0,0.2) [scale=0.9]{\textbf{MNIST}};
		\node at (-2.7,-1.4) [scale=0.7, rotate=90]{};
	\end{tikzpicture}
	\end{subfigure}\hspace{13pt}\begin{subfigure}{0.3\textwidth}
	\centering
	\begin{tikzpicture}
		\node at (0,-1.4) {\includegraphics[scale=0.26]{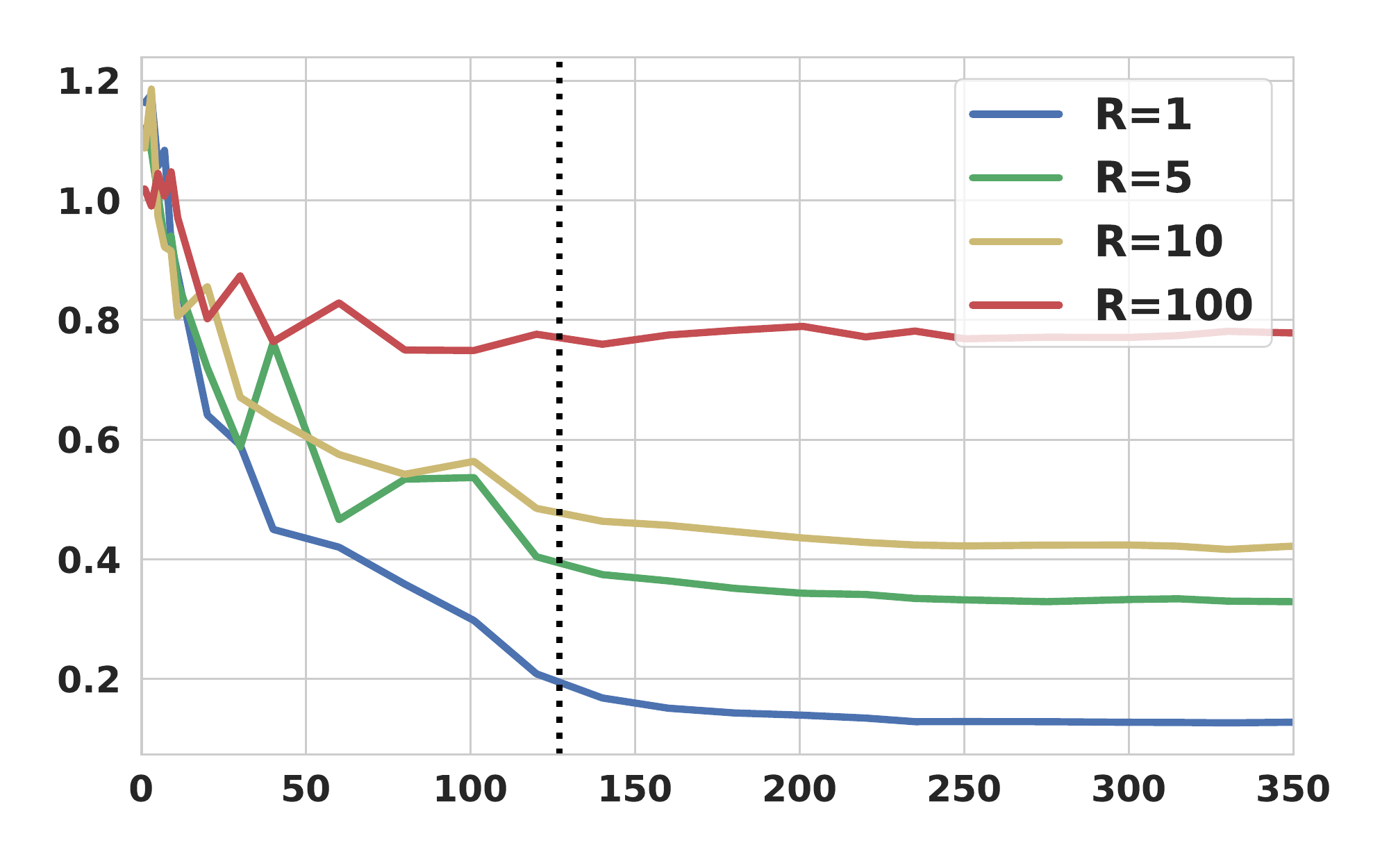}};
		\node at (-2.7,-1.4) [scale=0.7, rotate=90]{};
	\end{tikzpicture}
	\end{subfigure}\vspace{-10pt}
	
	\vspace{5pt}
	\centering
	\hspace{-40pt} \begin{subfigure}{0.3\textwidth}
		\centering
		\begin{tikzpicture}
			\node at (-1.4,-1.4)
			{\includegraphics[scale=0.26]{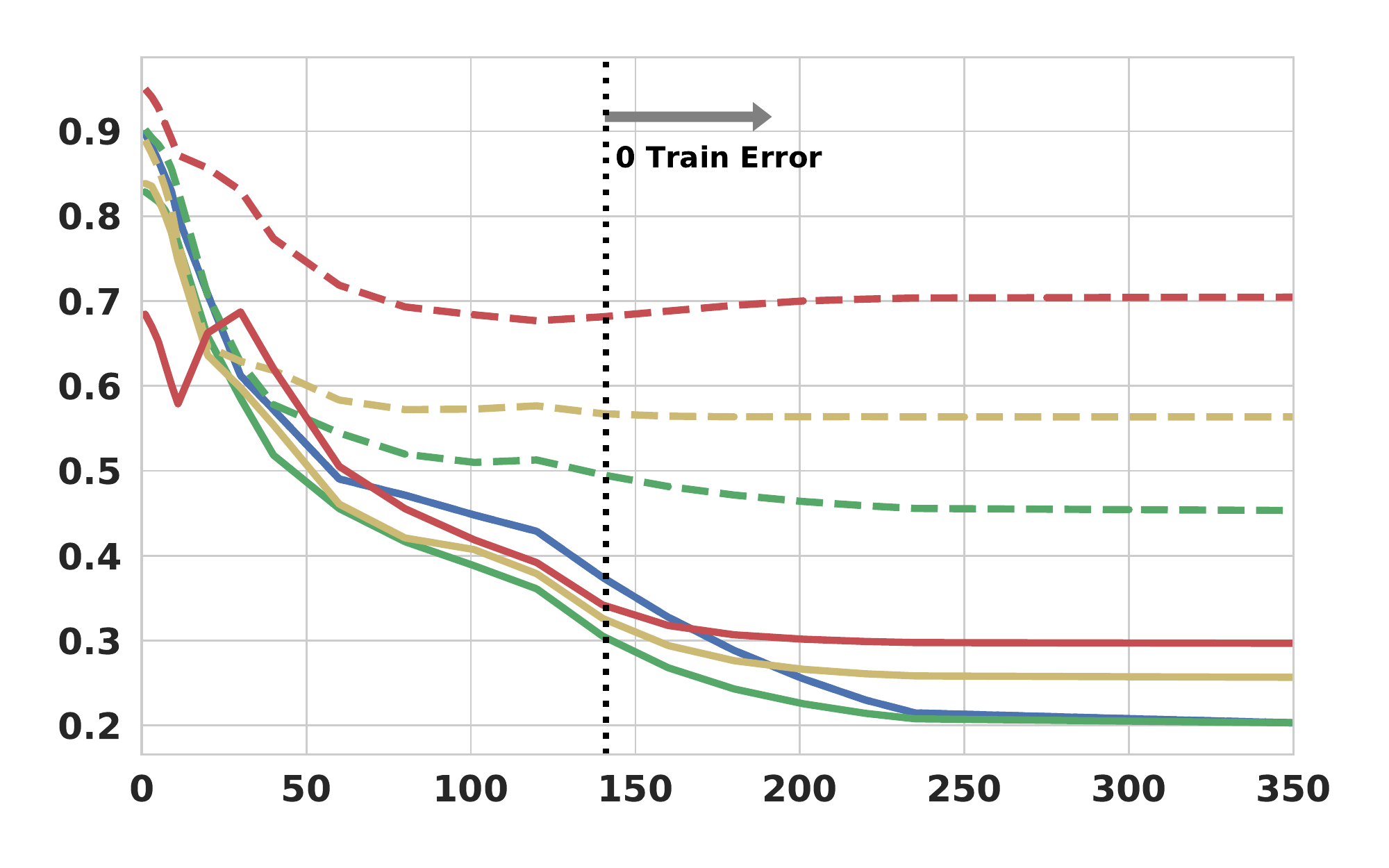}};
			\node at (-1.4,-3.1) [scale=0.7]{Epochs};
			\node at (-4.1,-1.4)  [scale=0.7, rotate=90]{Distance to SELI/ETF};
		\end{tikzpicture}\vspace{-0.2cm}
		\captionsetup{width=0.6\linewidth}\caption{Classifiers}
	\end{subfigure}\hspace{13pt}\begin{subfigure}{0.3\textwidth}
		\centering
		\begin{tikzpicture}
			\node at (0,-1.4) {\includegraphics[scale=0.26]{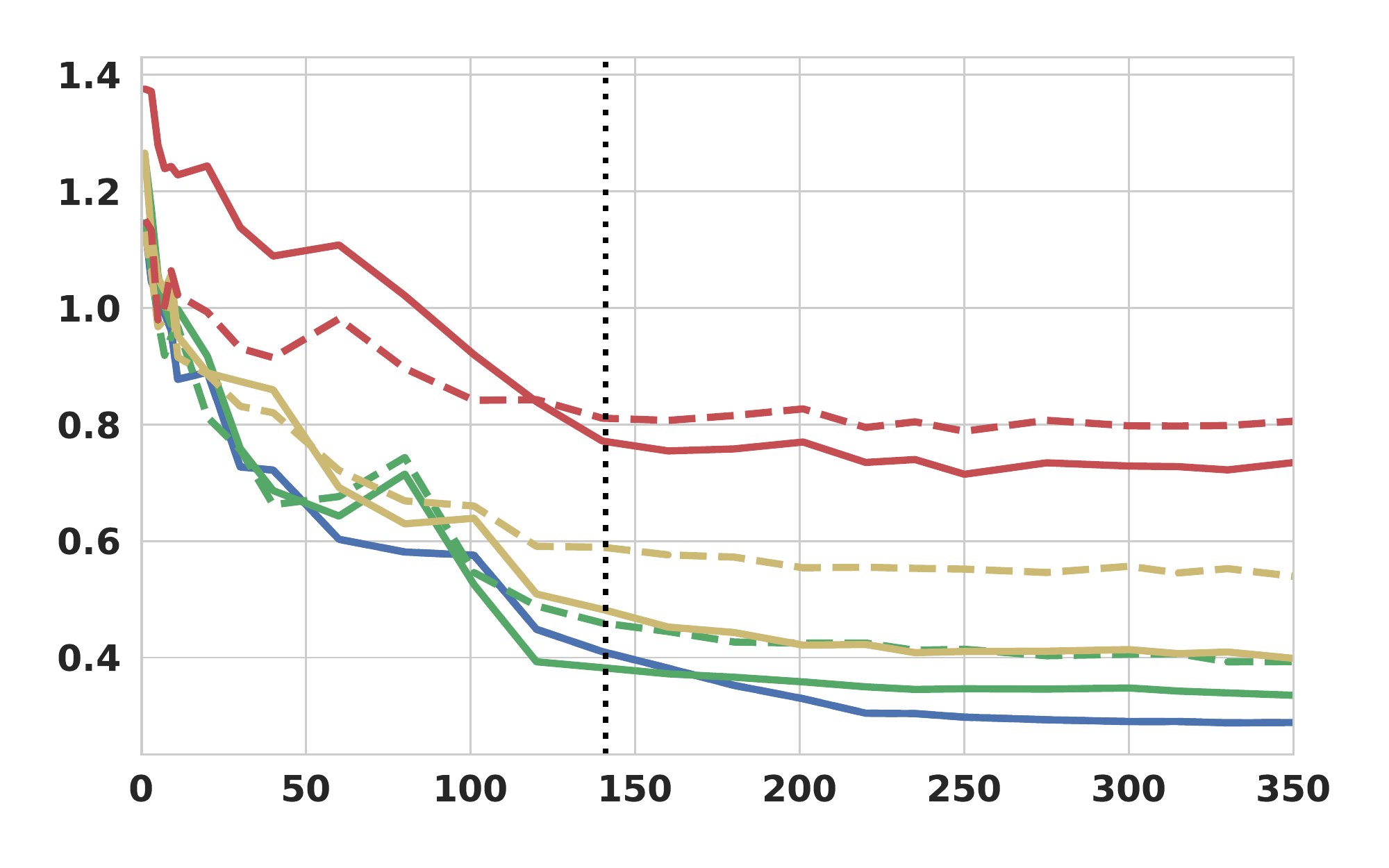}};
			\node at (0,0.2) [scale=0.9]{\new{\textbf{Fashion-MNIST}}};
			\node at (0,-3.1) [scale=0.7]{Epochs};
			\node at (-2.7,-1.4) [scale=0.7, rotate=90]{};
		\end{tikzpicture}\vspace{-0.2cm}
		\captionsetup{width=0.6\linewidth}\caption{Embeddings}
	\end{subfigure}\hspace{13pt}\begin{subfigure}{0.3\textwidth}
		\centering
		\begin{tikzpicture}
			\node at (0,-1.4) {\includegraphics[scale=0.26]{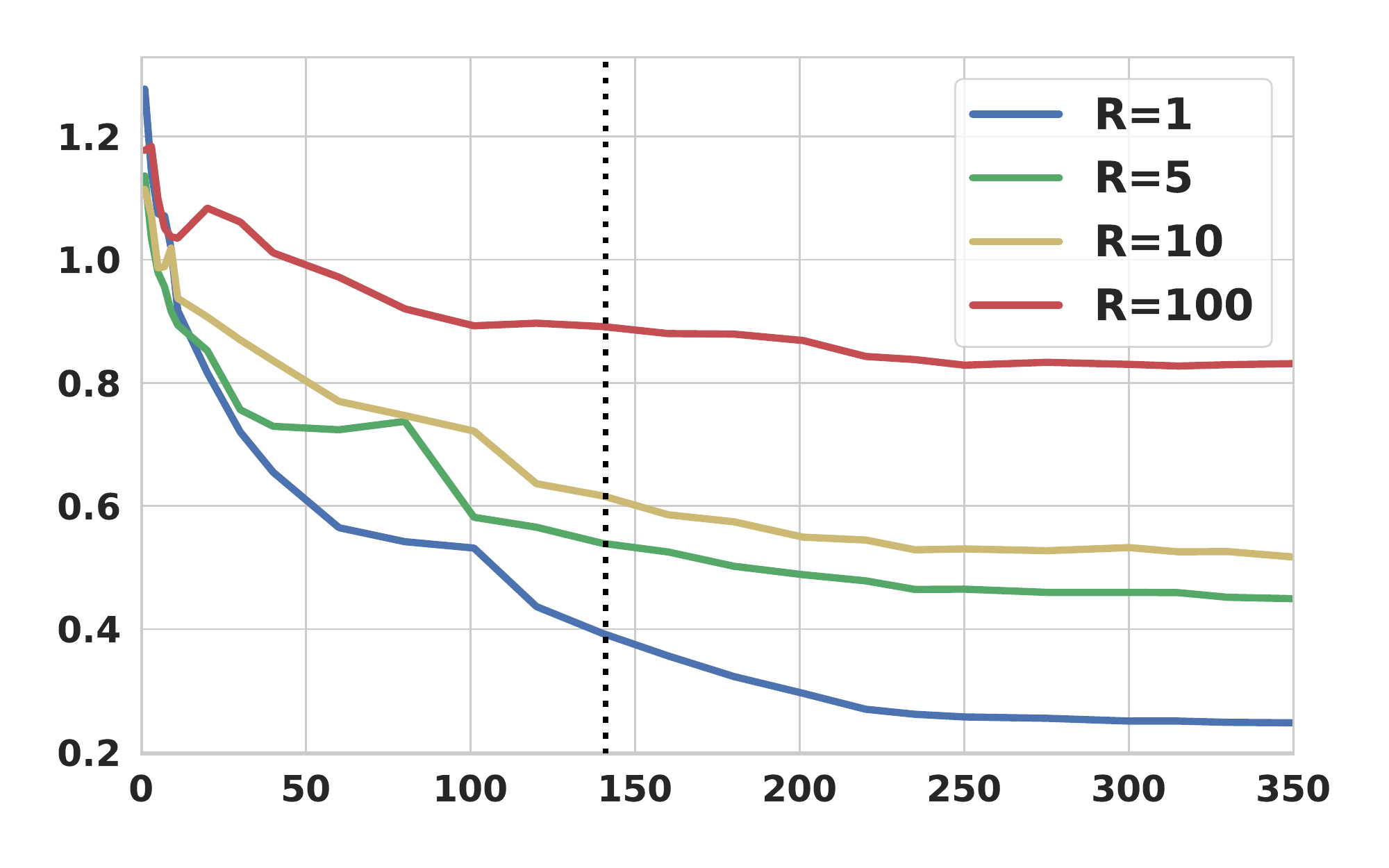}};
			\node at (0,-3.1) [scale=0.7]{Epochs};
			\node at (-2.7,-1.4) [scale=0.7, rotate=90]{};
		\end{tikzpicture}\vspace{-0.2cm}
		\captionsetup{width=0.6\linewidth}\caption{Logits}
	\end{subfigure}\vspace{-5pt}\caption{Same setup as Fig. \ref{fig:intro_CIFAR} only now training with a VGG-13 model.
	}
	\label{fig:vgg_}
	\vspace{-5pt}
\end{figure*}

\subsection{Deep-learning experiments}\label{sec:deep-net-main}

We investigate convergence to the proposed SELI geometry in deep-net training of $(R,\rho)$-STEP imbalanced MNIST, \new{Fashion-MNIST and} CIFAR10 datasets. \new{For concreteness, we consider here equal number of minorities and majorities, i.e. $\rho=1/2$. Additional experimental results for other values of the minority ratio are deferred to Sec. \ref{sec:rho}.}
 For all datasets, we keep the same total number of $n = 100\times 50 \times 5 + 50 \times 5 = 25250$ examples across all different imbalance ratios $R=1,5,10$ and $100$.  No data augmentation was used following \cite{NC}. 
We train two deep architectures, ResNet-18 \cite{he2015deep} and VGG-13 \cite{vgg}, and optimize the models using CE loss with SGD over $350$ epochs. In all experiments, the initial learning rate is set to $0.1$ and decreased by a factor of $10$ at epochs $120$ and $240$. For ResNet training on MNIST, we choose smaller initial learning rate $0.05$ which we empirically find that it interpolates data much faster. Weight decay and momentum are set to $5 \times 10^{-4}$ and $0.9$ respectively. Models are trained on a single GPU with a dataloader batchsize of $128$. 
To evaluate the learnt geometries, we track the three metrics that we defined at the beginning of this section. Recall also that following \cite{NC}, we first perform a global centering by subtracting from the mean embeddings their global average.


The convergence to SELI and ETF for the classifiers, (centered) mean-embeddings, and logits are illustrated in \Fig~\ref{fig:intro_CIFAR} and \ref{fig:vgg_} for ResNet and VGG models, respectively.
The vertical dashed lines mark the epoch at which the model reaches zero training error under all imbalance ratios; see \Sec~\ref{sec:SM_accuracies} for details. Note that in all plots, the distance to SELI geometry decreases as training evolves. Also, convergence to the SELI geometry is consistently better compared to the ETF geometry. However, convergence slows down for increasing imbalance (see $R=100$). Another interesting observation is that convergence is worse for the embeddings compared to classifiers. \new{In Sec. \ref{sec:SM_maj_min} we compare individual quadrants of the (normalized) $\underline{\G_\W}$ and $\underline{\G_\Hb}$ matrices,  which facilitates understanding the \emph{individual} behavior of majorities and minorities.}

In \Fig~\ref{fig:norm_ratio_exp_W} we focus on predicting the norms of the classifiers. In particular, we measure the norm ratio of the classifiers during training and compute its distance  from the predicted closed-form SELI characterization in Lemma \ref{lem:norms}. We  see that the (relative) magnitude of the classifiers  agrees with the value predicted by the SELI geometry. Analogous plots for the embeddings' norms are given in Sec.~\ref{sec:SM_exp_norm}.

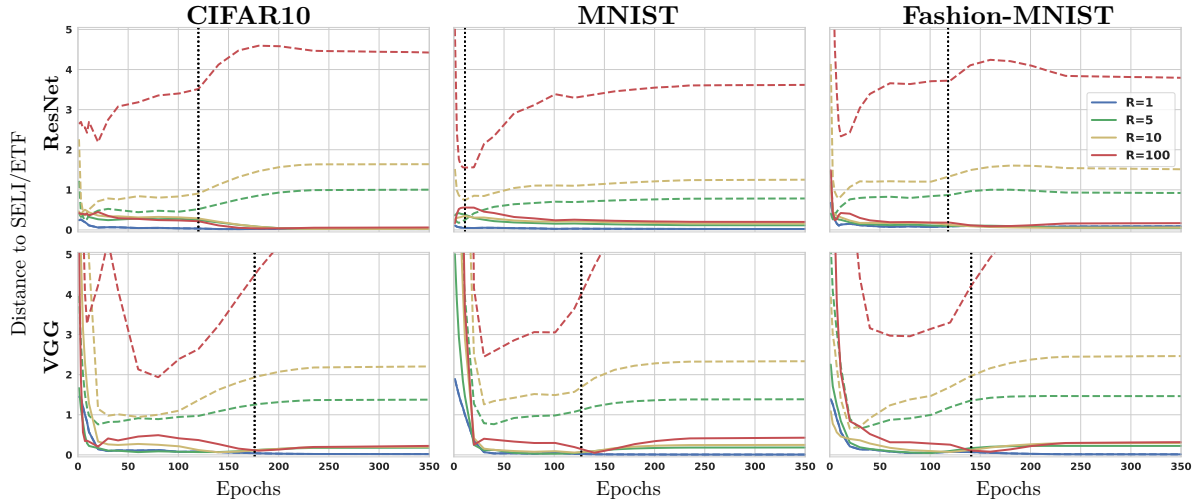
\begin{figure*}
	\vspace{-10pt}
	\centering
	\hspace{-60pt}
	\begin{subfigure}{0.9\textwidth}
		\centering
		\begin{tikzpicture}
			\node at (0,0) 
			{\includegraphics[scale=0.25]{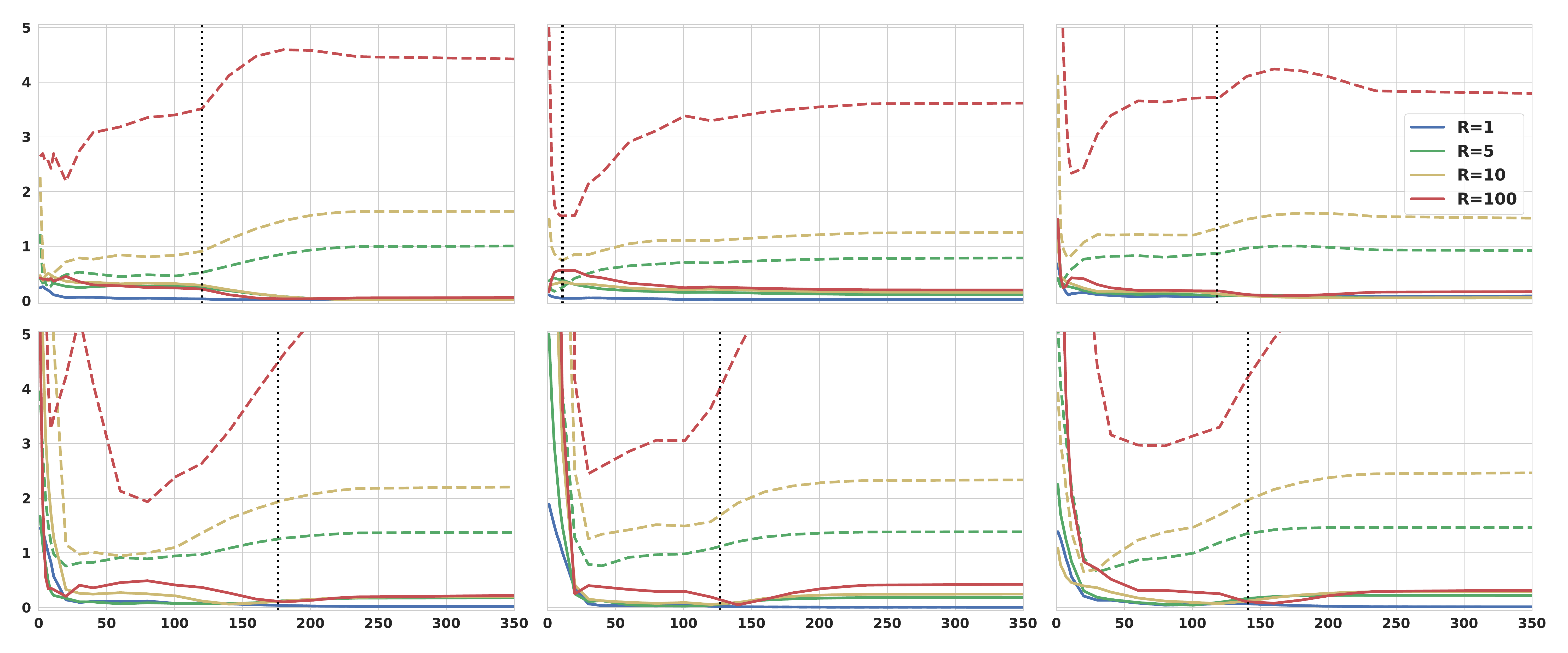}};
			\node at (5.0,-3.2) [scale=0.7] {Epochs};
			\node at (0,-3.2) [scale=0.7] {Epochs};
			\node at (-5.0,-3.2) [scale=0.7] {Epochs};
			\node at (-7.6,1.8) [scale=0.7, rotate=90]{\textbf{ResNet}};
			\node at (-7.6,-1.3) [scale=0.7, rotate=90]{\textbf{VGG}};
			\node at (-8.0,0.2) [scale=0.7, rotate=90] {Distance to SELI/ETF};
			\node at (-5.0,3.1) [scale=0.9] {\textbf{CIFAR10}};
			\node at (0.0,3.1) [scale=0.9] {\textbf{MNIST}};
			\node at (5.0,3.1) [scale=0.9] {\new{\textbf{Fashion-MNIST}}};
		\end{tikzpicture}
	\end{subfigure}	\vspace{-5pt}
	\caption{Convergence of classifier norm ratio ($\|\wmaj\|_2/\|\wmin\|_2$) to SELI (solid) vs ETF (dashed)
	for different imbalance levels $R$. The SELI geometry values are computed as per Lemma \ref{lem:norms}.}
	\label{fig:norm_ratio_exp_W}
\end{figure*}

\section{Outlook: Imbalance troubles and opportunities}\label{sec:out}
We propose \emph{\ref{SELI}} as the class-imbalance-invariant geometry of classifiers and embeddings learnt by overparameterized models when trained beyond zero training error. We arrive at it after showing that the UF-SVM global minimizers follow this geometry. Subsequently, we conjecture that: (C1) GD on the UFM leads to solutions approaching the SELI geometry asymptotically in the number of epochs; (C2) training of deep-nets beyond zero training error learns models that approach the SELI geometry. 

Statement (C1) is a conjecture and does not follow from Theorem~\ref{thm:SVM}. This is because GD is only known to converge to KKT points of the non-convex UF-SVM \cite{lyu2019gradient,ULPM}, which are not necessarily global minima.  Thus, while Proposition~\ref{propo:reg_path} combined with the benign landscape of the corresponding ridge-regularized minimization, as well as, our experiments suggest its validity, the conjecture remains to be further investigated.\footnote{Relying on \citep{ji2020directional}, \citet{vardi2021margin} showed gradient flow  finds global minimizers in deep linear nets. However, these only apply to binary classification. We show that interesting behaviors might occur (e.g. in the role of regularization and geometry) when $k>2$ and $R>1$ calling  further studies.} We also note that convergence rates appear slower with increasing imbalance levels; a feature which is interesting to study further. 

Regarding conjecture (C2), we deem our experiments encouraging: the classifiers' and embeddings' geometries get closer to the SELI geometry as neural-network training progresses. Specifically, the Gram matrices $\G_\W$ and $\G_\Hb$ of the classifiers and embeddings, respectively, align increasingly better with their SELI counterparts $\Vb\Lambdab\Vb^T$ and $\Ub\Lambdab\Ub^T$; see \Fig~\ref{fig:intro_CIFAR}. Also, \Fig~\ref{fig:norm_ratio_exp_W} shows that Lemma \ref{lem:norms} is able to make predictions regarding the norms of majorities versus minorities. We note that, similar to the UFM experiments, convergence appears slower for larger imbalance ratios. Also, we observe better convergence for classifiers compared to embeddings {and for norm predictions compared to angle predictions}. Overall, we hope our results motivate further theoretical and experimental investigations, especially since data imbalances appear frequently across applications. 

 Beyond that, we believe that further similar studies on identifying geometric structures of learned embeddings and classifiers could offer new perspectives on generalization. Our results could pave that way since they uncover different geometries (aka SELI for different $R$ values), each leading to different generalization (worse for increasing $R$ \cite{byrd2019effect}). 
 Relatedly, we envision that further such studies lead to algorithmic contributions in  imbalanced deep-learning as they can facilitate studying the implicit-bias effect of CE adjustments and post-hoc techniques tailored to imbalanced data \cite{TengyuMa,Menon,CDT,KimKim,kang2020decoupling,kini2021label,li2021autobalance}.

\label{sec:out}

\bibliography{refs}


\newpage
\appendix

\section*{Roadmap to the Supplementary material}

The SM contains the following. In Section \ref{sec:SEL_mat_properties} we derive several useful properties of the simplex-encoded-label (\SEL)~matrix, which we then use  in Section \ref{sec:SELI_properties} to present  closed-form characterizations of the embeddings and classifiers geometries. Notably, these include explicit expressions for the norms and angles of both majority and minority classes, which we accompany with numerical illustrations shedding further light on the features of the proposed \emph{\ref{SELI}} geometry. Next, in Section \ref{sec:proofs_UF-SVM}, we use the derived properties of the \SEL~matrix to prove our main Theorem \ref{thm:SVM} and its corollaries. In Section \ref{sec:nuc_norm_SM} we derive several useful properties of the nuclear-norm penalized CE minimization, which we then use in Section \ref{sec:proofs_regularization} to prove the statements of Section \ref{Sec:reg}. Section \ref{sec:UFM} contains additional numerical results on the UFM, such as experiments with varying weight-decay and regularization choices. Additional experiments on real data (such as, isolated minority/majority geometry investigations, and classifiers/embeddings norm ratios), as well as further implementation details, are included in Section \ref{sec:real_exp_SM}. In Section \ref{sec:minority} we show how results relate to minority collapse providing additional theoretical justifications and novel perspectives to empirical observations reported by previous work. Finally, an elaborate discussion on additional related works is contained in Section \ref{sec:rel2}.
\addtocontents{toc}{\protect\setcounter{tocdepth}{3}}
\tableofcontents

\section{Properties of the \SEL~matrix}\label{sec:SEL_mat_properties}
\subsection{Basic facts}

We gather some basic properties about the \SEL~matrix below. These properties hold without any assumptions on the number of examples  $n_1,\ldots,n_k$ per class other than $n_c\geq 1,\forall c\in[k].$

\begin{lemma}[\SEL~matrix --- Basic Facts]\label{lem:Zhat}
The following statements are true.

\begin{enumerate}[label={(\roman*)},itemindent=0em]

\item\label{state:SEL_prop_Y}
 Let $\Y$ be the zero-one hot encoding label matrix, i.e. $ \Y[c,i] = \begin{cases}
1 & ,\,c=y_i \\
0 & ,\,c\neq y_i
\end{cases}, ~
\forall c\in[k], i\in[n].
$ Then, $\Zhat=\Y-\frac{1}{k}\ones_k\ones_n^T$.

\item\label{state:SEL_prop_rank}
 $\Zhat^T\ones_k=0$ and $\operatorname{rank}(\Zhat)=k-1$.

\item\label{state:SEL_prop_SVD}
 $\Zhat^T$ admits a compact SVD $\Zhat^T=\Ub\Lambdab\Vb^T$, such that $\Lambdab$ is a $(k-1)$-diagonal, and $\Ub\in\R^{n \times (k-1)}$, $\Vb\in\R^{k \times (k-1)}$ are partial orthogonal matrices, i.e. $\Ub^T\Ub=\Vb^T\Vb=\Id_{k-1}$. Moreover $\Vb^T\ones_{k}=0,$ and the $k-1$ columns of $\Vb$ span the subspace orthogonal to $\ones_k$, i.e. $\Vb\Vb^T=\Id_{k}-\frac{1}{k}\ones_k\ones_k^T.$

\item\label{state:SEL_prop_Grams}
$\Zhat^T\Zhat = \Y^T\Y-\frac{1}{k}\ones\ones^T$  and  $\Zhat\Zhat^T = \diag(\nb)+\frac{n}{k^2}\ones_{k}\ones_k^T-\frac{1}{k}\nb\ones_k^T-\frac{1}{k}\ones_k\nb^T$, where we defined $\nb=[n_1,n_2,\ldots,n_k]^T$

\item \label{state:SEL_prop_CE}
Let $\Lc(\Z)=\sum_{i\in[n]}\log\big(1+\sum_{c\neq y_i}e^{-(\Z[y_i,i]-\Z[c,i])}\big)$. Then,  $\nabla_{\Z}\Lc(\alpha\Zhat) = -\frac{k}{e^\alpha+k-1}\Zhat$, for all $\alpha\in\R$.

\end{enumerate}
\end{lemma}

\begin{proof}

\noindent\underline{Proof of \ref{state:SEL_prop_Y}:} Follows directly from the definition.

\noindent\underline{Proof of \ref{state:SEL_prop_rank}:} The fact that $\Zhat^T\ones_k=0$ is easy to check. Hence, the rank is at most $k-1$. The fact that the rank is exactly $k-1$ follows by noting that the vectors $\eb_c-\frac{1}{k}\ones_k, c\in[k-1]$ are linearly independent.

\noindent\underline{Proof of \ref{state:SEL_prop_SVD}:} Follows directly from Statement \ref{state:SEL_prop_rank}.

\noindent\underline{Proof of \ref{state:SEL_prop_Grams}}: Follows easily by direct calculations. 

\noindent\underline{Proof of \ref{state:SEL_prop_CE}}: We show in Lemma \ref{lem:grad_ones} below that $-\nabla_{\Z} \Lc(\Z) =\Y-\A$ with $\A[c,i]:=\frac{e^{-(\Z[y_i,i]-\Z[c,i])}}{1+\sum_{c'\neq y_i}e^{-(\Z[y_i,i]-\Z[c',i])}}. 
$
It is easy to see that for $\Z=\alpha\Zhat$, the matrix $\A$ can be written as $\A=\frac{e^{\alpha}-1}{k-1+e^\alpha}\Y+\frac{1}{k-1+e^\alpha}\ones_k\ones_n^T$. The desired then follows by recalling Statement \ref{state:SEL_prop_Y}.
\end{proof}

\begin{lemma}[Auxiliary result---Gradient of CE]\label{lem:grad_ones}
The negative gradient of the cross-entropy loss $
\Lc(\Z) = \sum_{i=1}^{n}\log\left(1+\sum_{c\neq y_i}e^{-(\Z[y_i,i]-\Z[c,i])}\right)$ takes the form $-\nabla_{\Z} \Lc(\Z) =\Y-\A$ where $\Y$ is the one-hot encoding label matrix and for $i\in[n], c\in[k]$ we denote $\A[c,i]:=\frac{e^{-(\Z[y_i,i]-\Z[c,i])}}{1+\sum_{c'\neq y_i}e^{-(\Z[y_i,i]-\Z[c',i])}}
$. Thus, for any $\Z\in\R^{k\times n}$, it holds that $\ones_{k}^T\nabla_{\Z} \Lc(\Z) = \zeros$.
\end{lemma}
\begin{proof}
The proof is straightforward. Denote for convenience, $Z_{ci}:=\Z[c,i]$ and $s_{ci}:=\Z[y_i,i]-\Z[c,i]$.  Taking derivatives with respect to $Z_{ci}, c\in[k]$ we have for $c=y_i$, 
$$\frac{\partial \Lc}{\partial Z_{y_i i}}=-\frac{\sum_{c'\neq y_i}e^{-s_{c'i}}}{1+\sum_{c''\neq y_i}e^{-s_{c''i}}}\frac{\partial s_{c'i}}{\partial Z_{y_ii}} = - \frac{\sum_{c'\neq y_i}e^{-s_{c'i}}}{1+\sum_{c''\neq y_i}e^{-s_{c''i}}} = - \Big(1-\frac{1}{1+\sum_{c'\neq y_i}e^{-s_{c'i}}}\Big)=-(1-\A[y_i,i])
$$
and, for $c\neq y_i$,
$$
\frac{\partial \Lc}{\partial Z_{ci}} = - \frac{e^{-s_{ci}}}{1+\sum_{c'\neq y_i}e^{-s_{c'i}}}\frac{\partial s_{ci}}{\partial Z_{ci}} = \frac{e^{-s_{ci}}}{1+\sum_{c'\neq y_i}e^{-s_{c'i}}}=\A[c,i].
$$
Thus, $\sum_{c=1}^{k}\frac{\partial \Lc}{\partial Z_{ci}} = 0, \forall i \in [n]$. Hence, $\ones_{k}^T\nabla_{\Z} \Lc(\Z) = \zeros$ for every $\Z$.
\end{proof}

\subsection{Eigen-structure}\label{sec:eigen_SEL}
In this section, we explicitly compute the eigenstructure of the \SEL~matrix for $(R,\rho)$-STEP imbalanced data. To simplify the expressions, we assume $n_{\min}=1$. \footnote{It is rather easy to derive all formulas without this requirement. Concretely, the singular values in \eqref{eq:Lambda_general} are multiplied by $\sqrt{n_{\min}}$, the $(1,1)$-- and $(2,3)$-- blocks of $\Ub$ in \eqref{eq:U_general} are multiplied by   $1/\sqrt{n_{\min}}$ and the dimensions of $\Ub$ also adjust appropriately. Nevertheless, this does not change the values of $\Ub\Lambdab\Ub^T$ and of $\Vb\Lambdab\Vb^T$ aside from a global scaling. Besides, the assumption $\nmin=1$ is essentially without loss of generality because the \emph{\ref{NC}} property holds under the SELI property.} 
Also, we need the following definitions. For $m\in[k]$, let $\Pb_{m}\in\R^{m\times(m-1)}$ denote an orthonormal basis of the subspace orthogonal to $\ones_m$, i.e. $\Pb_m\Pb_m^T=\Id_m-\frac{1}{m}\ones_m\ones_m^T$ and $\Pb_m^T\Pb_m=\Id_{m-1}.$  We will also denote $\Sb_m:=\Id_m-\frac{1}{k}\ones_m\ones_m^T \in\R^{m}.$

\begin{lemma}[\SEL~matrix --- SVD for STEP imbalance]
\label{lem:Zhat_SVD_general}
 Assume $(R,\rho)$-STEP imbalanced data and $n_{\min}=1$. Also, denote $\rhobar=1-\rho$ and recall that the total number of examples is $n=(\rho + R\rhobar)k.$ Then, the SVD factors of $\Zhat=\Vb\Lambdab\Ub^T$ are given as follows:

\begin{align}
&\Lambdab=\diag{\left(\begin{bmatrix}
\sqrt{R}\ones_{(\rhobar k-1)}^T & \sqrt{\rhobar+R\rho} &\ones_{(\rho k-1)}^T
\end{bmatrix}
\right)}\,\label{eq:Lambda_general}
\\
\Vb&= \begin{bmatrix} 
\Pb_{\rhobar k}  & -\sqrt{\frac{\rho/\rhobar}{k}}\,\ones_{\rhobar k} & 0_{(\rhobar k)\times{(\rho k-1)}}
\\ 
0_{(\rho k)\times{(\rhobar k-1)}}  & \sqrt{\frac{\rhobar/\rho}{k}}\ones_{\rho k} & \Pb_{\rho k}
\end{bmatrix} \label{eq:V_general}
\\
\Ub&= \begin{bmatrix} 
\frac{1}{\sqrt{R}}\Pb_{\rhobar k}\otimes\ones_R  & -\sqrt{\frac{\rho/\rhobar}{(\rhobar+R\rho)k}}\ones_{R\rhobar k } & 0_{(R\rhobar k)\times{(\rho k-1)}}
\\ 
0_{(\rho k)\times{(\rhobar k-1)}}  & \sqrt{\frac{\rhobar/\rho}{(\rhobar+R\rho)k}}\ones_{\rho k} & \Pb_{\rho k}
\end{bmatrix} \,. \label{eq:U_general}
\end{align}

\end{lemma}

\begin{proof}
The challenging part is coming up with the formulas in \eqref{eq:Lambda_general}, \eqref{eq:V_general} and \eqref{eq:U_general} for the SVD factors. The lemma already does this for us. Hence, proving that the formulas are correct involves a few tedious calculations, which we present below. 

Let us define for convenience:
\begin{align*}
\Vb_{\text{maj}}:=\begin{bmatrix}
\Pb_{\rhobar k} \\
0_{(\rho k)\times{(\rhobar k-1)}} 
\end{bmatrix}
\quad
\vb = \frac{1}{\sqrt{k}}\begin{bmatrix}
-\sqrt{\frac{\rho}{\rhobar}}\ones_{\rhobar k} \\ \sqrt{\frac{\rhobar}{\rho}}\ones_{\rho k}
\end{bmatrix}
\quad
\Vb_{\text{min}}:=\begin{bmatrix}
0_{(\rhobar k)\times{(\rho k-1)}} \\
 \Pb_{\rho k}
\end{bmatrix}
\end{align*}
so that $\Vb=[\Vb_{\text{maj}},\vb,\Vb_{\text{min}}]$. 
Recalling for $m=\rhobar k$ or $m=\rho k$ that $\Pb_m^T\Pb_m=\Id_{m-1}$ and $\Pb_m^T\ones_m=0$ it is easy to check that $\Vb^T\Vb=\Id_{k-1}.$

Similarly, let
\begin{align*}
\Ub_{\text{maj}}:=\begin{bmatrix}
\frac{1}{\sqrt{R}}\Pb_{\rhobar k}\otimes\ones_R 
\\
0_{(\rho k)\times{(\rhobar k-1)}}  
\end{bmatrix}
\quad
\ub = \frac{1}{\sqrt{(\rhobar+R\rho)k}}\begin{bmatrix}
-\sqrt{\frac{\rho}{\rhobar}}\ones_{R\rhobar k} \\ \sqrt{\frac{\rhobar}{\rho}}\ones_{\rho k}
\end{bmatrix}
\quad
\Ub_{\text{min}}:=\begin{bmatrix}
0_{(R\rhobar k)\times{(\rho k-1)}} \\
 \Pb_{\rho k}
\end{bmatrix}
\end{align*}
so that $\Ub=[\Ub_{\text{maj}},\ub,\Ub_{\text{min}}]$. 
With same argument as above, it is  easy to check that $\Ub^T\Ub=\Id_{k-1}.$ Here, we also use that 
\begin{align*}
\frac{1}{R}(\Pb_{\rhobar k}\otimes\ones_R)^T(\Pb_{\rhobar k}\otimes\ones_R)=\frac{1}{R}(\Pb_{\rhobar k}^T\otimes\ones_R^T)(\Pb_{\rhobar k}\otimes\ones_R) = \frac{1}{R}\Pb_{\rhobar k}^T\Pb_{\rhobar k}\otimes \ones_R^T\ones_R = \Id_{\rhobar k}
\end{align*}
Thus, it  suffices to show that $\Vb\Lambdab\Ub^T=\Zhat.$ The key observation here is that $\Zhat$ can be written in block-form as follows
\begin{align}\label{eq:Zhat_block}
\Zhat = \begin{bmatrix}
\Sb_{\rhobar k}\otimes\ones_R^T
&
-\frac{1}{k}\ones_{\rhobar k}\ones^T_{\rho k}
\\
-\frac{1}{k}\ones_{\rho k}\ones^T_{R\rhobar k}
&
\Sb_{\rho k}
\end{bmatrix}.
\end{align} 
With these, we have the following direct calculations:
\begin{align*}
\Vb\Lambdab\Ub^T &= \sqrt{R}\Vb_{\text{maj}}\Ub_{\text{maj}}^T + \big(\sqrt{\rhobar + R\rho}\big)\vb\ub^T + \Vb_{\text{min}}\Ub_{\text{min}}^T
\\
&=\begin{bmatrix}
\Pb_{\rhobar k}\Pb_{\rhobar k}^T\otimes\ones_R^T & 0\\
0 & 0
\end{bmatrix}
+
\frac{1}{k}
\begin{bmatrix}
\frac{\rho}{\rhobar}\ones_{\rhobar k}\ones_{R \rhobar k}^T & -\ones_{\rhobar k}\ones_{\rho k}^T
\\
-\ones_{\rho k}\ones_{R \rhobar k}^T & \frac{\rhobar}{\rho}\ones_{\rho k}\ones_{\rho k}^T
\end{bmatrix}
+
\begin{bmatrix}
0 & 0
\\
0 & \Pb_{\rho k}\Pb_{\rho k}^T
\end{bmatrix}
\\&
=\begin{bmatrix}
\big(\Id_{\rhobar k}-\frac{1}{\rhobar k}\ones_{\rhobar k}\ones_{\rhobar k}^T\big)\otimes\ones_R^T & 0\\
0 & 0
\end{bmatrix}
+
\frac{1}{k}
\begin{bmatrix}
\frac{\rho}{\rhobar}\ones_{\rhobar k}\ones_{\rhobar k}^T\otimes\ones_R^T & -\ones_{\rhobar k}\ones_{\rho k}^T
\\
-\ones_{\rho k}\ones_{R \rhobar k}^T & \frac{\rhobar}{\rho}\ones_{\rho k}\ones_{\rho k}^T
\end{bmatrix}
+
\begin{bmatrix}
0 & 0
\\
0 & \Id_{\rho k}-\frac{1}{\rho k}\ones_{\rho k}\ones_{\rho k}^T
\end{bmatrix}
\\
&=
\begin{bmatrix}
\big(\Id_{\rhobar k}-\frac{1}{k}\ones_{\rhobar k}\ones_{\rhobar k}^T\big)\otimes\ones_R^T & -\frac{1}{k}\ones_{\rhobar k}\ones_{\rho k}^T\\
-\frac{1}{k}\ones_{\rho k}\ones_{R\rhobar k}^T & \Id_{\rho k}-\frac{1}{k}\ones_{\rho k}\ones_{\rho k}^T
\end{bmatrix}
\\
&= \Zhat.
\end{align*}
\end{proof}
\subsubsection{Special case: Balanced data}

When classes are  balanced, i.e. $R=1$, the following simple description of the SVD factors is immediate to see from Lemma \ref{lem:Zhat_SVD_general}.

\begin{corollary}
Assume balanced data and $n_{\min}=1$. Recall that $\Pb_k\in\R^{k\times(k-1)}$ denotes an orthonormal basis of the subspace orthogonal to $\ones_{k}$. Then, $\Zhat=\Pb_k\Pb_k^T$, that is
$$
\Lambdab=\Id_{k-1},\quad \Ub=\Vb=\Pb_k.
$$
\end{corollary}

\subsubsection{Special case: Equal minorities / majorities ($\rho=1/2$)}

Another special case of interest is when the numbers of minorities and majorities are the same, i.e. $\rho=1/2$. In this case, we get the following simplification of Lemma \ref{lem:Zhat_SVD_general}.

\begin{corollary}Consider the setting of step imbalance with even number $k=2m, m\geq 1$ of classes. Let $n_{\min}=1$. Then, the SVD of $\Zhat$  is as follows:
\begin{align*}
\Lambdab&=\diag{\left(\begin{bmatrix}
\sqrt{R}\ones_{(m-1)}^T & \sqrt{(R+1)/2} &\ones_{(m-1)}^T
\end{bmatrix}
\right)}\,,
\\
\Vb&= \begin{bmatrix} 
\Pb_{m}  & -\frac{1}{\sqrt{k}}\ones_m & 0_{m\times{(m-1)}}
\\ 
0_{m\times{(m-1)}}  & +\frac{1}{\sqrt{k}}\ones_m & \Pb_m
\end{bmatrix} 
\\
\Ub&= \begin{bmatrix} 
\frac{1}{\sqrt{R}}\Pb_{m}\otimes\ones_R  & -\frac{1}{\sqrt{n}}\ones_{Rm} & 0_{(Rm)\times{(m-1)}}
\\ 
0_{m\times{(m-1)}}  & +\frac{1}{\sqrt{n}}\ones_m & \Pb_m
\end{bmatrix} \,.
\end{align*}
\end{corollary}

\subsection{A useful property of the singular spaces}\label{sec:SVD_elementwise}

The following result is particularly important for the proof of Theorem \ref{thm:SVM}. It shows that the singular spaces $\Vb, \Ub$ of $\Zhat$ are such that the matrix $\Ub\Vb^T$ has entries that agree on their sign with the sign of the entries of $\Zhat^T.$

\begin{lemma}\label{eq:SVD_elementwise}
Recall the setting of Lemma \ref{lem:Zhat_SVD_general} and the SVD $\Zhat=\Vb\Lambdab\Ub^T$. The matrix $\hat\Bb=\Ub\Vb^T$ satisfies the following element-wise \emph{strict} inequalities: $\hat\Bb\odot\Zhat^T>0.$
\end{lemma}
\begin{proof}
From Lemma \ref{lem:Zhat_SVD_general}, we have explicit expressions for  the SVD factors $\Ub$ and $\Vb$. From these, we can directly compute that
\begin{align}
\begin{bmatrix}
\hat\Bb_{11} & \hat\Bb_{12}
\\
\hat\Bb_{21} & \hat\Bb_{22}
\end{bmatrix}:=\hat\Bb = \Ub\Vb^T = 
\begin{bmatrix}
\frac{1}{\sqrt{R}}\Pb_{\rhobar k}\Pb_{\rhobar k}^T\otimes\ones_R & 0
\\
0 & 0
\end{bmatrix} + 
\frac{1}{k\sqrt{\rhobar+R\rho}}
\begin{bmatrix}
\frac{\rho}{\rhobar}\ones_{R\rhobar k}\ones_{\rhobar k}^T & -\ones_{R\rhobar k}\ones_{\rho k}^T \\
-\ones_{\rho k}\ones_{\rhobar k}^T & \frac{\rhobar}{\rho}\ones_{\rho k}\ones_{\rho k}^T
\end{bmatrix}
+
\begin{bmatrix}
0& 0
\\
0 & \Pb_{\rho k}\Pb_{\rho k}^T 
\end{bmatrix}
\nn.
\end{align} 
To continue, recall again that for any integer $m$: $\Pb_m\Pb_m^T=\Id_m-\frac{1}{m}\ones_m\ones_m^T$. Also note that $\ones_{R\rhobar k}\ones_{\rhobar k}^T=(\ones_{\rhobar k}\ones_{\rhobar k}^T)\otimes\ones_R$. Hence continuing from the display above we find that
\begin{align*}
\hat\Bb_{11} &= 
\left(\frac{1}{\sqrt{R}}\Id_{\rhobar k}-\left(\frac{1}{\rhobar k\sqrt{R}}-\frac{\rho/\rhobar}{k\sqrt{\rhobar + R \rho}}\right)\ones_{\rhobar k}\ones_{\rhobar k}\right)
\otimes\ones_R = \frac{1}{\sqrt{R}}\left(\Id_{\rhobar k}-\frac{1}{\rhobar k}\left(1-\sqrt{\frac{{R \rho}}{R+{\rhobar}/{\rho}}}\right)\ones_{\rhobar k}\ones_{\rhobar k}\right)
\otimes\ones_R 
\\
\hat\Bb_{12} &= -\frac{1}{k\sqrt{\rhobar+R\rho}}\ones_{R\rhobar k}\ones_{\rho k}^T\qquad\text{ and }\qquad
\hat\Bb_{21} = -\frac{1}{k\sqrt{\rhobar+R\rho}}\ones_{\rho k}\ones_{\rhobar k}^T
\\
\hat\Bb_{22} &= 
\Id_{\rho k}-\left(\frac{1}{\rho k}-\frac{\rhobar/\rho}{k\sqrt{\rhobar + R \rho}}\right)\ones_{\rho k}\ones_{\rho k}^T
= 
\Id_{\rho k}-\frac{1}{\rho k}\left(1-\sqrt{\frac{\rhobar}{{1 + R \left({\rho}/{\rhobar}\right)}}}\right)\ones_{\rho k}\ones_{\rho k}^T.
\end{align*} 
Finally, recall from \eqref{eq:Zhat_block} the block-form of $\Zhat$ repeated here for convenience 
$$
\Zhat^T := \begin{bmatrix}
\Zhat_{11}^T
&
\Zhat_{12}^T
\\
\Zhat_{21}^T
&
\Zhat_{22}^T
\end{bmatrix} = \begin{bmatrix}
\Sb_{\rhobar k}\otimes\ones_R
&
-\frac{1}{k}\ones_{R \rhobar k}\ones^T_{\rho k}
\\
-\frac{1}{k}\ones_{\rho k}\ones^T_{\rhobar k}
&
\Sb_{\rho k}\end{bmatrix}.
$$
By inspection, the signs of $\hat\Bb_{12}$, $\hat\Bb_{21}$ are negative, same as the signs of $\Zhat_{21}^T$, $\Zhat_{12}^T$. To see that the signs of the diagonal blocks also agree it suffices to check that the following strict inequalities always hold
$$
1>1-\sqrt{\frac{{R \rho}}{R+{\rhobar}/{\rho}}}>0\qquad\text{and}\qquad 1>1-\sqrt{\frac{\rhobar}{1 + R \left({\rho}/{\rhobar}\right)}} >0.
$$
This completes the proof of the lemma.
\end{proof}


\section{The SELI geometry}\label{sec:SELI_properties}

As mentioned the SELI geometry is described in terms of the SVD of the \SEL~matrix $\Zhat$.  In this section, we show that it is in fact possible to get explicit closed-form expressions describing the SELI geometry in terms of the parameters $R,\rho,k$. Key to this is the explicit construction of the SVD factors in \Sec~\ref{sec:eigen_SEL}.

For concreteness, we focus on the case of equal numbers of minorities and majorities (i.e. $\rho=1/2$) since the formulas are somewhat simpler and the setting is of sufficient interest to convey main messages. Extension to the general case can be done in a similar fashion.

All results in this section hold under the following assumptions (assumed throughout without further explicit reference):
\begin{itemize}
\item A $(R,1/2)$-STEP imbalanced setting.
\item The classifiers $\w_c,c\in[k]$ and the embeddings $\h_i,i\in[n]$ follow the SELI geometry in Definition \ref{def:SELI}.
\end{itemize}

\subsection{Closed-form expressions}\label{sec:SELI_explicit}
\subsubsection{Norms}

The following two lemmas are essentially restatements of Lemma \ref{lem:norms}. 

\begin{lemma}[Norms of classifiers]\label{lem:norms_w}
The following statements are true about the norms of the classifiers.
\begin{enumerate}[label={(\roman*)},itemindent=0em]

\item The classifier norms across all majority / minority classes are all the same. That is, 
$$
\forall c=1,\ldots,k/2\,,\quad\|\w_c\|_2 =: \|\wmaj\|_2
\quad\text{and}\quad
\forall c=k/2+1,\ldots,k\,,\quad\|\w_c\|_2 =: \|\wmin\|_2
$$
where we let $\|\wmaj\|_2$ / $\|\wmin\|_2$ denote the majority / minority norm of an arbitrary class of the corresponding type.

\item It holds that
\begin{align}
\|\wmaj\|_2^2 = \sqrt{R}(1-2/k)+\frac{\sqrt{(R+1)/2}}{k}
\quad\text{and}\quad
\|\wmin\|_2^2 = (1-2/k)+\frac{\sqrt{(R+1)/2}}{k}.
\end{align}
Thus, $\|\wmaj\|_2\geq\|\wmin\|_2$ with equality if and only if $R=1$ or $k=2$.
\end{enumerate}
\end{lemma}
\begin{proof}
Recall, since the classifiers follow the SELI geometry, it holds that $\W^T\W=\Vb\Lambdab\Vb^T.$ Hence, it suffices to compute the diagonal of the matrix $\Vb\Lambdab\Vb^T$. We have
\begin{align}
\Vb\Lambdab\Vb^T &= \sqrt{R}
\begin{bmatrix}
\Pb_m\Pb_m^T & 0 \\ 0 & 0
\end{bmatrix}
+\frac{\sqrt{(R+1)/2}}{k}\begin{bmatrix}
\ones_{m}\ones_{m}^T & -\ones_{m}\ones_{m}^T
\\
-\ones_{m}\ones_{m}^T & \ones_{m}\ones_{m}^T
\end{bmatrix}
+
\begin{bmatrix}
0 & 0 \\ 0 & \Pb_m\Pb_m^T
\end{bmatrix} \nn
\\
&=
\begin{bmatrix}
\sqrt{R}\Id_{k/2}-\frac{1}{k}\left(2\sqrt{R}-\sqrt{(R+1)/2}\right)\ones_{k/2}\ones_{k/2}^T
&
-\frac{\sqrt{(R+1)/2}}{k}\ones_{k/2}\ones_{k/2}^T
\\
-\frac{\sqrt{(R+1)/2}}{k}\ones_{k/2}\ones_{k/2}^T
&
\Id_{k/2}-\frac{1}{k}\left(2-\sqrt{(R+1)/2}\right)\ones_{k/2}\ones_{k/2}^T
\end{bmatrix}\label{eq:SELI_Gw}\,.
\end{align}
From this, the statements of the lemma follow readily.
\end{proof}

\begin{lemma}[Norms of embeddings]\label{lem:norms_h}
The following statements are true about the norms of the embeddings:
\begin{enumerate}[label={(\roman*)},itemindent=0em]

\item The embedding norms across all majority / minority classes are all the same. That is, 
\begin{align*}
\forall j\in \{i\in[n]\,:\, y_i=1,\ldots,k/2\} \,,\quad\|\h_j\|_2 &=: \|\hmaj\|_2, \qquad\text{and}
\\
\forall j\in\{i\in[n]\,:\, y_i=k/2+1,\ldots,k\}\,,\quad\|\h_j\|_2 &=: \|\hmin\|_2.
\end{align*}
where we let $\|\hmaj\|_2$ / $\|\hmin\|_2$ denote the majority / minority norm of an arbitrary example of the corresponding type.

\item It holds that
\begin{align}
\|\hmaj\|_2^2 = \frac{1}{\sqrt{R}}(1-2/k)+\frac{1}{k\sqrt{(R+1)/2}}\quad\text{and}\quad\|\hmin\|_2^2 = (1-2/k)+\frac{1}{k\sqrt{(R+1)/2}}.
\end{align}
 Thus, $\|\hmaj\|_2\leq\|\hmin\|_2$ with equality if and only if  $R=1$ or $k=2$.

\end{enumerate}
\end{lemma}
\begin{proof}
Recall, since the classifiers follow the SELI geometry, it holds that $\Hb^T\Hb=\Ub\Lambdab\Ub^T.$ Hence, it suffices to compute the diagonal of the matrix $\Ub\Lambdab\Ub^T$. Recalling that $n=k(R+1)/2$ We have
\begin{align}
\Ub\Lambdab\Ub^T &= \sqrt{R}
\begin{bmatrix}
\frac{1}{R}\Pb_m\Pb_m^T\otimes\ones_R\ones_R^T & 0 \\ 0 & 0
\end{bmatrix}
+\frac{1}{k\sqrt{(R+1)/2}}\begin{bmatrix}
\ones_{Rm}\ones_{Rm}^T & -\ones_{Rm}\ones_{m}^T
\\
-\ones_{m}\ones_{Rm}^T & \ones_{m}\ones_{m}^T
\end{bmatrix}
+
\begin{bmatrix}
0 & 0 \\ 0 & \Pb_m\Pb_m^T
\end{bmatrix} \nn
\\
&=
\begin{bmatrix}
\left(\frac{1}{\sqrt{R}}\Id_{k/2}-\frac{1}{k}\left(\frac{2}{\sqrt{R}}-\frac{1}{\sqrt{(R+1)/2}}\right)\ones_{k/2}\ones_{k/2}^T\right)\otimes \ones_R\ones_R^T
&
-\frac{1}{k\sqrt{(R+1)/2}}\ones_{k/2}\ones_{k/2}^T
\\
-\frac{1}{k\sqrt{(R+1)/2}}\ones_{k/2}\ones_{k/2}^T
&
\Id_{k/2}-\frac{1}{k}\left(2-\frac{1}{\sqrt{(R+1)/2}}\right)\ones_{k/2}\ones_{k/2}^T
\end{bmatrix}\label{eq:SELI_Gh}\,.
\end{align}
From this, the statements of the lemma follow readily by reading out the diagonal.
\end{proof}

\subsubsection{Angles}\label{sec:SELI_angles}

\begin{lemma}[Angles of classifiers]\label{lem:angles_w}
The following statements are true about the angles of the classifiers.
\begin{enumerate}[label={(\roman*)},itemindent=0em]

\item The classifiers' angles across all majority / minority classes are all the same. That is, 
\begin{align*}
\forall c\neq c'\in\{1,\ldots,k/2\}\,,\quad\Cos{\w_c}{\w_{c'}} &=: \Cos{\wmaj}{\wmaj'}
\\
\forall c\neq c'\in\{k/2+1,\ldots,k\}\,,\quad\Cos{\w_c}{\w_{c'}} &=: \Cos{\wmin}{\wmin'}
\\
\forall c\in\{1,\ldots,k/2\}, c'\in\{k/2+1,\ldots,k\}\,,\quad\Cos{\w_c}{\w_{c'}} &=: \Cos{\wmaj}{\wmin}\,.
\end{align*}

\item It holds that
\begin{align*}
\Cos{\wmaj}{\wmaj'}&=\frac{-2\sqrt{R}+\sqrt{(R+1)/2}}{(k-2)\sqrt{R}+\sqrt{(R+1)/2}}
\\
\Cos{\wmin}{\wmin'}&=\frac{R-7}{R-7+2k\left(2+\sqrt{(R+1)/2}\right)}
\\
\Cos{\wmaj}{\wmin}&=-\frac{\sqrt{(R+1)/2}}{k\|\wmaj\|_2\|\wmin\|_2}\,.
\end{align*}
\end{enumerate}
\end{lemma}

\begin{proof} Recall under the SELI geometry that $\W^T\W=\Vb\Lambdab\Vb^T$. Hence, inspecting the off-diagonal entries of the matrix computed in Equation \eqref{eq:SELI_Gw} gives Statement (i) and  the following inner-product relations:
\begin{align*}
\wmaj^T\wmaj' &= -\frac{1}{k}\left({2}{\sqrt{R}}-{\sqrt{(R+1)/2}}\right)
\\
\wmin^T\wmin' &= -\frac{1}{k}\left(2-{\sqrt{(R+1)/2}}\right)
\\
\wmin^T\wmaj &= -\frac{\sqrt{(R+1)/2}}{k}\,.
\end{align*}
Combine these with the norm calculations in Lemma \ref{lem:norms_w} to prove Statement (ii).
\end{proof}

\begin{lemma}[Angles of classifiers]\label{lem:angles_h}
The following statements are true about the angles of the embeddings.
\begin{enumerate}[label={(\roman*)},itemindent=0em]

\item The embeddings' angles across all majority / minority classes are all the same. That is, 
\begin{align*}
\forall c\neq c'\in\{1,\ldots,k/2\}\,,\quad\Cos{\h_c}{\h_{c'}} &=: \Cos{\hmaj}{\hmaj'}
\\
\forall c\neq c'\in\{k/2+1,\ldots,k\}\,,\quad\Cos{\h_c}{\h_{c'}} &=: \Cos{\hmin}{\hmin'}
\\
\forall c\in\{1,\ldots,k/2\}, c'\in\{k/2+1,\ldots,k\}\,,\quad\Cos{\h_c}{\h_{c'}} &=: \Cos{\hmaj}{\hmin}\,.
\end{align*}

\item It holds that 
\begin{align*}
\Cos{\hmaj}{\hmaj'}&=\frac{-(R+2)}{-(R+2)+k\left(R+1+\sqrt{R}\sqrt{(R+1)/2}\right)}
\\
\Cos{\hmin}{\hmin'}&=\frac{1-\sqrt{2}\sqrt{R+1}}{1-\sqrt{2}\sqrt{R+1}+k\sqrt{(R+1)/2}}
\\
\Cos{\hmaj}{\hmin}&=-\frac{1}{\|\hmaj\|_2\|\hmin\|_2\,k\,\sqrt{(R+1)/2}}\,.
\end{align*}
\end{enumerate}
\end{lemma}

\begin{proof}
Recall under the SELI geometry that $\Hb^T\Hb=\Ub\Lambdab\Ub^T$. Hence, inspecting the off-diagonal entries of the matrix computed in Equation \eqref{eq:SELI_Gh} gives Statement (i) and  the following inner-product relations:
\begin{align*}
\hmaj^T\hmaj' &= -\frac{1}{k}\left(\frac{2}{\sqrt{R}}-\frac{1}{\sqrt{(R+1)/2}}\right)
\\
\hmin^T\hmin' &= -\frac{1}{k}\left(2-\frac{1}{\sqrt{(R+1)/2}}\right)
\\
\hmin^T\hmaj &= -\frac{1}{k\sqrt{(R+1)/2}}\,.
\end{align*}
Combine these with the norm calculations in Lemma \ref{lem:norms_h} to prove Statement (ii).
\end{proof}

\subsubsection{(Non)-alignment}
In the previous lemmas  we compute the angles between classifiers of different classes and between embeddings of different classes. Here, we also also compute the angles between classifiers and embeddings. Specifically, for each $c\in[k]$, we compute the angle $\Cos{\w_c}{\h_i}$ for an  example $i: y_i=c$ that belongs to the same class. These values can be thought of as the degree of alignment between classifiers and embeddings, as $\Cos{\w_c}{\h_i}=1$ corresponds to exact alignment between the two. 

\begin{lemma}[Alignment of classifiers and embeddings]\label{lem:align} The following statements are true about the degree of alignment between class-embeddings and their corresponding classifiers.

\begin{enumerate}[label={(\roman*)},itemindent=0em]\label{lem:angles_wh}

\item The angles between majority- / minority- class embeddings and their corresponding classifiers are all the same. That is,
\begin{align*}
\forall c\in\{1,\ldots,k/2\} \text{ and } i\,:\,y_i=c,~\Cos{\w_c}{\h_i} &= \Cos{\wmaj}{\hmaj}
\\
\forall c\in\{k/2+1,\ldots,k\} \text{ and } i\,:\,y_i=c,~\Cos{\w_c}{\h_i} &= \Cos{\wmin}{\hmin}.
\end{align*}

\item It holds that
\begin{align}
\Cos{\wmaj}{\hmaj} &= \frac{1-1/k}{\|\wmaj\|_2\|\hmaj\|_2} \label{eq:wh_maj}
\\
\Cos{\wmin}{\hmin} &= \frac{1-1/k}{\|\wmin\|_2\|\hmin\|_2}. \label{eq:wh_min}
\end{align}

\end{enumerate}

\end{lemma}
\begin{proof}
The proof is immediate by recognizing that under the SELI geometry for all $c\in[k]$ and $i:y_i= c$ it holds that $\w_c^T\h_i=1-1/k$.
\end{proof}

\subsubsection{Centering}\label{sec:SELIcentering}
It is easy to see that the classifiers $\w_c, c\in[k]$ following the SELI geometry are centered, i.e. $\sum_{c\in[k]}\w_c =0$. For example, we can see this from the facts that $\W^T\W=\Vb\Lambdab\Vb^T$ and $\Vb^T\ones_k=0.$ See also Lemma \ref{lem:W_ones} for an alternative proof. 

On the other hand, the embeddings $\h_i,i\in[n]$ are \emph{not} centered around zero in general. Instead, it holds that 
\begin{align}\label{eq:center_h}
\sum_{i\in[n]}\frac{1}{n_{y_i}} \h_i = 0.
\end{align}
Note that this reduces to $\sum_{i\in[n]}\h_i$ for balanced data, but is not true otherwise. To see \eqref{eq:center_h} recall first that $\Hb^T\Hb=\Ub\Lambdab\Ub^T$. Second, check that $\sum_{i\in[n]}\frac{1}{n_{y_i}}\zhat_i = 0$, i.e. $\Zhat\omegab=0$ where $\omegab[i]=1/n_{y_i}, i\in[n]$. Thus, $\Ub^T\omegab=0$. Combining these two it follows that $\Hb\omegab=0$, which gives the desired. 

Now, suppose that the \emph{\ref{NC}} property also holds, i.e. the embeddings collapse to their class means, that is, 
$$
\forall i\in[n],~~ \h_i = \mub_{y_i} := \frac{1}{n_{y_i}}\sum_{j\,:\,y_j=y_i} \h_j.
$$ 
Then, \eqref{eq:center_h} implies the following about the class means:
\begin{align}\label{eq:center_mu}
0 = \sum_{i\in[n]}\frac{1}{n_{y_i}} \h_i = \sum_{c\in[k]}\frac{1}{n_{c}}\sum_{i\,:\,y_i=c} \h_i = \sum_{c\in[k]}\mub_{c}.
\end{align}
Therefore, the class means are always centered around zero.

\subsection{Illustrations and discussion on dependence on $R$ and $k$}
In the previous section, we derived closed-form expressions for the features describing the SELI geometry. Here, we use these expressions to study how varying values of class-number $k$ and imbalance-ratio $R$ change the geometry.  We use the numerical illustration in \Fig~\ref{fig:SELI_theory} to guide the discussion. Specifically, in \Fig~\ref{fig:SELI_theory} we compute and plot the norm ratios, alignment and angles between embeddings and classifiers for $k=2,4,10,20$ classes and imbalance ratio varying from $1$ (aka balanced) to $100$.


\begin{figure}[t]
     \centering
     \begin{subfigure}[b]{0.49\textwidth}
         \centering
         \includegraphics[width=0.75\textwidth]{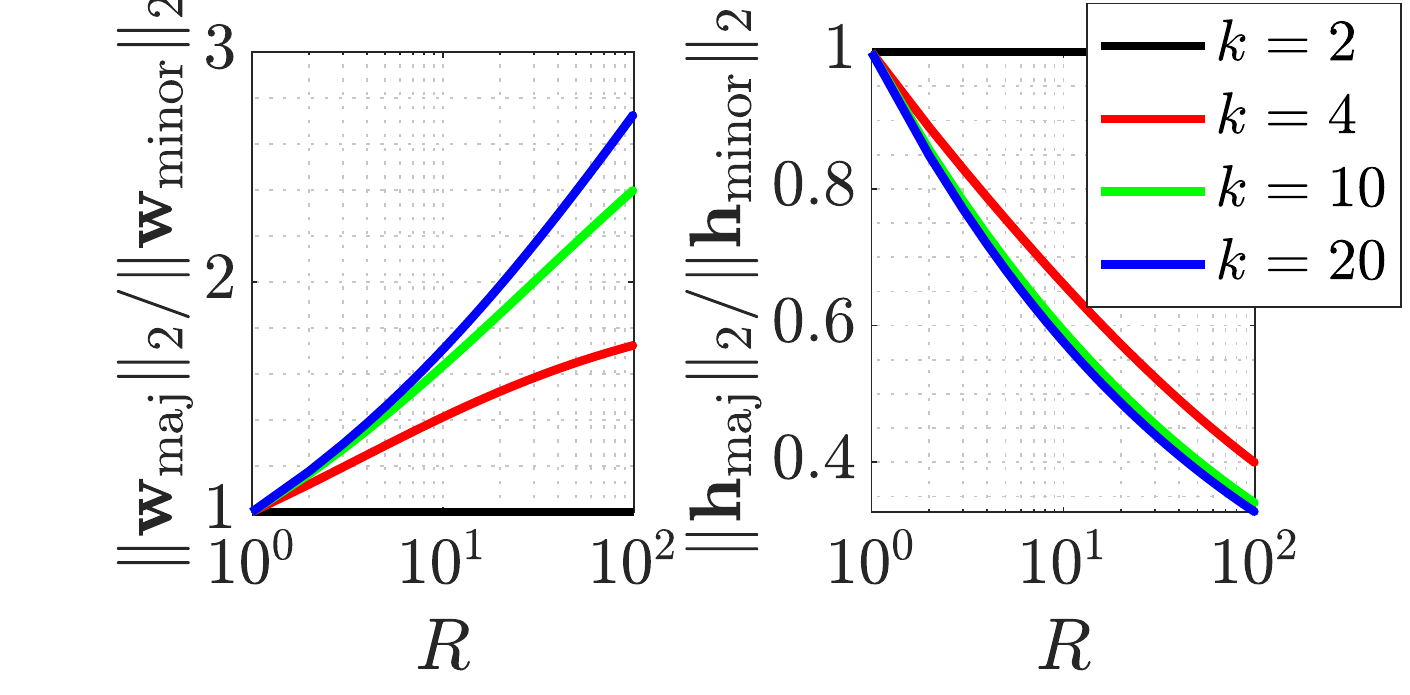}
         \caption{Norms of classifiers and embeddings.}
         \label{fig:SELI_theory_norms}
     \end{subfigure}
     \hfill
     \begin{subfigure}[b]{0.49\textwidth}
         \centering
         \includegraphics[width=0.75\textwidth]{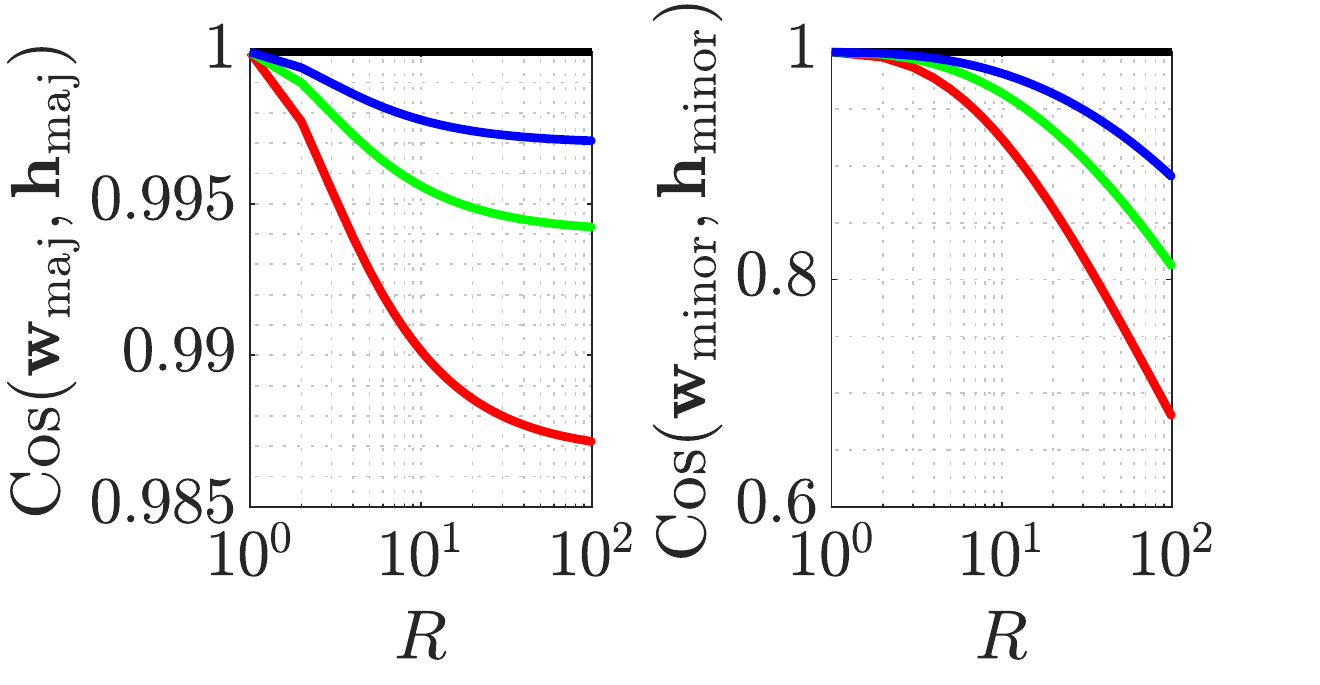}
         \captionsetup{width=0.9\linewidth}\caption{Alignment of classifiers and embeddings.}
         \label{fig:SELI_theory_angles_wh}
     \end{subfigure}
     \\
          \begin{subfigure}[b]{0.75\textwidth}
         \centering
              \includegraphics[width=\linewidth]{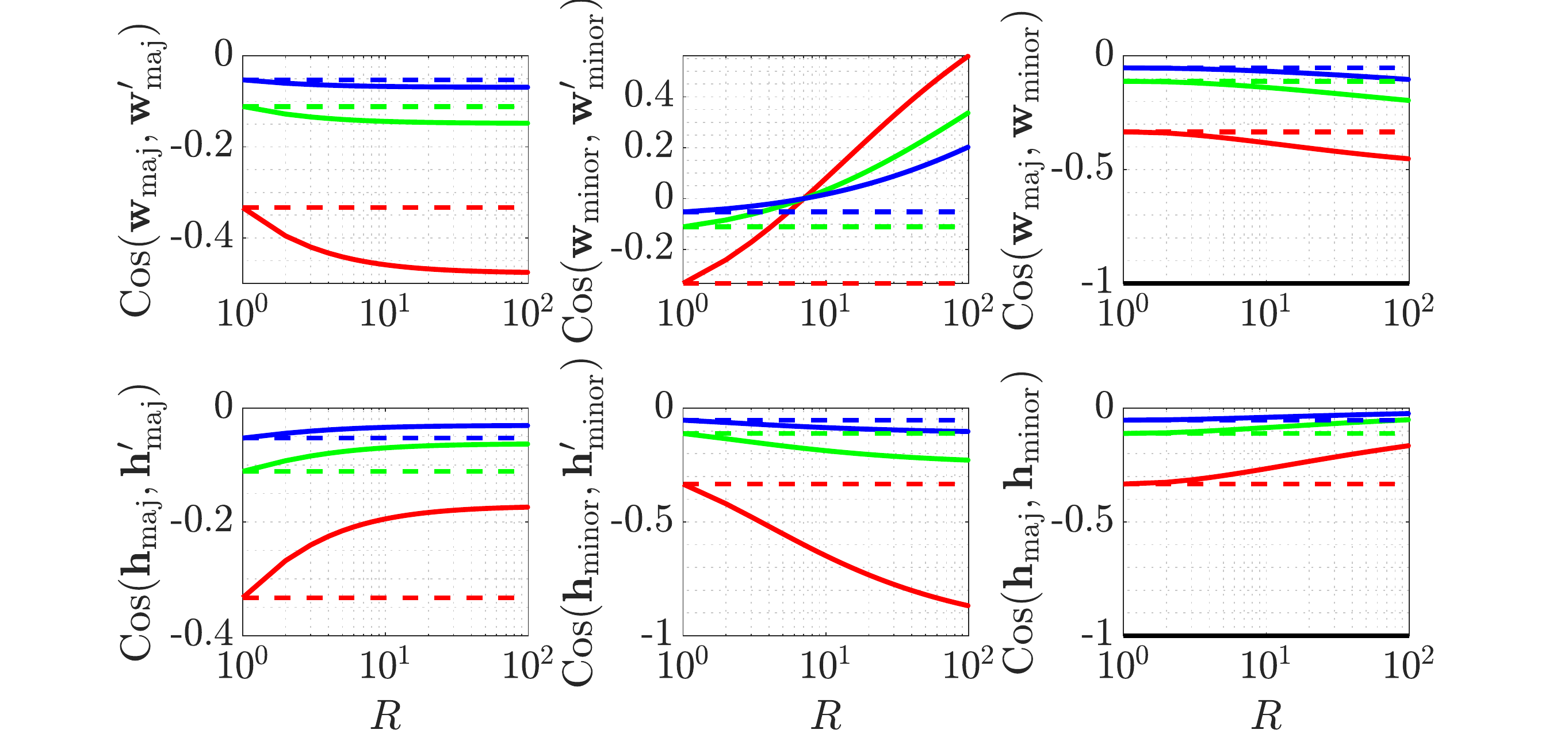}
         \captionsetup{width=0.9\linewidth}\caption{Angles of classifiers and embeddings. Dashed lines are are at $-1/{(k-1)}$.}
         \label{fig:SELI_theory_angles}
     \end{subfigure}
        \caption{
        Same as \Fig~\ref{fig:SELI_theory_inf} only $R$ varies from $1$ to $100$ to ease visualizations of how the geometry changes for finite $R$ values.}
        \label{fig:SELI_theory}
\end{figure}

\paragraph{Norms.}~\Fig~\ref{fig:SELI_theory_norms} shows the ratios of majority vs minority norms for both classifiers and embeddings. The values are computed using Lemmas \ref{lem:norms_w} and \ref{lem:norms_h}. For binary problems (aka $k=2$), the norm ratio is always equal to one irrespective of imbalance. For larger values of $k$, the norm ratio is equal to one only for balanced classes (aka $R=1$). Recall that equal norm ratios is a feature of the ETF geometry \citep{NC}. On the other hand, for $k\neq 2$ and $R>1$, the majorities have strictly larger norms for the classifiers and strictly smaller norms for the embeddings. Thus, in general the SELI geometry is very different from the ETF geometry. Interestingly, the difference is already evident when going from $k=2$ to $k=4$ classes. Also, the change in the ratios for classifiers is more pronounced than that for embeddings, which is changing progressively slower as $k$ increases (see how close are the green and blue curves in the right plot).

\paragraph{Alignment.}~\Fig~\ref{fig:SELI_theory_angles_wh} shows the degree to which the geometries of classifiers and embeddings are aligned to each other. Specifically, we plot the cosine between any majority (left) / minority (right) classifier and corresponding embeddings belonging to the same class; see also Lemma \ref{lem:align}.    For $k=2$ and $R=1$ the cosines are equal to one indicating that classifiers and embeddings align for both majorities and minorities. This is consistent with the ETF geometry. On the other hand, the alignment property breaks when $k>2$ and $R>1$. The effect is more drastic for the minorities (right plot), while for majorities the alignment is approximately preserved (note the y-axis scale is different in the two plots). Interestingly, alignment is in fact favored for larger number of classes, but deteriorates with increasing $R$ consistently for all values of $k$. 

\paragraph{Angles.}~\Fig~\ref{fig:SELI_theory_angles} shows the angles between majority/majority (left), minority/minority (center), and minority/majority (right) for both classifiers (top)  and embeddings (bottom). The values are computed using Lemmas \ref{lem:angles_w} and \ref{lem:angles_h}. For binary problems (aka $k=2$), there is only one majority and one minority class. Thus, we only plot the  minority/majority cosines, which are always equal to $-1/(k-1)$ irrespective of imbalance. For larger values of $k$, the cosines are equal to that same value  $-1/(k-1)$ only for balanced classes (aka $R=1$). Recall that cosine value equal to  $-1/(k-1)$ is a unique feature of the ETF geometry. On the other hand, for $k\neq 2$ and $R>1$, the cosines are different. For reference, we plot the values of $-1/(k-1)$ in dashed lines. For the classifiers, the majority angles increase, while the minority angles decrease. The rate of change is more drastic for minorities. In both cases, the rate of change across $R$ is more pronounced for smaller $k$. The majority-minority angles also increase with $R$. The trend is reversed for embeddings. For example, the angles between minority embeddings become larger with increasing $R$. Again, the effect of imbalance (at least for the values of $R$ shown) is more pronounced here for smaller values of $k$.


\subsection{Special cases}

\subsubsection{Balanced classes and binary classification}

The following result follows directly by combining Lemmas \ref{lem:norms_w}, \ref{lem:norms_h}, \ref{lem:angles_w}, \ref{lem:angles_h} and \ref{lem:angles_wh}. 

\begin{corollary}[$k=2$ or $R=1$: \emph{\ref{SELI}}$\equiv$\emph{\ref{ETF}}]\label{cor:SELI_to_ETF}
Assume $k=2$ or $R=1$. The following hold.
\begin{align*}
&\|\w_1\|_2=\ldots=\|\w_k\|_2\quad\text{ and }\quad
\|\h_1\|_2=\ldots=\|\h_n\|_2\,,
\\
&\forall c\neq c'\in[k],~\Cos{\w_c}{\w_c'}=-1/(k-1)\quad\text{ and }\quad &\forall i\neq i'\in[n],~\Cos{\h_i}{\h_i'}=-1/(k-1)\,,
\\
&\forall c\in[k], i\,:\,y_i=c,~\Cos{\w_c}{\h_i}=1.
\end{align*}
Thus, the SELI geometry is same as the ETF geometry in balanced and in binary classification.
\end{corollary}

\subsubsection{Asymptotics}
%

We can also use the results of \Sec~\ref{sec:SELI_explicit} to understand the SELI geometric features asymptotically as $R$ increases. These are included in Corollary \ref{cor:R_inf} below. See also \Fig~\ref{fig:SELI_theory_inf} for a numerical illustration of the limiting behavior for large imbalance ratios.

\begin{corollary}[$R\rightarrow\infty$]\label{cor:R_inf} Fix even $k>2$. Then, the following limits hold.

\begin{enumerate}[label={(\roman*)},itemindent=0em]

\item $\lim_{R\rightarrow\infty} \frac{\|\wmaj\|_2^2}{\|\wmin\|_2^2} = 1+(k-2)\sqrt{2}$

\item $\lim_{R\rightarrow\infty} \frac{\|\hmaj\|_2^2}{\|\hmin\|_2^2} = 0$

\item $\lim_{R\rightarrow\infty} \Cos{\wmaj}{\wmaj'} = \frac{-4+\sqrt{2}}{\sqrt{2}+2(k-2)}$

\item $\lim_{R\rightarrow\infty} \Cos{\wmin}{\wmin'} = 1$

\item $\lim_{R\rightarrow\infty} \Cos{\wmaj}{\wmin} = -\frac{1}{\sqrt{\sqrt{2}+2(k-2)}}$

\item $\lim_{R\rightarrow\infty} \Cos{\hmaj}{\hmaj'} = \frac{1}{k\left(1+\sqrt{2}/2\right)-1}$

\item $\lim_{R\rightarrow\infty} \Cos{\hmin}{\hmin'} = -\frac{2}{k-2}$

\item $\lim_{R\rightarrow\infty} \Cos{\hmaj}{\hmin} = 0$

\item $\lim_{R\rightarrow\infty} \Cos{\wmaj}{\hmaj} = \frac{k-1}{\sqrt{k+\sqrt{2}-2}\sqrt{k+\sqrt{2}/2-2}}$

\item $\lim_{R\rightarrow\infty} \Cos{\wmin}{\hmin} = 0$

\end{enumerate}
\end{corollary}

\begin{proof} The expressions that appear in the lemmas in \Sec~\ref{sec:SELI_explicit} hold for any value of $R$. Take the limit of $R\rightarrow\infty$ to yield the expressions above. We omit the details for brevity.
\end{proof}

Several interesting conclusions are immediate from the formulas above. For example, Statement (iv) shows that, asymptotically in $R$, the minority classes collapse to the same vector. This is a manifestation of the minority collapse phenomenon discovered by \citet{fang2021exploring}. Asymptotic minority collapse had not been shown before for the UF-SVM. See also \Sec~\ref{sec:minority} for an extended discussion. We note that beyond minority classifiers, Corollary \ref{cor:R_inf} is further conclusive about the behavior of majority classifiers, as well as, minority and majority embeddings. For example, Statement (viii) suggests that the minority and majority embeddings become orthogonal to each other asymptotically. Further investigations of the validity of such asymptotic conclusions in deep-net training is beyond our scope.

\section{Proofs for UF-SVM Section \ref{sec:UF-SVM}}\label{sec:proofs_UF-SVM}
\subsection{Proof of Theorem \ref{thm:SVM}}\label{sec:proof_SVM}
We consider a relaxation of the non-convex SVM in \eqref{eq:svm_original} by setting
\begin{align}\label{eq:X_rel}
\X = \begin{bmatrix} \W^T \\ {\Hb}^T \end{bmatrix} \begin{bmatrix} \W & \Hb \end{bmatrix}= \begin{bmatrix} \W^T\W & \W^T\Hb \\ \Hb^T\W & \Hb^T\Hb\end{bmatrix} \in\R^{(k+n)\times (k+n)} .
\end{align}
With this consider the following semidefinite program:
\begin{align}\label{eq:cvx_X}
\mathrm{q}_*  &= \min_{\X\succeq 0}~~~~~~
\frac{1}{2}\tr\Big(\X\Big)\\
\nn&~~~\text{sub.~to}~~~~\X[y_i,k+i]-\X[c,k+i]\geq 1,~ \forall i\in[n], c\neq y_i.
\end{align}
It is not hard to see that $\mathrm{q}_* \leq \mathrm{p}_* $. In what follows, we will compute the optimal set of \eqref{eq:cvx_X} and use this to show that the relaxation is in fact tight. This will allow us to characterize the solution of the original problem. 

\vp
\noindent\underline{\emph{Dual of the convex relaxation:}}~The optimization in \eqref{eq:cvx_X} is convex and satisfies Slater's conditions. (Since constraints are affine, it suffices to check feasibility, which is easily verified.)
 Hence, strong duality holds and KKT conditions are necessary and sufficient for optimality. The dual of \eqref{eq:cvx_X} is written as follows: 
\begin{align}\label{eq:dual_X}
\mathrm{d}_*  &=\max_{\{a_{i,c}\}_{i\in[n], c\neq y_i}}~~\sum_{i\in[n]}\sum_{c\neq y_i}\alpha_{ic}
\\&~~~~~\text{sub.~to}~~~~~~ \begin{bmatrix} \Id_{k} & \A^T \\ \A & \Id_{n}
\end{bmatrix}\succeq 0 \nn
\\&~~~~~~~~~~~~~~~~~~~~\alpha_{ic}\geq 0, i\in[n], c\neq y_i,~~\forall i\in[n]:  \A[{i,c}] = \begin{cases}
-\sum_{c'\neq y_i}\alpha_{ic'}, &y_i=c,
\\\alpha_{ic},&y_i\neq c.
\end{cases}
\nn
\end{align}
Also, the complementary slackness conditions are
\begin{align}\label{eq:slackness}
 \begin{bmatrix} \Id_{k} & \A^T \\ \A & \Id_{n}
\end{bmatrix} \X = 0 
\qquad\text{and}\qquad
\forall i\in[n],c\neq y_i\,:\,\alpha_{ic}(1-\X[y_i,k+i]+\X[c,k+i])=0.
\end{align}
Instead of working with the dual in the above standard form, it is convenient to work with an alternative representation by (re)-defining dual variables $\beta_{ic}, i\in[n], c\in[k]$ such that 
\footnote{The same re-parameterization trick, but for a simpler (convex) max-margin program was used by \citet{wang2021benign}.}
\begin{align}\label{eq:betas}
\beta_{ic}=\begin{cases}
\sum_{c'\neq y_i}\alpha_{ic'}, &y_i=c,
\\-\alpha_{ic},&y_i\neq c.
\end{cases}
\end{align}
Specifically, arrange these new dual variable representations in a matrix $\Bb\in\R^{n\times k}$ with entries $\Bb[i,c]=\beta_{ic}, i\in[n], c\in[k]$ and recall the \SEL~matrix $\Zhat\in\R^{k\times n}$ in Definition \ref{def:SEL}. With these, we can rewrite the dual in \eqref{eq:dual_X} in the following more convenient form: 
\begin{align}\label{eq:dual2_X}
\mathrm{d}_*  &=\max_{\Bb\in\R^{n\times k} }~~~~ \tr(\Zhat\Bb)
\\&~~~~\text{sub.~to}~~ \begin{bmatrix} \Id_{k} & -\Bb^T \\ -\Bb & \Id_{n}
\end{bmatrix}\succeq 0 \nn
\\&~~~~~~~~~~~~~~~~ 
 \forall i\in[n]: 
\Bb[i,y_i] = - \sum_{c\neq y_i}\Bb[i,c]\nn
\\&~~~~~~~~~~~~~~~~ \Bb\odot\Zhat^T \geq 0 \label{eq:other}\,.
\end{align}
Analogously, the complementary slackness conditions in \eqref{eq:slackness} are equivalent to the following:
\begin{align}\label{eq:slackness2}
 \begin{bmatrix} \Id_{k} & -\Bb^T \\ -\Bb & \Id_{n}
\end{bmatrix} \X = 0 
\quad\text{and}\quad
\forall i\in[n],c\neq y_i\,:\,\Bb[i,c]\,(1-\X[y_i,k+i]+\X[c,k+i])=0.
\end{align}
To see the equivalence of the objective of \eqref{eq:dual2_X} denote $A=\sum_{i\in[n]}\sum_{c\neq y_i}\alpha_{ic}$ the objective in \eqref{eq:dual_X} and note that  we get simultaneously  $A=\sum_{i\in[n]}\beta_{iy_i}=\sum_{i\in[n]}\sum_{c\neq y_i}\big(-\beta_{ic}\big)$ following the definition in Equation \eqref{eq:betas}. Then, we have $A=\frac{k-1}{k}A+\frac{1}{k}A=\sum_{i\in[n]}\sum_{c\in[k]}\hat{z}_{ic}\beta_{ic}=\tr\left(\Zhat\Bb\right)$.

\vp
\noindent\underline{\emph{Solution to the dual:}}~To continue, we consider the following relaxation of the dual problem \eqref{eq:dual2_X} (by removing the constraint in Equation \eqref{eq:other}):
\begin{align}
\hat{\mathrm{d}} = \max_{\Bb\in\R^{n\times k}}~~ \tr(\Zhat\Bb)\qquad\text{sub. to}~~\|\Bb\|_2\leq 1\quad\text{ and }\quad \Bb\ones_k=0.\label{eq:unconstrained_B}
\end{align}
Here, recall that $\|\Bb\|_2$ denotes the spectral norm, hence the first constraint is equivalent to the first constraint in the maximization in \eqref{eq:dual2_X} by Schur-complement argument. Using standard  arguments, it can be shown that $\hat{\mathrm{d}}\leq\|\Zhat\|_*$ and equality holds by setting 
\begin{align}\label{eq:uncon_feasibility}
\hat\Bb &= \Ub\Vb^T,
\end{align}
where we recalled the compact SVD $\Zhat=\Vb\Lambdab\Ub^T$. It is also not hard to see that $\hat\Bb$ is feasible in \eqref{eq:unconstrained_B} (recall here that $\Vb^T\ones_k=0$). Hence, $\hat\Bb$ is a maximizer and $\hat{\mathrm{d}}=\|\Zhat\|_*$. The following lemma, the proof of which we defer to the end of this section, proves something stronger: $\hat\Bb$ is in fact the only maximizer in \eqref{eq:unconstrained_B}.

\begin{lemma}\label{lem:unconstrained_B}
The optimal cost of the maximization in \eqref{eq:unconstrained_B} is $\|\Zhat\|_*$ and $\hat\Bb=\Ub\Vb^T$ is its unique maximizer.
\end{lemma}

Next, we use Lemma \ref{lem:unconstrained_B} to show that $\hat\Bb$ 
also satisfies the inequality constraints in \eqref{eq:other}. To do this, we use an explicit construction of the singular factors $\Ub$ and $\Vb$ presented in \Sec~\ref{sec:eigen_SEL}, which is possible thanks to the special structure of $\Zhat$. Specifically, we prove in Lemma \ref{lem:Zhat_SVD_general} 
that
\begin{align}\label{eq:dual_feas}
\hat\Bb\odot \Zhat^T = (\Ub\Vb^T)\odot \Zhat^T> 0.
\end{align}
Hence,  $\hat\Bb$ is feasible in \eqref{eq:dual2_X}. In fact, the feasibility inequalities are strict, which we will use soon. For now, note that feasibility of  $\hat\Bb$ in \eqref{eq:dual2_X} guarantees that it is its unique maximizer (since it is the unique maximizer of the program's relaxation, as established in Lemma \ref{lem:unconstrained_B}.) Therefore, 
\begin{align}
\mathrm{d}_*=\hat{\mathrm{d}}=\|\Zhat\|_*\,,
\end{align}
and $\hat\Bb$ in \eqref{eq:uncon_feasibility} is dual optimal for the semidefinite program in \eqref{eq:cvx_X}.

\vp
\noindent\underline{\emph{Solution to the primal relaxation:}}~By strong duality, this implies  $\mathrm{q}_*=\mathrm{d}_*=\|\Zhat\|_*$ and that any primal minimizer $\Xhat=\begin{bmatrix}\hat\X_{11} & \hat\X_{12} \\ \hat\X_{12}^T & \hat\X_{22} \end{bmatrix}$ satisfies the complementary slackness conditions in \eqref{eq:slackness2} with $\Bb=\hat\Bb$, i.e.,
\begin{align}
\forall i\in[n], c\neq y_i\,:\,\Xhat_{12}[y_i,i]-\Xhat_{12}[c,i]-1&=0
\label{eq:comp1a}
\\
\hat\X_{11} - \hat\Bb^T\hat\X_{12}^T = \hat\X_{12} - \hat\Bb^T\hat\X_{22} = -\hat\Bb\hat\X_{12}+ \hat\X_{22}  &=0  \label{eq:comp1b}\,.
\end{align}
In \eqref{eq:comp1a} we used from \eqref{eq:dual_feas} that inequalities are strict; thus, $\beta_{ic}>0$ for all $i\in[n], c\neq y_i$. Now, from \eqref{eq:comp1b} it follows that
\begin{align*}
\hat\X_{11} = \hat\Bb^T\hat\X_{12}^T,\quad \hat\X_{22} = \hat\Bb\hat\X_{12}, \quad \hat\X_{12} = \hat\Bb^T\hat\X_{22}\,.
\end{align*}
Combining the last two equations shows 
\begin{align*}
(\Id_n-\hat\Bb\hat\Bb^T)\hat\X_{22}=0 
&\implies (\Id_n-\Ub\Ub^T)\hat\X_{22}=0\implies \hat\X_{22}=\Ub\Db\Ub^T,
\\
&\implies\hat\X_{12}=\hat\Bb^T\hat\X_{22}=\Vb\Db\Ub^T \qquad\text{and}\qquad \hat\X_{11}=\hat\Bb^T\hat\X_{12}^T=\Vb\Db\Vb^T
\end{align*}
for some 
(k-1)--dimensional 
matrix $\Db\in\R^{(k-1)\times(k-1)}$.

Next, we will use $\hat\X_{12}=\Vb\Db\Ub^T$ in \eqref{eq:comp1a} to compute $\Db$.  For convenience denote  $\eb_{k,c}\in\R^k$  the $c$-th standard basis vector in $\R^k$ and $\eb_{n,i}\in\R^n$ the $i$-th standard basis vector in $\R^n$. Then, starting with \eqref{eq:comp1a}, we get the following chain of implications:
\begin{align}
\forall i\in[n], c\neq y_i\,&:\, (\eb_{k,y_i}-\eb_{k,c})^T\Vb\Db\Ub^T\eb_{n,i}=1\label{eq:comp1a_2}
\\
\nn\stackrel{(\sum_{c\neq y_i}\cdot)}{\implies}\forall i\in[n]\,&:\, \big((k-1)\eb_{k,y_i} - (\ones_k-\eb_{k,y_i})\big)^T\Vb\Db\Ub^T\eb_{n,i}=k-1\\
\nn\implies\forall i\in[n]\,&:\, \big(k\eb_{k,y_i} - \ones_k\big)^T\Vb\Db\Ub^T\eb_{n,i}=k-1
\\
\nn\stackrel{(\Vb^T\ones_k=0)}{\implies}\forall i\in[n]\,&:\, \eb_{k,y_i}^T \Vb\Db\Ub^T\eb_{n,i}=1-1/k
\\
\nn\stackrel{\eqref{eq:comp1a_2}}{\implies}\forall i\in[n]\,&:\, \eb_{k,y_i}^T \Vb\Db\Ub^T\eb_{n,i}=1-1/k \qquad\text{and}\qquad \eb_{k,c}^T \Vb\Db\Ub^T\eb_{n,i}=-1/k,\,c\neq y_i
\end{align}
Specifically, the last equation written in matrix form gives \footnote{
Note that the derivation and conclusion following \eqref{eq:comp1a_2} is general. Specifically, the display above shows that $\Zhat$ is the only $k\times n$ matrix for which it holds simultaneously that $\Z[y_i,i]-\Z[c,i]=1, \forall i\in[n], c\neq y_i$ and $\ones_k^T\Z=0$.
}
\begin{align}
\Vb\Db\Ub^T=\Zhat \stackrel{}{\implies}
\Db=\Vb^T\Zhat\Ub\implies \Db=\Lambdab,
\end{align}
where the second and third equalities used $\Vb^T\Vb=\Ub^T\Ub=\Id_{k-1}$.

To conclude, we have shown that any optimal point $\Xhat$ of \eqref{eq:cvx_X} satisfies
\begin{align}\label{eq:unique_X}
\Xhat = \begin{bmatrix}\Vb \\ \Ub\end{bmatrix}\Lambdab \begin{bmatrix}\Vb^T & \Ub^T\end{bmatrix}.
\end{align}

\vp
\noindent\underline{\emph{Solving the original problem:}}~Now, we show that the convex relaxation in \eqref{eq:cvx_X} is tight. For some partial orthonormal matrix $\Rb\in\R^{(k-1)\times d}$ (recall that $d\geq k-1$) with $\Rb\Rb^T=\Id_{k-1}$, let 
\begin{align}\label{eq:WH_fin}
\What=\Rb^T\Lambdab^{1/2}\Vb^T \qquad\text{and} \qquad\Hhat=\Rb^T\Lambdab^{1/2}\Ub^T.
\end{align}
By construnction $\What^T\Hhat=\Xhat_{12}$. Hence, $(\What,\Hhat)$ is feasible in \eqref{eq:svm_original}. Thus, $ \mathrm{p}_* \leq \frac{1}{2}\|\What\|_F^2+\frac{1}{2}\|\Hhat\|_F^2=\frac{1}{2}\tr(\Xhat)=\mathrm{q}_*$. But, we have already argued that $ \mathrm{p}_*\geq  \mathrm{q}_*$. Hence, $ \mathrm{p}_*= \mathrm{q}_*$.

Take now any minimizer $(\What,\Hhat)$ of the original  problem \eqref{eq:svm_original}. By feasibility of  $(\What,\Hhat)$, the matrix $\boldsymbol{\Omega}:=\begin{bmatrix} \What^T \\ {\Hhat}^T \end{bmatrix} \begin{bmatrix} \What & \Hhat \end{bmatrix}$ is feasible in \eqref{eq:cvx_X}. Also, $\frac{1}{2}\tr(\boldsymbol{\Omega})=\frac{1}{2}\|\What\|_F^2+\frac{1}{2}\|\Hhat\|_F^2= \mathrm{p}_*$. But, $ \mathrm{p}_*= \mathrm{q}_*$. Hence, $\frac{1}{2}\tr(\boldsymbol{\Omega})=\mathrm{q}_*$, which implies that $\boldsymbol{\Omega}$ is optimal. It must then be from \eqref{eq:unique_X} that $$\begin{bmatrix} \What^T \\ {\Hhat}^T \end{bmatrix} \begin{bmatrix} \What & \Hhat  \end{bmatrix} =\boldsymbol{\Omega}=\Xhat=\begin{bmatrix}\Vb \\ \Ub\end{bmatrix}\Lambdab \begin{bmatrix}\Vb^T & \Ub^T\end{bmatrix}.$$
 This completes the proof of the theorem.

\vp
\noindent\underline{\emph{Proof of Lemma \ref{lem:unconstrained_B}:}}~It only remains to prove Lemma \ref{lem:unconstrained_B}, which we do here.
Any feasible $\Bb$ satisfies $\|\Bb\|_2\leq 1$. Hence, also recalling that $\Zhat$ has rank $k-1$ (because $\Vb^T\ones_k=0$): 
\begin{align}
\tr(\Zhat\Bb) = \sum_{i=1}^{k-1}\Lambdab[i,i](\ub_i^T\Bb\vb_i) \leq \sum_{i=1}^{k-1}\Lambdab[i,i]  = \|\Zhat\|_*.
\end{align}
The inequality above is tight if and only if 
\begin{align}\label{eq:Bb_opt}
\forall i\in[k-1]\,:\,\eps_i:=\ub_i^T\Bb\vb_i=1,
\end{align}
which is indeed satisfied by $\hat\Bb=\Ub\Vb^T$. Clearly, $\hat\Bb$ is also feasible. Hence, $\hat\Bb$ is a maximizer of \eqref{eq:unconstrained_B}. 

Next, we will show that there is no other maximizer, say $\wt\Bb$. Indeed, since $\wt\Bb$ is optimal, it must satisfy \eqref{eq:Bb_opt}. Hence, it has rank at least $k-1$. But since $\wt\Bb\ones_k=0$, we find that $\wt\Bb$ has rank exactly $k-1$. Let $\wt\Bb=\Ub_\Bb\Sigmab_\Bb\Vb_\Bb^T$ be its compact SVD and denote $\sigma_i:=\Sigmab_\Bb[i,i],~i\in[k-1]$ its singular values. Finally, define $(k-1)\times (k-1)$ matrices 
\begin{align}
\Pbf=\Ub_\Bb^T\Ub\qquad\text{and}\qquad\Qbf=\Vb_\Bb^T\Vb
\end{align}
with columns $\pbf_i,\qbf_i,\, i\in[k-1].$ By Equation \eqref{eq:Bb_opt} we have the following chain of inequalities for all $i\in[k-1]$:
\begin{align*}
1=\ub_i^T\Bb\vb_i=\pbf_i^T\Sigmab_\Bb\qbf_i &= \sum_{j\in[k-1]}\sigma_j\pbf_i[j]\qbf_i[j] \leq \sum_{j\in[k-1]}\sigma_j|\pbf_i[j]\qbf_i[j]| 
\\&\leq \sum_{j\in[k-1]}|\pbf_i[j]\qbf_i[j]| \leq \|\pbf_i\|_2\|\qbf_i\|_2\leq 1.
\end{align*}
Inspecting this, we note that all inequalities must be equalities.
The first inequality in the second line follows because 
$
\|\wt\Bb\|_2\leq1\implies \forall j\in[k-1]\,:\,\sigma_j\leq1.
$
Hence, $\sigma_j=1$ for all $j\in[k-1]$. Equivalently $\Sigmab_\Bb=\Id_{k-1}.$
The second inequality in that same line is Cauchy-Schwarz and equality implies 
\begin{align}\label{eq:norms_pbf_equal}
\forall i\in[k-1]\,:\,\pbf_i=\pm\qbf_i.
\end{align}
 The last inequality follows because 
 \begin{align}\label{eq:norms_pbf}
 \|\pbf_i\|_2=\|\Ub_\Bb^T\ub_i\|_2\leq\|\Ub_\Bb\|_2\|\ub_i\|_2\leq 1 \qquad\text{and}\qquad
 \|\qbf_i\|_2=\|\Vb_\Bb^T\vb_i\|_2\leq\|\Vb_\Bb\|_2\|\vb_i\|_2\leq 1.
 \end{align}
Since $\Ub_\Bb$ (resp. $\Vb_\Bb$) has orthonormal columns, equality in \eqref{eq:norms_pbf} holds
if and only if for all $i\in[k-1]$, $\ub_i$ (resp. $\vb_i$) is a column of $\Ub_\Bb$ (resp. $\Vb_\Bb$). Then, it must be that $\Pbf$ and $\Qbf$ are permutation matrices. Combined with \eqref{eq:norms_pbf_equal} this gives \begin{align}
\Pbf=\Qbf=\Pibf \implies \Ub_\Bb^T\Ub=\Vb_\Bb^T\Vb=\Pibf
\end{align}
for some permutation matrix $\Pibf$. Continuing from this,
$$
\Ub_\Bb^T\Ub\Pibf^T = \Vb_\Bb^T\Vb \Pibf^T = \Id_{k-1} \implies \Ub_\Bb=\Ub\Pibf^T \text{ and } \Vb_\Bb=\Vb\Pibf^T.
$$
Putting things together, we conclude that 
$$
\wt\Bb = \Ub_\Bb\Sigmab_\Bb\Vb_\Bb^T = \Ub_\Bb\Vb_\Bb^T = \Ub\Pibf^T \Pibf\Vb^T = \Ub\Vb^T=\hat\Bb.
$$
This concludes the proof of the lemma.


\subsection{Proofs of Corollaries \ref{cor:NC}, \ref{cor:SELI}, \ref{cor:balanced} and of Lemma \ref{lem:norms}}
The proofs of the rest of the results of \Sec~\ref{sec:UF-SVM}} are presented in the following sections.
\begin{itemize}
\item \textbf{Corollary \ref{cor:NC}:} See paragraph below the statement of the corollary in \Sec~\ref{sec:NC_and_SELI}.

\item \textbf{Corollary \ref{cor:SELI}:} As already mentioned in \Sec~\ref{sec:NC_and_SELI} this is nothing but a reformulation of Theorem\ref{thm:SVM}\ref{thm:global_logits}--\ref{thm:global_gram} in view of Definition 
\ref{def:SELI} of the SELI property.

\item \textbf{Corollary \ref{cor:balanced}:} See Corollary \ref{cor:SELI_to_ETF} in \Sec~\ref{sec:SELI_properties}.

\item \textbf{Lemma \ref{lem:norms}:} See Lemma \ref{lem:norms_w} in \Sec~\ref{sec:SELI_properties}.
\end{itemize}
\new{\subsection{On different regularization hyperparameters between embeddings and classifiers}\label{sec:different_hyperparams}}
\new{
Thus far in our analysis of the UFM, we assumed same regularization strength $\lambda$ for the embeddings and classifiers; see Eqn. \eqref{eq:CE}. Here, we discuss a slight generalization allowing different regularizations $\lambda_W$ and $\lambda_H$ for the classifiers and embeddings, respectively. This is motivated by previous studies of the UFM in the literature, e.g. \cite{zhu2021geometric}. As we show, this modification does \emph{not} change the SELI geometry modulo a relative scaling factor between the embeddings and classifiers.

Concretely, consider the following slight generalization version of Eqn.~\eqref{eq:CE}:
\begin{align}\nn
	(\What_{(\la_W,\la_H)},\Hhat_{(\la_W,\la_H)}):=\arg\min_{\W,\Hb}~\Lc(\W^T\Hb) + \frac{\la_W}{2} \|\W\|_F^2+\frac{\la_H}{2}\|\Hb\|_F^2,
\end{align}
which can be more conveniently reparameterized as follows:
\begin{align}\nn
	(\What_{(\la_W,\la_H)},\Hhat_{(\la_W,\la_H)}):=\arg\min_{\W,\Hb}~\Lc(\W^T\Hb) + \sqrt{\la_W\la_H}\left( \frac{1}{{2\beta}}\|\W\|_F^2+\frac{\beta}{2}\|\Hb\|_F^2\right),
\end{align}
where $\beta := \sqrt{\frac{\la_H}{\la_W}}$ and $\la_H,\la_W>0$. Now, consider the limit of vanishing regularization $\la_W \rightarrow 0, \la_H \rightarrow 0$, with a fixed finite and non-zero ratio $\beta^2=\la_H/\la_W$. Entirely analogous to Eqn. \eqref{eq:svm_original}, this leads to the a $\beta$-parameterized UF-SVM as follows:
\begin{align}\label{eq:svm_diff_reg}
	(\What_\beta,\Hhat_\beta)\in\arg\min_{\W,\Hb}~\frac{1}{2\beta}\|\W\|_F^2+ \frac{\beta}{2}\|\Hb\|_F^2
	\quad\text{sub. to}\quad (\w_{y_i}-\w_c)^T\h_i \geq 1,~i\in[n], c\neq y_i.
\end{align}
Note that the UF-SVM we study in Eqn. \eqref{eq:svm_original} is a special case of the above for $\beta=1$. For general values of $\beta>0$, it is not hard to see that there is a one-to-one mapping of global solutions $(\What_\beta,\Hhat_\beta)$ of \eqref{eq:svm_diff_reg} to global solutions $(\What_\beta,\Hhat_\beta)$ of \eqref{eq:svm_original} as follows:
$$
(\What_\beta,\Hhat_\beta) = (\sqrt{\beta}\What,\frac{1}{\sqrt{\beta}}\Hhat_\beta).
$$
Hence, from Theorem \ref{thm:SVM}, it follows that the solutions of \eqref{eq:svm_diff_reg} satisfy for any $\beta>0$ the following:
$$
\What_\beta^T\Hhat_\beta=\Zhat,\quad \Hhat_\beta^T\Hhat_\beta = \frac{1}{\beta}\Ub\Lambdab\Ub^T,\quad \text{and}\quad
\What_\beta^T\What_\beta = \beta\Vb\Lambdab\Vb^T.
$$ 
Therefore, different regularization between embeddings and classifiers only affects the geometry of global minimizers of the corresponding UF-SVM (i.e., at vanishing regularization) up to introducing an extra scaling factor between the Gram matrices of embeddings and classifiers. Specifically, this only affects the relative scaling between the norms of embeddings and classifiers. 
}

\section{Nuclear-norm relaxations of the UFM}\label{sec:nuc_norm_SM}

In this section, we gather useful properties of the nuclear-norm-regularized CE minimization \eqref{eq:nuc_norm_reg}, repeated here for convenience:
$$
\Zhat_\la=\arg\min_\Z \Lc(\Z)+\la\|\Z\|_*\,.
$$
 These properties are useful in the proof of the results that appear in \Sec~\ref{Sec:reg}. As a reminder, the nuclear-norm-regularized CE minimization is relevant to us because of Theorem \ref{thm:regularized}, i.e. it forms a tight convex relaxation of the non-convex ridge-regularized CE-minimization for the UFM.

The following lemma gathers some basic properties, which we use later to study the behavior of the solutions to \eqref{eq:nuc_norm_reg} for different regularization strengths.
\begin{lemma}[Basic properties of \eqref{eq:nuc_norm_reg}]\label{lem:nuc_norm_reg_app}
The following statements are true.

\begin{enumerate}[label={(\roman*)},itemindent=1em]

\item There is a unique minimizer $\Zhat_\la$.

\item $\Zhat_\la$ satisfies the following necessary and sufficient first-order optimality conditions:
\begin{align}\label{eq:KKT_nuc_norm}
\nabla_\Z\Lc(\Zhat_\la)\Ub_\la = -\la\Vb_\la,\qquad\left(\nabla_\Z\Lc(\Zhat_\la)\right)^T\Vb_\la = -\la\Ub_\la,\qquad \|\nabla_\Z\Lc(\Zhat_\la)\|_2\leq \la.
\end{align}
where, $\Zhat_\la=\Vb_\la\Lambdab_\la\Ub_\la^T$ is its compact SVD.

\item It holds $\ones_k^T\Vb_\la=0.$ Thus also, $\ones_k^T\Zhat_\la=0.$

\end{enumerate}
\end{lemma}
\begin{proof}
We prove each statement separately below.

\vp
\noindent\underline{Proof of (ii):}~This is straightforward from first-order optimality of the convex program \eqref{eq:nuc_norm_reg}. For example, see \citep[Lemma C.3]{zhu2021geometric} for details. 

\vp
\noindent\underline{Proof of (iii):}~Start from Statement (ii) and use from Lemma \ref{lem:grad_ones} that $\ones_k^T\nabla_\Z\Lc(\Zhat_\la)=0$. Then, it must then be that $\ones_k^T\Vb_\la=\frac{1}{\la}\ones_k^T\nabla_\Z\Lc(\Zhat_\la)\Ub_\la=0$, which in turn implies  $\ones_k^T\Zhat_\la=0$. 

\vp
\noindent\underline{Proof of (i):}~Now, we prove that $\Zhat_\la$ is unique. This is a consequence of Lemma \ref{lem:strict_CE} stated below  and the fact that $\Lc(\Z)=\Lc(\Z)=\sum_{i\in[n]}\log\big(1+\sum_{c\neq y_i}e^{-(\Z[y_i,i]-\Z[c,i])}\big)=\sum_{i\in[n]}\ell_{y_i}(\z_i)$ where $\z_i$ are the columns of $\Z$. Specifically, suppose there were two minimizers $\Zhat_\la=[\z_1,\ldots,\z_n]$ and $\Zhat_\la+\Deltab$ for $\Deltab=[\deltab_1,\ldots,\deltab_n]\neq 0$. From Statement (iii), it must be that $\ones_k^T\Deltab=0$. Thus, there exists $j\in[n]$ such that $(\Id_k-\frac{1}{k}\ones_k\ones_k^T)\deltab_j\neq0$. With these, we have from Lemma \ref{lem:strict_CE} that
\begin{align*}
\Lc(\Zhat_\la+\Deltab) &= \ell_{y_j}(\z_j+\deltab_j)+ \sum_{i\neq j} \ell_{y_i}(\z_i+\deltab_i) 
\\
&> \ell_{y_j}(\z_j) + \big(\nabla_\z\ell_{y_j}(\z_j)\big)^T\deltab_j + \sum_{i\neq j} \ell_{y_i}(\z_i) + \sum_{i\neq j}\big(\nabla_\z\ell_{y_i}(\z_i)\big)^T\deltab_i 
\\
&=\Lc(\Zhat_\la) + \tr\left(\big(\nabla_\Z\Lc(\Zhat_\la)\big)^T\Deltab\right)=\Lc(\Zhat_\la),
\end{align*}
where the last equality uses optimality of $\Zhat_\la$. The above display  contradicts optimality of $\Zhat_\la+\Deltab$ and completes the proof.
\end{proof}

\begin{lemma}[Auxiliary result---Strict convexity of CE]\label{lem:strict_CE}
For some  $y\in[k]$ and $\z=[z_1,\ldots,z_k]$, let $\ell_y(\z):=\log\left(1+\sum_{c\neq y}e^{-(\z[y]-\z[c])}\right)$. The function $\ell_y$ is strictly convex along any direction on the $(k-1)$-dimensional subspace orthogonal to $\ones_k$. That is, for all \emph{non-zero} $\vb\in\R^k$ such that $(\Id_k-\frac{1}{k}\ones_k\ones_k^T)\vb\neq0$, it holds for all $\z\in\R^k$ that
$
\ell_y(\z+\vb) > \ell_y(\z) + \big(\nabla_\z\ell_y(\z)\big)^T\vb.
$
\end{lemma}
\begin{proof}
Define univariate function $g(t)=\ell_y(\z+t\vb)$. From Taylor's expansion, for some $\theta\in(0,1)$
$$
\ell_y(\z+\vb) = \ell_y(\z) + \big(\nabla_\z\ell_y(\z)\big)^T\vb + \frac{1}{2}\vb^T\nabla_\z^2\ell_y(\z+\theta\vb)\vb.
$$
Hence, it will suffice showing that  $\vb^T\nabla_\z^2\ell_y(\z)\vb>0$ for any $\z$. Denote for convenience vector $\ab\in\R^k$ with $\ab[c]:=\frac{e^{-(\z[y]-\z[c])}}{1+\sum_{c'\neq y}e^{-(\z[y]-\z[c'])}}$ for all $c\in[k]$. From Lemma \ref{lem:grad_ones}, we have 
$
\nabla_\z\ell_y(\z)=-\eb_y+\ab.
$
Thus, a straightforward calculation yields 
$
\nabla^2_\z\ell_y(\z)=\diag(\ab)-\ab\ab^T.
$ 
For the sake of contradiction assume $\vb^T\nabla_\z^2\ell_y(\z)\vb=0$. Then, since $\sum_{c\in[k]}\ab[c]=1$ and $\ab[c]\geq0$, it must be from Cauchy-Schwartz:
$$
\sum_c \vb[c]^2\ab[c]  = \Big(\sum_c \vb[c]\ab[c]\Big)^2 \leq \Big(\sum_c \vb[c]^2\ab[c]\Big)^2\Big(\sum_c \ab[c]\Big)^2=\sum_c \vb[c]^2\ab[c]
$$ 
 that
$\vb = C\ab$. But then, $\vb^T\ones_k=0$, which violates the lemma's assumption.
\end{proof}

\subsection{Large $\la$}

\begin{lemma}[Large $\la$ behavior of \eqref{eq:nuc_norm_reg}]
The minimizer is zero, i.e. $\Zhat_\la=0$, if and only if $\la\geq \|\Zhat\|_2$, where $\Zhat$ is the \SEL~matrix. 
\end{lemma}
\begin{proof}
From the KKT conditions in Lemma \ref{lem:nuc_norm_reg_app}(ii), $\Zhat_\la=0$ is optimal if and only if $\|\nabla_\Z\Lc(0)\|_2\leq \la$. 
The key observation here is that $\nabla_\Z\Lc(0)=\Y-\frac{1}{k}\ones_k\ones_n^T=\Zhat.$ (See Lemma \ref{lem:grad_ones} for a formula evaluating the gradient.) Thus, $\|\nabla_\Z\Lc(0)\|_2= \|\Zhat\|_2$, which completes the proof.
\end{proof}

We note that, for $(R,\rho)$-STEP imbalance the necessary and sufficient condition of the lemma for $\Zhat_\la=0$ becomes $\la\geq \|\Zhat\|_2=\sqrt{R}$ (see Lemma \ref{lem:Zhat_SVD_general}.)

\subsection{Small $\la$}\label{sec:small_la}

\begin{lemma}[Small $\la$ behavior of \eqref{eq:nuc_norm_reg}]\label{lem:small_la}
 If $\la<\frac{1}{2}$, then $\Zhat_\la$ is such that it linearly separates the data, i.e. it holds $\Zhat_\la[y_i,i]-\Zhat_\la[c,i]>0$ for all $i\in[n], c\neq y_i$. 
\end{lemma}
\begin{proof}
For the sake of contradiction assume $\la<1/2$ and there exists $j\in[n]$ and $c\neq y_j$ such that $\Zhat_\la[y_j,j]-\Zhat_\la[c,j]\leq 0$.  Denote $\G=\nabla_\Z\Lc(\Zhat_\la)$ the gradient. From KKT conditions in Lemma \ref{lem:nuc_norm_reg_app}(ii), it must be that 
\begin{align}\label{eq:2contra_large_la}
\|\G\|_2\leq \la < 1/2.
\end{align}
A simple calculation (e.g., see Lemma \ref{lem:grad_ones}) gives 
$\left|\G[y_j,j]\right|=1-\frac{1}{1+\sum_{c'\neq y_j}e^{-(\Zhat_\la[y_j,j]-\Zhat_\la[c',j])}}$. But, from assumption on $j,c$: 
$$1+\sum_{c'\neq y_j}e^{-(\Zhat_\la[y_j,j]-\Zhat_\la[c',j])} = 1+ e^{-(\Zhat_\la[y_j,j]-\Zhat_\la[c,j])} + \sum_{c'\not\in \{y_j,c\}}e^{-(\Zhat_\la[y_j,j]-\Zhat_\la[c',j])} \geq 2.$$
Hence, $\left|\G[y_j,j]\right| \geq 1/2$. 
Since $\|\G\|_2\geq |\G[y_j,j]|$, this in turn implies that $\|\G\|_2\geq 1/2$, which contradicts \eqref{eq:2contra_large_la}.
\end{proof}

In \Sec~\ref{sec:minority} we combine Theorem \ref{thm:regularized} with Lemma \ref{lem:small_la} above to show that the solution $(\What_\la,\Hhat_\la)$ of \eqref{eq:CE} cannot satisfy the minority collapse property.


\subsection{Vanishing $\la$}

\begin{propo}[Vanishing $\la$ behavior of \eqref{eq:nuc_norm_reg}]\label{propo:vanish_cvx}
Assume $(R,\rho)$-STEP imbalance. Then 
\begin{align}\label{eq:la_zero}
\lim_{\la\rightarrow0}\frac{\Zhat_\la}{\|\Zhat_\la\|_*}=\frac{\Zhat}{\|\Zhat\|_*}\,,
\end{align}
where $\Zhat$ is the \SEL~matrix.
\end{propo}

The proposition is an extension of \citep[Theorem 3.1]{rosset2003margin} to nuclear-norm regularization. Its proof follows the exact same steps as in \citet{rosset2003margin} who studied $\ell_p$-regularization. The critical observation allowing this is that $\Zhat$ is the unique minimizer of \eqref{eq:svm_original} thanks to Theorem \ref{thm:SVM}. Because of that, we can get the following max-margin formulation for (the normalized) $\Zhat$, which is key in the proof of Proposition \ref{propo:vanish_cvx}.

\begin{lemma}\label{lem:max-margin}
Assume $(R,\rho)$-STEP imbalance so that $\Zhat$ is the unique solution of \eqref{eq:svm_original} (see Theorem \ref{thm:SVM}). Then, it holds that
\begin{align}\label{eq:max-margin}
\frac{\Zhat}{\|\Zhat\|_*} := \arg\max_{\|\Z\|_*\leq 1}~\min_{i\in[n],c\neq y_i}~\Z[y_i,i]-\Z[c,i].
\end{align}
\end{lemma}

We present the proof of Proposition \ref{propo:vanish_cvx} together with the proof of Lemma \ref{lem:max-margin} in the next section.

\subsubsection{Proof of Proposition \ref{propo:vanish_cvx}}

First, by KKT conditions (see Lemma \ref{lem:nuc_norm_reg_app}(ii)), as $\la\rightarrow0$, we have $\|\nabla\Lc(\Zhat_\la)\|_2\rightarrow 0$. This, in turn, implies $\|\nabla\Lc(\Zhat_\la)\|_*\rightarrow 0$ since $\|\nabla\Lc(\Zhat_\la)\|_*\leq (k-1)\|\nabla\Lc(\Zhat_\la)\|_2$. Moreover, it implies that $$m_\la:=\min_{i\in[n],c\neq y_i}~\Zhat_\la[y_i,i]-\Zhat_\la[c,i]$$ diverges. To see this, denote $\G_\la:=\nabla\Lc(\Zhat_\la)$ and note from Lemma \ref{lem:grad_ones} that for all $i\in[n]$ $\G_\la[y_i,i]:=-\big(1+1/\sum_{c\neq y_i}e^{-(\Zhat_\la[y_i,i]-\Zhat_\la[c,i])}\big)^{-1}$. But, 
\begin{align*}
&\|\nabla\Lc(\Zhat_\la)\|_2 \geq \max_{i\in[n]}|\G_\la[y_i,i]| = \max_{i\in[n]} \big(1+1/\sum_{c\neq y_i}e^{-(\Zhat_\la[y_i,i]-\Zhat_\la[c,i])}\big)^{-1} \geq \big(1+e^{m_\la}\big)^{-1} \\
&\qquad\implies m_\la \geq \log\big(\|\nabla\Lc(\Zhat_\la)\|_2^{-1} -1\big).
\end{align*}
Thus, $\|\nabla\Lc(\Zhat_\la)\|_2\rightarrow 0\implies m_\la \rightarrow+\infty.$ 

Assume for the sake of contradiction that $\lim_{\la\rightarrow0}\frac{\Zhat_\la}{\|\Zhat_\la\|_*}=\wt\Z$ for $\wt\Z\neq \Zhat/\|\Zhat\|_*.$ Then, by Lemma \ref{lem:max-margin}, it must be that 
$$
m =\min_{i\in[n],c\neq y_i}~\wt\Z[y_i,i]-\wt\Z[c,i]<\frac{1}{\|\Zhat\|_*}\,\min_{i\in[n],c\neq y_i}~\Zhat[y_i,i]-\Zhat[c,i]=:m_*.
$$
Moreover, since $m_\la\rightarrow+\infty,$ we also have that $m>0$. The rest of the argument follows mutatis-mutandis the proof of \citep[Theorem 2.1]{rosset2003margin}. We repeat here for completeness. By continuity of the minimum margin in $\Z$, there exists open neighborhood of $\wt\Z$ on the nuclear-norm sphere:
$$
\Nc_{\wt\Z}:=\{\Z\,:\,\|\Z\|_*=1, \|\Z-\wt\Z\|_2\leq \delta\}
$$
and an $\eps>0$ such that $\min_{i\in[n],c\neq y_i}~\Z[y_i,i]-\Z[c,i]<m_*-\eps$ for all $\Z\in\Nc_{\wt\Z}.$  To continue, we use the following lemma

\begin{lemma}\label{lem:compare_margins}
Assume $\Z_1,\Z_2$ such that $\|\Z_1\|_*=\|\Z_2\|_*=1$ and 
$$
0< m_2 =\min_{i\in[n],c\neq y_i}~\Z_2[y_i,i]-\Z_2[c,i]<\min_{i\in[n],c\neq y_i}~\Z_1[y_i,i]-\Z_1[c,i]=m_1.
$$
Then, there exists $T:=T(m_1,m_2)$ such that 
$$
\forall t>T\,:\, \Lc(t\Z_1) < \Lc(t\Z_2).
$$
\end{lemma}

By Lemma \ref{lem:compare_margins}, there exists $T:=T(m_*,m_*-\eps)$ such that for all $t>T$ and all $\Z\in\Nc_{\wt\Z}$,  $\Lc(t\Zhat/\|\Zhat\|_*)<\Lc(t\Z)$. Therefore, $\wt\Z$ cannot be a convergence point of $\Zhat_\la/\|\Zhat_\la\|_*.$


We finish the proof by showing how to get Lemmas \ref{lem:max-margin} and \ref{lem:compare_margins}.

\paragraph{Proof of Lemma \ref{lem:max-margin}.}
Clearly, $\frac{\Zhat}{\|\Zhat\|_*}$ is feasible in \eqref{eq:max-margin}. Thus, $\max_{\|\Z\|_*\leq 1}~\min_{i\in[n],c\neq y_i}~\Z[y_i,i]-\Z[c,i]\geq 1/\|\Zhat\|_*.$
Now, suppose there exists $\widetilde\Z\neq \frac{\Zhat}{\|\Zhat\|_*}$ such that $\|\wt\Z\|_*\leq 1$ and is a maximizer in the max-min problem in \eqref{eq:max-margin}. Then, $\min_{i\in[n],c\neq y_i}~\wt\Z[y_i,i]-\wt\Z[c,i]\geq 1/\|\Zhat\|_*.$ This means, $\|\Zhat\|_*\wt\Z$ is feasible in \eqref{eq:svm_original}. But then it must be optimal therein since $\left\|\|\Zhat\|_*\wt\Z\right\|_*=\|\Zhat\|_*\|\wt\Z\|_*\leq \|\Zhat\|_*$. From Theorem \ref{thm:SVM}, $\Zhat$ is the unique minimizer of \eqref{eq:max-margin}. Thus, we have shown that $\|\Zhat\|_*\wt\Z=\Zhat$, which contradicts the assumption on $\wt\Z$ and completes the proof.

\paragraph{Proof of Lemma \ref{lem:compare_margins}.}
Define $\eps:=m_1/m_2-1>0$ and let $T$ large enough such that $e^{-T\eps m_2}n(k-1)< 1/2$. We then have the following chain of inequalities for $t>T$
\begin{align*}
\Lc(t\Z_1) &= \sum_{i\in[n]}\log\left(1+\sum_{c\neq y_i}e^{-t(\Z_1[y_i,i]-\Z_1[c,i])}\right)\leq n\log(1+(k-1)e^{-t\,m_1}) \nn
\\
&\leq n(k-1)e^{-t\,m_1} =n(k-1)e^{-t\frac{m_1}{m_2}\,m_2}=n(k-1)e^{-t\,\eps\,m_2}e^{-t\,m_2} 
\\
&< n(k-1)e^{-T\,\eps\,m_2}e^{-t\,m_2} 
\\
&< \frac{e^{-t\,m_2}}{2}\leq \frac{e^{-t\,m_2}}{1+e^{-t\,m_2}}
\\
&\leq \log(1+e^{-t\,m_2}) 
\\
&\leq \Lc(t\Z_2).
\end{align*}
The  inequality in the third line used that $t>T$ and $\eps>0,m_2>0$. The next inequality follows by our choice of $T$. Throughout, we also used both sides of the inequality $\frac{x}{1+x}\leq\log(1+x)\leq x, x\geq 0.$


\section{Proofs for Regularized CE Section \ref{Sec:reg}}\label{sec:proofs_regularization}
\subsection{Proof of Theorem \ref{thm:regularized}}\label{sec:proof_regularized_thm}

As mentioned in Remark \ref{rem:nuc_norm_ridge} the theorem is drawn from \citep[Theorem 3.2]{zhu2021geometric} with the following three small adjustments. Since the main proof argument remains essentially unaltered, we refer the reader to \citep[Sec.~C]{zhu2021geometric} for detailed derivations. Instead here, we only overview the necessary adjustments. 

First, Theorem \ref{thm:regularized} holds for imbalanced classes. Technically, \citet{zhu2021geometric} only consider balanced data. However, a close inspection of their proof shows that such a restriction is not necessary. 

Second, Theorem \ref{thm:regularized} further shows that the nuclear-norm CE minimization in \eqref{eq:nuc_norm_reg} has a unique solution. We prove this in Lemma \ref{lem:nuc_norm_reg_app}(i) in \Sec~\ref{sec:nuc_norm_SM}.

Finally, Theorem \ref{thm:regularized} relaxes an assumption $d> k$ in \citep[Theorem 3.2]{zhu2021geometric} to $d> k-1$. (In fact, \citet{zhu2021geometric} conjecture that this relaxation is possible. We close the gap.) The assumption $d> k$ is only used by \citet{zhu2021geometric} to show there exists nonzero $\ab\in\R^d$ such that $\What_\la^T\ab=0$ for a stationary point $(\What_\la,\Hhat_\la)$ of \eqref{eq:CE}. (This step is necessary to construct a negative curvature direction at stationary points for which $\|\nabla_\Z\Lc(\What_\la^T\Hhat_\la)\|_2>\la$; see \citep[Sec.~C.1]{zhu2021geometric}.) Indeed, if $d>k$, then existence of $\ab$ is guaranteed because $\What_\la$ has $k$ columns implying  $\rank(\What_\la)\leq k$. To relax this requirement to $d>k-1$, we show in Lemma \ref{lem:W_ones} below that $\What_\la\ones_k=0$. Hence, $\rank(\What_\la)\leq k-1$.

\begin{lemma}\label{lem:W_ones}
Let $\What_\la,\Hhat_\la$ be a stationary point of \eqref{eq:CE}. Then, $\What_\la\ones_k=0.$ 
Similarly,  any global minimizer $(\hat\W,\hat\Hb)$ of \eqref{eq:svm_original} is such that $\What\ones_k=0.$
\end{lemma}
\begin{proof}If $(\What_\la,\Hhat_\la)$ is a stationary point of \eqref{eq:CE} and we denote $\Zhat_\la=\What_\la^T\Hhat_\la$, then by stationarity condition and Lemma \ref{lem:grad_ones}:
\begin{align}
 \nabla_\W \Lc(\What_\la^T\Hhat_\la) = -\la \What_\la \implies \Hhat_\la\left(\nabla_\Z \Lc(\Zhat_\la)\right)^T = -\la \What_\la \stackrel{(\text{Lem.}~\ref{lem:grad_ones})}{\Longrightarrow} \What_\la\ones=0.
\end{align}

Next, consider any global minimizer of \eqref{eq:svm_original}. For the sake of contradiction assume that $\bar\w:=\frac{1}{k}\sum_{c\in[k]}\hat\w_c\neq \zeros$, that is $\|\bar\w\|_2>0$. Consider the pair $({\Vb},\hat\Hb)$ where the columns of $\Vb$ are defined each such that $\vb_c=\hat\w_c-\bar\w$. Clearly, the new pair is feasible, since $(\hat\W,\hat\Hb)$ is feasible. But, 
$$
\sum_{c\in[k]}\|\vb_c\|^2 = \sum_{c\in[k]}\|\w_c\|^2 + k\|\bar\w\|^2 - \inp{\bar\w}{\sum_{c\in[k]}\w_c} =  \sum_{c\in[k]}\|\w_c\|^2 - k\|\bar\w\|^2 <\sum_{c\in[k]}\|\w_c\|^2,
$$
which contradicts optimality of $(\hat\W,\hat\Hb)$. Hence, it must be that $\bar\w=\zeros$.
\end{proof}

\subsection{Proof of Proposition \ref{propo:imbalance_reg}}
We prove the  proposition right below its statement in \Sec~\ref{sec:reg_matters}. The only remaining thing to show is that for $R>1$ and $k>2$, \emph{not} all eigenvalues of the \SEL~matrix are same. This follows immediately from Lemma \ref{lem:Zhat_SVD_general} in \Sec~\ref{sec:eigen_SEL}. Specifically, we show in \eqref{eq:Lambda_general}, that if $R>1$ and $k>2$, then the maximum eigenvalue of $\Lambdab$ is $\sqrt{R}$ and the minimum one is $1$ (thus, different).

\subsection{Proof of Proposition \ref{propo:reg_path}}\label{sec:proof_reg_path_noncvx}
The proposition follows by combining Propisition \ref{propo:vanish_cvx} and Theorem \ref{thm:regularized}. First, thanks to Theorem \ref{thm:regularized}, for all $\la>0$:
$$
\frac{\What_\la^T\Hhat_\la}{{\|\What_\la\|_F^2/2+\|\Hhat_\la\|_F^2/2}}= \frac{\Zhat_\la}{\|\Zhat_\la\|_*}, 
$$
where $\Zhat_\la$ is the solution to \eqref{eq:nuc_norm_reg}. Here, we used the fact from Equation \eqref{eq:geometry_reg} that 
$$
{\|\What_\la\|_F^2}+{\|\Hhat_\la\|_F^2} = {\tr(\What_\la^T\What_\la)}+{\tr(\Hhat_\la^T\Hhat_\la)} ={\tr(\Vb_\la\Lambdab_\la\Vb_\la^T)}+{\tr(\Ub_\la\Lambdab_\la\Ub_\la^T)} = 2\tr(\Lambdab_\la)=2\|\Zhat_\la\|_*.
$$
Next, from Proposition \ref{propo:vanish_cvx}
$$
\lim_{\la\rightarrow0}\frac{\Zhat_\la}{\|\Zhat_\la\|_*}=\frac{\Zhat}{\|\Zhat\|_*}.
$$
The desired follows by combining the above two displays.

\subsection{Comparison to one-layer linear model}\label{sec:2linear}
In \Sec~\ref{sec:linear_ridge}, we compared the solution to the non-convex minimization \eqref{eq:CE}, corresponding to a two-layer linear model, to the solution found by an one-layer linear model. Specifically, the one-layer linear model trains $k$-class linear classifier $\Xibf_{k\times n}$ by solving the following convex ridge-regularized CE:
\begin{align}\label{eq:ridge_reg_CE}
\min_{\Xibf}~\Lc(\Xibf)+\frac{\la}{2}\|\Xibf\|_F^2.
\end{align}
The following lemma computes the solution of this minimization.
\begin{lemma}[Linear model: Ridge-regularized CE]\label{lem:ridge_reg_CE}
Let $\hat{\Xibf}_\la$ denote the solution of the convex ridge-regularized CE minimization. Irrespective of imbalances, for all $\la>0$, there exists $\alpha_\la>0$ such that $\hat{\Xibf}_\la=\alpha_\la\Zhat$. 
 Moreover, in the limit of $\la\rightarrow0$, $\lim_{\la\rightarrow 0}\hat{\Xibf}_\la/\|\hat{\Xibf}_\la\|_F=\hat{\Xibf}_0/\|\hat{\Xibf}_0\|_F,$
 where 
 \begin{align}\label{eq:cvx_svm}
 \hat{\Xibf}_0:=\arg\min_{\Xibf}~ \|\Xibf\|_2\qquad\text{sub. to}~~\Xibf[y_i,i]-\Xibf[c,i]\geq 1,~\forall i\in[n],c\neq y_i.
 \end{align}
 In fact, the solution to \eqref{eq:cvx_svm} is the \SEL~matrix, i.e.,  $\hat{\Xibf}_0=\Zhat.$
\end{lemma}

According to the lemma above, the one-layer linear model always finds the \SEL~matrix: irrespective of imbalances and for any value of $\la$ (including vanishing ones). On the other hand, by Proposition \ref{propo:imbalance_reg}, we know that the end-to-end models minimizing ridge-regularized CE for a two-layer linear network correspond to the \SEL~matrix only if data are balanced, or there is two classes, or regularization is vanishing. Specifically, the solution is different when $k>2$, $R=1$ and $\la>0$. 

It is also worth noting the following connection between the one- and two-layer models. Thanks to the convex relaxation of Theorem \ref{thm:regularized} the end-to-end model $\W^T\Hb$ found by the two-layer model solves CE minimization with nuclear norm minimization (cf. \eqref{eq:nuc_norm_reg}) compared to the ridge regularization in \eqref{eq:ridge_reg_CE}. Correspondingly, for vanishing regularization, the two-layer model corresponds to the ``nuclear-norm SVM'' in \eqref{eq:nuc_norm_SVM} compared to the vanilla SVM in \eqref{eq:cvx_svm}. Notably, Theorem \ref{thm:SVM} proves that the solution to the former is always $\Zhat$, i.e., the same as that of the latter. 

\begin{proof}[Proof of Lemma \ref{lem:ridge_reg_CE}]
The minimization in \eqref{eq:ridge_reg_CE} is convex. Hence, it suffices to prove there exists $\alpha_\la$ such that that $\nabla_{\Z}\Lc(\alpha_\la\Zhat)+\la\alpha_\la\Zhat=0$. Thanks to Lemma \ref{lem:Zhat}\ref{state:SEL_prop_CE}, it suffices that $\exists\alpha_\la$ such that
$
\la\alpha_\la=\frac{k}{e^{\alpha_\la}+k-1}.
$
It is easy to check that this equation always has a positive solution $\alpha_\la>0$ since the LHS is increasing in $(0,\infty)$, the RHS is decreasing in the same interval, and they both take values in $(1,\infty)$. The fact that following the regularization path $\la\rightarrow0$ leads (in direction) to the SVM solution follows from \cite{rosset2003margin}. These two combined also show that the solution to \eqref{eq:cvx_svm} is $\Zhat.$
\end{proof}


\section{Additional results on UFM}\label{sec:UFM}
\begin{figure*}[t]
	\centering
	\hspace{-1.5cm}
	\begin{subfigure}{0.22\textwidth}
		\centering
		\begin{tikzpicture}
			\node at (-2,-1.4) 
			{\includegraphics[scale=0.26]{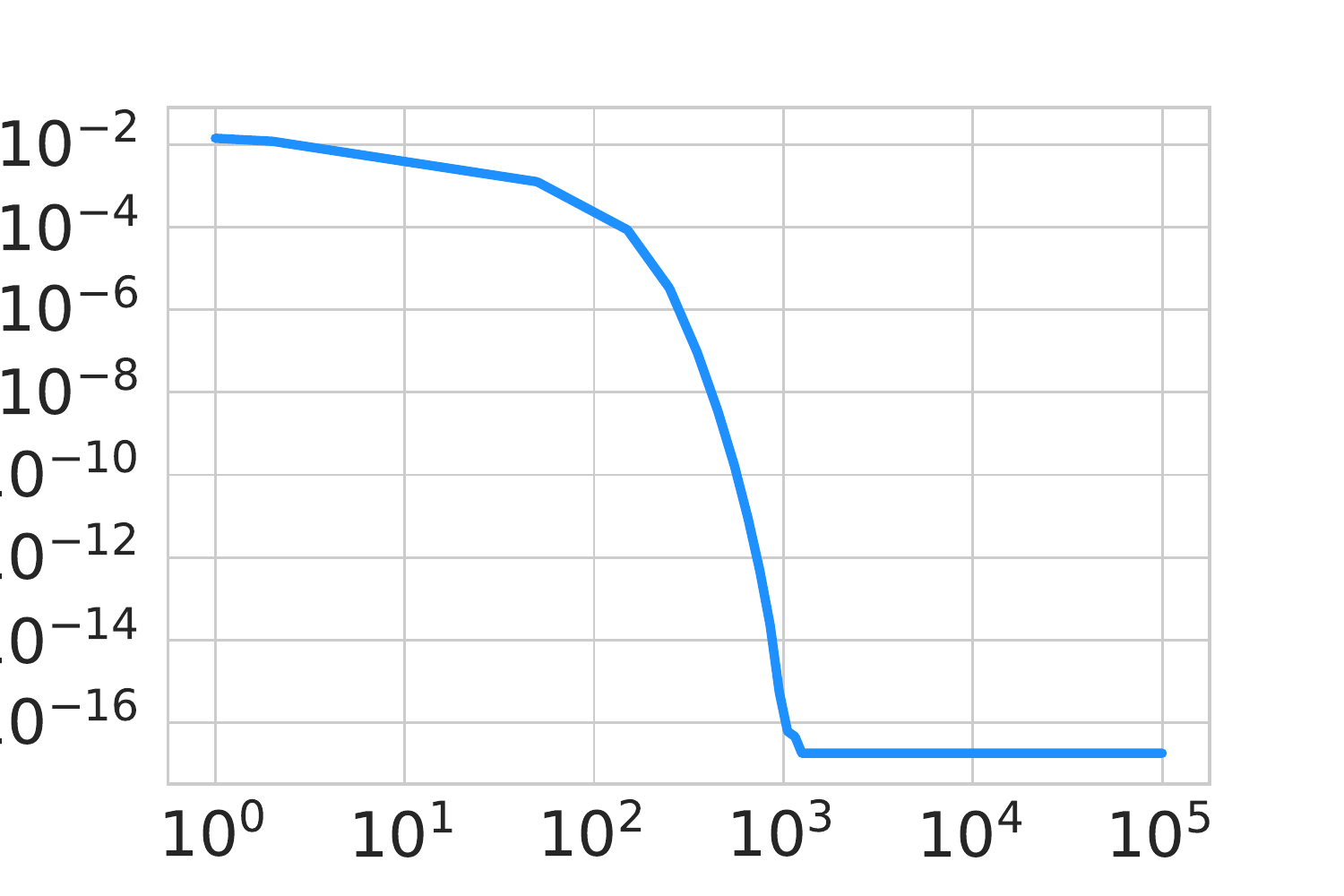}};
			\node at (-2,-3) [scale=0.7]{Epochs};
			\node at (-4.2,-1.4)  [scale=0.7, rotate=90]{NC error};
		\end{tikzpicture}\caption{\small{NC}}\label{fig:UFM_ridge_NC}
	\end{subfigure}\hspace{20pt}\begin{subfigure}{0.22\textwidth}
		\centering
		\begin{tikzpicture}
			\node at (-1.4,-1.4) 
			{\includegraphics[scale=0.26]{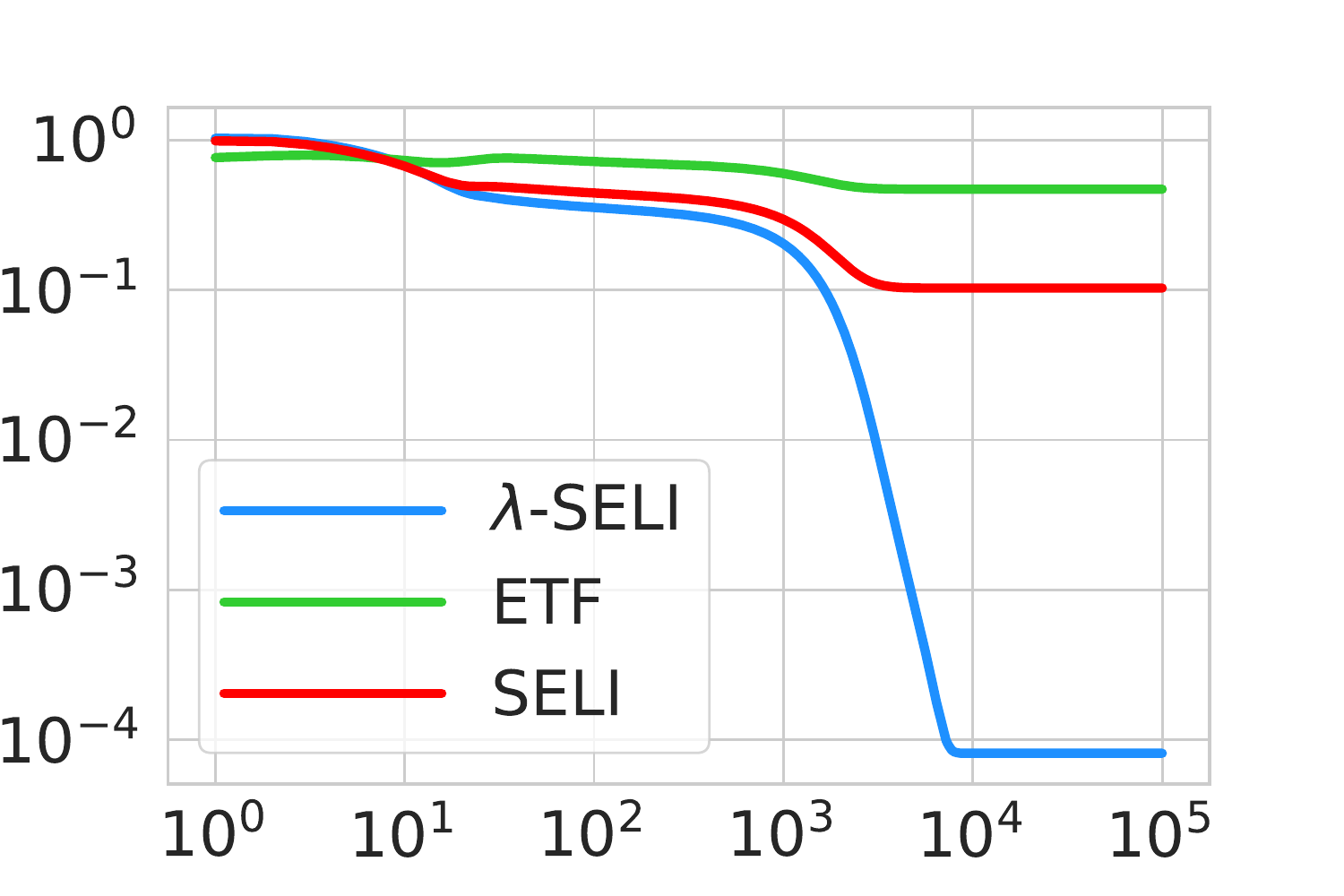}};
			\node at (-1.4,-3) [scale=0.7]{Epochs};
			\node at (-3.7,-1.4)  [scale=0.7, rotate=90]{Distance to SELI/ETF};
		\end{tikzpicture}\captionsetup{width=0.9\linewidth}\caption{\small{Classifiers}}\label{fig:UFM_ridge_W_hat}
	\end{subfigure}\hspace{1pt}\begin{subfigure}{0.22\textwidth}
		\centering
		\begin{tikzpicture}
			\node at (0,-1.4) {\includegraphics[scale=0.26]{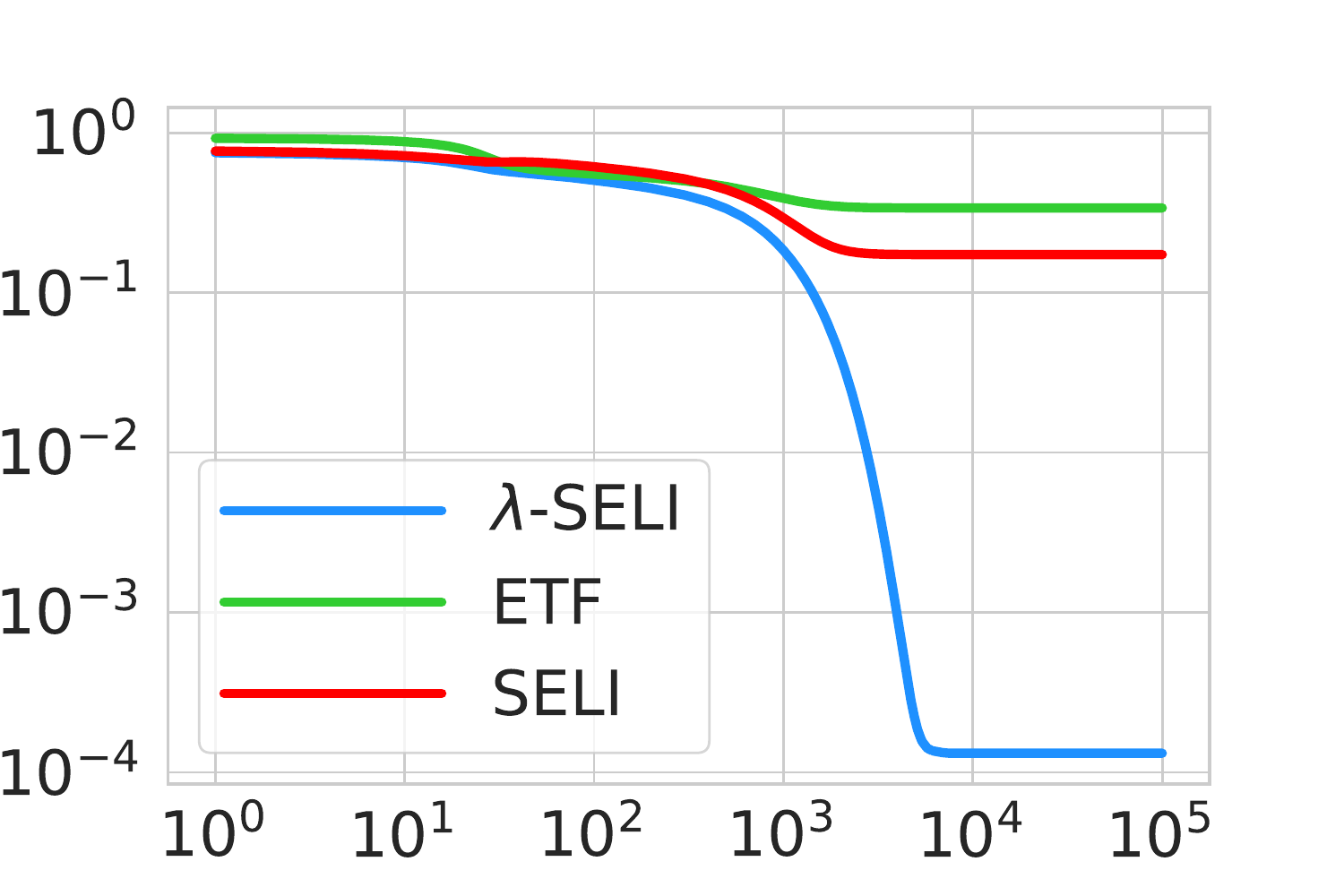}};
			\node at (0,-3) [scale=0.7]{Epochs};
			\node at (-2.7,-1.4) [scale=0.7, rotate=90]{};
		\end{tikzpicture}\captionsetup{width=0.9\linewidth}\caption{\small{Embeddings}}\label{fig:UFM_ridge_H_hat}
	\end{subfigure}\hspace{10pt}\begin{subfigure}{0.22\textwidth}
		\centering
		\begin{tikzpicture}
			\node at (0,-1.4) {\includegraphics[scale=0.26]{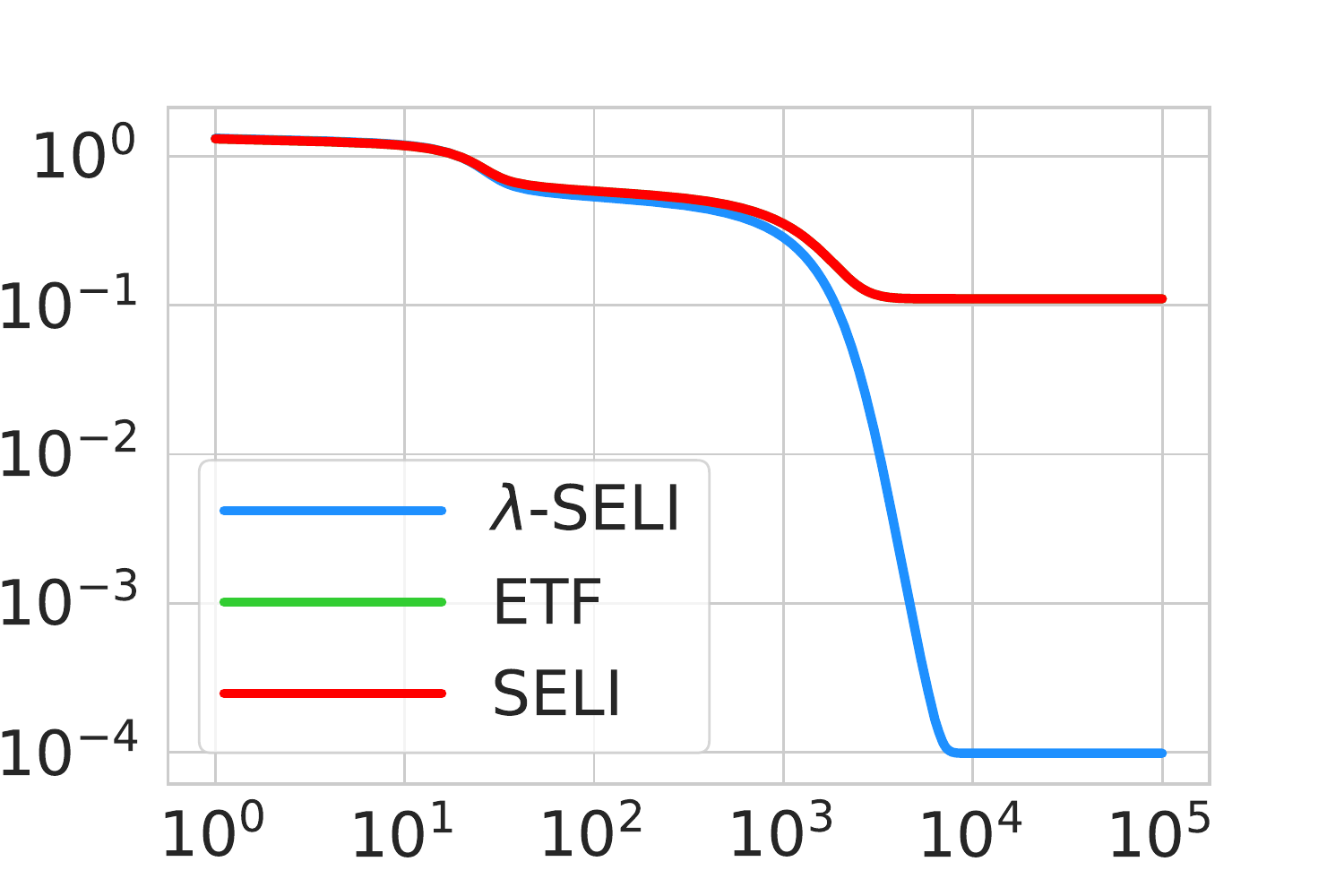}};
			\node at (0,-3) [scale=0.7]{Epochs};
			\node at (-2.7,-1.4) [scale=0.7, rotate=90]{};
		\end{tikzpicture}\captionsetup{width=0.9\linewidth}\caption{\small{Logits}}\label{fig:UFM_ridge_Z_hat}
	\end{subfigure}\caption{ SGD solution with ridge-regularized CE converging to $(\Zhat_\la,\What_\la,\Hhat_\la)$ for $R=10$ and $\lambda/n=10^{-3}$. The iterates clearly converge towards $\lambda$-SELI (see Eqn. \eqref{eq:geometry_reg} in Theorem \ref{thm:regularized}) and are far away from ETF. The distance to SELI also stays large as training progresses. However, note that it tracks the curves measuring distance to $\la$-SELI until $\sim5e3$ epochs.
	}
	\label{fig:UFM_ridge}
\end{figure*}
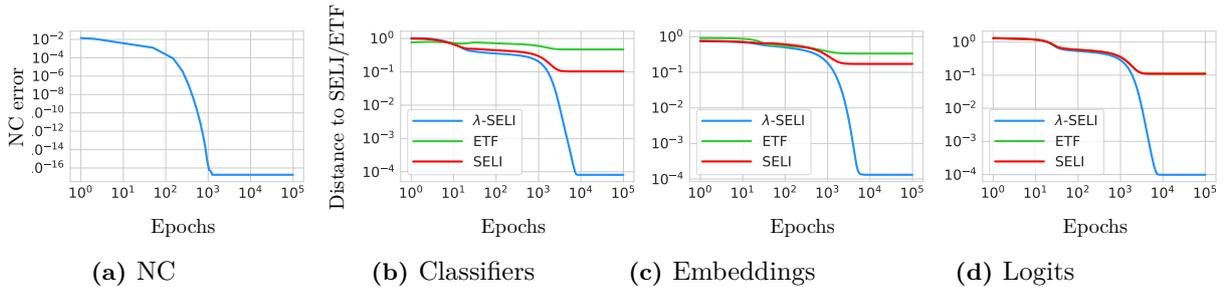

\subsection{Experiments with weight-decay}
In this section, we show experiments on the UFM supporting the claim of Theorem \ref{thm:regularized}: the global solution $(\W_\lambda,\Hb_\lambda)$ of ridge-regularized CE in \eqref{eq:CE}, call it ``$\lambda$-SELI'' for convenience (with some abuse of the term SELI, as the logit matrix does not represent the simplex encoding anymore for $\lambda > 0$), satisfies \eqref{eq:geometry_reg}. Moreover any first-order method that avoids strict saddles converges to that global optimum \cite{zhu2021geometric}.

\Fig~\ref{fig:UFM_ridge} investigates the above claims in a setting of $(R=10,\rho=\nicefrac{1}{2})$-STEP imbalance, $k=4$ classes and $\nmin=1$, where we ran SGD on the ridge-regularized CE with $\lambda/n=10^{-3}$. We set learning rate to $1$ and implement ridge-regularization as weight-decay on the parameters. 
We observe the following. \Fig~\ref{fig:UFM_ridge_W_hat}, \ref{fig:UFM_ridge_H_hat}, and \ref{fig:UFM_ridge_Z_hat} verify convergence to $\lambda$-SELI, while \Fig~\ref{fig:UFM_ridge_NC} verifies that the \emph{\ref{NC}} property also holds. Also, the solution is clearly away from ETF geometry (see green lines in \Fig~\ref{fig:UFM_ridge}). This is a noteworthy difference of the behavior of learning with imbalanced data, compared to that with balanced data. With balanced data, the geometry with ridge-regularization $\lambda>0$ is always ETF. On the contrary, the geometry for learning from imbalanced data is sensitive to $\lambda$, as discussed in \Sec~\ref{sec:reg_matters}.  While the distance from $\lambda$-SELI is the least, the distance to SELI is smaller compared to ETF. Thus, while SELI is not the ``true'' characterization when training with finite ridge-regularization, it is nevertheless a significantly better approximation than the ETF. In addition, compared to the $\la$-SELI solution, the SELI one admits explicit closed-form expressions (see \Sec~\ref{sec:SELI_properties}), rather than requiring numerical solution to a nuclear-norm CE minimization. Finally, note that, up to a certain point of time during the training, the distances from SELI and $\lambda$-SELI are comparable. Interestingly, the divergence between the two distances becomes more prominent at epoch count that also corresponds to a sharp fall of the NC error-curve in \Fig~\ref{fig:UFM_ridge_NC}.
 
 \begin{remark}[A note on $\lambda$ scaling] In Equation \eqref{eq:CE}, the CE loss is \emph{not} normalized by the number of examples. On the other hand, in all our experiments, we normalize the CE loss by $1/n$. 
 This is why, when denoting the regularization used in experiments, we write $\nicefrac{\lambda}{n}$ in the axis-labels of all figures.
 \end{remark}
\subsection{Logit regularization and Ridge-decay}\label{sec:logit_reg}
Our Proposition \ref{propo:reg_path} proves that the optimal logits of 
CE minimization for the UFM with vanishing ridge-regularization is the SEL matrix. However, we also find in \Fig~\ref{fig:UFM_convergence} that the convergence of SGD to the SEL matrix is rather slow. Motivated by this slow convergence we discuss here \emph{logit-regularized} CE minimization, i.e., the solution to the following (non-convex) program:
\begin{align}\label{eq:logreg_CE}
\min_{\W,\Hb}~\Lc(\W^T\Hb) + \frac{\la}{2} \|\W\|_F^2+\frac{\la}{2}\|\Hb\|_F^2+\la_L\|\W^T\Hb\|_F^2. 
\end{align}
Note that when $\la=0$, the global solutions of the above minimization give a logit matrix that aligns with the SEL matrix $\Zhat$. This is easy to see by noting the resemblance to the convex program in \eqref{eq:ridge_reg_CE} and invoking Lemma \ref{lem:ridge_reg_CE}.

Empirically, we observe that the above logit regularization helps achieve SGD convergence to solutions with SEL matrix as logits converge much faster, even without additional ridge-regularization. Specifically, for  $(R=10,\rho=\nicefrac{1}{2})$-STEP imbalance, $k=4$ classes and $\nmin=1$, we run SGD on the logit-regularized CE with $\lambda_L/n=10^{-3}$ and with zero ridge-regularization ($\lambda=0$). We also set the learning rate to $1$. \Fig~\ref{fig:UFM_logreg_Z_only} depicts convergence of the logit-matrix $\Z$ in the direction of the SEL matrix $\Zhat$. However, we find that this does \emph{not} ensure of convergence for the individual geometries of $\W$ and $\Hb$ towards SELI, although their inner product $\W^T\Hb$ aligns well to the SEL matrix. 
\begin{figure*}[t]
	\centering
		\centering
		\begin{tikzpicture}
			\node at (-3,-2) {\includegraphics[scale=0.3]{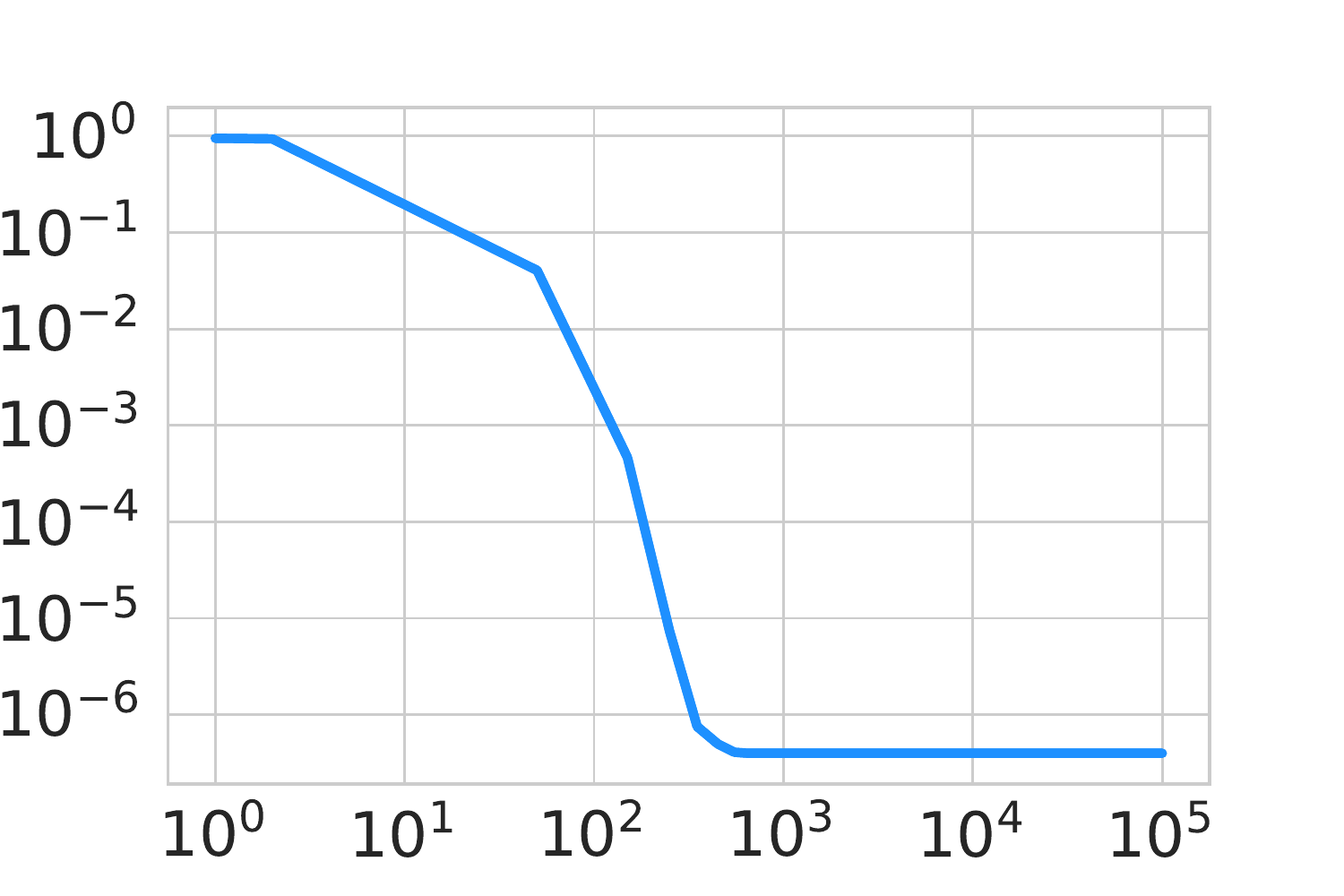}};
			\vspace{0.5cm}\node at (-2.7,-3.6) [scale=0.9]{epochs};
			\node at (-5.5,-1.9)  [scale=0.8, rotate=90]{Distance to SEL};
		\end{tikzpicture}\caption{\small{GD iterates on logit-regularized CE without ridge-regularization favor yielding logit matrix in the direction of SEL matrix $\Zhat$. }
	}\label{fig:UFM_logreg_Z_only}
\end{figure*}

\begin{figure*}[t]
	\centering
	\hspace{-1.5cm}
	\begin{subfigure}{0.22\textwidth}
		\centering
		\vspace{0.5cm}
		\begin{tikzpicture}
			\node at (-2,-1.8) 
			{\includegraphics[scale=0.25]{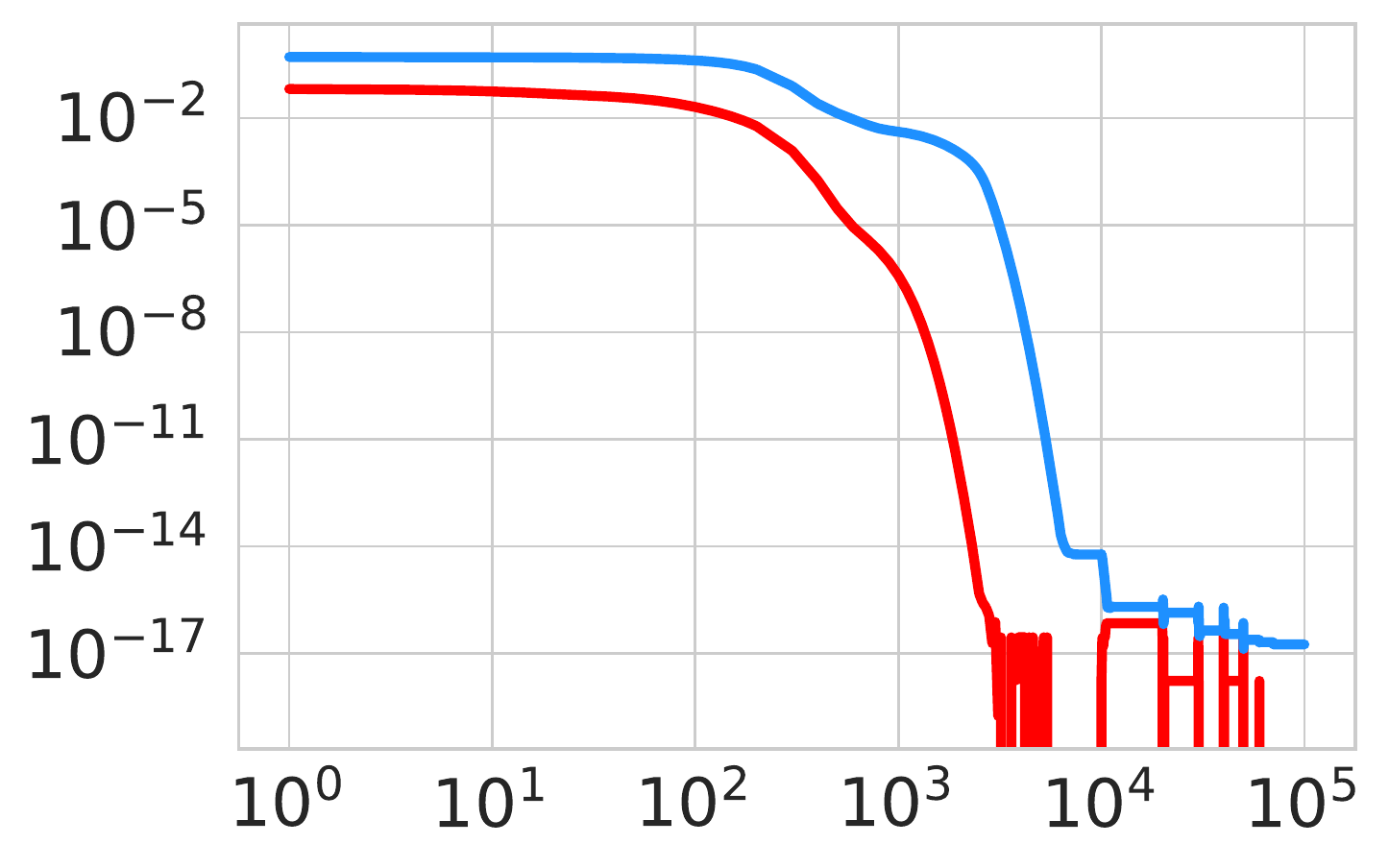}};
			\node at (-2,-3.4) [scale=0.7]{Epochs};
			\node at (-4.2,-1.8)  [scale=0.7, rotate=90]{NC error};
		\end{tikzpicture}\caption{\small{NC}}\label{fig:UFM_logreg_NC}
	\end{subfigure}\hspace{20pt}\begin{subfigure}{0.22\textwidth}
		\centering
		\begin{tikzpicture}
			\node at (-1.4,-1.4) 
			{\includegraphics[scale=0.26]{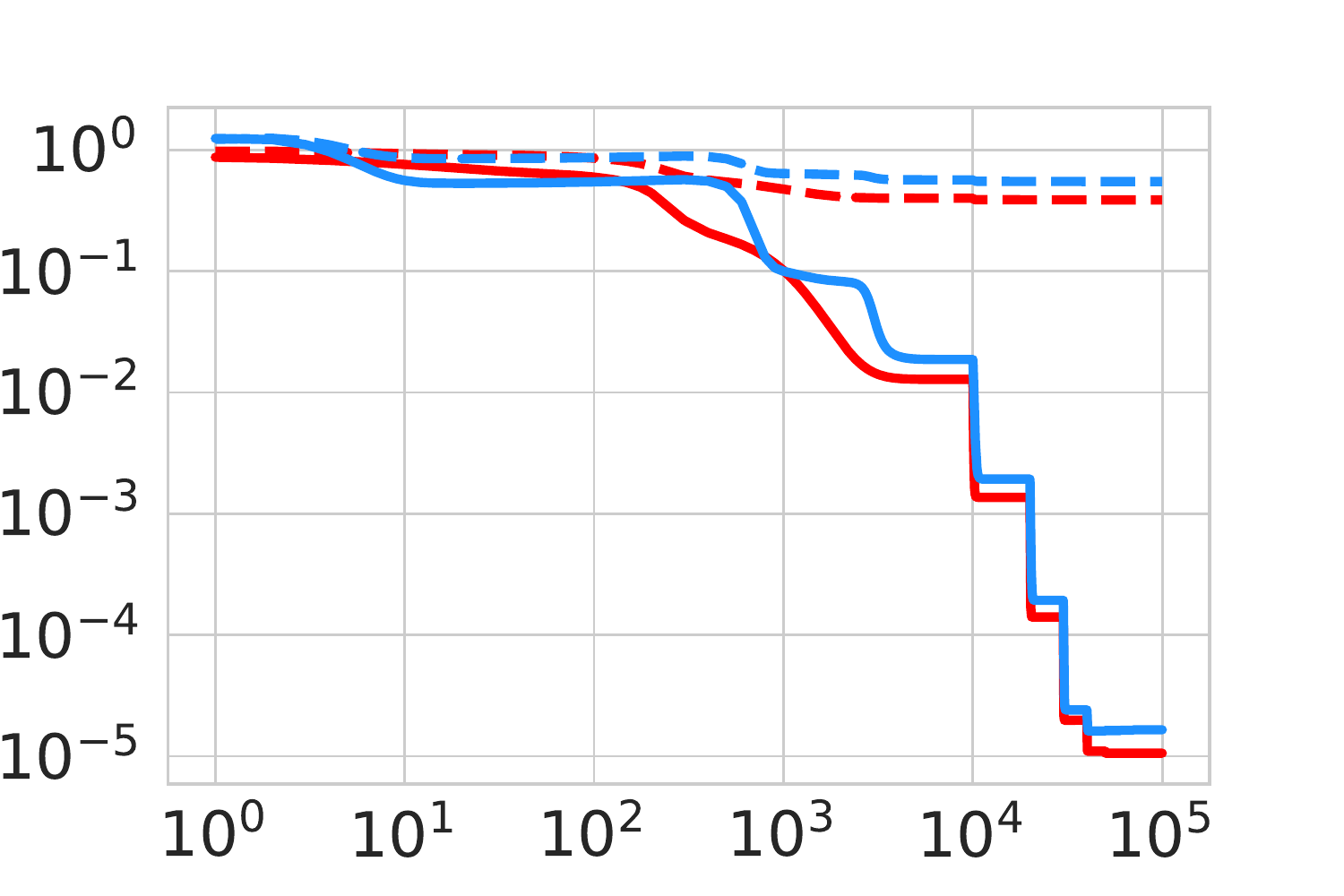}};
			\node at (-1.4,-3) [scale=0.7]{Epochs};
			\node at (-3.7,-1.4)  [scale=0.7, rotate=90]{Distance to SELI/ETF};
		\end{tikzpicture}\captionsetup{width=0.9\linewidth}\caption{\small{Classifiers}}\label{fig:UFM_logreg_W_hat}
	\end{subfigure}\hspace{1pt}\begin{subfigure}{0.22\textwidth}
		\centering
		\begin{tikzpicture}
			\node at (0,-1.4) {\includegraphics[scale=0.26]{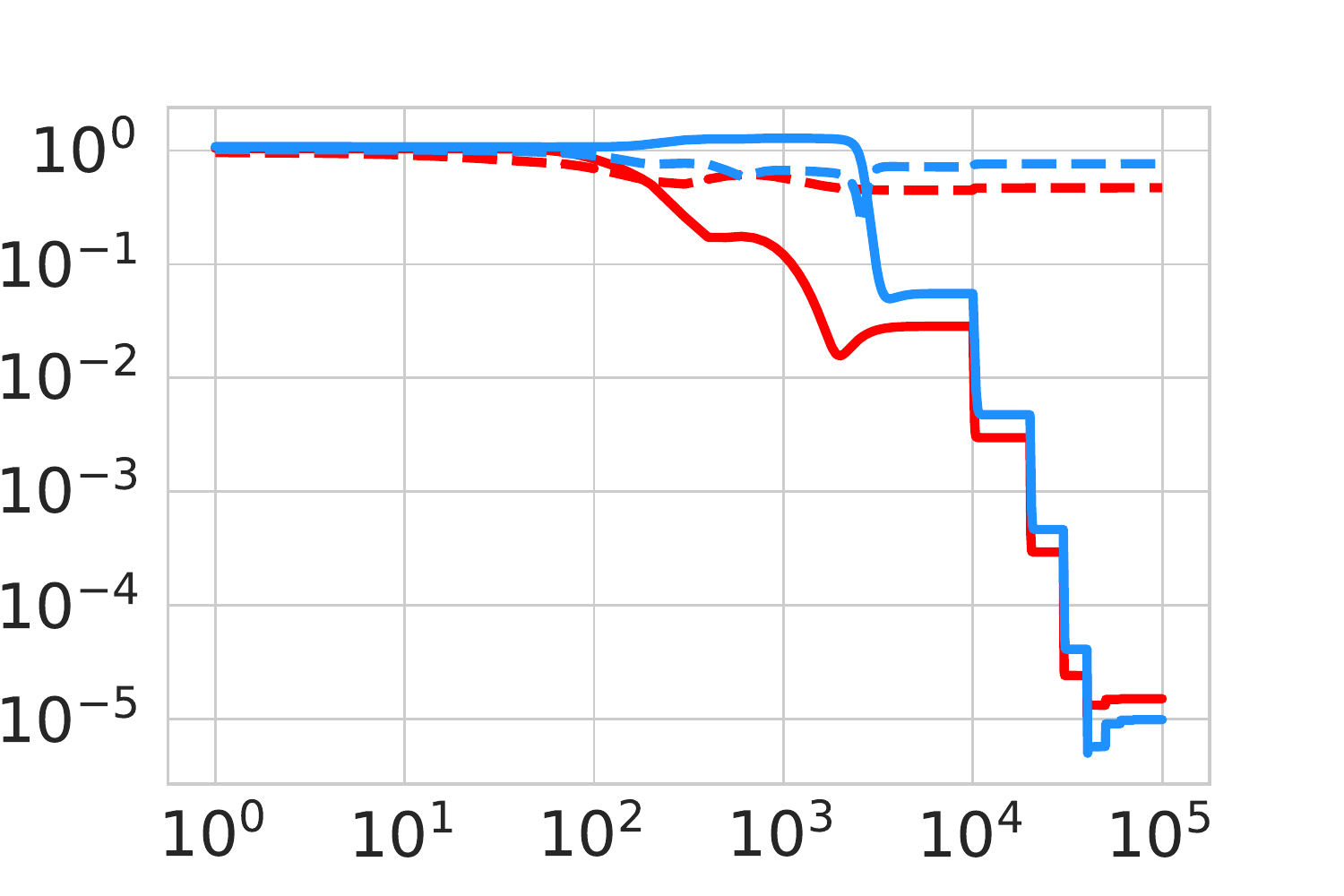}};
			\node at (0,-3) [scale=0.7]{Epochs};
			\node at (-2.7,-1.4) [scale=0.7, rotate=90]{};
		\end{tikzpicture}\captionsetup{width=0.9\linewidth}\caption{\small{Embeddings}}\label{fig:UFM_logreg_H_hat}
	\end{subfigure}\hspace{10pt}\begin{subfigure}{0.22\textwidth}
		\centering
		\begin{tikzpicture}
			\node at (0,-1.4) {\includegraphics[scale=0.26]{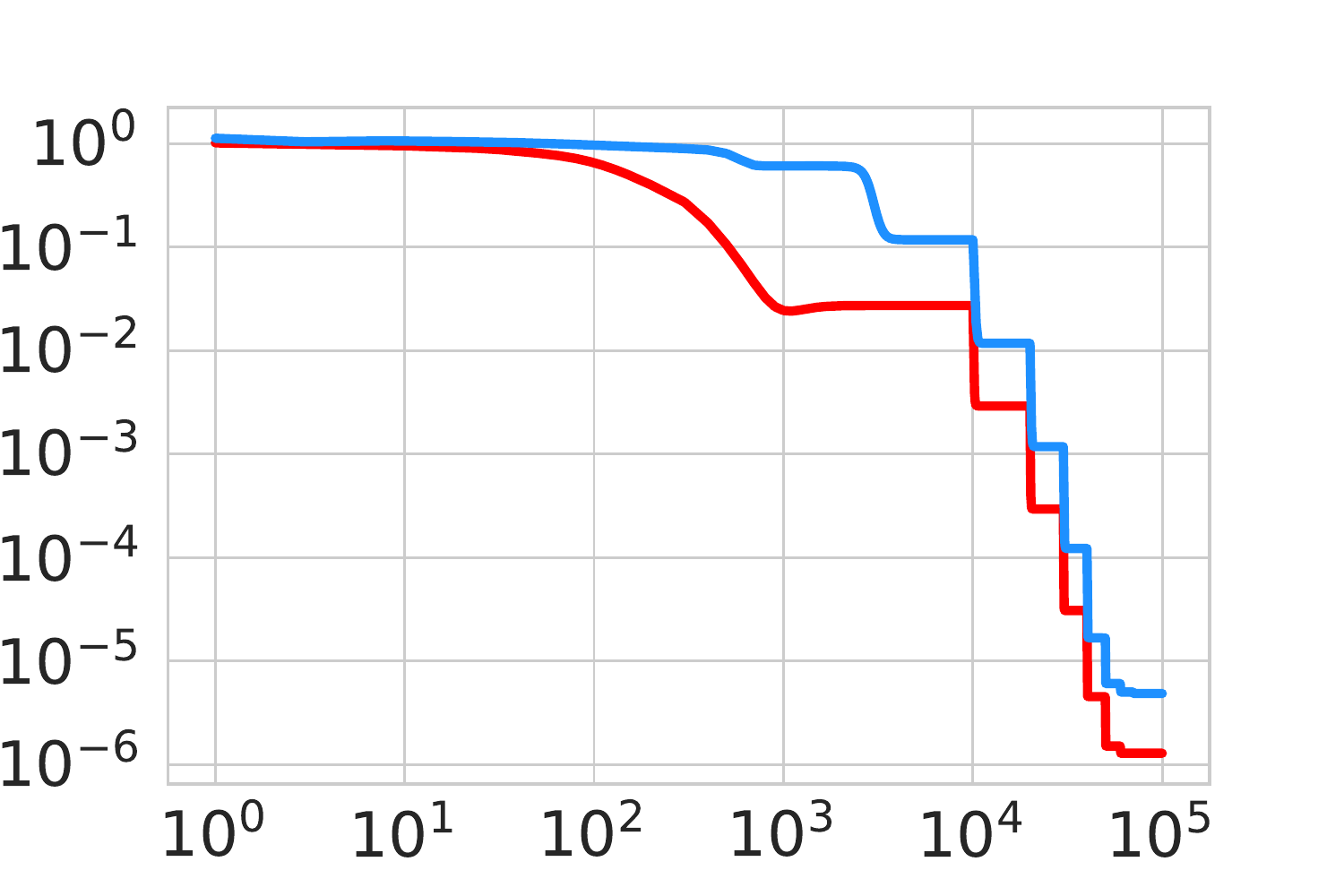}};
			\node at (0,-3) [scale=0.7]{Epochs};
			\node at (-2.7,-1.4) [scale=0.7, rotate=90]{};
		\end{tikzpicture}\captionsetup{width=0.9\linewidth}\caption{\small{Logits}}\label{fig:UFM_logreg_Z_hat}
	\end{subfigure}\caption{ Geometry of SGD solutions minimizing logit-regularized CE with ridge-decay on  UFM; SELI(Solid)/ETF(Dashed), $R=10$(\red{Red})/$R=100$(\blue{Blue}); See Sec. \ref{sec:logit_reg} for details.
	}
	\label{fig:UFM_logreg}
\end{figure*}
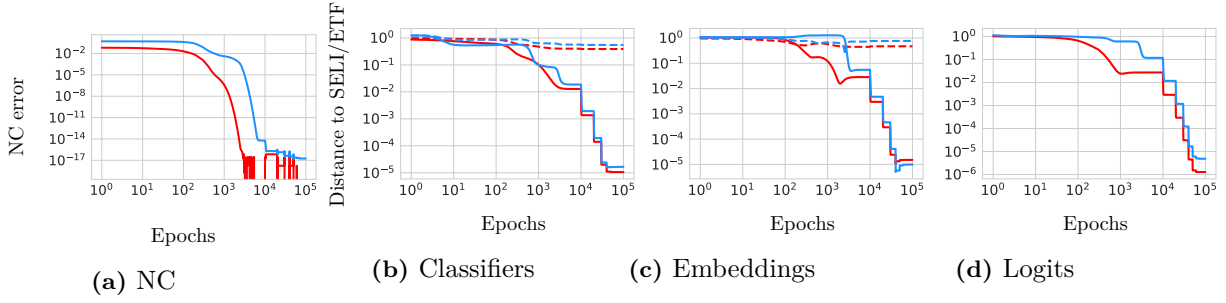

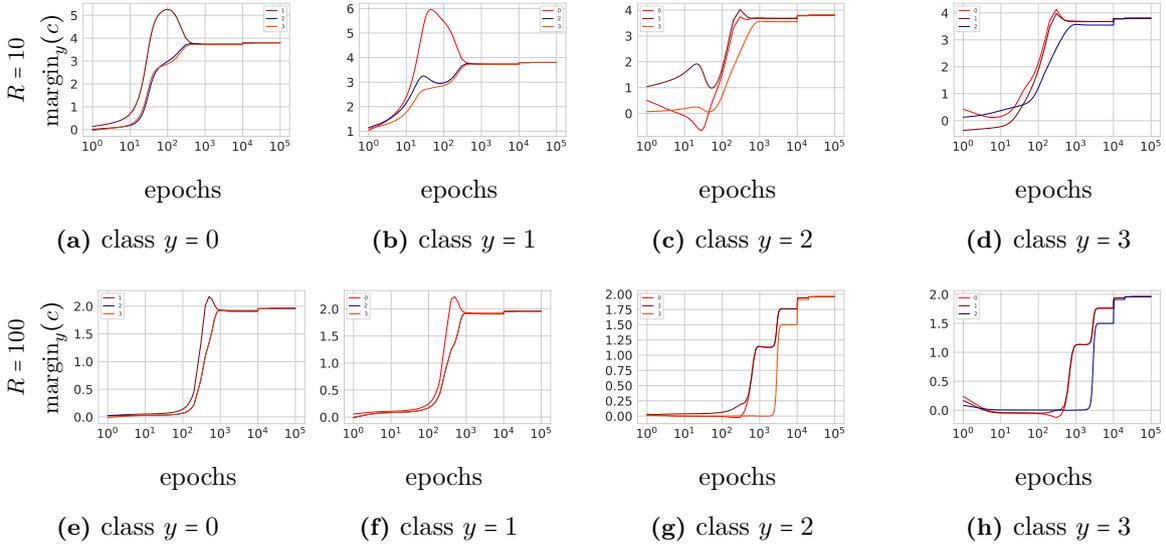
\begin{figure*}[t]
	\begin{subfigure}{0.23\textwidth}
		\centering
		\begin{tikzpicture}
			\node at (0,0) {\includegraphics[scale=0.23]{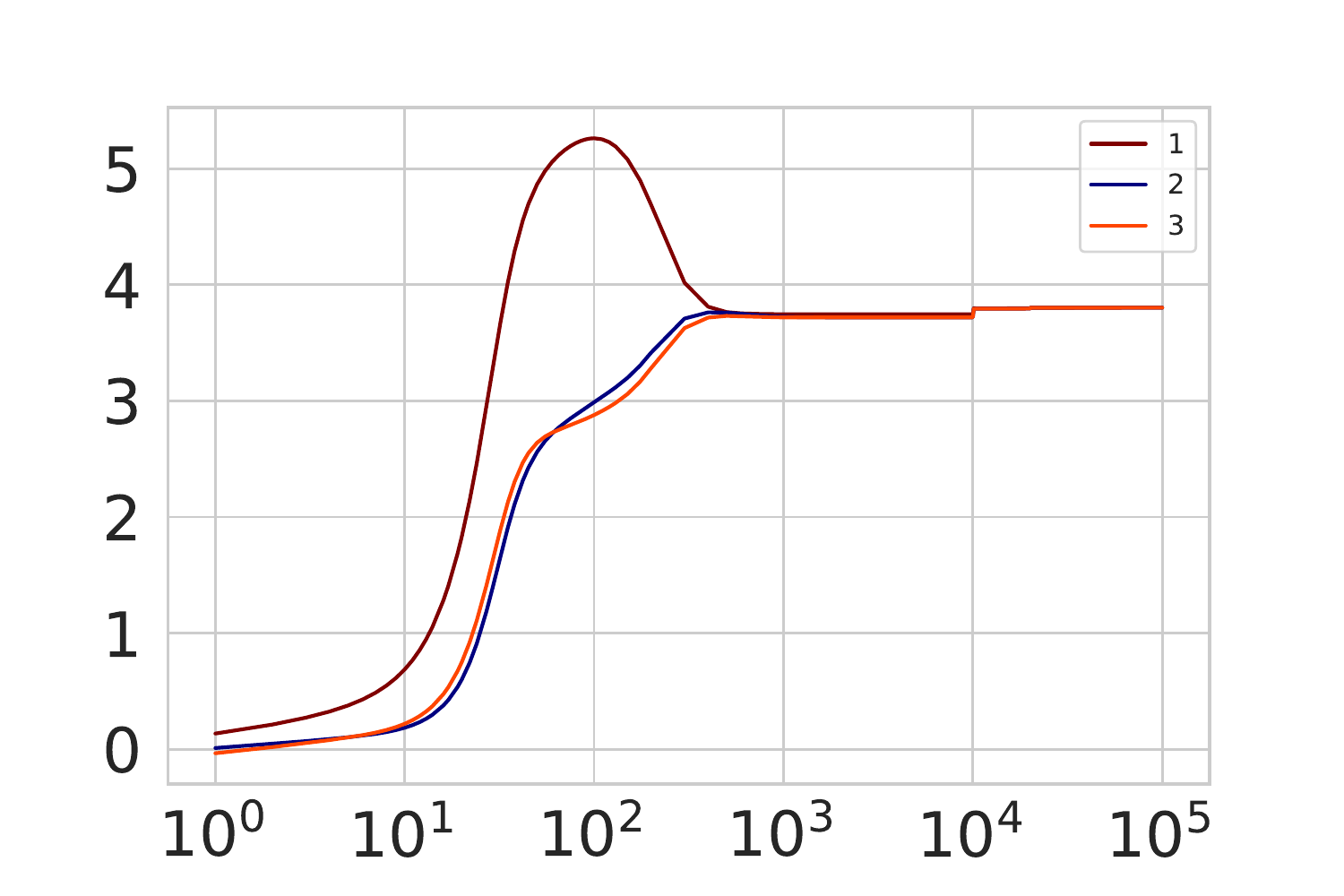}};
			\node at (0,-1.6) [scale=0.9]{epochs};
			\node at (-1.7,-0)  [scale=0.8, rotate=90]{$\text{margin}_y(c)$};
			\node at (-2.2,-0)  [scale=0.8, rotate=90]{$R=10$};
		\end{tikzpicture}\captionsetup{width=0.9\linewidth}\caption{\small{class $y=0$}}\label{fig:UFM_logreg_margins0}
	\end{subfigure}\hspace{0.5cm}\begin{subfigure}{0.23\textwidth}
		\centering
		\begin{tikzpicture}
			\node at (0,0.0) {\includegraphics[scale=0.23]{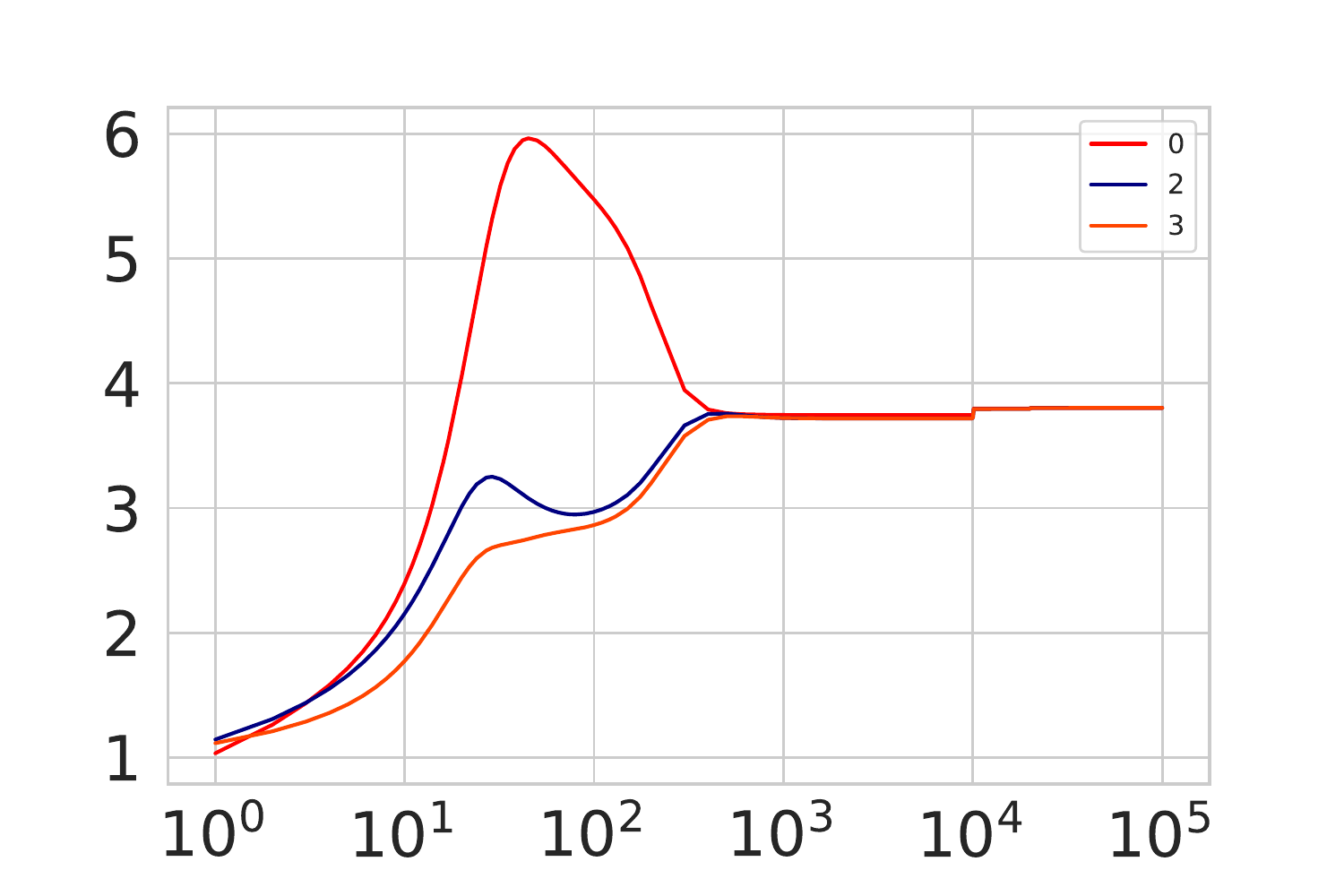}};
			\node at (0,-1.6) [scale=0.9]{epochs};
		\end{tikzpicture}\captionsetup{width=0.9\linewidth}\caption{\small{class $y=1$}}\label{fig:UFM_logreg_margins1}
	\end{subfigure}\hspace{0pt}\begin{subfigure}{0.23\textwidth}
		\centering
		\begin{tikzpicture}
			\node at (0,0) {\includegraphics[scale=0.23]{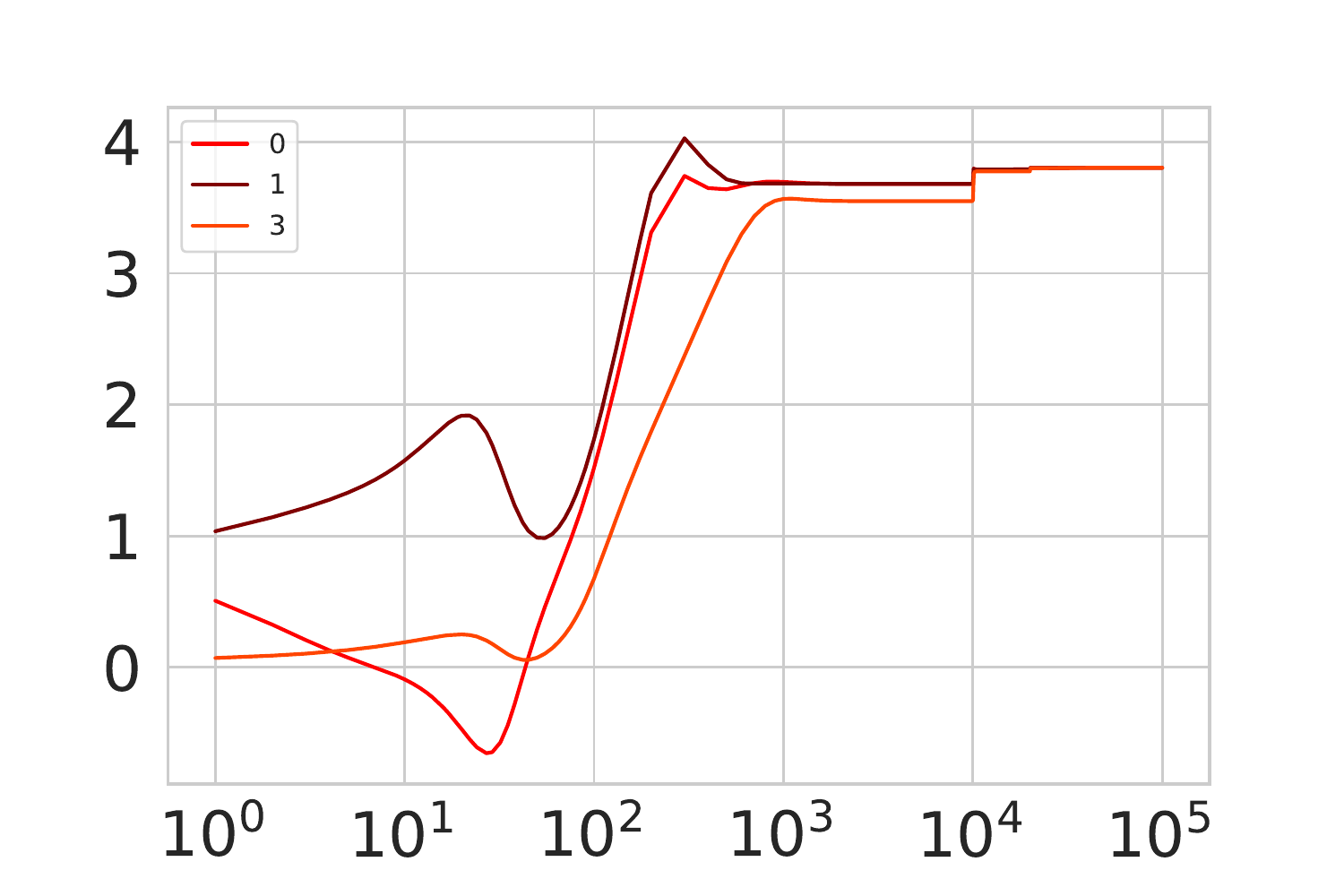}};
			\node at (0,-1.6) [scale=0.9]{epochs};
		\end{tikzpicture}\captionsetup{width=0.9\linewidth}\caption{\small{class $y=2$}}\label{fig:UFM_logreg_margins2}
	\end{subfigure}\hspace{0.5cm}\begin{subfigure}{0.23\textwidth}
		\centering
		\begin{tikzpicture}
			\node at (0,0) {\includegraphics[scale=0.23]{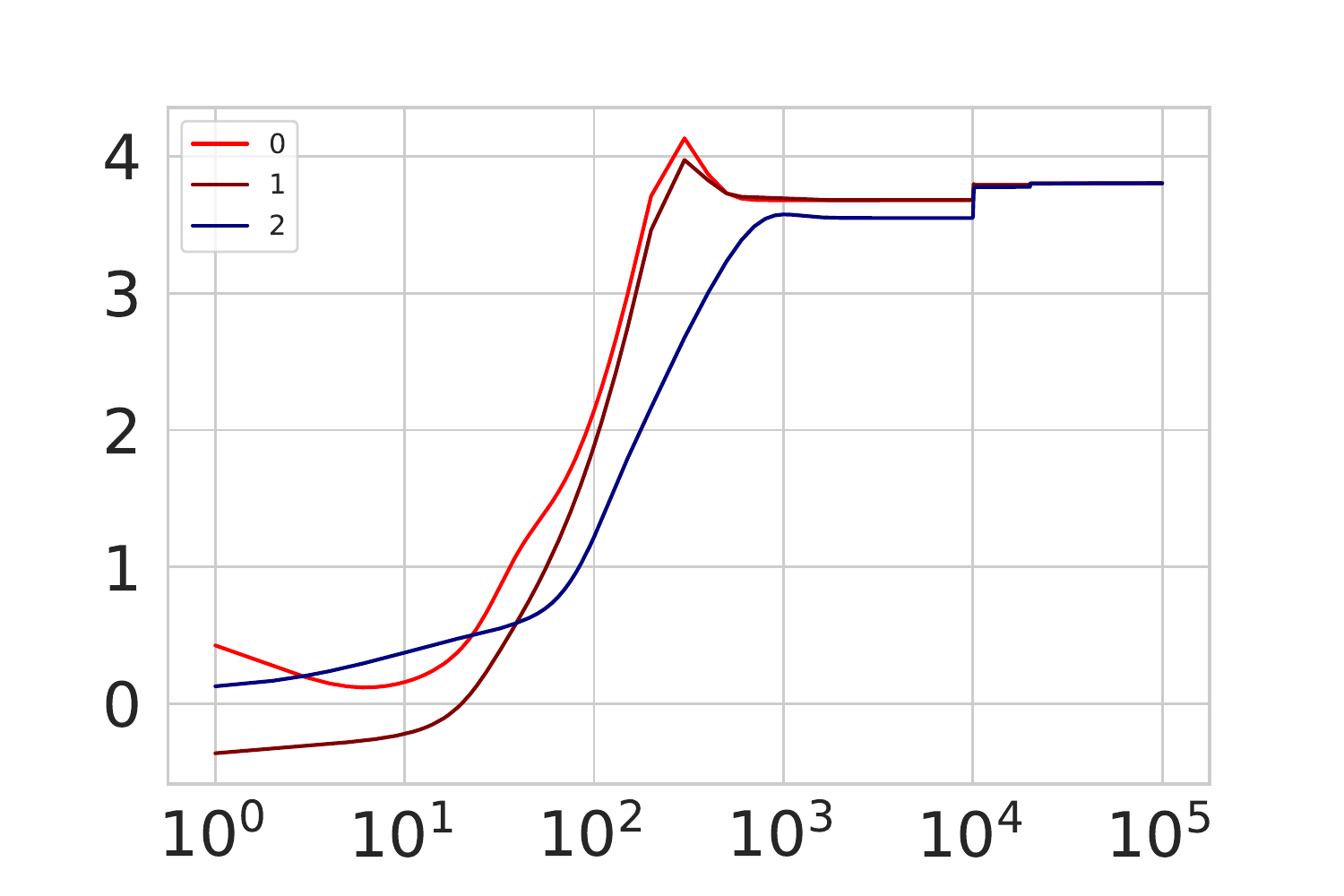}};
			\node at (0,-1.6) [scale=0.9]{epochs};
		\end{tikzpicture}\captionsetup{width=0.9\linewidth}\caption{\small{class $y=3$}}\label{fig:UFM_logreg_margins3}
	\end{subfigure}
	\begin{subfigure}{0.23\textwidth}
		\centering
		\begin{tikzpicture}
			\node at (0,0) {\includegraphics[scale=0.23]{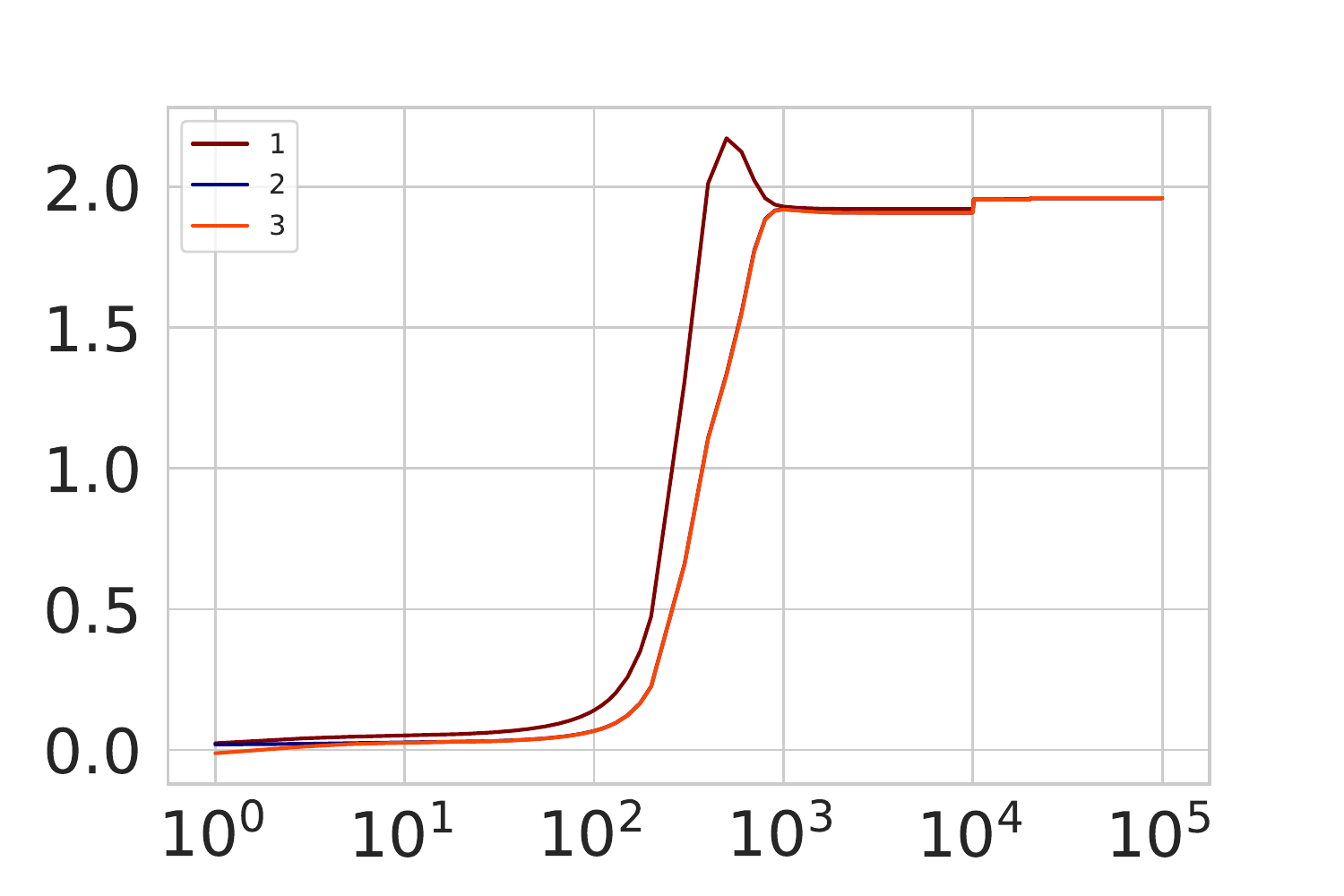}};
			\node at (0,-1.6) [scale=0.9]{epochs};
			\node at (-1.9,-0)  [scale=0.8, rotate=90]{$\text{margin}_y(c)$};
			\node at (-2.4,-0)  [scale=0.8, rotate=90]{$R=100$};
		\end{tikzpicture}\captionsetup{width=0.9\linewidth}\caption{\small{class $y=0$}}\label{fig:UFM_logreg_margins0_R_100}
	\end{subfigure}\hspace{0.3cm}\begin{subfigure}{0.23\textwidth}
		\centering
		\begin{tikzpicture}
			\node at (0,0.0) {\includegraphics[scale=0.23]{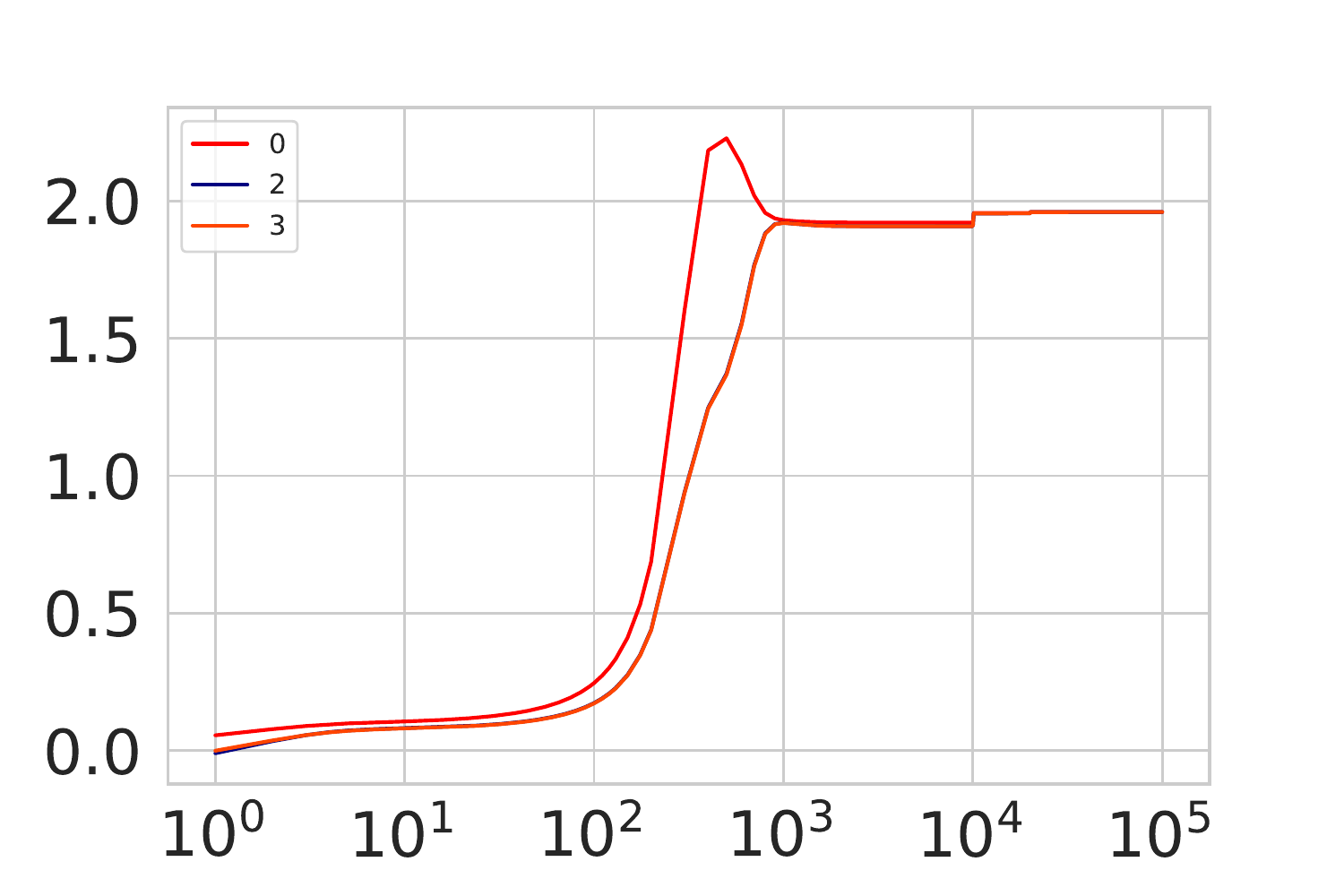}};
			\node at (0,-1.6) [scale=0.9]{epochs};
		\end{tikzpicture}\captionsetup{width=0.9\linewidth}\caption{\small{class $y=1$}}\label{fig:UFM_logreg_margins1_R_100}
	\end{subfigure}\hspace{0.2cm}\begin{subfigure}{0.23\textwidth}
		\centering
		\begin{tikzpicture}
			\node at (0,0) {\includegraphics[scale=0.23]{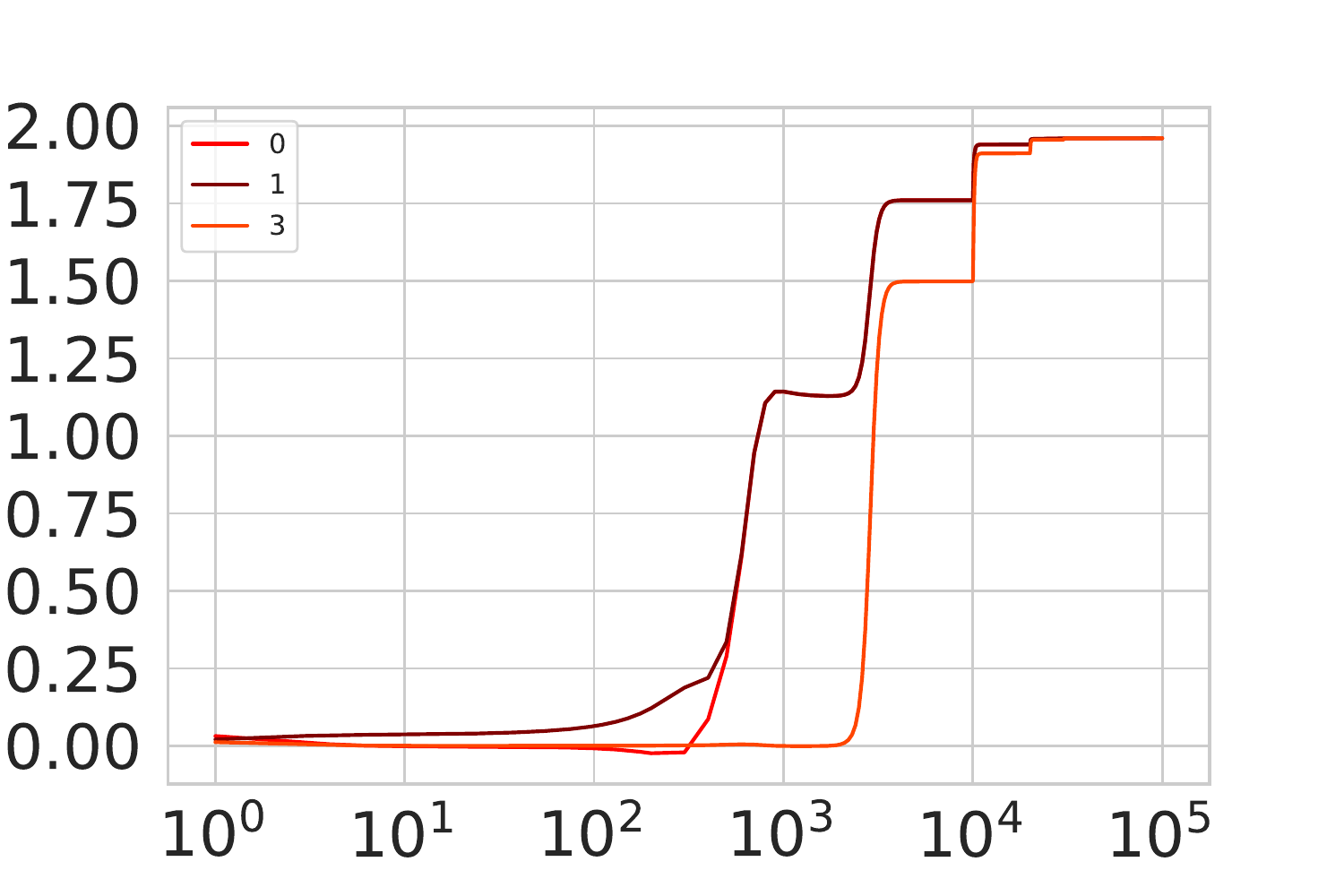}};
			\node at (0,-1.6) [scale=0.9]{epochs};
		\end{tikzpicture}\captionsetup{width=0.9\linewidth}\caption{\small{class $y=2$}}\label{fig:UFM_logreg_margins2_R_100}
	\end{subfigure}\hspace{0.5cm}\begin{subfigure}{0.23\textwidth}
		\centering
		\begin{tikzpicture}
			\node at (0,0) {\includegraphics[scale=0.23]{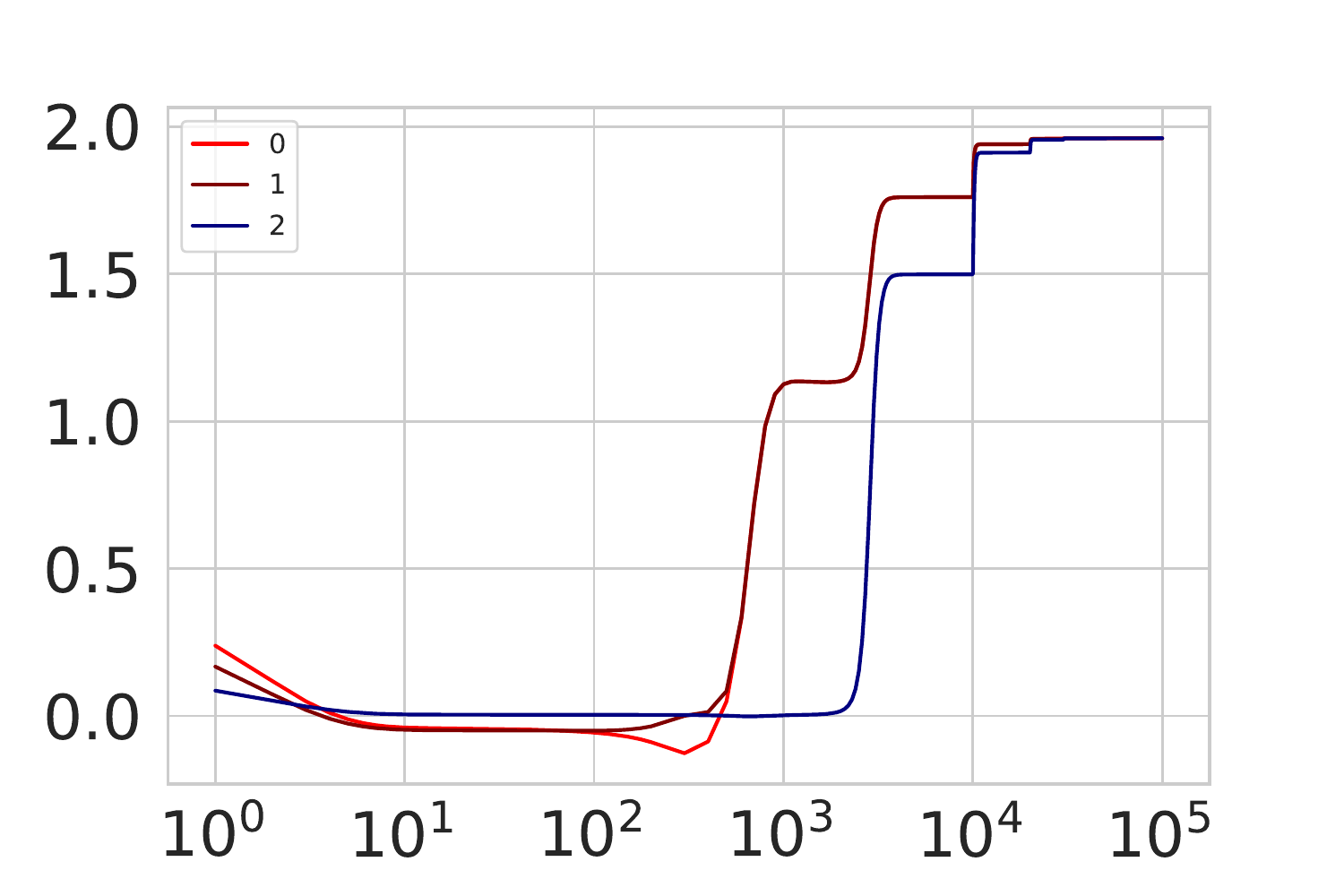}};
			\node at (0,-1.6) [scale=0.9]{epochs};
		\end{tikzpicture}\captionsetup{width=0.9\linewidth}\caption{\small{class $y=3$}}\label{fig:UFM_logreg_margins3_R_100}
	\end{subfigure}\caption{\small{Average Margins for 4 classes corresponding to the experiments of $R=10$ (top row) and $R=100$ (bottom row) from \Fig~\ref{fig:UFM_logreg}. We run GD on logit-regularized ($\lambda_L/n=10^{-3}$) CE with ridge-decay $\lambda_\text{initial}/n=10^{-2}$ (see Sec. \ref{sec:logit_reg} for details).}
	}\label{fig:UFM_logreg_margins}
\end{figure*}

On the other hand, we also find that logit-regularized ERM on the UFM yields the SELI geometry simultaneously for logits, classifiers and embeddings when we follow a specific decaying schedule for the ridge-regularization parameter. 
Specifically, in the experiments shown in \Fig~\ref{fig:UFM_logreg} and \ref{fig:UFM_logreg_margins} below, we start with a large initial value $\la/n = 10^{-2}$ and progressively decay $\la$ by a factor of $10$ after every few epochs. We also set a finite strength of logit-regularization $\la_L/n=10^{-3}$, although the convergence direction is not sensitive to that choice. 
We term this scheduling ``ridge-decay''. While the exact dynamics followed by the SGD with such a scheme require further analysis, ``ridge-decay'' can be thought of as emulating the regularization path of ridge-regularization with $\la \rightarrow 0$.

\Fig~\ref{fig:UFM_logreg} show convergence of the GD solution with the above described ridge-decay to the SELI geometry. Here, we choose $k=4$ classes, with $(R=10,100,\rho=\nicefrac{1}{2})$-STEP imbalance and  $n_\text{min} = 1$. The learning rate is again fixed to $1$. In this experiment, we use GD instead of SGD since the number of examples is small. This is also useful as it shows that GD is able to drive the solution towards the SELI geometry without stochastic updates being necessary. 

In summary, we make the following observations from \Fig~\ref{fig:UFM_logreg}. First, GD iterates favor the SELI, instead of the ETF geometry, suggesting an implicit bias towards global minimizers of the UF-SVM. Second, with logit-regularization and ridge-decay, convergence towards SELI is achieved at a faster rate than in the case of unregularized CE. Third, while \Fig~\ref{fig:UFM_logreg_Z_only} showed that the logit matrix converges fast in direction of the SEL matrix $\Zhat$ with only logit-regularization, \Fig~\ref{fig:UFM_logreg_W_hat} and \ref{fig:UFM_logreg_H_hat} show that ridge-decay promotes convergence of classifiers and embeddings to their respective SELI geometries as well.
Finally, for completeness, we also present the NC convergence in \Fig~\ref{fig:UFM_logreg_NC}. The metric used to measure NC is as described in \Sec~\ref{sec:NC_for_Imbalance}.

The corresponding margins for the $4$ classes are shown in \Fig~\ref{fig:UFM_logreg_margins}. For examples belonging to class $y$, we define average margin with respect to another class $c\neq y$ as follows:
\begin{align*}
	\text{margin}_y(c) := (\w_{y} - \w_{c})^T\mub_y.
\end{align*}
Note that this is an average over examples from the class $y$ since by NC property $\h_i\approx\mub_y, \forall i\,:\,y_i=c$. 

We make the following remarks regarding \Fig~\ref{fig:UFM_logreg_margins}. First, as training progresses, the (average) margins are positive, thus zero training error is achieved for all classes.
Second, the average margins for a class $y$ with respect to classes ${c:c\neq y}$ converge to a common value, even though their initial values  differ. This can be seen from the convergence of the four different colored curves within a plot (e.g. \Fig~\ref{fig:UFM_logreg_margins0}). Third, all margins for all pairs of classes converge to the same quantity, irrespective of being majority or minority classes. Note, for instance, from \Fig~\ref{fig:UFM_logreg_margins0},\ref{fig:UFM_logreg_margins1},\ref{fig:UFM_logreg_margins2}, and \ref{fig:UFM_logreg_margins3} that the final value of all graphs is the same. Finally, the value of margins stagnates to a level that is governed by the strength of the logit-regularization $\lambda_L$. 
\section{Additional results on real data}\label{sec:real_exp_SM}
\new{This section complements Sec. \ref{sec:deep-net-main} with  additional experiments.}


\subsection{Missing details on experiments of Section \ref{sec:deep-net-main}}
\subsubsection{Implementation details}\label{subsec:app_implementation_details}


Our experiments on deep models build on the code provided by \cite{NC}\footnote{\url{https://colab.research.google.com/github/neuralcollapse/neuralcollapse/blob/main/neuralcollapse.ipynb}}. We train ResNet-18 \cite{he2015deep} and VGG-13 \cite{vgg} models, on three 10-class datasets, CIFAR10, MNIST and Fashion-MNIST. {Following \cite{NC}, we use batch normalization in place of the dropout layers in VGG-13.} For both models, we disable the biases of all the fully-connected layers, similar to the experiments on UFM. We adopt the same training strategy as \cite{NC}, namely SGD on CE loss, with momentum ($0.9$), small weight decay ($5\times10^{-4}$), and learning rate $0.1$ decayed at two stages (epochs $120$ and $240$) by a factor of 10. We train the network on a $(R,1/2)$-STEP imbalance setting. To create imbalanced data, we use the data sampler provided by \citet{TengyuMa} \footnote{\url{https://github.com/kaidic/LDAM-DRW}}. Following \cite{NC} we do not use any data augmentation. In all the experiments, we fix the first 5 classes to be majorities, and the rest as minorities. To have a fair comparison between the models with different imbalance ratios $R$, we sample the datasets to have $n=25250$ training images in all cases. While the training set is imbalanced, when measuring test performance of a trained model we do so on a balanced test set, e.g. just like \cite{byrd2019effect,TengyuMa}. We measure the metrics at certain epochs, and similar to \cite{NC}, we sample epochs more frequently at the start of the training as the network parameters change more quickly in the beginning.


\subsubsection{Model accuracies}\label{sec:SM_accuracies}
Consistent with the requirements of the neural collapse phenomenon by \citet{NC}, all models are trained well beyond zero error. Specifically, as illustrated in \Fig~\ref{fig:intro_CIFAR} and \ref{fig:vgg_}, most of the models achieve zero error around epoch 120, while training continues until epoch $350$. Table \ref{Table 2} presents the first epochs at which each model achieves 100 percent accuracy for each value of the imbalance ratio $R$. For practical purposes, the minimal requirement for majority classes to be declared having achieved zero training error is set to $0.2\%$ error. For minority classes  $n_{min} \leq 500$ we set $0.00\%$ for the same requirements.
\begin{table}[h]
	\begin{center}
		\small
		\begin{tabular}[5pt]{ |c||c|c|c|c| } 
			\hline \newline
			Zero Error Epoch & R = 1 & R = 5 & R = 10 & R = 100 \\
			\hline\hline
			\multicolumn{5}{|c|}{ ResNet (VGG)} \\
			\hline
			CIFAR10 & 118 (121) & 117 (119) & 119 (119) & 120 (176)\\
			\hline
			MNIST & 10 (117) & 8 (125) & 11 (127) & 7 (117) \\
			\hline
			\new{Fashion-MNIST} & \new{117 (119)} & \new{117 (119)} & \new{118 (120)} & \new{117 (141)} \\
			\hline
		\end{tabular}
		\vspace{3pt}
		\caption{First epoch that zero training error is achieved in deep-net eperiments of Sec. \ref{sec:real_exp_SM}.}
		\label{Table 2}
	\end{center}
\end{table}

Balanced test errors are reported in Table \ref{Table 3}. Majority and minority errors are calculated by averaging the per-class majority and minority errors respectively. The total error is calculated by averaging all per-class accuracies. The averaging is done with equal weights. The test error for the balanced case ($R=1$) is larger than errors reported in \cite{NC} since the model is not trained on the whole dataset, but on $n=25250$ samples. 
\begin{table}[h]
	\begin{center}
		\small
		\begin{tabular}[5pt]{ |c||c|c|c|c| } 
			\hline
			Test Error & R = 1 & R = 5 & R = 10 & R = 100 \\
			\hline\hline
			\multicolumn{5}{|c|}{ CIFAR10 - ResNet (VGG) } \\
			\hline
			Total  & 16.47\% (17.27\%) & 26.17\% (21.08\%) & 34.68\% (30.03\%) & 53.89\% (53.21\%)\\
			\hline
			Majority  & 19.18\% (20.60\%) & 11.26\% (12.64\%) & 10.58\% (11.30\%) & 9.68\% (10.28\%) \\
			\hline
			Minority  & 13.76\% (13.94\%) & 41.08\% (30.96\%) & 58.78\% (48.76\%) & 98.1\% (96.14\%) \\
			\hline\hline
			\multicolumn{5}{|c|}{ MNIST - ResNet (VGG) } \\
			\hline
			Total  & 0.55\% (0.52\%) & 0.78\% (0.70\%) & 0.96\% (0.81\%) & 3.04\% (4.75\%)\\
			\hline
			Majority  & 0.38\% (0.35\%) & 0.18\% (0.25\%) & 0.06\% (0.06\%) & 0.02\% (0.08\%) \\
			\hline
			Minority  & 0.73\% (0.70\%) & 1.41\% (1.16\%) & 1.87\% (1.60\%) & 6.18\% (9.62\%) \\
			\hline\hline
			\multicolumn{5}{|c|}{ \new{Fashion-MNIST - ResNet (VGG)} } \\
			\hline
			Total  & 7.64\% (7.54\%) & 10.27\% (9.26\%) & 11.63\% (10.35\%) & 16.68\% (16.67\%)\\
			\hline
			Majority  & 8.56\% (8.40\%) & 5.72\% (5.50\%) & 5.12\% (4.58\%) & 4.78\% (4.94\%) \\
			\hline
			Minority  & 6.72\% (6.68\%) & 14.82\% (13.02\%) & 18.14\% (16.12\%) & 28.58\% (28.40\%) \\
			 \hline
		\end{tabular}
		\vspace{3pt}
		\caption{Balanced test error of deep-net experiments in Sec. \ref{sec:real_exp_SM}.}
		\label{Table 3}
	\end{center}
\end{table}

\subsubsection{NC property}
From the \emph{\ref{NC}} property, we expect that the embeddings collapse to their class means. In order to quantify validity of this property, we follow \cite{NC}. Specifically, we compute the {within-class covariance} ($\Sigmab_W$) and {between-class covariance} ($\Sigmab_B$) as, 
$
\Sigmab_W = \sum_{i \in [n]}  (\h_{i} - \mub_{y_i})(\h_{i} - \mub_{y_i})^T \in \mathbb{R}^{d \times d},
$ 
 and 
$
\Sigmab_B = \sum_{c \in [k]}  (\mub_c - \mub_G)(\mub_c - \mub_G)^T \in \mathbb{R}^{d \times d},
$
where $\mub_c = \frac{1}{n_c} \sum_{i\,:\,y_i\in[c]} \h_{i}$ is the mean embedding of class $c$ and $\mub_G = \frac{1}{k} \sum_{c \in [k]} \mub_c$ is their (blanaced) global mean. We can now measure NC by computing $\tr(\Sigmab_W \Sigmab_B^\dagger) / k$. \Fig~\ref{fig:NC} illustrates how this quantity indeed decreases as training evolves. This confirms that feature embeddings converge to their class means, regardless of the imbalance ratio $R$. 
\begin{figure*}[t]
	\vspace{-10pt}
	\centering
	\hspace{-60pt}
	\begin{subfigure}{0.9\textwidth}
		\centering
		\begin{tikzpicture}
			\node at (0,0) 
			{\includegraphics[scale=0.25]{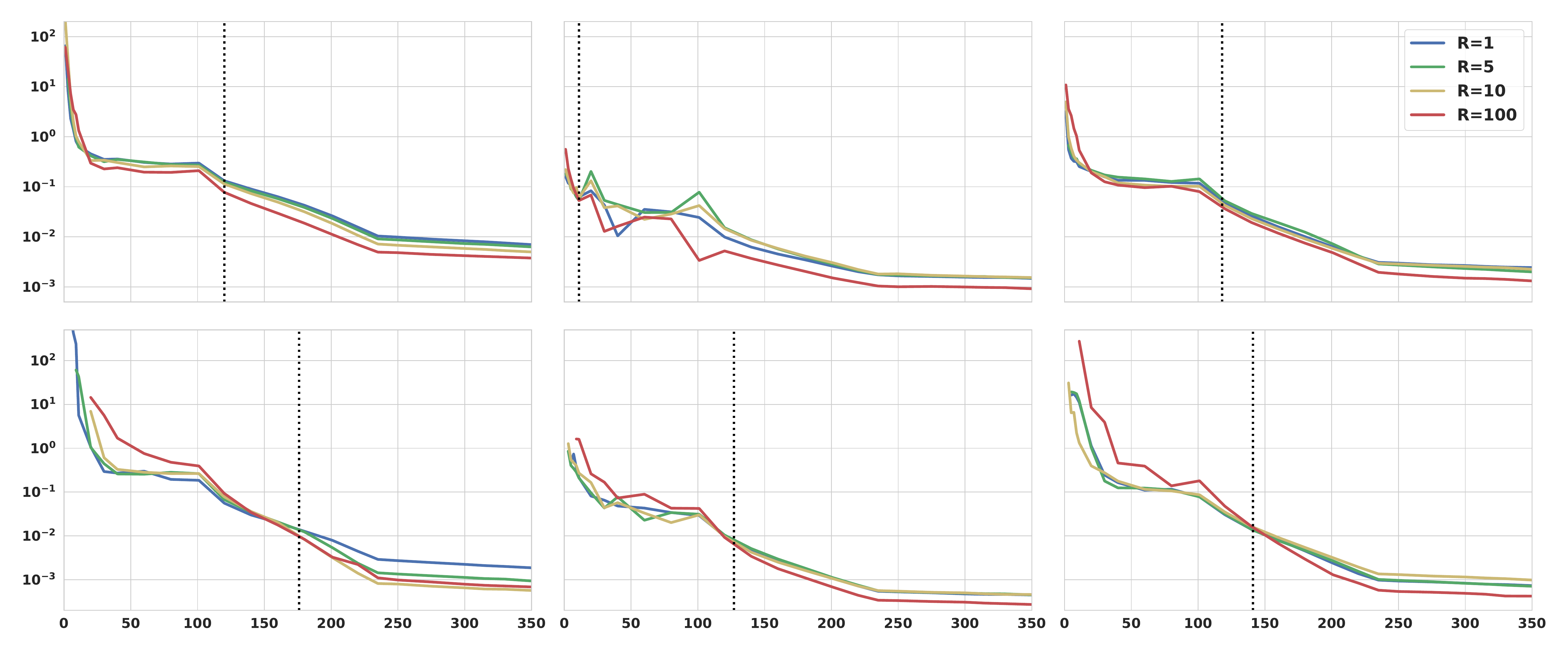}};
			\node at (5.0,-3.2) [scale=0.7] {Epochs};
			\node at (0,-3.2) [scale=0.7] {Epochs};
			\node at (-5.0,-3.2) [scale=0.7] {Epochs};
			\node at (-7.6,1.8) [scale=0.7, rotate=90]{\textbf{ResNet}};
			\node at (-7.6,-1.3) [scale=0.7, rotate=90]{\textbf{VGG}};
			\node at (-8.0,0.2) [scale=0.7, rotate=90] {Distance to SELI/ETF};
			\node at (-5.0,3.2) [scale=0.9] {\textbf{CIFAR10}};
			\node at (0.0,3.2) [scale=0.9] {\textbf{MNIST}};
			\node at (5.0,3.2) [scale=0.9] {\new{\textbf{Fashion-MNIST}}};
		\end{tikzpicture}
	\end{subfigure}	\vspace{-5pt}
	\caption{NC property for 
		different imbalance levels $R$.}
	\label{fig:NC}
\end{figure*}
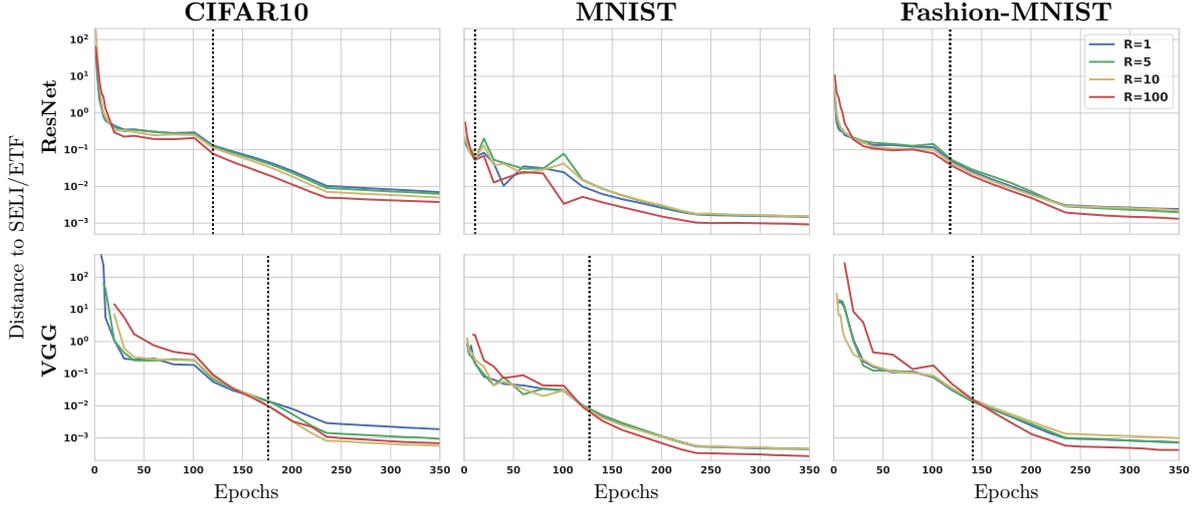

\label{sec:NC_for_Imbalance}

\subsubsection{Norms of classifiers / embeddings}\label{sec:SM_exp_norm}
\begin{figure*}
	\vspace{-10pt}
	\centering
	\hspace{-60pt}
	\begin{subfigure}{0.9\textwidth}
		\centering
		\begin{tikzpicture}
			\node at (0,0) 
			{\includegraphics[scale=0.25]{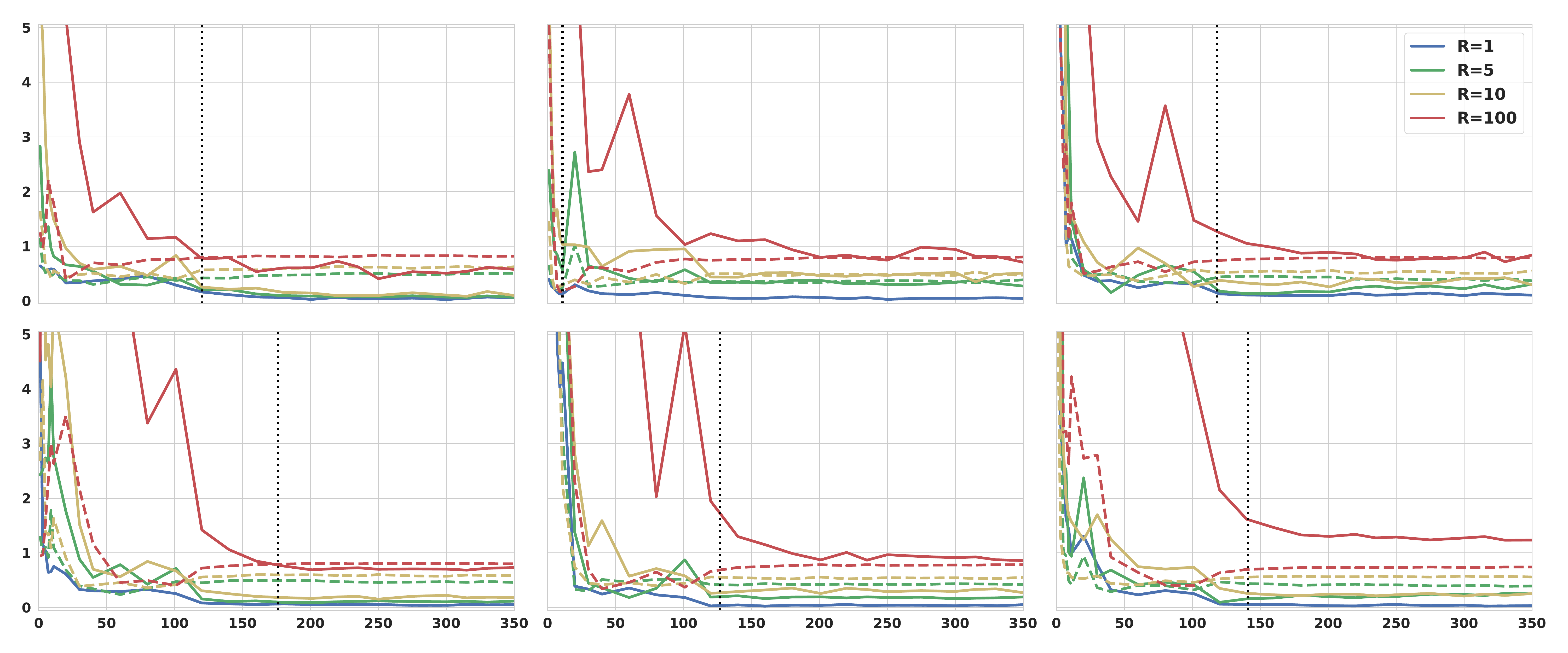}};
			\node at (5.0,-3.2) [scale=0.7] {Epochs};
			\node at (0,-3.2) [scale=0.7] {Epochs};
			\node at (-5.0,-3.2) [scale=0.7] {Epochs};
			\node at (-7.6,1.8) [scale=0.7, rotate=90]{\textbf{ResNet}};
			\node at (-7.6,-1.3) [scale=0.7, rotate=90]{\textbf{VGG}};
			\node at (-8.0,0.2) [scale=0.7, rotate=90] {Distance to SELI/ETF};
			\node at (-5.0,3.1) [scale=0.9] {\textbf{CIFAR10}};
			\node at (0.0,3.1) [scale=0.9] {\textbf{MNIST}};
			\node at (5.0,3.1) [scale=0.9] {\new{\textbf{Fashion-MNIST}}};
		\end{tikzpicture}
	\end{subfigure}	\vspace{-5pt}
	\caption{Convergence of embeddings norm ratio ($\|\mubmaj\|_2/\|\mubmin\|_2$) to SELI (solid) vs ETF (dashed)
			for different imbalance levels $R$.}
	\label{fig:norm_ratio_exp_H}
\end{figure*}
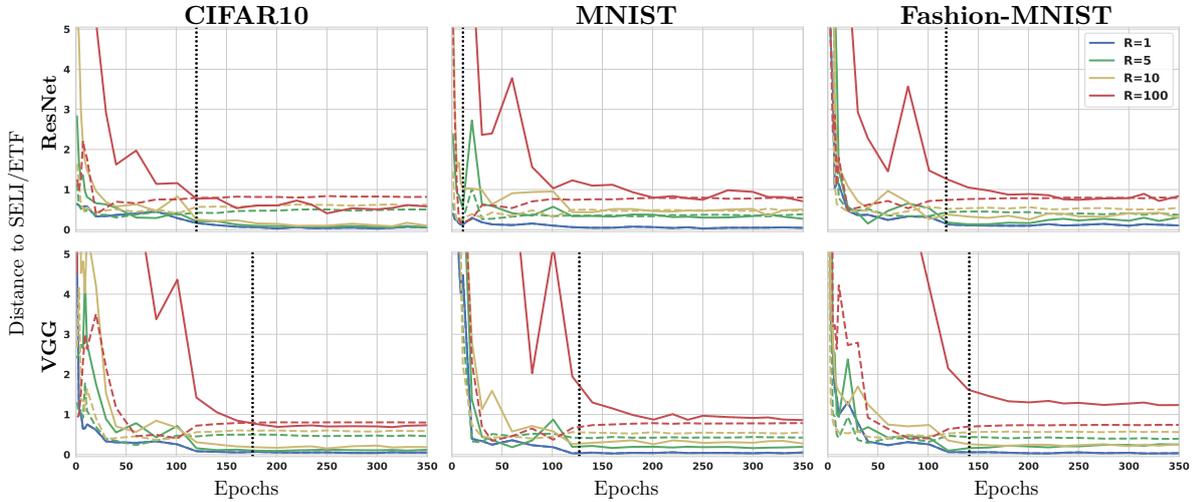
Here, we further investigate the geometry of learned embeddings and classifiers, by focusing on their norms. In particular, we study the ratios $\tau_\h := \|\hmaj\|_2 / \|\hmin\|_2$ and $\tau_\w := \|\wmaj\|_2 / \|\wmin\|_2$.  Assuming the classifiers and embeddings follow the SELI geometry, those ratios admit explicit closed-form expressions thanks to Lemmas \ref{lem:norms_w} and \ref{lem:norms_h}. To determine deviations of the measured norm-ratios $\tau_\w$ and $\tau_\h$ compared to those reference closed-form expressions, we calculate and report the following quantity:
\begin{align*}
    \text{Average}_{c,c'} \big(|\tau_{\w}(c,c') - \hat{\tau}_{\w}| / \hat{\tau}_{\w}\big),
\end{align*}
where $\hat{\tau}_\w$ is given by Lemma \ref{lem:norms_w} for SELI and is equal to $1$ for ETF, and,  $\tau_{\w}(c,c') = \|\w_{c}\|_2 / \|\w_{c'}\|_2$ with $c$ being a majority and $c'$ a minority class. Similarly, we compute distances for the norm-ratios of centered mean embeddings $\overline{\mub_c}$. \Fig~\ref{fig:norm_ratio_exp_W} and \ref{fig:norm_ratio_exp_H} depict these metrics during training of the ResNet and VGG networks. The results confirm once more that the SELI geometry accurately captures features of the learned geometries. On the contrary, this is not the case for ETF when data are imbalanced. We observe that convergence to SELI geometry (and respective deviation from ETF) is more pronounced for the classifier weights and becomes elusive for the embeddings particularly for large imbalance ratios ($R=100$.)

\subsubsection{Non-alignment of classifiers and embeddings} \label{sec:SM_exp_nonalign}
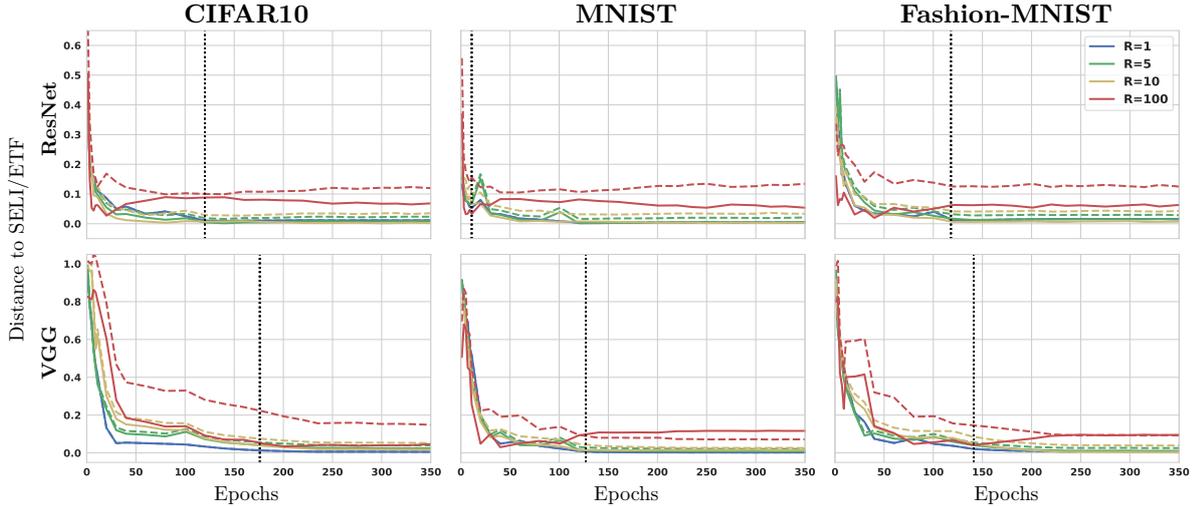
\begin{figure*}[t]
	\vspace{-10pt}
	\centering
	\hspace{-60pt}
	\begin{subfigure}{0.9\textwidth}
		\centering
		\begin{tikzpicture}
			\node at (0,0) 
			{\includegraphics[scale=0.25]{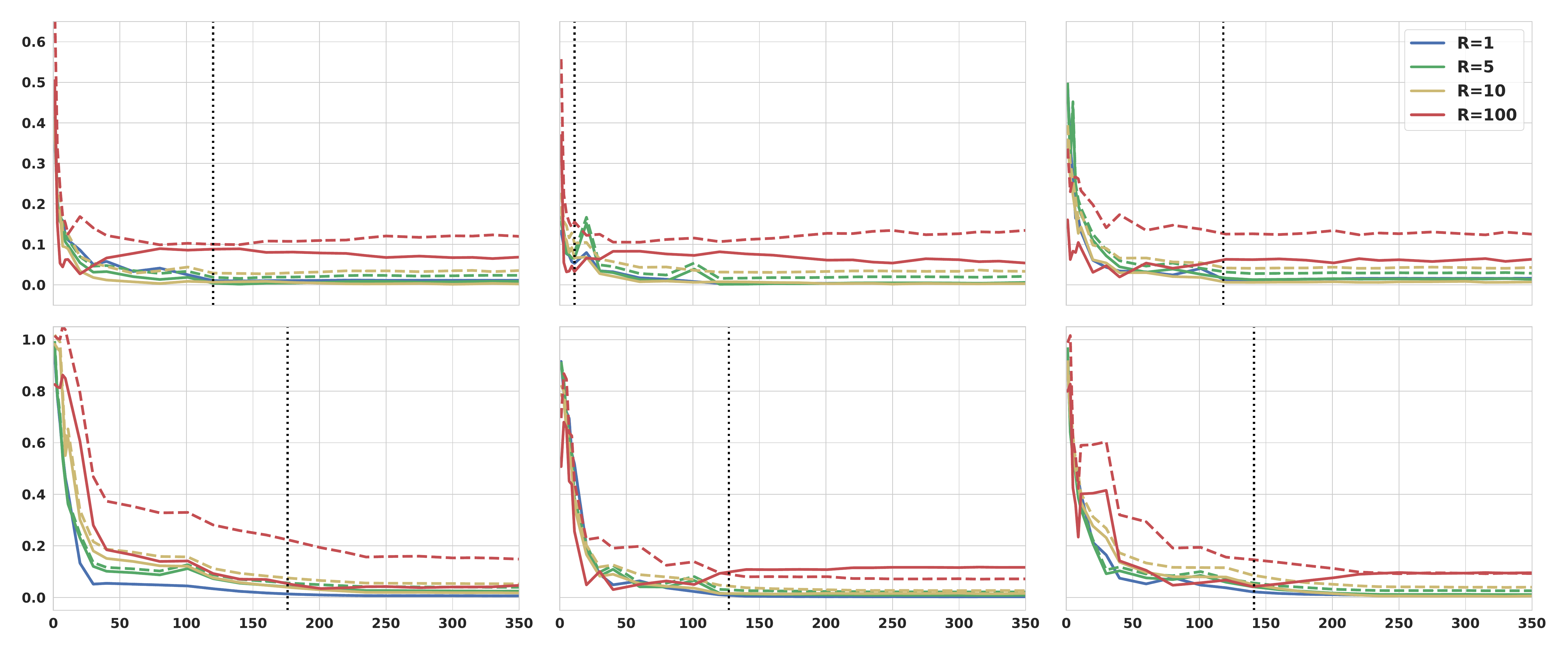}};
			\node at (5.0,-3.2) [scale=0.7] {Epochs};
			\node at (0,-3.2) [scale=0.7] {Epochs};
			\node at (-5.0,-3.2) [scale=0.7] {Epochs};
			\node at (-7.6,1.8) [scale=0.7, rotate=90]{\textbf{ResNet}};
			\node at (-7.6,-1.3) [scale=0.7, rotate=90]{\textbf{VGG}};
			\node at (-8.0,0.2) [scale=0.7, rotate=90] {Distance to SELI/ETF};
			\node at (-5.0,3.2) [scale=0.9] {\textbf{CIFAR10}};
			\node at (0.0,3.2) [scale=0.9] {\textbf{MNIST}};
			\node at (5.0,3.2) [scale=0.9] {\new{\textbf{Fashion-MNIST}}};
		\end{tikzpicture}
	\end{subfigure}	\vspace{-5pt}
	\caption{Alignment of minority embeddings and classifiers. Convergence of angles to SELI (solid) and ETF (dashed) for different imbalance levels $R$.}
	\label{fig:cosWH}
\end{figure*}
%
While from previous empirical results on balanced datasets \cite{NC}, we expect an alignment between the classifiers and embeddings, Lemma \ref{lem:align} suggests these two geometries deviate as data becomes more imbalanced. 
To verify this property, we compute the angle between mean embeddings and their corresponding classifiers, and measure the deviation from the SELI geometry. Namely, let $\theta_c=\Cos{\w_c}{\h_c}$. Then, similar to the previous section, we compute,
\begin{align*}
	\text{Average}_{c} \big(|\theta_c - \hat{\theta}| / \hat{\theta}\big),
\end{align*}
where $c$ ranges over minority classes and $\hat{\theta}$ is given by \eqref{eq:wh_min}. \Fig~\ref{fig:cosWH} shows how this quantity evolves during training. From \Fig~\ref{fig:SELI_theory_angles_wh}, we know that the embeddings and classifiers of the majority classes remain aligned even for highly imbalanced data, thus we only analyze the impact of imbalance ratio on the minority classes.

\subsubsection{Majority vs minority geometry}\label{sec:SM_maj_min}
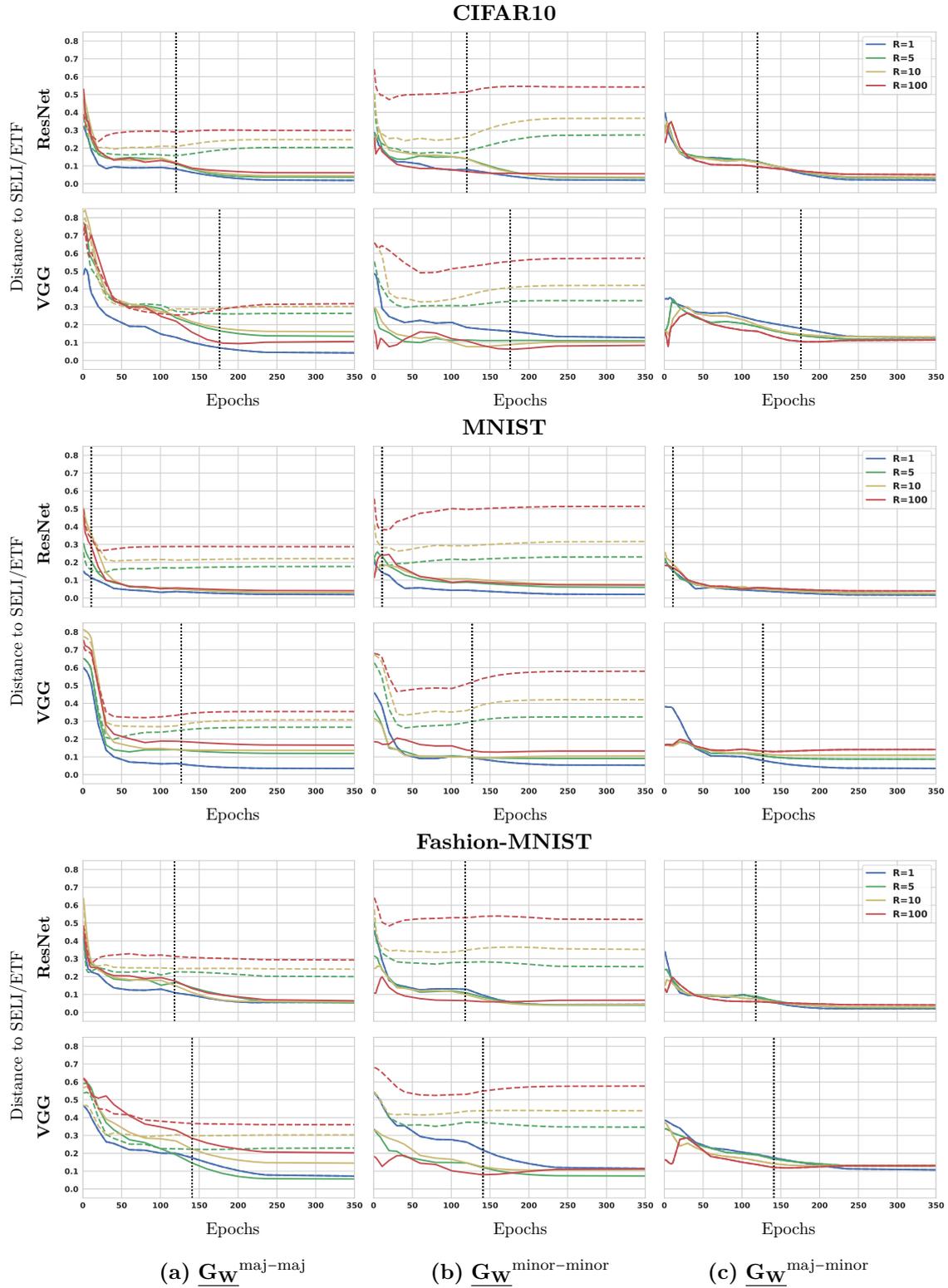
\begin{figure}
	\centering
	\begin{subfigure}[b]{1.0\textwidth}
		\centering
		\begin{tikzpicture}
			\node at (0,-1.4) {\includegraphics[width=0.9\textwidth]{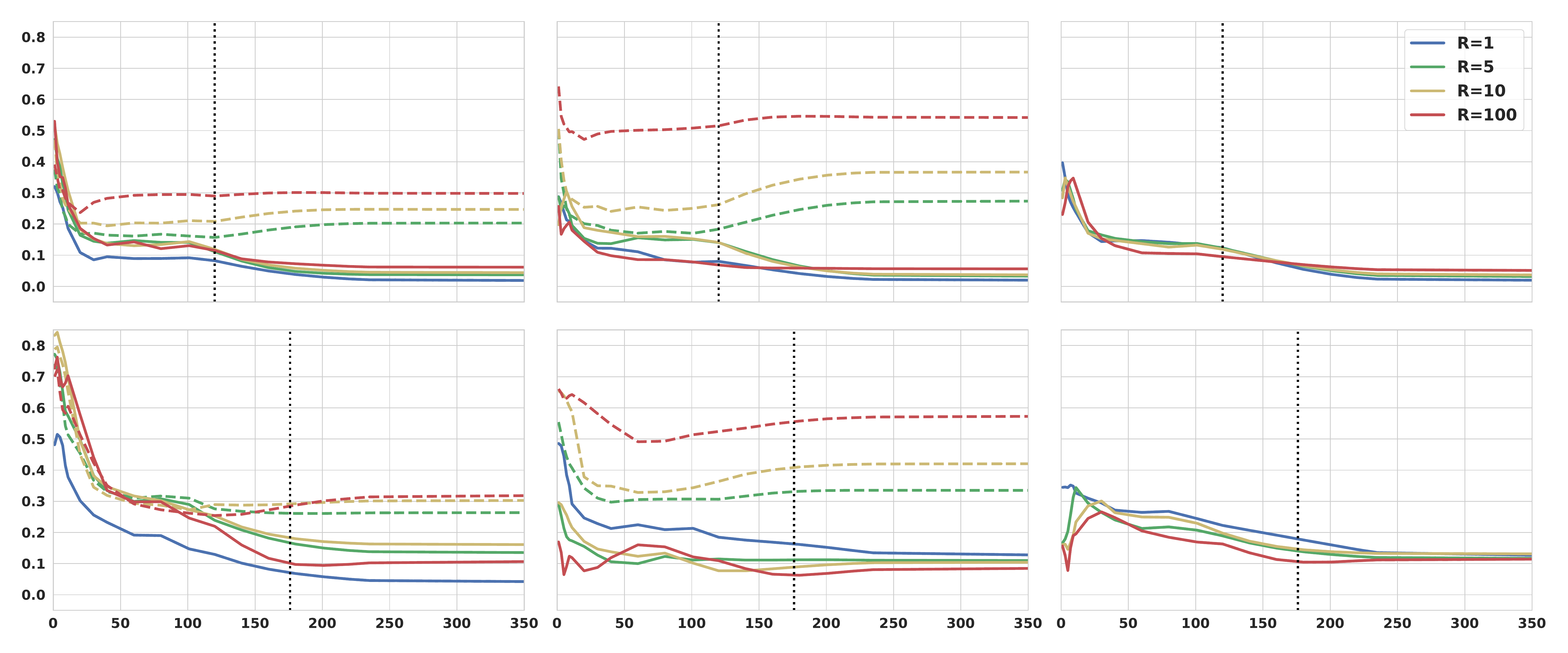}};
			\node at (0,1.7) [scale=0.9]{\textbf{CIFAR10}};
			\node at (0.25,-4.5) [scale=0.7]{Epochs};
			\node at (-4.3,-4.5) [scale=0.7]{Epochs};
			\node at (4.5,-4.5) [scale=0.7]{Epochs};
			\node at (-7.3,0.0) [scale=0.7, rotate=90]{\textbf{ResNet}};
			\node at (-7.3,-2.7) [scale=0.7, rotate=90]{\textbf{VGG}};
			\node at (-7.7,-1.4) [scale=0.7, rotate=90]{Distance to SELI/ETF};
		\end{tikzpicture}
		\label{fig:GW_minmaj_cifar10}
	\end{subfigure}
	\\
	\vspace{-2pt}
	\begin{subfigure}[b]{1.0\textwidth}
		\centering
		\begin{tikzpicture}
			\node at (0,-1.4) {\includegraphics[width=0.9\textwidth]{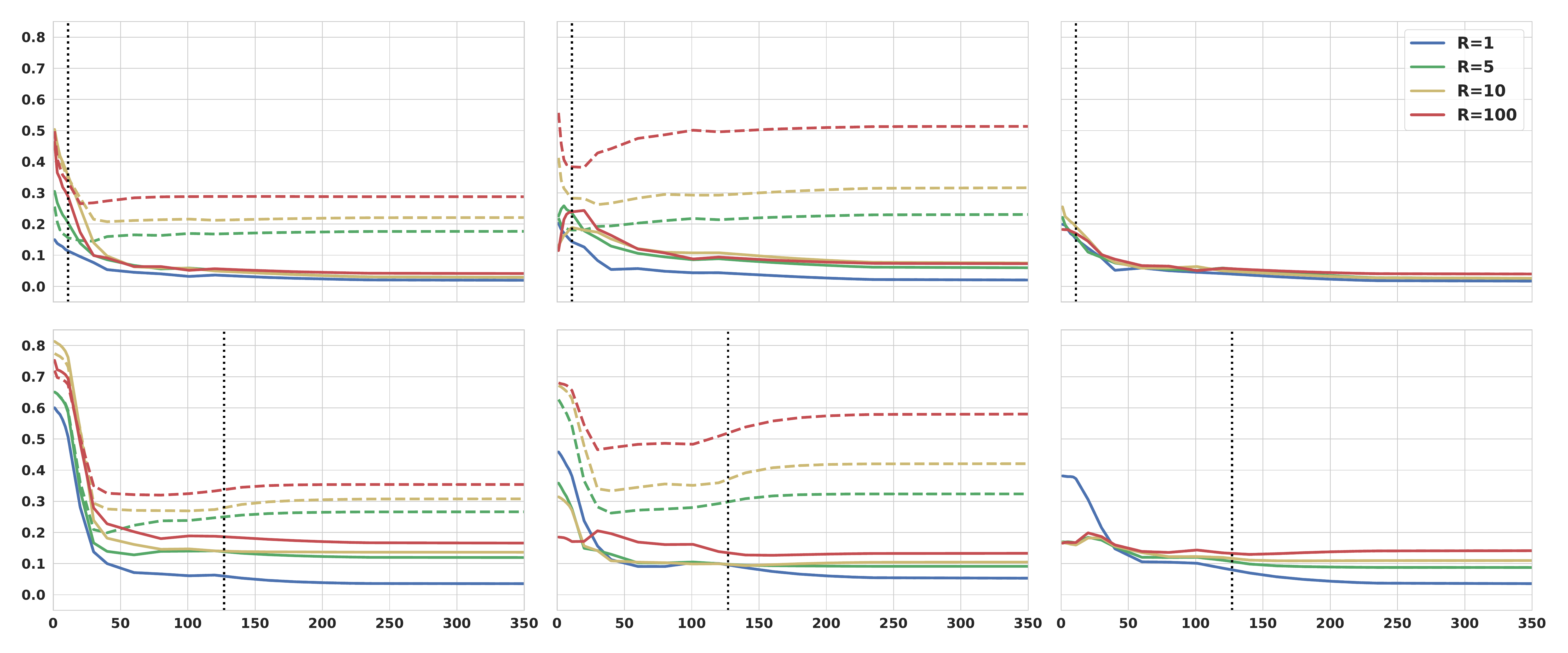}};
			\node at (0,1.7) [scale=0.9]{\textbf{MNIST}};
			\node at (0.25,-4.5) [scale=0.7]{Epochs};
			\node at (-4.3,-4.5) [scale=0.7]{Epochs};
			\node at (4.5,-4.5) [scale=0.7]{Epochs};
			\node at (-7.3,0.0) [scale=0.7, rotate=90]{\textbf{ResNet}};
			\node at (-7.3,-2.7) [scale=0.7, rotate=90]{\textbf{VGG}};
			\node at (-7.7,-1.4) [scale=0.7, rotate=90]{Distance to SELI/ETF};
		\end{tikzpicture}
		\label{fig:GW_minmaj_mnist}
	\end{subfigure}
	\\
	\vspace{-2pt}
	\begin{subfigure}[b]{1.0\textwidth}
	\centering
	\begin{tikzpicture}
		\node at (0,-1.4) {\includegraphics[width=0.9\textwidth]{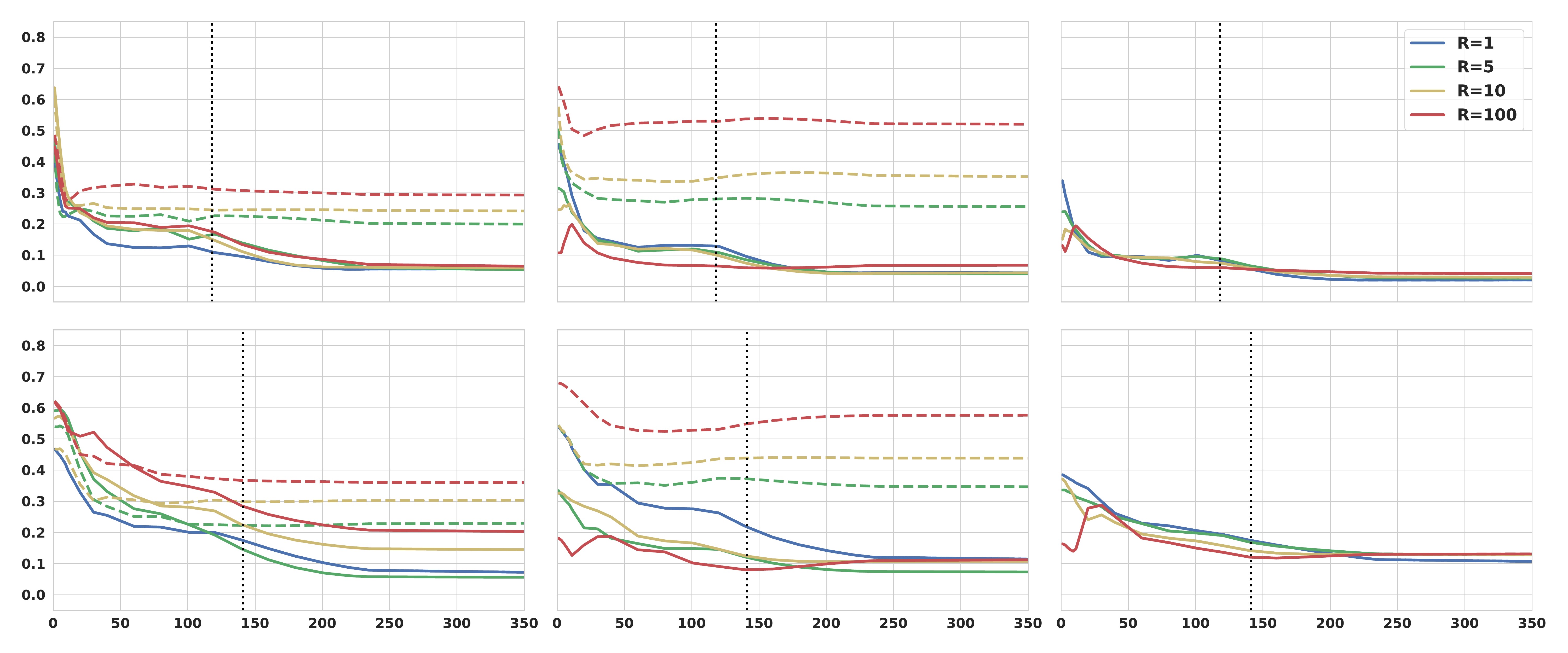}};
		\node at (0,1.7) [scale=0.9]{\new{\textbf{Fashion-MNIST}}};
		\node at (0.25,-4.5) [scale=0.7]{Epochs};
		\node at (-4.3,-4.5) [scale=0.7]{Epochs};
		\node at (4.5,-4.5) [scale=0.7]{Epochs};
		\node at (-7.3,0.0) [scale=0.7, rotate=90]{\textbf{ResNet}};
		\node at (-7.3,-2.7) [scale=0.7, rotate=90]{\textbf{VGG}};
		\node at (-7.7,-1.4) [scale=0.7, rotate=90]{Distance to SELI/ETF};
		\node at (0.25,-5.2) [scale=0.9]{\textbf{(b)} $\underline{\G_\W}^{\minor-\minor}$};
		\node at (-4.3,-5.2) [scale=0.9]{\textbf{(a)} $\underline{\G_\W}^{\maj-\maj}$};
		\node at (4.5,-5.2) [scale=0.9]{\textbf{(c)} $\underline{\G_\W}^{\maj-\minor}$};
	\end{tikzpicture}
	\label{fig:GW_minmaj_fmnist}
	\end{subfigure}
	\vspace{-20pt}
	\caption{Convergence of learned classifiers' majority and minority individual geometries to the SELI (solid lines) vs ETF (dashed lines) geometries. 
	}
	\label{fig:appx_GW_MNIST}
\end{figure}
\begin{figure}
	\centering
	\begin{subfigure}[b]{1.0\textwidth}
		\centering
		\begin{tikzpicture}
			\node at (0,-1.4) {\includegraphics[width=0.9\textwidth]{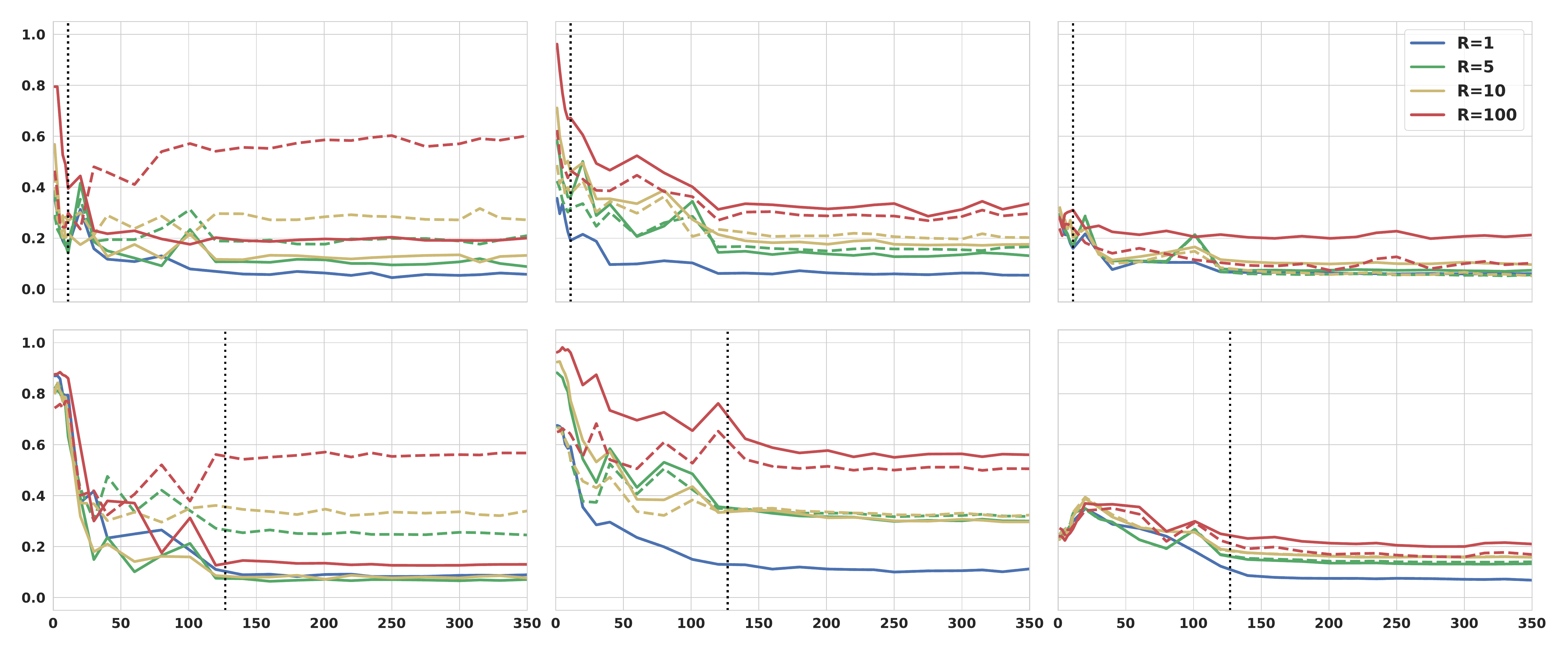}};
			\node at (0,1.7) [scale=0.9]{\textbf{CIFAR10}};
			\node at (0.25,-4.5) [scale=0.7]{Epochs};
			\node at (-4.3,-4.5) [scale=0.7]{Epochs};
			\node at (4.5,-4.5) [scale=0.7]{Epochs};
			\node at (-7.3,0.0) [scale=0.7, rotate=90]{\textbf{ResNet}};
			\node at (-7.3,-2.7) [scale=0.7, rotate=90]{\textbf{VGG}};
			\node at (-7.7,-1.4) [scale=0.7, rotate=90]{Distance to SELI/ETF};
		\end{tikzpicture}
		\label{fig:GH_minmaj_cifar10}
	\end{subfigure}
	\\
	\vspace{-2pt}
	\begin{subfigure}[b]{1.0\textwidth}
		\centering
		\begin{tikzpicture}
			\node at (0,-1.4) {\includegraphics[width=0.9\textwidth]{./figs/PAPER_H_MNIST_GH_fro_minmaj.pdf}};
			\node at (0,1.7) [scale=0.9]{\textbf{MNIST}};
			\node at (0.25,-4.5) [scale=0.7]{Epochs};
			\node at (-4.3,-4.5) [scale=0.7]{Epochs};
			\node at (4.5,-4.5) [scale=0.7]{Epochs};
			\node at (-7.3,0.0) [scale=0.7, rotate=90]{\textbf{ResNet}};
			\node at (-7.3,-2.7) [scale=0.7, rotate=90]{\textbf{VGG}};
			\node at (-7.7,-1.4) [scale=0.7, rotate=90]{Distance to SELI/ETF};
		\end{tikzpicture}
		\label{fig:GH_minmaj_mnist}
	\end{subfigure}
	\\
	\vspace{-2pt}
	\begin{subfigure}[b]{1.0\textwidth}
		\centering
		\begin{tikzpicture}
			\node at (0,-1.4) {\includegraphics[width=0.9\textwidth]{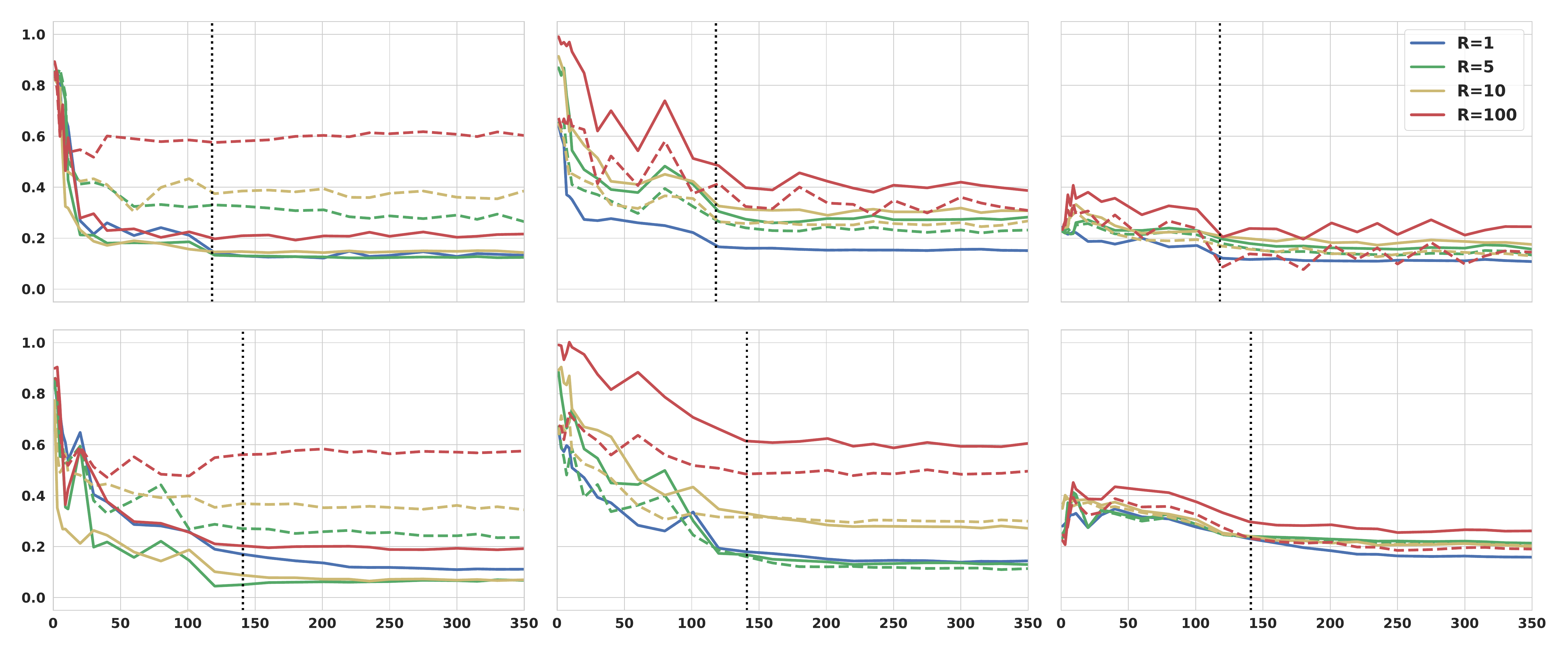}};
			\node at (0,1.7) [scale=0.9]{\new{\textbf{Fashion-MNIST}}};
			\node at (0.25,-4.5) [scale=0.7]{Epochs};
			\node at (-4.3,-4.5) [scale=0.7]{Epochs};
			\node at (4.5,-4.5) [scale=0.7]{Epochs};
			\node at (-7.3,0.0) [scale=0.7, rotate=90]{\textbf{ResNet}};
			\node at (-7.3,-2.7) [scale=0.7, rotate=90]{\textbf{VGG}};
			\node at (-7.7,-1.4) [scale=0.7, rotate=90]{Distance to SELI/ETF};
			\node at (0.25,-5.2) [scale=0.9]{\textbf{(b)} $\underline{\G_\M}^{\minor-\minor}$};
			\node at (-4.3,-5.2) [scale=0.9]{\textbf{(a)} $\underline{\G_\M}^{\maj-\maj}$};
			\node at (4.5,-5.2) [scale=0.9]{\textbf{(c)} $\underline{\G_\M}^{\maj-\minor}$};
		\end{tikzpicture}
		\label{fig:GH_minmaj_fmnist}
	\end{subfigure}
	\vspace{-20pt}
	\caption{Convergence of learned classifiers' majority and minority individual geometries to the SELI (solid lines) vs ETF (dashed lines) geometries. 
	}
	\label{fig:appx_GH_MNIST}
\end{figure}
In order to better understand the individual behavior of majorities and minorities, we now compare individual quadrants of the (normalized) $\underline{\G_\W}$ and $\underline{\G_\Hb}$ matrices. Concretely, let
$$
	\underline{\G_\W} = 
	\begin{bmatrix}
		\underline{\G_\W}^{\maj-\maj} & \underline{\G_\W}^{\maj-\minor} \\
		\big(\underline{\G_\W}^{\maj-\minor}\big)^T & \underline{\G_\W}^{\minor-\minor}
	\end{bmatrix}
$$
be a partition of the normalized $\underline{\G_\W}=\G_\W/\|\G_\W\|_F$ to $(k/2)\times (k/2)$ sub-blocks.  Comparing quadrants $\underline{\G_\W}^{\maj-\maj},  \underline{\G_\W}^{\maj-\minor} $ and $\underline{\G_\W}^{\minor-\minor}$ to the corresponding quadrants of the reference SELI/ETF matrix $\hat{\G}_\W$ allows us to ``zoom-in'' the majority-majority, majority-minority and minority-minority structures. Entirely analogous calculations allow the same for the embeddings.

\vspace{3pt}
\noindent\textbf{Classifiers Geometry.}~
\Fig~\ref{fig:appx_GW_MNIST} confirms that both majority and minority geometries converge to SELI properly. Interestingly, we see that the minorities diverge the most from the equiangular structure of ETF geometry.

\vspace{3pt}
\noindent\textbf{Embeddings Geometry.}~We find thanks to \Fig~\ref{fig:appx_GH_MNIST}  that the ``error'' in convergence of embeddings to SELI geometry (compare to the better convergence for classifiers) shown previously in \Fig~\ref{fig:intro_CIFAR} and \ref{fig:vgg_} is primarily due to the minority class geometries. However, an overall inspection of the subfigures shows that embeddings also tend to align better to SELI compared to ETF. This alignment property is pronounced in the case of majority geometries (see $\underline{\G_\M}^{\maj-\maj}$).


\subsection{Capturing weight-decay with regularized UFM}\label{sec:exp_wd_ufm}

\begin{figure*}[t!]
	\vspace{-10pt}
	\centering
	\hspace{-40pt} 
	\begin{subfigure}{0.3\textwidth}
		\centering
		\begin{tikzpicture}
			\node at (-1.4,-1.4) 
			{\includegraphics[scale=0.18]{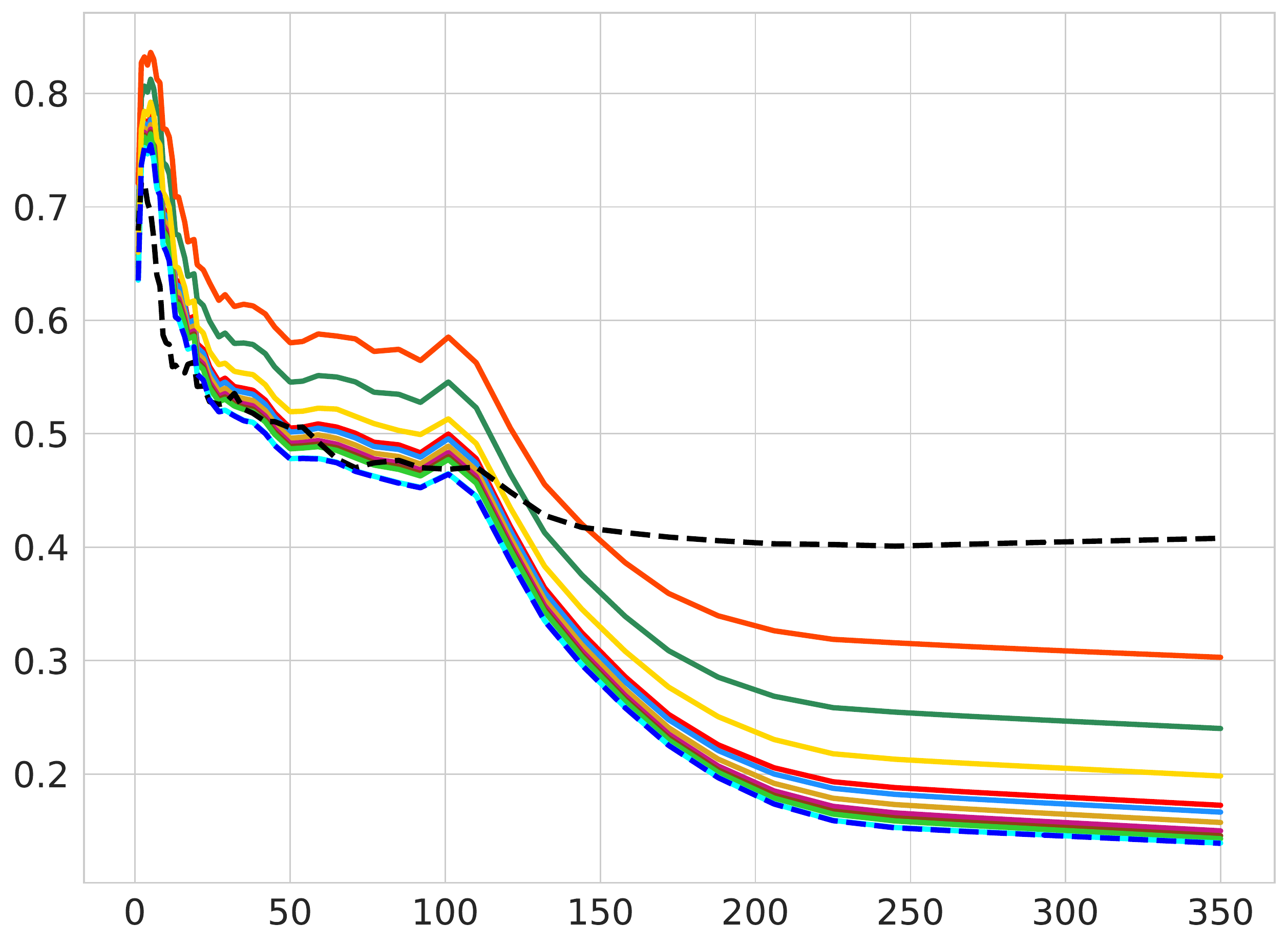}};
			 			\node at (-1.4,-3.5) [scale=0.7]{Epochs};
			\node at (-4.1,-1.4)  [scale=0.7, rotate=90]{Distance to $\lambda$ -SELI/SELI/ETF};
		\end{tikzpicture}\caption{Classifiers}\label{fig:cifar_ridge_ufm_W}
	\end{subfigure}\hspace{25pt}
	\begin{subfigure}{0.3\textwidth}
		\centering
		\begin{tikzpicture}
			\node at (0,-1.4) {\includegraphics[scale=0.18]{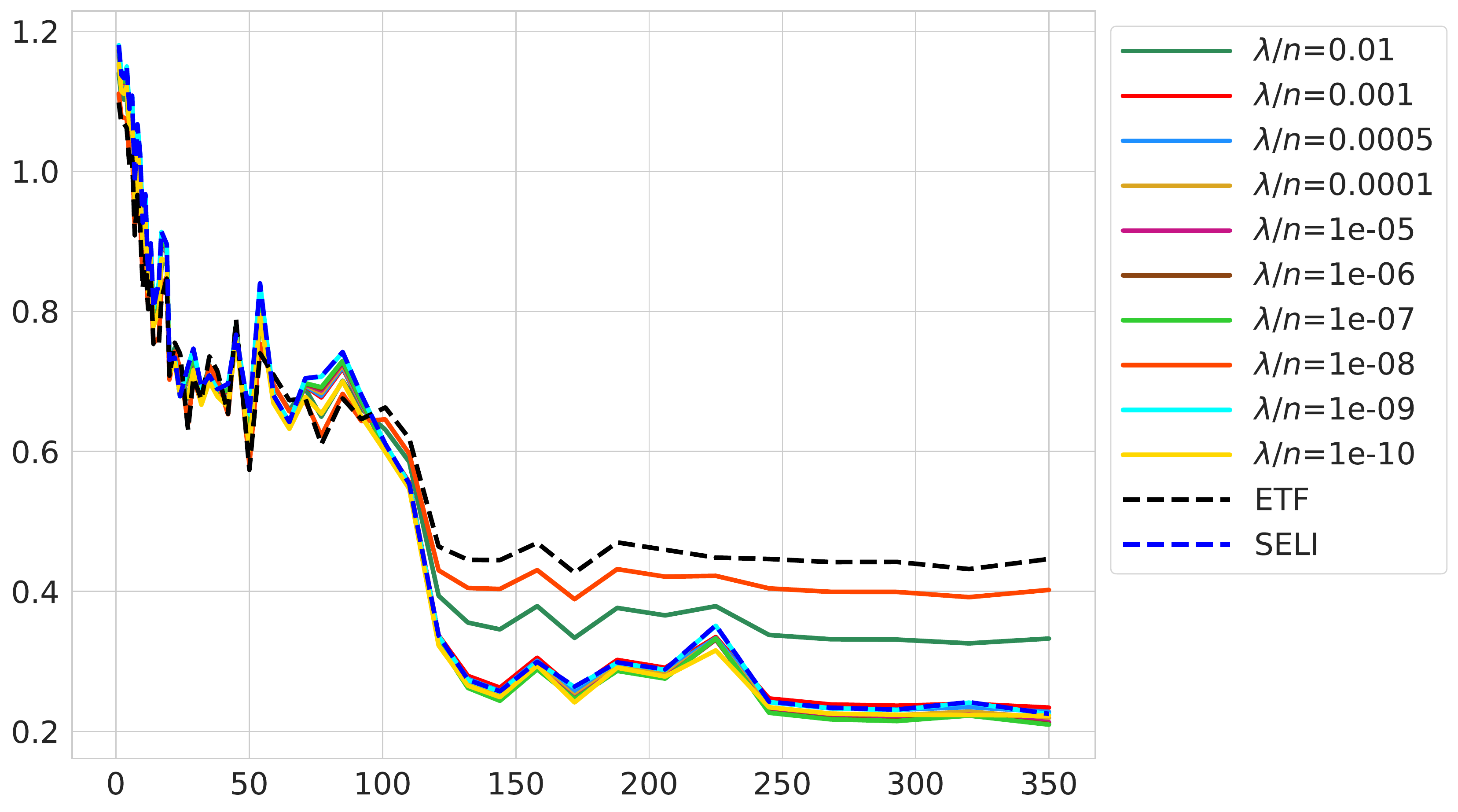}};
			 			\node at (0,-3.5) [scale=0.7]{Epochs};
			\node at (-2.7,-1.4) [scale=0.7, rotate=90]{};
		\end{tikzpicture}\caption{Embeddings}\label{fig:cifar_ridge_ufm_H}
	\end{subfigure}
\caption{Distances of ResNet learned classifiers and mean-embeddings trained on CIFAR10 data with imbalance ratio $R=10$ from the $\lambda$-SELI geometry (aka solution to \eqref{eq:nuc_norm_reg} as defined in Theorem \ref{thm:regularized}) for different values of $\lambda$, as well as, from the SELI  and ETF geometries.}
\label{fig:cifar_ridge_ufm}
\vspace{-5pt}
\end{figure*}
In \Fig~\ref{fig:cifar_ridge_ufm} we compare for $R=10$ the distances of $\underline{\G_\W}$ and $\underline{\G_\M}$ matrices to various $\lambda$-SELI geometries for different values of regularization $\lambda$. Specifically, for the  $\lambda$-SELI geometries, we obtained the reference matrices $\hat{\G}_\W^\la,\hat{\G}_\M^\la$ as follows. For each value of $\la$, we solve the nuclear-norm CE minimization in \eqref{eq:nuc_norm_reg} to find $\Zhat_\la$ for $k=10$ and $R=10$ (to match the CIFAR10 settings). We then form $\hat{\G}_\W^\la, \hat{\G}_\M^\la$ as described in the first paragraph of \Sec~\ref{sec:rel} only now using $\Zhat_\la$ instead of the SELI $\Zhat$. For comparison, we also plot the distances to the SELI and ETF geometries.

Note that we have experimented with a wide range of values for the regularization $\la$. Among those is the value $5e-4$ that matches with the choice of weight-decay parameters in the deep-net experiments. Recall also from the discussion in \Sec~\ref{Sec:reg} that $\Zhat_\la$ is sensitive to $\lambda$ in this setting where $R=10, k=10.$

 
We make the following interesting observations from \Fig~\ref{fig:cifar_ridge_ufm}. First, note that for classifiers the minimum distance  is that to SELI. The distance of the embeddings' geometry to SELI is also among the lowest ones. Specifically, despite being slightly larger than that from $\lambda$-SELI for a few values of $\lambda$, the difference is very small. This suggests that SELI is indeed a good approximation for the learned geometry even when training with finite weight-decay. Besides, there are two key advantages of the SELI over the $\lambda$-SELI geometry. First, it is unclear what the mapping ought to be (if such a mapping exists) between training-implementation choices (such as weight-decay) and $\la$. Second, even if such a mapping was known, the SELI geometry has the unique advantage of being expressed simply in terms of the (SVD of the) \SEL~matrix. In fact, as we show in \Sec~\ref{sec:SELI_properties} it is possible to get closed-form expressions for the norms and angles describing the geometry. This not only makes calculations much easier, but also it allows further analysis of the properties (e.g. quantifying norm-ratios as in \Fig~\ref{fig:norm_ratio_exp_W}).
%


\new{\subsection{Additional experiments for minority ratios $\rho\neq 1/2$}\label{sec:rho}} 
\new{
Thus far, in our previous experiments we considered imbalanced data with minority ratio $\rho=1/2$, i.e. same number of minorities and majorities. However, note that our theoretical results hold for any value of $\rho$. Specifically, Theorem \ref{thm:SVM} shows that the solution of the UF-SVM follows the SELI geometry irrespective of $\rho$. Here, we empirically study convergence to the SELI geometry for a ResNet-18 model on ($R$,$\rho$)-STEP imbalanced CIFAR10 data, for two values of minority ratio $\rho=0.3$ and $0.7$. These experiments complement our previous demonstrations for $\rho=0.5$. Specifically, we create training set of the same size of $n=15350$ in all experiments for a fair comparison.\footnote{This is smaller than the total number $n=25250$ of examples used previously in our experiments for $\rho=1/2$. The reason is that there is a limited number of $5000$ images per class in CIFAR-10 making it impossible to have $(R=100,\rho=0.7)$-STEP imbalanced CIFAR-10 data.}  All  other experimental settings are as described in Section \ref{subsec:app_implementation_details}. Figs. \ref{fig:rho_gram}, \ref{fig:rho_minmaj_w} and \ref{fig:rho_minmaj_m} demonstrate how classifiers and embeddings converge to the proposed SELI geometry. Consistent with the previous experiments, the learnt geometries of deep networks are well-captured by the SELI. We also observe empirically that larger values of $\rho$ exhibit slightly faster convergence.
}
\vspace{5pt}
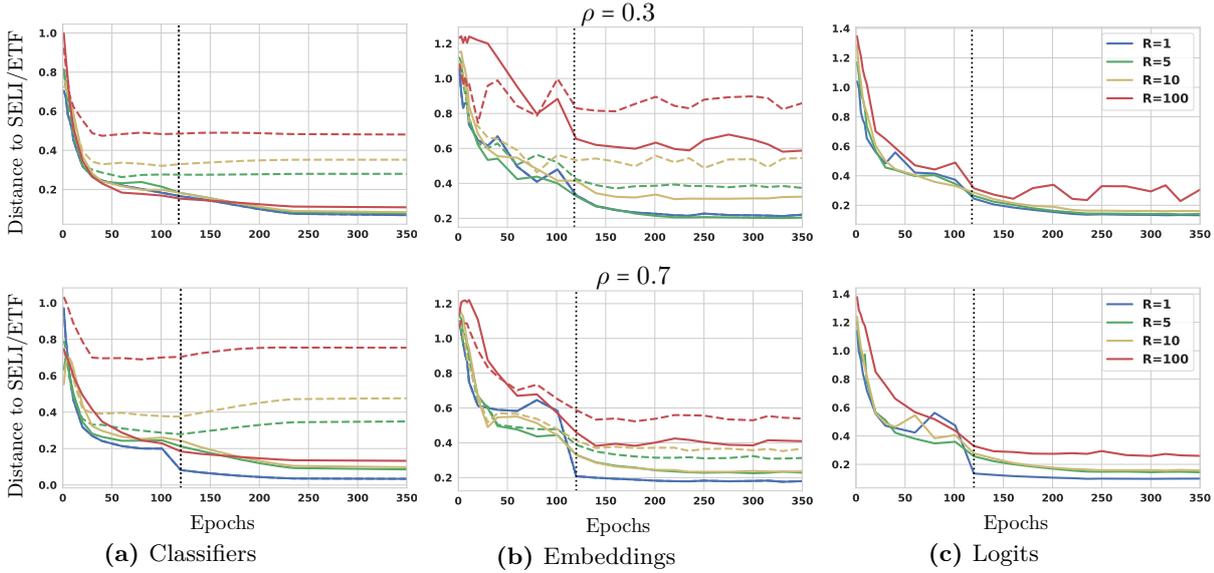
\begin{figure*}[h!]
	\vspace{-10pt}
	\centering
	\hspace{-40pt} \begin{subfigure}{0.3\textwidth}
		\centering
		\begin{tikzpicture}
			\node at (-1.4,-1.4) 
			{\includegraphics[scale=0.26]{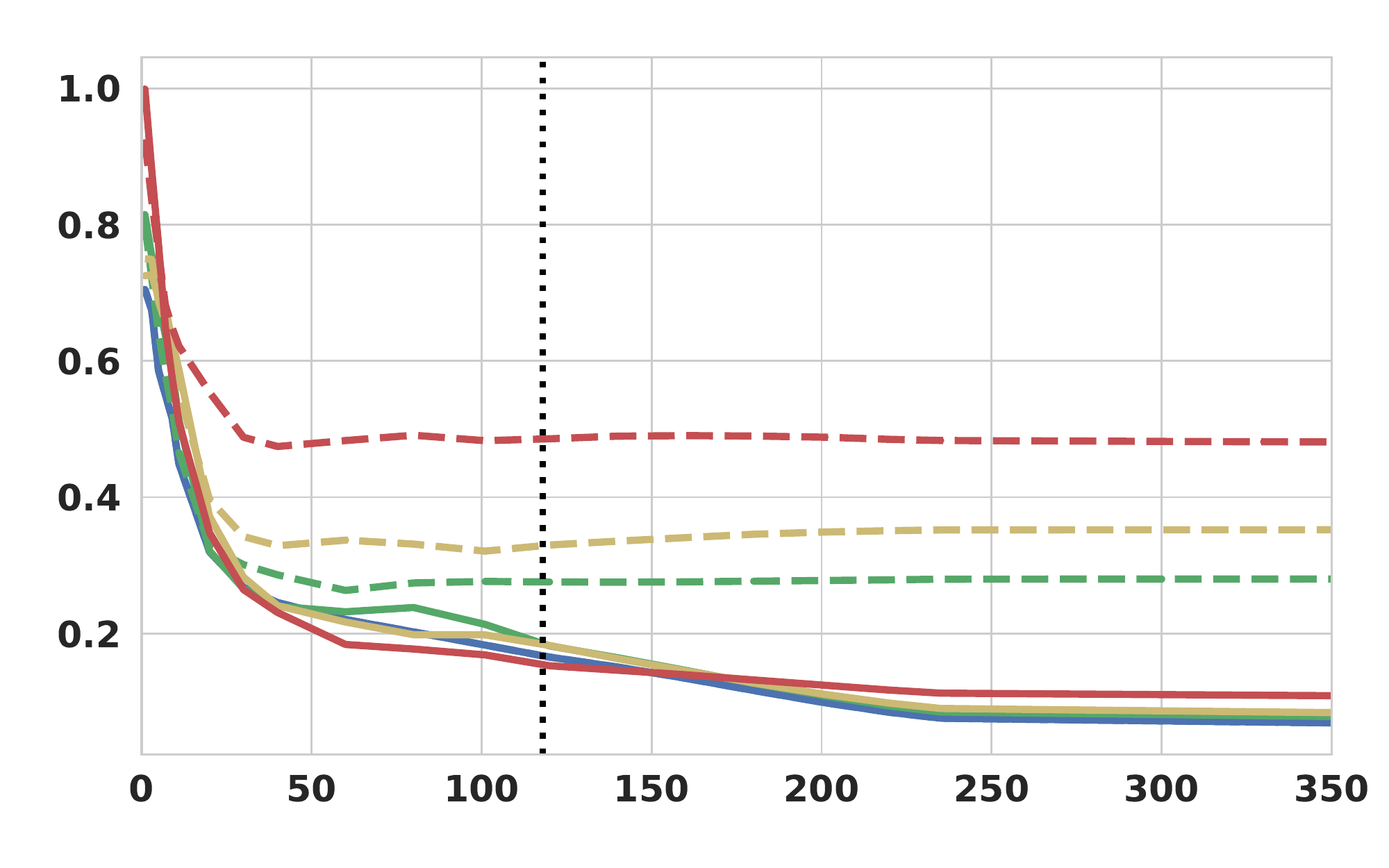}};
			\node at (-4.1,-1.4)  [scale=0.7, rotate=90]{Distance to SELI/ETF};
		\end{tikzpicture}
	\end{subfigure}\hspace{16pt}\begin{subfigure}{0.3\textwidth}
		\centering
		\begin{tikzpicture}
			\node at (0,-1.4) {\includegraphics[scale=0.26]{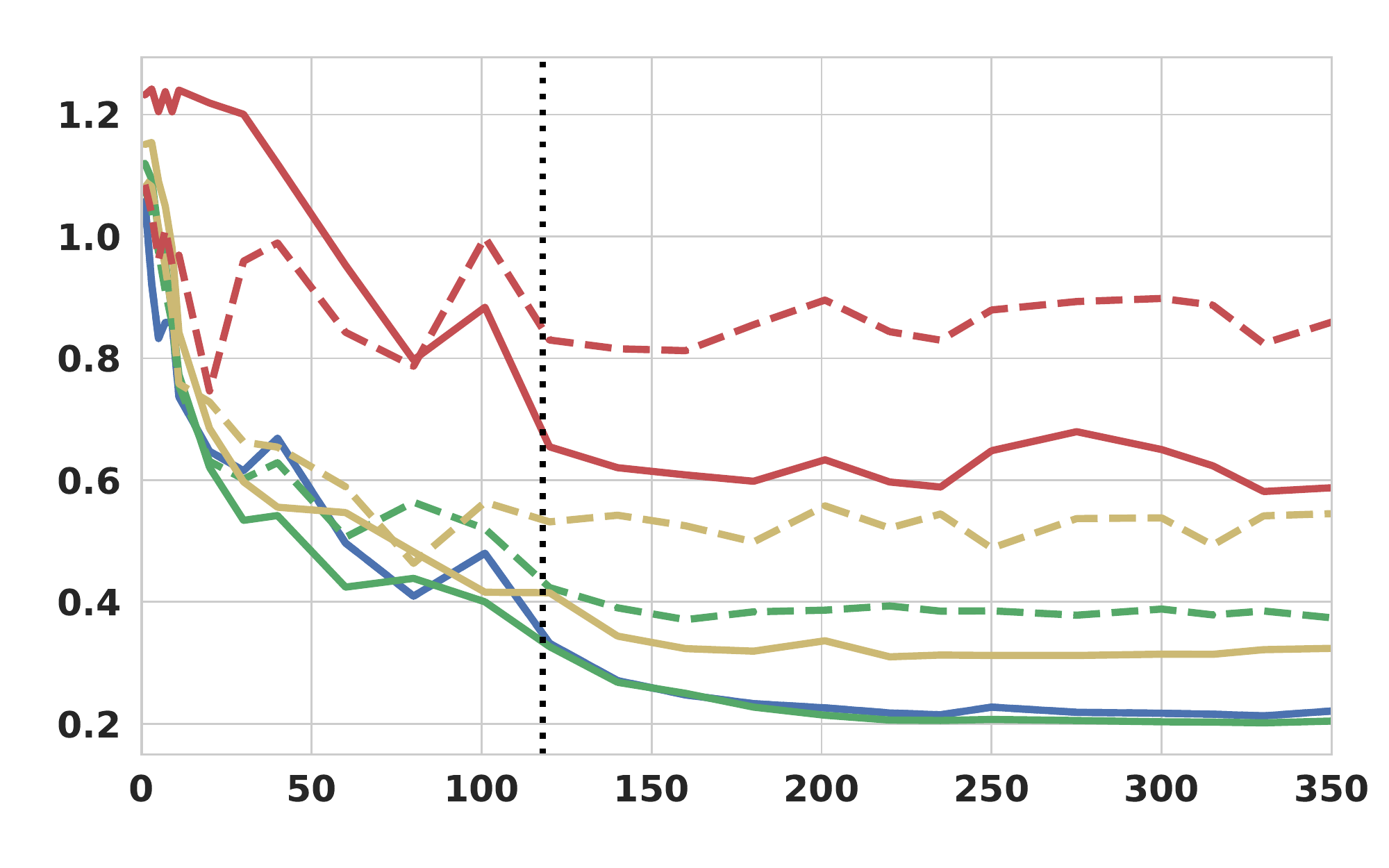}};
			\node at (0,0.2) [scale=0.9]{\textbf{$\rho = 0.3$}};
			\node at (-2.7,-1.4) [scale=0.7, rotate=90]{};
		\end{tikzpicture}
	\end{subfigure}\hspace{13pt}\begin{subfigure}{0.3\textwidth}
		\centering
		\begin{tikzpicture}
			\node at (0,-1.4) {\includegraphics[scale=0.26]{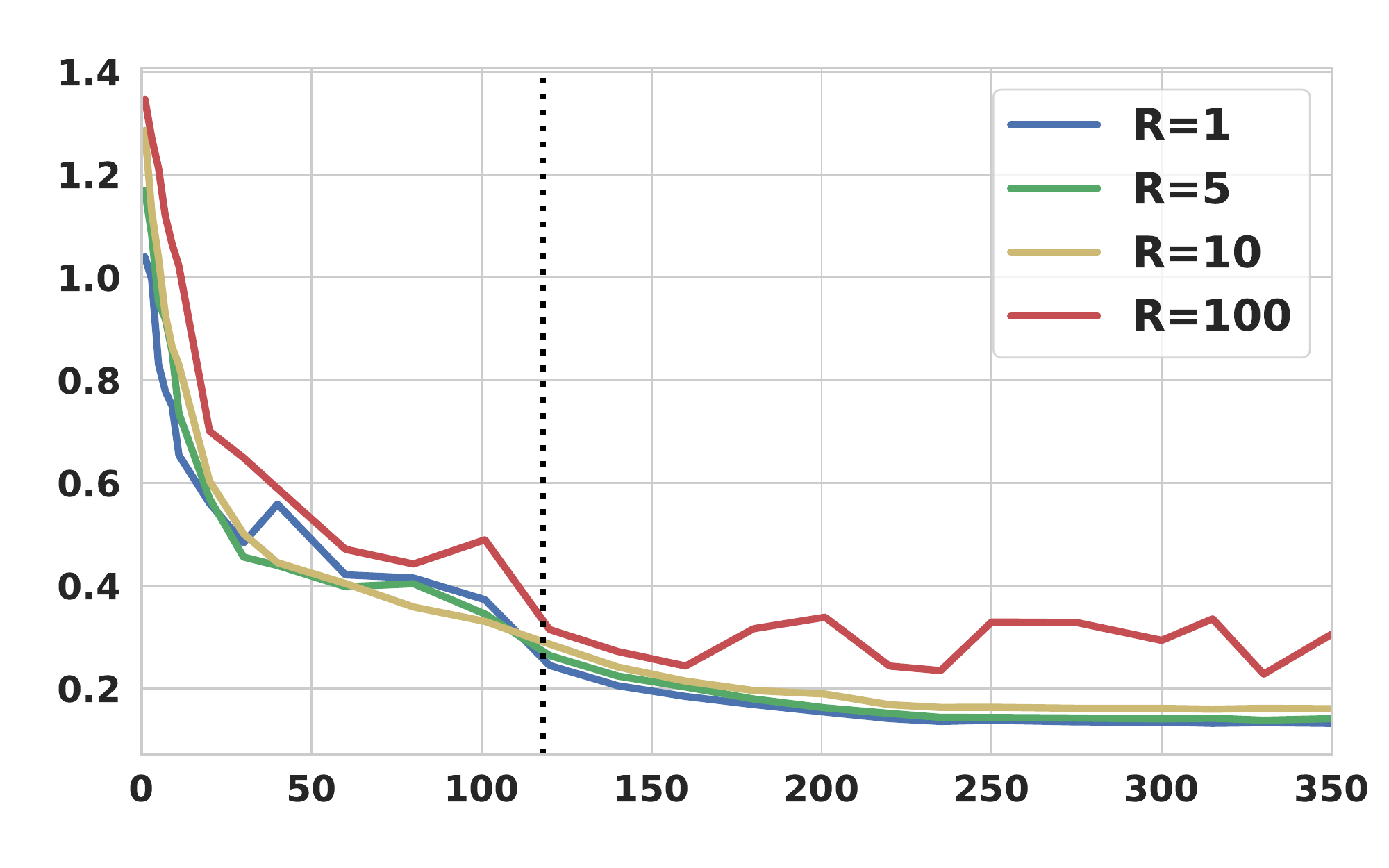}};
			\node at (-2.7,-1.4) [scale=0.7, rotate=90]{};
		\end{tikzpicture}
	\end{subfigure}\vspace{-5pt}
	\vspace{-5pt}
	
	\vspace{5pt}
	\centering
	\hspace{-40pt} \begin{subfigure}{0.3\textwidth}
		\centering
		\begin{tikzpicture}
			\node at (-1.4,-1.4)
			{\includegraphics[scale=0.26]{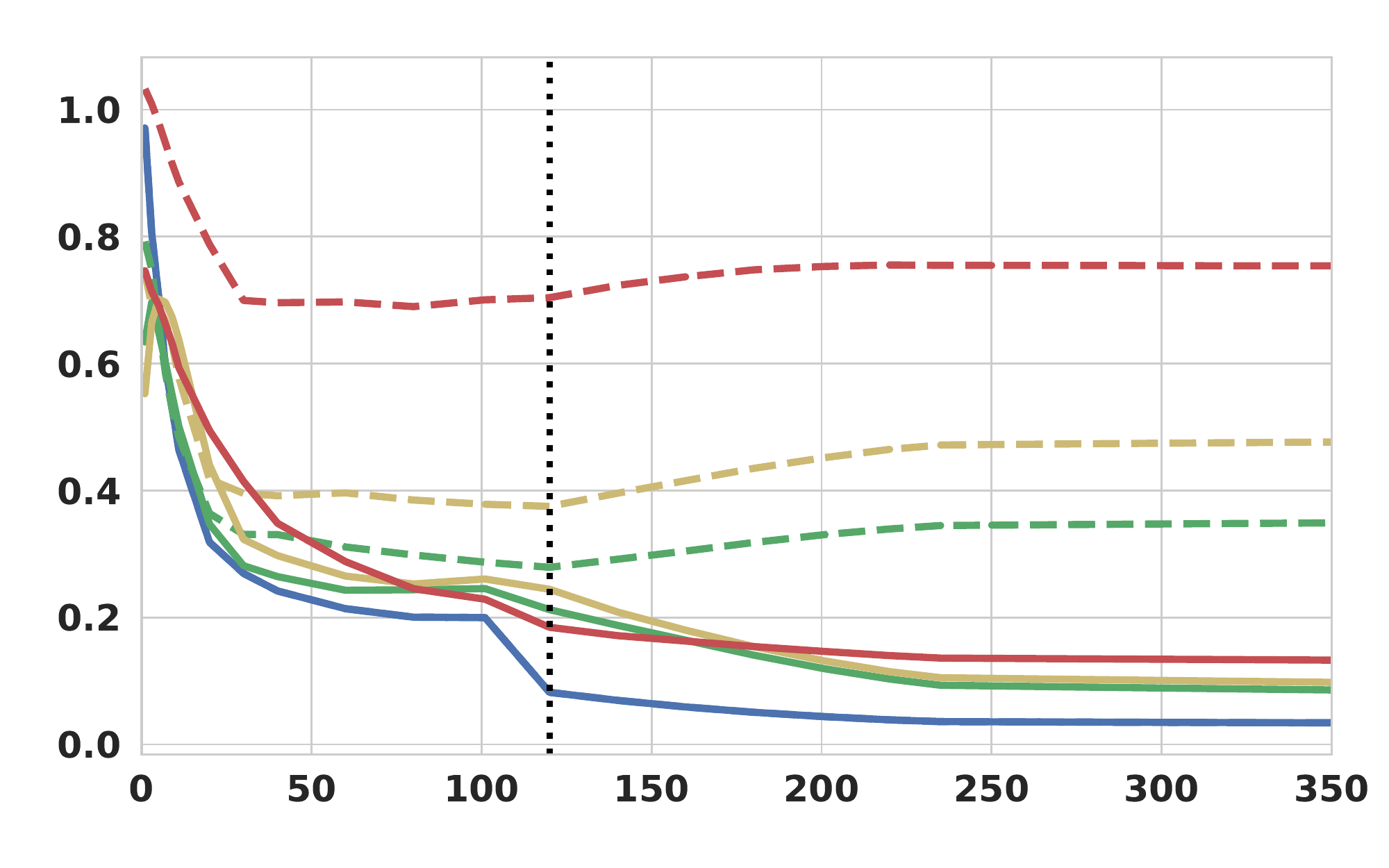}};
			\node at (-1.4,-3.1) [scale=0.7]{Epochs};
			\node at (-4.1,-1.4)  [scale=0.7, rotate=90]{Distance to SELI/ETF};
		\end{tikzpicture}\vspace{-0.2cm}\caption{Classifiers}
	\end{subfigure}\hspace{16pt}\begin{subfigure}{0.3\textwidth}
		\centering
		\begin{tikzpicture}
			\node at (0,-1.4) {\includegraphics[scale=0.26]{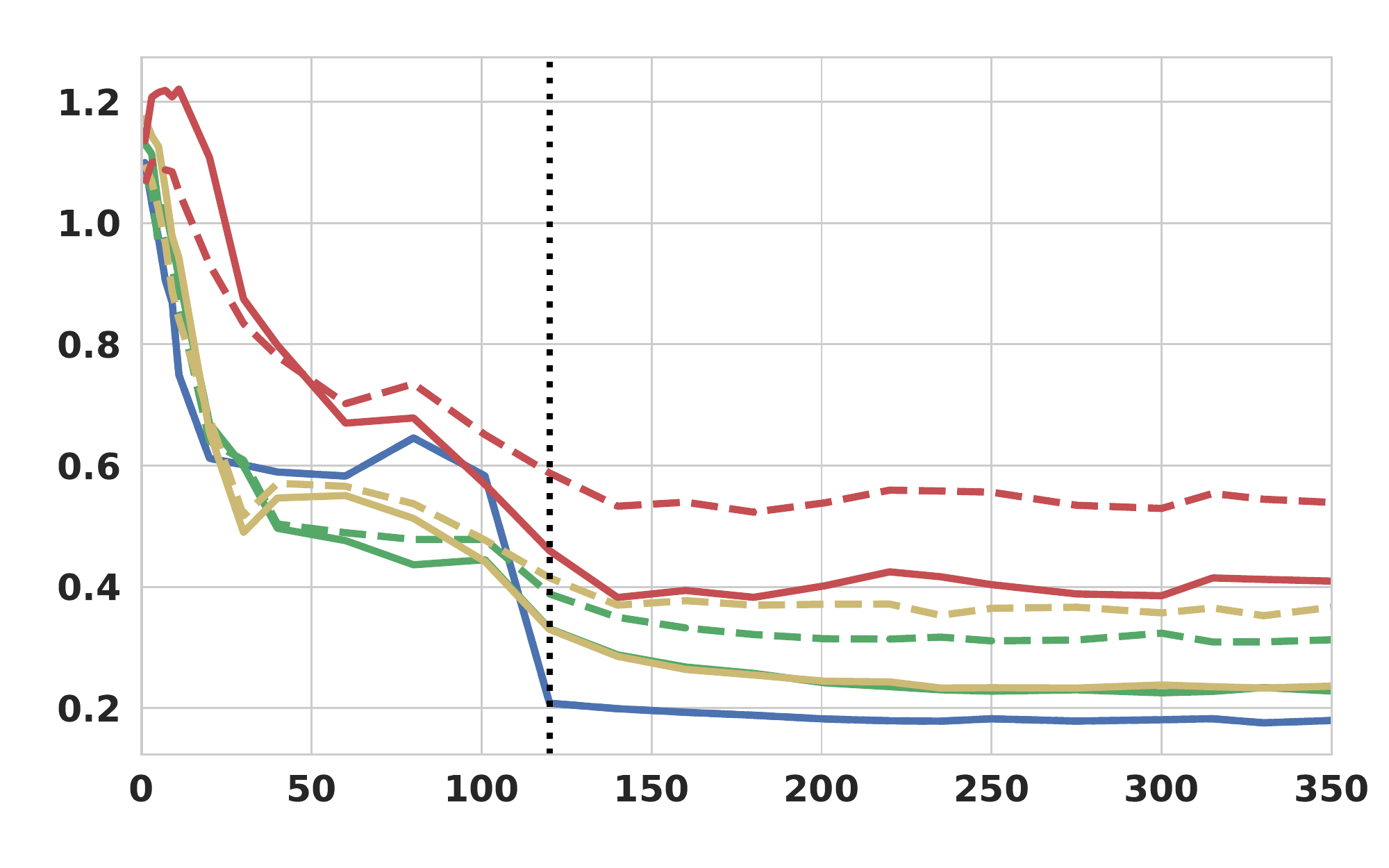}};
			\node at (0.2,0.2) [scale=0.9]{\textbf{$\rho=0.7$}};
			\node at (0,-3.1) [scale=0.7]{Epochs};
			\node at (-2.7,-1.4) [scale=0.7, rotate=90]{};
		\end{tikzpicture}\vspace{-0.2cm}\caption{Embeddings}
	\end{subfigure}\hspace{13pt}\begin{subfigure}{0.3\textwidth}
		\centering
		\begin{tikzpicture}
			\node at (0,-1.4) {\includegraphics[scale=0.26]{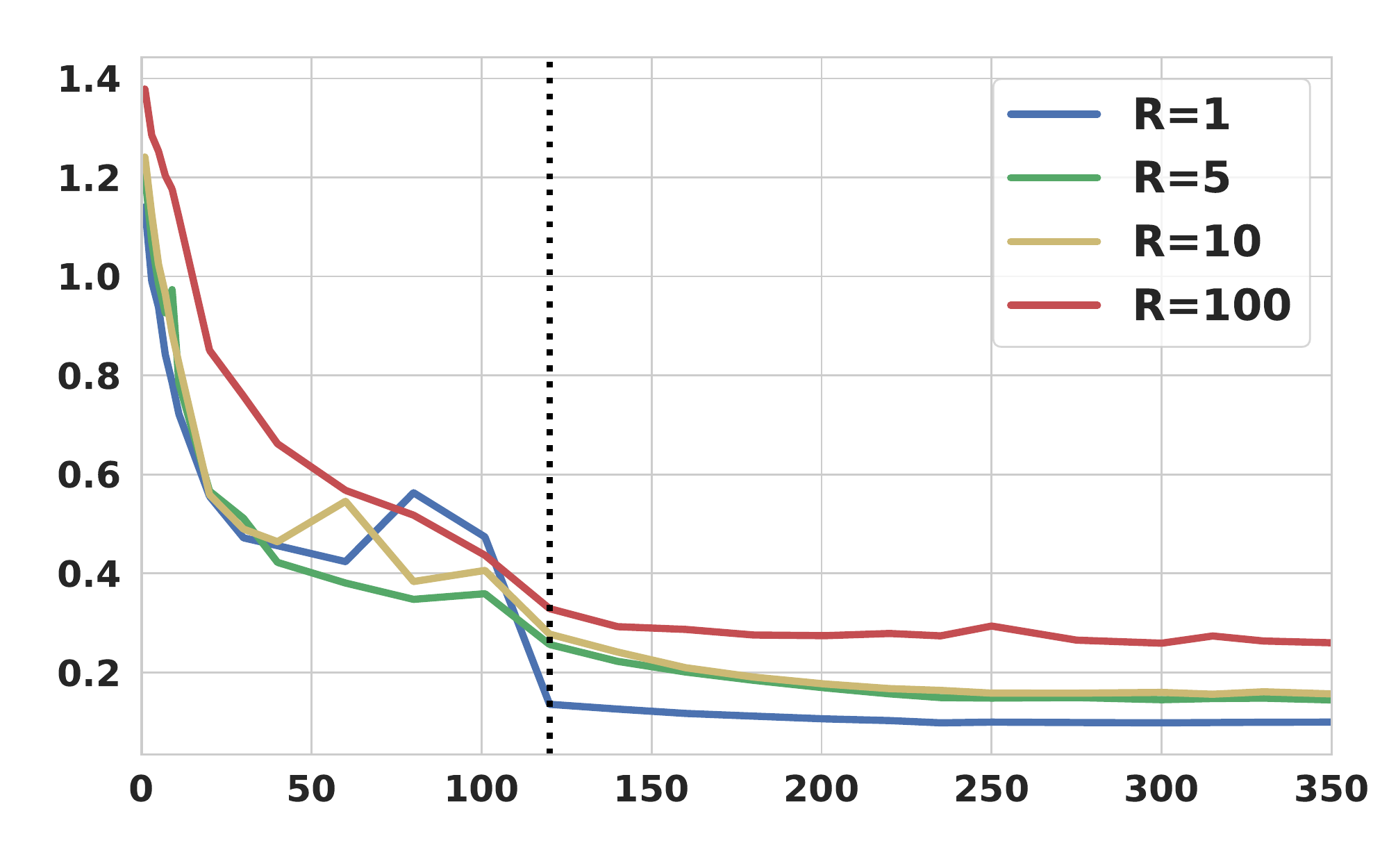}};
			\node at (0,-3.1) [scale=0.7]{Epochs};
			\node at (-2.7,-1.4) [scale=0.7, rotate=90]{};
		\end{tikzpicture}\vspace{-0.2cm}\caption{Logits}
	\end{subfigure}
	\vspace{-1pt}\caption{\new{Impact of $\rho$ on the convergence of ResNet geometry to the SELI, trained on CIFAR10.}
	}
	\label{fig:rho_gram}
	\vspace{-5pt}
\end{figure*}

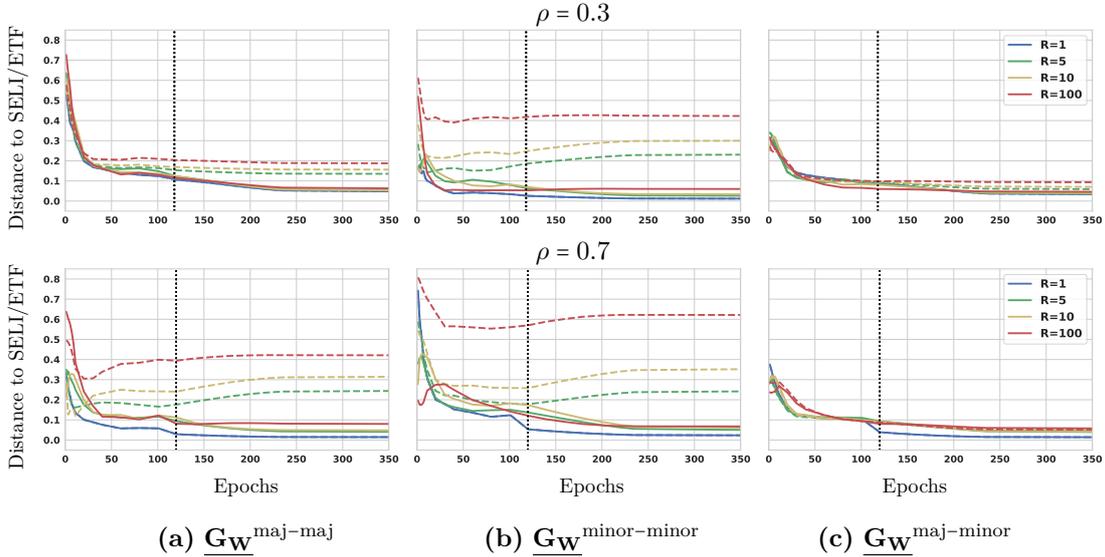
\begin{figure}[h!]
	\centering
	\begin{subfigure}[b]{1.0\textwidth}
		\centering
		\begin{tikzpicture}
			\node at (0,-1.4) {\includegraphics[width=0.9\textwidth]{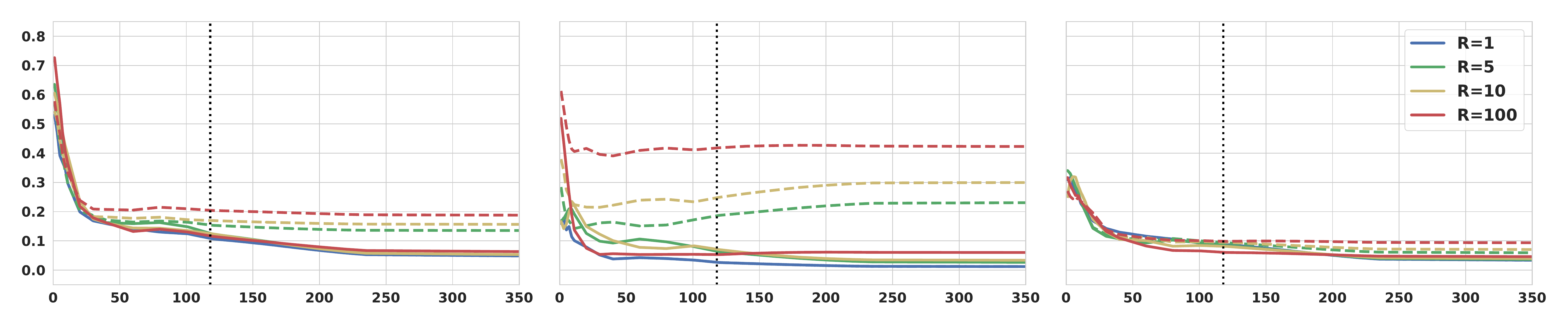}};
			\node at (0,0.1) [scale=0.9]{\textbf{$\rho=0.3$}};
			\node at (-7.3,-1.4) [scale=0.7, rotate=90]{Distance to SELI/ETF};
		\end{tikzpicture}
	\end{subfigure}
	\\
	\vspace{-7pt}
	\begin{subfigure}[b]{1.0\textwidth}
		\centering
		\begin{tikzpicture}
			\node at (0,-1.4) {\includegraphics[width=0.9\textwidth]{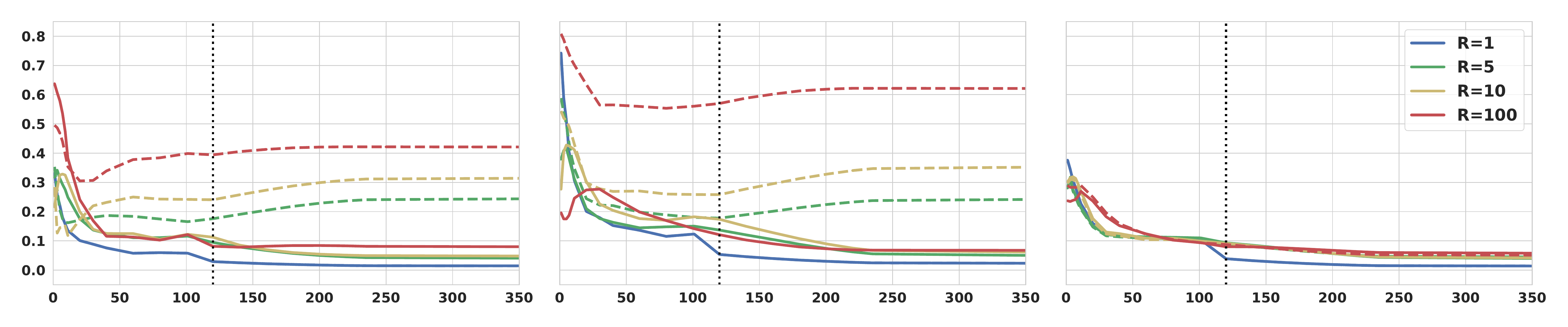}};
			\node at (0,0.1) [scale=0.9]{\textbf{$\rho=0.7$}};
			\node at (0.25,-3.0) [scale=0.7]{Epochs};
			\node at (-4.3,-3.0) [scale=0.7]{Epochs};
			\node at (4.5,-3.0) [scale=0.7]{Epochs};
			\node at (-7.3,-1.4) [scale=0.7, rotate=90]{Distance to SELI/ETF};
			\node at (0.25,-3.7) [scale=0.9]{\textbf{(b)} $\underline{\G_\W}^{\minor-\minor}$};
			\node at (-4.3,-3.7) [scale=0.9]{\textbf{(a)} $\underline{\G_\W}^{\maj-\maj}$};
			\node at (4.5,-3.7) [scale=0.9]{\textbf{(c)} $\underline{\G_\W}^{\maj-\minor}$};
		\end{tikzpicture}
	\end{subfigure}\vspace{-5pt}\caption{\new{Impact of $\rho$ on the convergence of majority/minority classifiers of ResNet model trained on CIFAR10.}
		}
	\label{fig:rho_minmaj_w}
	\vspace{-5pt}
\end{figure}

\begin{figure}[h!]
	\centering
	\begin{subfigure}[b]{1.0\textwidth}
		\centering
		\begin{tikzpicture}
			\node at (0,-1.4) {\includegraphics[width=0.9\textwidth]{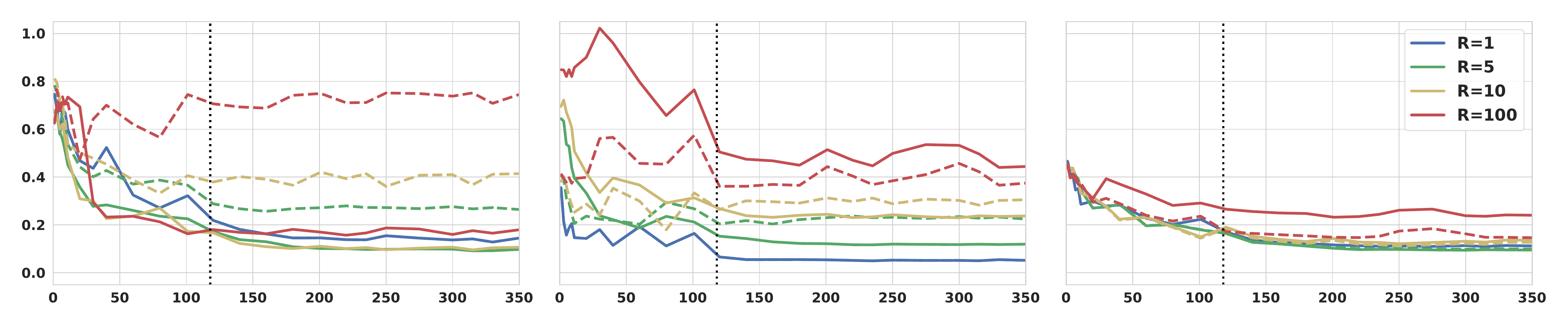}};
			\node at (0,0.1) [scale=0.9]{\textbf{$\rho=0.3$}};
			\node at (-7.3,-1.4) [scale=0.7, rotate=90]{Distance to SELI/ETF};
		\end{tikzpicture}
	\end{subfigure}
	\\
	\vspace{-7pt}
	\begin{subfigure}[b]{1.0\textwidth}
		\centering
		\begin{tikzpicture}
			\node at (0,-1.4) {\includegraphics[width=0.9\textwidth]{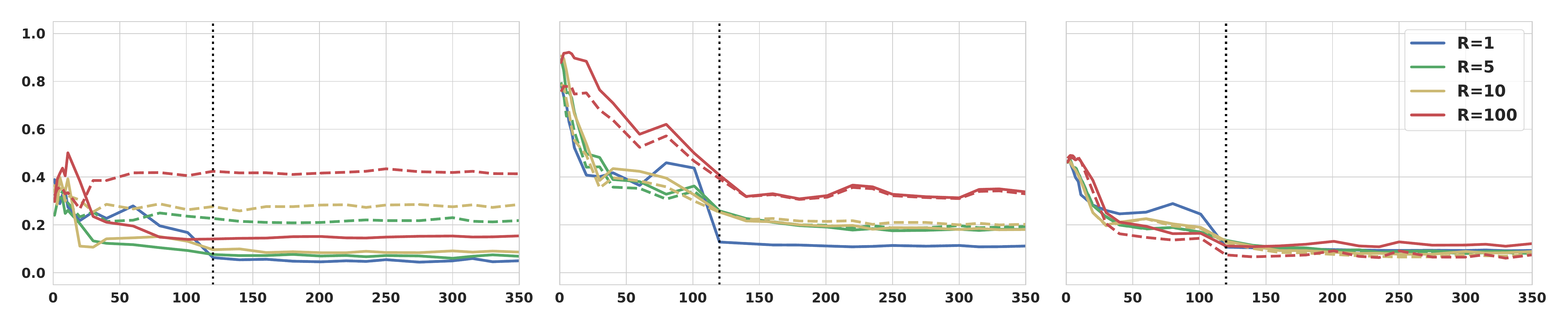}};
			\node at (0,0.1) [scale=0.9]{\textbf{$\rho=0.7$}};
			\node at (0.25,-3.0) [scale=0.7]{Epochs};
			\node at (-4.3,-3.0) [scale=0.7]{Epochs};
			\node at (4.5,-3.0) [scale=0.7]{Epochs};
			\node at (-7.3,-1.4) [scale=0.7, rotate=90]{Distance to SELI/ETF};
			\node at (0.25,-3.7) [scale=0.9]{\textbf{(b)} $\underline{\G_\M}^{\minor-\minor}$};
			\node at (-4.3,-3.7) [scale=0.9]{\textbf{(a)} $\underline{\G_\M}^{\maj-\maj}$};
			\node at (4.5,-3.7) [scale=0.9]{\textbf{(c)} $\underline{\G_\M}^{\maj-\minor}$};
		\end{tikzpicture}
	\end{subfigure}\vspace{-5pt}\caption{\new{Impact of $\rho$ on the convergence of majority/minority embeddings of ResNet model trained on CIFAR10.}
	}
	\label{fig:rho_minmaj_m}
	\vspace{-5pt}
\end{figure}
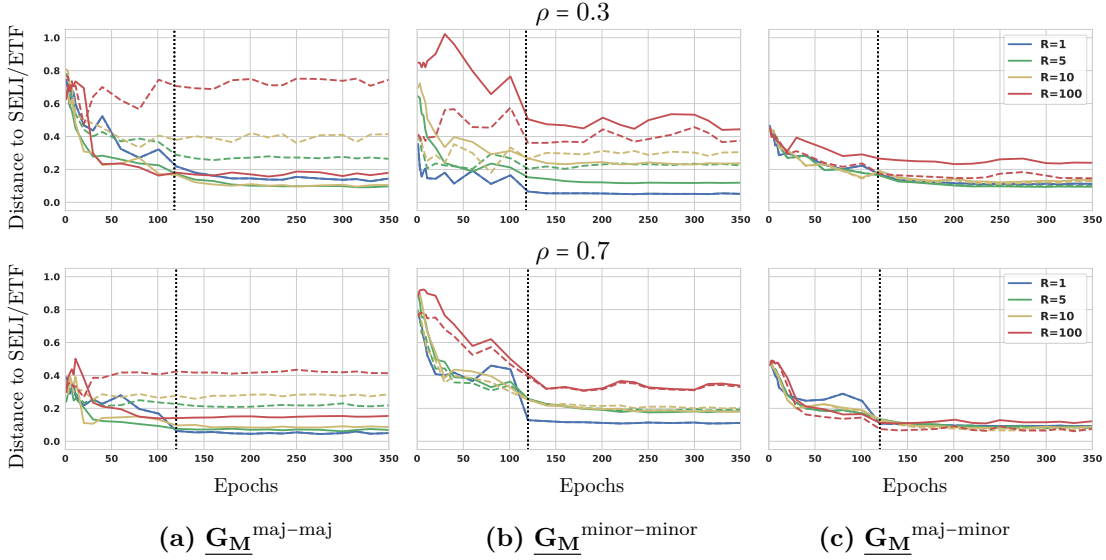

\section{Implications on minority collapse}\label{sec:minority}
In this section, we further elaborate on how our results relate to the \emph{minority collapse} phenomenon, which is defined by \citet{fang2021exploring} as the phenomenon during which \emph{minority classifiers become completely indistinguishable}. Notably, \citet{fang2021exploring} discover its occurrence both in the UFM and in real datasets trained with deep-nets.  Below, we first state their concrete findings and then we discuss how our results extend them. 

\paragraph{Summary of findings by \citet{fang2021exploring}.} 
\citet{fang2021exploring} make the following key findings. 

\begin{enumerate}[label={\textbf{FHLS(\arabic*)}},itemindent=5em]

\item\label{FHLS1} In \cite[Theorem 5]{fang2021exploring} they prove for a constrained UFM, that $\lim_{R\rightarrow \infty}\Cos{\wmin}{\wmin'}=1$ (here, $\wmin,\wmin'$ are any two distinct minority classifiers.) 

\item\label{FHLS2} They also find numerically that the solution to the same constrained UFM gives $\Cos{\wmin}{\wmin'}=1$ for any $R>R_0$ for some \emph{finite} threshold $R_0$. 

\item\label{FHLS3} Their experiments, specifically \cite[\Fig~3]{fang2021exploring}, suggest that for fixed number of classes $k$, the value of the threshold $R_0$ increases as the constraint parameter gets relaxed and also as the ratio $\rho$ of minority classes increases.

\item\label{FHLS4} Finally, they validate the minority collapse phenomenon on real imbalanced datasets trained with deep-nets. Their real-data experiments (e.g.  \cite[\Fig~2,4]{fang2021exploring}) suggest the following. 
\begin{enumerate}
\item Consistently, as $R$ increases, the cosine similarity between minority classes increases until it reaches one.
\item The value of $R$ after which the cosine becomes one (i.e. minority collapse is reached) depends critically on the weight-decay. For small weight decay ($\sim5e-4$), it takes $R>1000$ to reach minority collapse. It is only for larger weight decay ($\sim5e-3$) that minority collapse occurs for $R\sim100$.
\end{enumerate}
\end{enumerate}

\paragraph{Our novelties.}
Before discussing implications of our results for minority collapse, we highlight the following key features of our study. 
\begin{itemize}
\item \textbf{Entire geometry:} We describe the \emph{entire geometry of classifiers and embeddings for both majority and minority classes} (not only the geometry of minority classes.) 

\item \textbf{Finite imbalance levels:} Our geometric characterizations (aka SELI geometry) hold for \emph{all finite values of the imbalance ratio $R$} (not only asymptotically.)

\item \textbf{Vanishing regularization:} We focus on CE training with vanishing regularization. (As such, our geometry characterizations result from analyzing the UF-SVM.) 
\end{itemize}

\paragraph{Contact points: What do our results say about minority collapse?}

\begin{itemize}
\item \textbf{For zero regularization (aka UF-SVM), there is no minority collapse for any finite value of $R$.} Specifically, we show in Lemma \ref{lem:angles_w} that for $(R,1/2)$-STEP imbalance $\Cos{\wmin}{\wmin'}=\nicefrac{(R-7)}{\left(R-7+2k(2+\sqrt{(R+1)/2})\right)}<1.$ This does not contradict Finding \ref{FHLS2} since here we consider zero regularization; all their numerical evaluations are with finite regularization. 

\item \textbf{Minority collapse occurs (only) \emph{asymptotically} in $R$ for the UF-SVM.} Specifically, we show in Corollary \ref{cor:R_inf} that $\lim_{R\rightarrow\infty}\Cos{\wmin}{\wmin'} = 1$; see also \Fig~\ref{fig:SELI_theory_angles_inf}.  This can be seen as a different manifestation of the Finding \ref{FHLS1}, this time for the UF-SVM, instead of  constrained CE minimization.

\item \textbf{There is no minority collapse for small regularization strength when training the UFM with regularized CE minimization.}  Specifically, we show in Lemma \ref{lem:small_la} that for any finite imbalance values $R, \rho$, there is no minority collapse in the solutions of  Eqn. \eqref{eq:CE} when  $\la<1/2.$ Instead, the CE solution perfectly separates the training data. Our result theoretical justifies the numerical Findings \ref{FHLS3} and \ref{FHLS4}(b). To establish the connection to the setting of \citet{fang2021exploring}, we need to normalize the CE loss in \eqref{eq:CE} by a factor of $1/n$ (see \cite[Eqn.~(15)]{fang2021exploring}). For this normalized minimization, Lemma \ref{lem:small_la} ensures \emph{no} minority collapse provided that 
\begin{align}\label{eq:explain_minority}
2\la<\frac{1}{n} = \frac{1}{(\rho+R\rhobar)k}=\frac{1}{(R-\rho(R-1))k}\,.
\end{align}
Put in terms of imbalance ratio, we prove that 
\begin{align}
\text{Minority collapse requires}~~R>f(\la,\rho):=\frac{\frac{1}{2k\la}-\rho}{1-\rho}\,.
\end{align}
It is straightforward to check that $f(\la,\rho)$ is increasing in $\rho$ and decreasing in $\la$. Thus, we show that the minority collapse threshold $R_0$ (see Finding \ref{FHLS3}) increases with the minority ratio $\rho$ and with the inverse regularization parameter $1/\la$. This finding explains the behavior reported empirically in \cite[\Fig~3]{fang2021exploring} for the UFM and in \cite[\Fig~2,4]{fang2021exploring} for real data. \footnote{When referring to \cite[\Fig~3]{fang2021exploring} keep in mind that they simulate constrained, rather than regularized, CE minimization. Hence, larger constraint parameters mean larger \emph{inverse} regularization parameter $1/\la.$}

\item \textbf{The angle between minority classifiers decreases monotonically with the imbalance ratio $R$ for the UF-SVM.} Specifically, it can be checked easily by direct differentiation that the formula in Lemma \ref{lem:angles_w} giving $\Cos{\wmin}{\wmin'}$ is increasing in $R$. This can be seen as a theoretical justification of the empirical Finding \ref{FHLS4}(a).

\end{itemize}



\section{Additional related work}\label{sec:rel2}

As discussed in \Sec~\ref{sec:intro_related}, our work is inspired and is most closely related to the recent literature on  Neural Collapse. In \Sec~\ref{sec:intro_related} we reviewed the most closely related of these works. A few others are referenced in \Sec~\ref{sec:NC_rel_SM} below. Beyond Neural collapse, our results and analysis tools are also related to the literatures on implicit bias, matrix factorization and imbalanced deep-learning. We elaborate on these connections below.

\subsection{Additional works on Neural Collapse}\label{sec:NC_rel_SM}
Beyond CE minimization, a series of recent works study and analyze the neural collapse phenomenon when training with square loss. 
Interestingly, \citet{graf2021dissecting,fang2021exploring} discover and analyze  neural collapse for similarity-type losses, such as the self-supervised contrastive loss, which trains only for embeddings. To the best of our knowledge, all these works restrict attention to balanced classification. Our work shows that it is possible to obtain explicit geometric characterizations in class-imbalanced settings when training with CE. Hence, it also opens the way to extending the analysis to square-loss minimization. 

Potential connections of neural collapse to generalization and transferability  is a less well-understood topic. Some initial investigations appear in the recent works \cite{hui2022limitations,galanti2021role,han2021neural}. Our results are not immediately conclusive about generalization. However, as mentioned in \Sec~\ref{sec:out} our results have the potential to offer new perspectives on generalization, since they uncover different geometries (aka SELI for different $R$ values), each leading to different generalization (worse for increasing $R$ \cite{byrd2019effect,TengyuMa}).

\begin{remark}[Last-layer peeled model (LPM)]Around the same time that the UFM was proposed by \citet{mixon2020neural} (see also \citet{lu2020neural,graf2021dissecting}), the same model was independently formulated and analyzed by \citet{fang2021exploring} under the alternative name of ``last-layer peeled model (LPM)". 
\end{remark}

\begin{remark}[\emph{\ref{NC}} and \emph{\ref{ETF}} properties in \cite{NC}]
The formalization of Neural collapse by \citet{NC} involved four NC properties. The first property concerns the collapse of class embeddings to their corresponding means. Properties two and three concern the geometry of class-means and classifier-weights, specifically their alignment and convergence to a simplex ETF geometry. Property four is a consequence of the other three properties, hence is less important in the formalization and we do not discuss it further. 
Motivated by our findings, we propose and use here a regrouping/renaming of the aforementioned three properties. We refer to the first property as the \emph{\ref{NC}} property, and, to the second and third properties as the \emph{\ref{ETF}} property. We argue that this distinction is important towards a formulation that is invariant across class-imbalances, by showing that the ETF property is \emph{not} invariant and replacing it with the \emph{\ref{SELI}}~property. For balanced data, the latter simplifies to the ETF property. 
\end{remark}



\subsection{Implicit bias}
Neural collapse is intimately related to the recent literature on implicit bias, which started from a series of influential works \cite{soudry2018implicit,ji2018risk,gunasekar2018characterizing,ji2019implicit,nacson2019convergence}. (This connection is already recognized by the seminal work of \citet{NC}.) For example,  \cite[Theorem 7, Remark~2]{gunasekar2018characterizing} concerns a bilinear non-convex SVM-type minimization that bears similarities to the UF-SVM and establishes a connection to a convex nuclear-norm minimization problem. However, the two factors in their bilinear formulation are the same (unlike the UF-SVM) and also they restrict attention to binary classification. Another very closely related work is that by \citet{lyu2019gradient} who show that gradient descent on deep homogeneous networks converges (in direction) to a KKT point of a corresponding non-convex max-margin classifier. While their focus is again on binary classification, they briefly discuss extension to multiclass settings in their appendix. Recently, \citet{ULPM} leveraged their results and formally showed that gradient descent on \eqref{eq:CE} with zero regularization converges to a KKT point of the UF-SVM. The max-margin classifiers corresponding to multi-layer linear networks (with the UF-SVM being a special case) are non-convex. Hence there is no guarantee that the KKT point where gradient descent converges to is a global optimum. Whether this is the case or not is investigated for various settings by \citet{vardi2021margin}. Specifically for linear fully-connected networks, which are of interest to us, they show that, when trained on binary data, the point  of convergence is always a global optimum \cite[Theorem  3.1]{vardi2021margin}. Their proof uses another nice result on implicit bias by \citet{ji2020directional}. In fact, their results are more intimately connected to neural collapse as it can be checked that \cite[Proposition~4.4]{ji2020directional} provides a direct proof that gradient descent on unregularized CE for the UFM and binary data finds embeddings and a classifier that satisfy the NC and ETF properties (irrespective of imbalance). Note here that all these works focus almost exclusively on binary settings. A salient message of our results is that rich and possibly complicated behaviors can occur in multiclass ($k>2$) settings. There are several findings supporting this. For example, we show that regularization in the UFM only matters when data are multiclass and imbalanced. Similarly, it is only then that the model found by the two-layer UFM can differ from what a one-layer convex network would find. We also show empirically for the UFM that convergence rates in the absence of regularization can be heavily impacted by imbalances. Related to this, we highlight a missing piece in our analysis: we characterize the global optimum of the UF-SVM, but we do not prove, or are aware of a proof, that gradient descent converges to this global optimum. Proving or disproving this can be of great interest in its own way. On the one hand, if the conjecture holds, then our results warn that imbalance levels can severely impact convergence rates. On the other hand, if the conjecture is refuted, then this would be the simplest model to have been discovered where convergence to global optimum fails.

%


\subsection{Matrix factorization and low-rank recovery}
The connection between the study of the UFM for the purpose of neural collapse analysis and the literatures on matrix factorization and low-rank matrix recovery (see for example \cite{wright2022high}) is uncovered and first exploited by \citet{zhu2021geometric} and \citet{fang2021exploring}. Thus, we refer the interested reader to those papers for a list of references and detailed discussion (specifically see \citet[Sec.~3.2]{zhu2021geometric}). 
Specializing this discussion to the UF-SVM that is of main interest to us, we note its close ties to the formulation of the hard-margin matrix factorization problem as studied by \citet{srebro2004maximum}. The author formulated the problem of fitting a binary target matrix $\Y$ (ie. with entries $\pm1$) with a low-rank matrix $\W^T\Hb$ as the minimization 
\begin{align}\label{eq:Srebro}
\min_{\W,\Hb} \|\W^T\Hb\|_* \quad\text{sub. to}~~ \Big(\Y\odot\W^T\Hb\Big)[c,i]\geq1, ~\forall (c,i)\in\Sc,
\end{align}
where $\Sc$ is a given subset of observed entries of $\Y$.
Despite being non-convex, they derived a convex reformulation based on duality and a corresponding procedure for finding a solution to \eqref{eq:Srebro} via essentially solving an SDP and an appropriate system of linear equations corresponding to the active constraints.  The non-convex max-margin problem \eqref{eq:svm_original} that we investigate bears similarities to \eqref{eq:Srebro}. Specifically, for an one-hot encoding and fully observed $\Y$, \eqref{eq:Srebro} is essentially the binary analogue of \eqref{eq:svm_original}. 
Importantly in our setting, we are able to calculate the solution to \eqref{eq:svm_original} in closed form, that is without requiring numerically solving an SDP. Finally, as mentioned in \Sec~\ref{sec:2linear} a byproduct of Theorem \ref{thm:SVM} is that the nuclear-norm max-margin minimization \eqref{eq:nuc_norm_SVM} has the same solution as the vanilla max-margin with Frobenius norm. Although different, somewhat related settings were frobenius-norm penalized problems give same solutions as nuclear-norm penalized ones have been studied in the low-rank representation literature, e.g. \cite{vidal2014low,peng2016connections}.

\subsection{Class-imbalanced deep learning}
The past few years have seen a surge of research activity towards substituting the vanilla CE empirical risk minimization, which leads to poor accuracy for minorities, with better alternatives that are particularly suited for training large models, e.g. \cite{kang2020decoupling,KimKim,TengyuMa,Menon,CDT}. Among the many solutions suggested in the recent literature, most closely related to the topic of neural collapse are \cite{kang2020decoupling,KimKim}. Interestingly (as it happened chronologically before the conception of Neural collapse by \citet{NC}), \citet{kang2020decoupling} and \citet{KimKim} observed that the classifier weights found by deep-nets when trained with CE on class-imbalanced data yield larger norms for majority rather than minority classes. This empirical observation led them to propose post-hoc schemes that normalize the logits before deciding on the correct class, thus leading to better performance on the minorities. Our Lemma \ref{lem:norms_w}, not only proves this behavior for the unconstrained feature model, but it also precisely quantifies the norm-ratio between minorities and majorities. Interestingly, our deep-net experiments in \Fig~\ref{fig:norm_ratio_exp_W} confirm the predicted behavior. Evidently then, Lemma \ref{lem:norms} offers a plausible theoretical justification of the empirical findings by \citet{kang2020decoupling,KimKim} and also quantifies the norm ratio. It is conceivable, that this characterization in terms of a simple formula that only involves the imbalance ratio and number of classes, can be used to turn some of the heuristic post-hoc normalizations of  \citet{kang2020decoupling,KimKim} to principled methods. Beyond that, our results are conclusive not only about classifiers and norms, but also about embeddings and angles. It is exciting to investigate leveraging these findings to design better techniques for class-imbalanced deep learning (see also discussion by the end of \Sec~\ref{sec:out}.) 

\new{Related ideas appeared very recently in the contemporaneous works \cite{yang2022we,xie2022neural}, where the authors  design loss functions for class-imbalanced learning in an attempt to enforce a geometry that is alike the ETF geometry for balanced data. However, they do not characterize the joint geometry of classifiers and embeddings under class imbalances as we do here.}

\subsection{On the Simplex-encoding interpolation}
Our analysis of the UF-SVM uncovers the unique role played by the \SEL~matrix. Specifically, Theorem \ref{thm:SVM} shows that the global minimizers $(\What,\Hhat)$ of the non-convex UF-SVM are such that the resulting logit matrix $\What^T\Hhat$ satisfies all constraints with equality and also it equals the \SEL~matrix $\Zhat$. Our finding is related to (in fact, can be seen as an extension of) a recent result by \citet{wang2021benign}. In a different context and with different research objective,  \citet{wang2021benign} derived a deterministic condition under which the solution to the (convex) vanilla SVM (like the one in \eqref{eq:cvx_svm} but with general inputs, not necessarily the basis vectors $\eb_i$ resulting in the UFM) finds logits that interpolate a simplex encoding of the labels. Theorem \ref{thm:SVM} goes far beyond: it studies a non-convex SVM and since it applies directly to the UFM, it does not involve deterministic input conditions for the result to hold. That said, it might be interesting future work to derive similar conditions on the inputs of the non-convex max-margin problem such that it gives optimal logits interpolating the simplex encoding of the labels.

\newpage




%




\newpage


\end{document}